%% file: main.tex
\pdfoutput=1 
\documentclass[10pt, reqno]{amsart}

\usepackage[T1]{fontenc}
\usepackage[utf8]{inputenc}
\usepackage{lmodern}
\usepackage[USenglish]{babel}
\usepackage{amsmath,amssymb,amsthm,upgreek}
\usepackage{mathrsfs}
\usepackage{aligned-overset}
\usepackage{mathtools}
\usepackage{colortbl,color} %
\usepackage[dvipsnames]{xcolor}
\usepackage{graphicx,tikz,pgfplots,pgfplotstable} %
\usepackage{enumitem}
\pgfplotsset{compat=1.6}
\usepackage{subcaption,float,longtable}
\usepackage{anyfontsize}
\usepackage[bookmarks=true,
            bookmarksnumbered,
            plainpages=false,
            linktocpage,
            colorlinks=true,
            citecolor=green!80!black,
            linkcolor=red!70!black,
            filecolor=magenta,
            urlcolor=magenta,
            hidelinks,
            breaklinks,
            pdfauthor={Sjoerd Dirksen, Patrick Finke, Paul Geuchen, Dominik Stöger, Felix Voigtlaender},
            pdftitle={Near-optimal estimates for the ℓᵖ-Lipschitz constants of deep random ReLU neural networks},
            unicode=true]{hyperref}
\usepackage[nameinlink,capitalize,noabbrev]{cleveref}
\crefname{appendix}{Appendix}{Appendices}
\usepackage{adforn} %
\usepackage{bbm}
\usepackage{csquotes}
\MakeOuterQuote{"}
\usepackage{comment}
\usepackage{titlesec}
\usepackage[margin=3cm]{geometry}
\usepackage[style = numeric, url = true, giveninits=true, eprint = true, sortcites=true,backend = biber, maxbibnames=99]{biblatex}

\usepackage{booktabs}
\usepackage{multirow}
\usepackage{aliascnt}
\usepackage{url}

\AtBeginBibliography{\small}
\AtEveryBibitem{\clearlist{language}}
\AtEveryBibitem{\clearfield{issn}}
\AtEveryCitekey{\clearfield{issn}}

\DeclareBibliographyCategory{needsurl}
\renewbibmacro*{url+urldate}{%
  \ifcategory{needsurl}
    {\printfield{url}%
     \iffieldundef{urlyear}
       {}
       {\setunit*{\addspace}%
        \printurldate}}
    {}}

\newcommand{\entryneedsurl}[1]{\addtocategory{needsurl}{#1}}
\entryneedsurl{vershynin_high-dimensional_2018}
\usepackage{ifthen}

\DeclareRobustCommand{\abbrevcrefs}{%
\crefname{lemma}{Lem.}{Lems.}%
\crefname{corollary}{Cor.}{Cors.}%
\crefname{equation}{Eqn.}{Eqns.}%
\crefname{proposition}{Prop.}{Props.}%
}
\DeclareRobustCommand{\cshref}[1]{{\abbrevcrefs\cref{#1}}}

\let\emptyset\varnothing

\newcommand{\ee}{\mathrm{e}}
\newcommand{\dd}{\mathrm{d}}
\newcommand{\NN}{\mathbb{N}}                                     %
\newcommand{\RR}{\mathbb{R}}                                     %
\newcommand{\ZZ}{\mathbb{Z}}                                     %

\newcommand{\PP}{\mathbb{P}}
\newcommand{\EE}{\mathbb{E}}
\newcommand{\VV}{\mathbb{V}}
\newcommand{\E}{\mathcal{E}}
\renewcommand{\SS}{\mathbb{S}}
\newcommand{\N}{\mathcal{N}}

\newcommand{\op}[1]{\left\Vert #1 \right\Vert_{2 \to 2}}                                    %
\newcommand{\abs}[1]{\left\lvert#1\right\rvert}                  %
\newcommand{\mnorm}[1]{\left\lVert#1\right\rVert}                %
\newcommand{\defeq}{\mathrel{\mathop:}=}                               %

\newcommand{\skpr}[2]{\left\langle#1, #2\right\rangle}               %
\newcommand{\norel}{\mathrel{\phantom{=}}}                                           %

\newcommand{\iid}{\overset{\mathrm{i.i.d.}}{\sim}}
\newcommand{\B}{\overline{B}}
\renewcommand{\d}{\overset{d}{=}}

\newcommand{\neps}{\mathcal{N}_\eps}

\newcommand{\remove}[1]{{\color{black} #1}}

\renewcommand{\hat}{\widehat}
\renewcommand{\tilde}{\widetilde}
\newcommand*{\fres}[2]{ {\left.\kern-\nulldelimiterspace #1 \vphantom{\big|} \right|_{\kern-1pt #2} }}
\newcommand{\arrow}[1]{\overset{\rightarrow}{#1}}
\newcommand{\nablaa}{\overline{\nabla}}
\newcommand{\act}{\nabla}
\newcommand{\jac}{\overline{\mathrm{D}}}
\newcommand{\metalambda}{%
  \mathop{%
    \rlap{$\lambda$}%
    \mkern2mu
    \raisebox{.275ex}{$\lambda$}%
  }%
}
\newcommand{\fewdotsfill}{%
	\leavevmode
	\cleaders\hbox{.\kern5pt}\hfill\kern0pt
}
\newcommand{\apptocline}[2]{%
	\noindent
	\makebox[\textwidth][l]{%
		\hyperref[#1]{\textbf{#2}}%
		\fewdotsfill
		\pageref{#1}%
	}\par
}

\newcommand{\sjoerd}[1]{\textcolor{black}{#1}}

\DeclareMathOperator{\diam}{diam}
\DeclareMathOperator{\inte}{int}

\DeclareMathOperator{\supp}{supp}
\DeclareMathOperator{\lip}{Lip}

\DeclareMathOperator{\relu}{ReLU}
\DeclareMathOperator{\rang}{rank}
\DeclareMathOperator{\sgn}{sgn}

\DeclareMathOperator{\Tr}{Tr}
\DeclareMathOperator{\diag}{diag}

\DeclareMathOperator{\ang}{ang}

\DeclareMathOperator{\Unif}{Unif}
\DeclareMathOperator{\cone}{cone}

\let \eps \varepsilon
\let \epsilon \varepsilon

\theoremstyle{plain} 

\newtheorem{theorem}{Theorem}[section]

\newaliascnt{corollary}{theorem}
\newtheorem{corollary}[corollary]{Corollary}
\aliascntresetthe{corollary}

\newaliascnt{assumption}{theorem}
\newtheorem{assumption}[assumption]{Assumption}
\aliascntresetthe{assumption}

\newaliascnt{lemma}{theorem}
\newtheorem{lemma}[lemma]{Lemma}
\aliascntresetthe{lemma}

\newaliascnt{proposition}{theorem}
\newtheorem{proposition}[proposition]{Proposition}
\aliascntresetthe{proposition}

\theoremstyle{definition} %
\newaliascnt{definition}{theorem}
\newtheorem{definition}[definition]{Definition}
\aliascntresetthe{definition}

\newaliascnt{remark}{theorem}
\newtheorem{remark}[remark]{Remark}
\aliascntresetthe{remark}
\AtEndEnvironment{remark}{\hfill$\diamond$}

\newtheorem{rem}[theorem]{Remark}
\newif\ifshowdiamond
\showdiamondtrue

\AtEndEnvironment{rem}{%
  \ifshowdiamond
    \hfill$\diamond$
  \fi
}

\crefname{theorem}{theorem}{theorems}
\crefname{Prop}{Proposition}{Propositions}
\crefname{Lem}{Lemma}{Lemmas}
\crefname{Kor}{Corollary}{Corollaries}
\crefname{rem}{Remark}{Remarks}
\crefname{Bsp}{Example}{Examples}
\crefname{Def}{Definition}{Definitions}
\crefname{Alg}{Algorithm}{Algorithms}

\numberwithin{equation}{section}

\makeatletter
\renewcommand*{\eqref}[1]{%
  \hyperref[{#1}]{\textup{\tagform@{\ref*{#1}}}}%
}
\makeatother

\makeatletter
\@ifundefined{textcommabelow}{%
  \DeclareTextCommandDefault\textcommabelow[1]
    {\hmode@bgroup\ooalign{\null#1\crcr\hidewidth\raise-.31ex
     \hbox{\check@mathfonts\fontsize\ssf@size\z@
     \math@fontsfalse\selectfont,}\hidewidth}\egroup}%
}{}
\makeatother

\titleformat{\section}
  {\normalfont\LARGE\bfseries} %
  {\thesection}{1em}           %
  {}                           %

\titleformat{\subsection}
  {\normalfont\large\bfseries}
  {\thesubsection}{1em}
  {}

\titlespacing*{\section}
  {0pt}      %
  {6ex plus 1ex minus .2ex}   %
  {3ex plus .2ex}             %

\titlespacing*{\subsection}
  {0pt}
  {4.5ex plus 1ex minus .2ex}
  {2.5ex plus .2ex}

\makeatletter
\def\@settitle{%
  \begin{center}%
    \normalfont
    \fontsize{17}{19}\bfseries %
    \@title
  \end{center}%
}
\makeatother

\makeatletter
\renewcommand{\enddoc@text}{}

\def\@setaddresses{%
  \par\nobreak
  \begingroup
  \footnotesize

  \def\author##1{\nobreak\addvspace\bigskipamount}%
  \def\\{\unskip , \ignorespaces}%
  \interlinepenalty\@M

  \def\address##1##2{%
    \begingroup
    \par\addvspace\bigskipamount
    \noindent
    \makebox[9em][l]{%
      \@ifnotempty{##1}{(\ignorespaces##1\unskip)}{}%
    }%
    \parbox[t]{\dimexpr\linewidth-9em\relax}{%
      \scshape\ignorespaces##2%
    }%
    \par
    \endgroup
  }%

  \def\curraddr##1##2{%
    \begingroup
    \@ifnotempty{##2}{%
      \nobreak\indent\curraddrname
      \@ifnotempty{##1}{, \ignorespaces##1\unskip}\/:\space
      ##2\par
    }%
    \endgroup
  }%

  \def\email##1##2{%
  \begingroup
  \@ifnotempty{##2}{%
    \vspace{0.3\baselineskip}%
    \nobreak\noindent
    \makebox[9em][l]{}%
    \emailaddrname
    \@ifnotempty{##1}{, \ignorespaces##1\unskip}\/:\space
    \ttfamily##2\par
  }%
  \endgroup
  }%

  \def\urladdr##1##2{%
    \begingroup
    \def~{\char`\~}%
    \@ifnotempty{##2}{%
      \nobreak\indent\urladdrname
      \@ifnotempty{##1}{, \ignorespaces##1\unskip}\/:\space
      \ttfamily##2\par
    }%
    \endgroup
  }%

  \addresses
  \endgroup
}
\makeatother

\begin{document}

\allowdisplaybreaks

\hyphenation{Lip-schitz}

\title[Optimal estimates for the $\ell^p$-Lipschitz constants of random neural networks]{Optimal estimates for the $\ell^p$-Lipschitz constants \\of deep random ReLU neural networks}

\author{Sjoerd Dirksen}
\address[S. Dirksen]{Mathematical Institute,
Utrecht University,
Budapestlaan 6,
3584 CD Utrecht,
Netherlands}
\email{s.dirksen@uu.nl}
\thanks{}

\author{Patrick Finke}
\address[P. Finke]{Mathematical Institute,
Utrecht University,
Budapestlaan 6,
3584 CD Utrecht,
Netherlands}
\email{p.g.finke@uu.nl}
\thanks{}

\author{Paul Geuchen}
\address[P. Geuchen]{
Mathematical Institute,
Utrecht University,
Budapestlaan 6, 3584 CD Utrecht, Netherlands}
\email{p.geuchen@uu.nl}
\thanks{}

\author{Dominik Stöger}
\address[D. Stöger]{Mathematical Institute for Machine Learning and Data Science (MIDS),
Catholic University of Eichstätt--Ingolstadt (KU), Auf der Schanz 49, 85049 Ingolstadt, Germany}
\email{dominik.stoeger@ku.de}
\thanks{}

\author{Felix Voigtlaender}\thanks{}
\address[F. Voigtlaender]{Mathematical Institute for Machine Learning and Data Science (MIDS),
Catholic University of Eichstätt--Ingolstadt (KU), Auf der Schanz 49, 85049 Ingolstadt, Germany}
\email{felix.voigtlaender@ku.de}



\date{\today}

\begin{abstract}    
This paper studies \sjoerd{the} $\ell^p$-Lipschitz constants of \sjoerd{ReLU neural networks $\Phi: \RR^d \to \RR$ with random parameters} for $p \in [1,\infty]$. 
The distribution of the weights follows a variant of the He initialization.
In the case of zero-bias networks, we derive matching high probability upper and lower bounds for wide networks that differ at most by an absolute constant.  
The bounds reveal two remarkable phenomena. First, the bounds grow root-logarithmically in terms of the depth of the net. Second, the behavior of the $\ell^p$-Lipschitz constants 
varies significantly between the regimes $ p \in [1,2) $ and $ p \in [2,\infty] $.
For $p \in [2,\infty]$, the $\ell^p$-Lipschitz constant behaves similarly to $\mnorm{g}_{p'}$, 
where $g \in \RR^d$ is a standard Gaussian vector and $1/p + 1/p' = 1$.
In contrast, for $p \in [1,2)$, the $\ell^p$-Lipschitz constant aligns more closely to $\mnorm{g}_{2}$.
For networks with possibly nonzero biases, we obtain the same lower bounds without assumptions on the bias distributions.
For symmetric bias distributions with bounded densities, we also derive high-probability
upper bounds for the $\ell^p$-Lipschitz constants.
For fixed sub-exponential bias distributions, these upper bounds differ from 
the lower bounds by a factor that has square-root dependence on the depth and logarithmic 
dependence on the width. 
\end{abstract}
\maketitle
\vspace{-0.75cm}
\input{introduction.tex}

\input{setting.tex}

\input{preliminaries.tex}

\input{pointwise.tex}

\input{upper_hom.tex}

\input{lower_hom.tex}

\input{bias.tex}

\medskip

\textbf{Acknowledgements.}
FV and PG are grateful for the hospitality of Utrecht University, where parts of this project were developed.   
FV and PG acknowledge support by
the German Science Foundation (DFG) in the context of the Emmy Noether junior research
group VO 2594/1-1.
FV acknowledges support by the Hightech Agenda Bavaria. 

\textbf{Use of AI tools.}
We describe how AI tools assisted in the development and revision of the mathematical content 
of this paper, using our previous preprint \cite{dirksen2026nearoptimalestimatesellplipschitzconstants},
which contains no AI-generated mathematical content, as a point of comparison.

AI tools were not used to develop the mathematical results or proof strategies in
\Cref{sec:prelim,sec:pw}
and \Cref{sec:prelim_proofs,sec:rand_tess,sec:distr_equiv,sec:pw_proofs,sec:auxil}.

The upper-bound argument in \Cref{sec:upper,sec:upper_proof} was developed 
in \cite{dirksen2026nearoptimalestimatesellplipschitzconstants}. 
In preparing the present version of this manuscript, OpenAI's GPT-6 Astra was used for mathematical 
proofreading, local corrections, and an explicit reformulation of the width assumption.
The underlying proof strategy remained unchanged.

The content of \Cref{sec:lower,sec:lower_proof} 
was largely suggested by GPT-6 Astra. In \cite{dirksen2026nearoptimalestimatesellplipschitzconstants},
we established that with high probability,
\[
\lip_p(\Phi) \gtrsim \begin{cases}\sqrt{d}&\text{if }p \in [1,2), \\d^{1-1/p}&\text{if } p \in [2,\infty]\end{cases}
\]
for random zero-bias ReLU networks following \Cref{assum:1}.
More specifically, the bound in the regime $p \in [2,\infty]$ was shown under the assumption $N \gtrsim L^2$ (up to logarithmic factors),
whereas the bound in the regime $p \in [1,2)$ was proven under the condition $N \gtrsim C^L \cdot dL^{4L+10}$
(up to logarithmic factors, with an absolute constant $C>0$).
It remained open, however, to close the logarithmic gap to the upper bound and to 
improve the assumption on $N$ that is sufficient in order to obtain the lower bound for $p \in [1,2)$.
Given \cite{dirksen2026nearoptimalestimatesellplipschitzconstants} and asked whether 
the logarithmic gap could be closed, ChatGPT-6 Astra suggested studying directional 
finite differences and produced a draft proof of the result now stated as \Cref{thm:main_lower}. 
The authors of the present manuscript independently checked, validated, and rewrote this argument.

The upper-bound approach in \Cref{sec:bias,sec:bias_bounds_proofs,sec:bias_bounds_proofs,sec:hom}
was developed in \cite{dirksen2026nearoptimalestimatesellplipschitzconstants}.
During the preparation of the present manuscript, ChatGPT-6 Astra assisted with mathematical
proofreading, corrections to individual estimates, parameter choices, 
and simplification of the sufficient width conditions. 
The underlying proof strategy for the upper bound remained unchanged. 
ChatGPT-6 Astra also suggested that the $\ell^p$-Lipschitz constants of a network with biases can be lower bounded 
by the $\ell^p$-Lipschitz constants of its zero-bias counterpart, which simplified both the statement and the 
proof of \Cref{thm:main_lower_bias}. 

In addition to the mathematical assistance described above, AI tools were used for wording, language editing, and literature searches. 
The authors of the present paper take full responsibility for its content.

\vspace{-0.4cm}
\printbibliography[heading=bibintoc]
\markboth{\MakeUppercase{\scshape Sjoerd Dirksen, Patrick Finke, Paul Geuchen, Dominik Stöger, and Felix Voigtlaender}}{\shorttitle}
\begingroup
\raggedbottom  %
\makeatletter
\@setaddresses
\makeatother
\endgroup
\clearpage
\appendix
\crefalias{section}{appendix}
\vspace*{0.05cm}
\begin{center}
	\Huge\bfseries Appendix
\end{center}

\vspace{1em} %
\section*{Appendix Contents}
\apptocline{sec:prelim_proofs}{A.   Proofs of preliminary results regarding the Jacobian of ReLU networks}
\apptocline{sec:distr_equiv}{B.   Proof of distributional equivalence}
\apptocline{sec:rand_tess}{C.   Proofs for the results on random hyperplane tessellations}
\apptocline{sec:pw_proofs}{D.   Proof of pointwise bounds}
\apptocline{sec:upper_proof}{E.   Proof of \Cref{thm:main_upper}}
\apptocline{sec:lower_proof}{F.   Proof of \Cref{thm:main_lower}}
\apptocline{sec:hom}{G.   Proof of \Cref{thm:grad_diff}}
\apptocline{sec:bias_bounds_proofs}{H.   Bounds on the Lipschitz constants for networks with biases}
\apptocline{sec:auxil}{I.   Other auxiliary results}
\input{prelim_proofs.tex}

\input{distr_equiv}

\input{tesselations.tex}

\input{pw_proof.tex}

\input{upper_proof.tex}

\input{sharp_lipschitz_section.tex}

\input{homogeneous_proof.tex}

\input{bias_bounds_proofs.tex}

\input{auxiliary.tex}

\end{document}

%% file: introduction.tex
\section{Introduction}
\label{sec:intro}

\sjoerd{From empirical research it is well known that neural networks are} vulnerable to 
adversarial attacks \cite{szegedy2013intriguing,goodfellow_adversarial}, meaning that by
introducing small perturbations to the input, the output of the neural network can be significantly altered.
One way to assess the
worst-case robustness of a neural network
$\Phi: \RR^d \rightarrow \RR$
\sjoerd{to adversarial attacks}
is by considering its 
\sjoerd{$\ell^p$-Lipschitz constants}
\[
\lip_p(\Phi) \defeq \underset{x \neq y}{\underset{x,y \in \RR^d}{\sup}} \ \frac{\abs{\Phi(x) - \Phi(y)}}{\mnorm{x-y}_p},
\quad p \in [1,\infty].
\]
For example,
in image classification
the case $p=\infty$
is especially relevant.
This scenario allows for 
perturbations of each pixel
which are imperceptible to the human eye.
Consequently,
several algorithms  were proposed  to estimate
bounds for \sjoerd{Lipschitz constants} of a trained deep neural network,
see, e.g.,
\cite{ebihara2024local,fazlyab2019efficient}.

In this paper,
we aim to understand the Lipschitz constant\sjoerd{s} 
of neural networks with \sjoerd{\textit{random parameters}},
\sjoerd{i.e.,}
the bias\sjoerd{es} and weights are chosen at random. \sjoerd{These networks appear naturally 
as the random initialization from which the training process starts. 
Understanding their properties is hence an important component in building a rigorous theory for deep learning. For instance, the $\ell^2$-Lipschitz constant of a random neural network has been used in the study of 
neural networks in the Neural Tangent Kernel (NTK) regime, 
see, e.g., \cite{buchanan2021deep}.
More generally, random neural networks are often initially considered 
when trying to rigorously understand empirical phenomena in deep learning, see, e.g., \cite{bartlett2021adversarial,dirksen2022separation}.}

In this paper,
we consider fully-connected ReLU neural networks $\Phi: \ \RR^d \to \RR $ with $L$ hidden
layers of width $N$, i.e., \sjoerd{of the form}
\begin{equation*}
  \Phi(x)
  \defeq \left( V^{(L)} \circ \relu \circ V^{(L-1)} \circ  \dots \circ \relu \circ V^{(0)}\right) (x),
\end{equation*}
where 
$\relu (x) \defeq \max \{0,x\}$
\sjoerd{is applied componentwise and}
the affine maps $V^{(\ell)}$ are defined by
\begin{equation*}
  V^{(\ell)} (x) = W^{(\ell)}x + b^{(\ell)}
  \quad \text{for every} \quad 0 \leq \ell \leq L.
\end{equation*}
We assume that both the weight matrices 
$W^{(0)} \in \mathbb{R}^{N \times d}$,
$ (W^{(\ell)})_{\ell=1}^{L-1} \subset \mathbb{R}^{N \times N}$,
$W^{(L)} \in \mathbb{R}^{1 \times N}$
and the bias vectors
$ (b^{(\ell)})_{\ell=0}^{L-1} \subset \RR^N$ and $b^{(L)} \in \RR$
are initialized at random
\sjoerd{according to} a variant of the \sjoerd{popular} He initialization
\cite{he_initialization}.
Namely,
we assume that the entries of $W^{(\ell)}$ 
are i.i.d.\ \sjoerd{Gaussian} with zero mean and standard deviation 
$ \sqrt{ 2/N} $
for $ 0 \le \ell \le L-1$
and i.i.d.\ \sjoerd{standard Gaussian}
for $\ell = L$. \sjoerd{This variant of the} He initialization is a natural choice for our 
\sjoerd{study: in the case of zero biases, each layer acts as an isometry in expectation,
i.e.,
$ \mathbb{E} \left[ \mnorm{ \text{ReLU} (W^{(\ell)} x )  }^2_2 \right] 
= \mnorm{x}_2^2  $.
Due to the homogeneity of the $\text{ReLU}$-function,
our results can be easily generalized to other 
initialization schemes. For the biases $(b^{(\ell)})_{\ell=0}^{L}$,
we allow for more general distributions.}

In this paper, we are interested in 
upper and lower bounds for
the Lipschitz constant
$\lip_p\sjoerd{(\Phi)}$
for all $p \in [1,\infty]$. \sjoerd{Previous works made notable progress on this question in the case $p=2$. It was shown in \cite{geuchen2024upper} that
for shallow nets (i.e., $L=1$) with arbitrary width, we have
$ \lip_2 (\Phi ) \asymp \sqrt{d} $
with high probability.
For deeper networks, 
i.e., $L \ge 2$, they showed that}
\begin{equation*}
    \sqrt{d}
    \lesssim
    \lip_2 (\Phi)
    \lesssim 
    C^{L}
    \sqrt{Ld \ln (N/d) }
\end{equation*}
for some absolute constant $C>0$ \sjoerd{if $N\gtrsim d$. Hence, the gap between the lower bound and the upper bound grows exponentially in $L$ and logarithmically in $N$. An improved upper bound in the special case of sufficiently wide \emph{zero-bias} networks was derived in 
\cite{buchanan2021deep}:
in this case it was shown that $ \lip_2 (\Phi) \lesssim \sqrt{d\log(N)}$ if $N \gtrsim  d^4 L$ (up to logarithmic factors).}
\sjoerd{In this work we derive near-matching upper and lower bounds
 for $\lip_p (\Phi)$ for all $p \in [1,\infty]$. 
 In the process, we also obtain (partial) improvements of the above results 
 in the case $p=2$. 
 In the following summary of our main findings} we distinguish between the 
zero-bias case and the more general case.

\subsection{Main results for the zero-bias case}
We first consider the zero-bias case,
that is,
we assume that
$b^{(\ell)}=0$ 
for all $0 \le \ell \le L$.
In this scenario,
we show that, with high probability, 
the Lipschitz constant\sjoerd{s satisfy}
\[
 \lip_p(\Phi) \asymp
\begin{cases}
\sqrt{d \cdot \ln(\ee L) }& \text{for } p \in [1,2), \\
d^{1-1/p} \cdot \sqrt{\ln(\ee L)}& \text{for } p \in [2,\infty],
\end{cases}
\]
provided that $N\gtrsim dL^6 \ln(\ee L)$.
Remarkably, we obtain sharp bounds for every $p \in [1,\infty]$ and thus \emph{completely characterize}
the Lipschitz constants of random zero-bias ReLU neural networks. 
In particular, in the case $p=2$, we improve the estimate from \cite{buchanan2021deep} in the natural 
regime $N \gtrsim L$.

Our proof of the upper bounds (see \Cref{thm:main_upper,corr:lp_upper}) relies on upper bounding the gradient of the network using Dudley's inequality \cite[Theorem~8.1.3]{vershynin_high-dimensional_2018}. 
The main technical innovation lies in adapting
recent methods 
developed for analyzing random 
hyperplane tessellations  \sjoerd{\cite{oymak2015near,dirksen2021non,dirksen2022sharp} to} 
control the deviation between gradients at points within $\varepsilon$-distance of each other.

To establish the lower bounds,
we obtain from a straightforward duality argument 
that $\lip_p (\Phi) $ can be bounded from below by 
$ \mnorm{\nabla \Phi (x) }_{p'}  $
for \sjoerd{any} fixed $x \in \RR^d$,
where $p'$ is the H\"older conjugate of $p$,
i.e., $1/p +1/p' =1 $.
By estimating $  \mnorm{\nabla \Phi (x) }_{p'} $
at a single point $x \in \RR^d$,
we obtain a lower bound for $ \lip_p (\Phi) $
that, for $ p \in [2,\infty] $, matches the upper bound
up to a logarithmic factor; see \Cref{corr:pw}.
Interestingly, this approach does \emph{not} yield optimal lower bounds when $p \in [1,2)$.
Moreover, it leaves open the question whether the logarithmic dependence on the number of hidden
layers in the upper bound is a proof artifact or in fact necessary. 
To address this,
we fix a direction $\nu \in \SS^{d-1}$ and directly consider finite differences of the form 
$(\Phi(x+h\nu)-\Phi(x-h\nu))/(2h)$
and lower bound these using 
Sudakov minoration; see \Cref{thm:main_lower}.

\subsection{Main results for the general (non-zero bias) case}
For networks with non-zero biases, 
we derive near-matching bounds on the Lipschitz constants under suitable 
assumptions on the biases. For instance, 
if they are drawn independently from a \emph{fixed} symmetric sub-exponential
distribution with bounded density, we show that
\begin{equation}\label{eq:bias_lower_upper}
\begin{cases} 
\sqrt{d\ln(\ee L)}, \\
d^{1-1/p} \cdot \sqrt{\ln(\ee L)}
\end{cases}
\lesssim \lip_p(\Phi) \lesssim 
\begin{cases}
\sqrt{d} \left( \sqrt{L} +  \sqrt{\ln(N) } \right)& \text{for } p \in [1,2), \\
d^{1-1/p} \left( \sqrt{L} + \sqrt{ \ln(N)  } \right) & \text{for } p \in [2,\infty],
\end{cases}
\end{equation}
with high probability; see \Cref{rem:sub_exponential} and \Cref{thm:main_lower_bias}.
This includes the case of a Gaussian distribution with zero mean and 
the case of a symmetric uniform distribution. 
Remarkably, the bounds in \eqref{eq:bias_lower_upper} are invariant under scaling of the biases. 
More generally, for (not necessarily sub-exponential)
symmetric bias distributions with bounded densities, our upper 
bound depends explicitly on an upper bound of the PDF and 
a high-probability bias threshold; see \Cref{thm:main_upper_bias,corr:upper_bias}. 
The lower bounds require no 
assumptions on the bias distributions. 
For an overview, 
we refer to \Cref{table:results_bias}, 
which lists our bounds along with the corresponding assumptions
and locations within the manuscript where these bounds can be found. 
\sjoerd{Compared to \cite{geuchen2024upper}, 
in the case $p=2$ we obtain an upper bound with square-root rather than exponential 
dependence on the depth, under slightly more stringent assumptions on the biases. 
Moreover, we establish near-matching bounds for all $p \in [1,\infty]$. 
For shallow networks, our work complements \cite{geuchen2024upper} to yield matching 
bounds for all $p\in [1,\infty]$ 
that hold for completely arbitrary bias distributions; see \Cref{corr:shallow}.}

\sjoerd{The upper bounds are derived via a reduction to the zero-bias case. 
We observe that for any signal 
$ x\in M= \left\{ x \in \RR^d : \mnorm{x}_2 \ge R  \right\}  $,}
for sufficiently large $R>0$,
as $x$ propagates through the individual 
layers of the network,
the network activations are relatively large compared to the biases.
Consequently,
the perturbation of the signal induced by the biases remains minimal.
In particular, on the set $M$ the network $\Phi$
can be approximated \sjoerd{well} by a zero-bias network $\tilde{\Phi} $ 
\sjoerd{obtained from $\Phi$ by setting all its biases to zero}.
On $M^c$, however, we can no longer expect that
the zero-bias network $\tilde{\Phi}$ approximates $\Phi$.
We deal with this using a different 
approach which causes the factor $\sqrt{\ln(\ee L)}$ that appears in the zero-bias case 
to be replaced by $\sqrt{L} + \sqrt{\ln(\lambda C_\tau N/d)}$ in the non-zero bias case.
This reflects the exponential factor $3^L$ in the comparison radius $R$, together
with its dependence on the bias threshold and the auxiliary parameter $\delta$;
see \Cref{thm:grad_diff}.

\begin{table}[t]
\begin{tabular}{@{} c 
>{\centering\arraybackslash}m{1.5cm} 
>{\centering\arraybackslash}m{5.2cm} 
>{\centering\arraybackslash}m{4cm} 
c @{}}
\toprule
\textbf{$p$ range} & \textbf{Type} & \textbf{Bound} 
& \textbf{Assumption on biases}
& \textbf{Reference} \\
\midrule
\multirow{2}{*}{$p \in [1, 2)$} 
    & $\gtrsim$
    & $\sqrt{d \ln(\ee L)}$ 
    &none
    & \Cref{thm:main_lower_bias} \\
    &$\lesssim$
    & $\sqrt{d} \left( \sqrt{L} + \sqrt{ \ln(\lambda C_{\tau} N / d)} \right)$ 
    &bounded pdf+symmetric
    & \Cref{corr:upper_bias} \\
\midrule
\multirow{2}{*}{$p \in [2, \infty]$} 
    & $ \gtrsim$
    & $d^{1 - 1/p} \cdot \sqrt{\ln(\ee L)}$ 
    & none
    & \Cref{thm:main_lower_bias} \\
    & $ \lesssim$
    & $d^{1 - 1/p} \left( \sqrt{L} + \sqrt{ \ln(\lambda C_{\tau} N / d)} \right)$ 
    &bounded pdf+symmetric
    & \Cref{corr:upper_bias} \\
\bottomrule
\end{tabular}
\caption{Our results for neural networks
with non-zero biases. The constant $C_\tau$ is an upper bound for the PDF
of the biases and $\lambda$ is a high-probability bias threshold.}
\label{table:results_bias}
\end{table}

The lower bounds in the case of arbitrary biases follow from the fundamental observation that 
the $\ell^p$-Lipschitz constants of a ReLU network with biases can be deterministically lower bounded by the 
Lipschitz constants of the corresponding zero-bias network.
Therefore, they can directly be deduced from \Cref{thm:main_lower}.

\input{related_work}

\subsection{Structure of the paper}
In \Cref{sec:prelim}, we provide background
regarding random \sjoerd{neural} networks and
describe our main technical assumptions.
Moreover, we describe \sjoerd{several results on} random hyperplane \sjoerd{tessellations},
which are a key component in our proofs.
\Cref{sec:pw} focuses on proving both 
upper and lower bounds for the $\ell^p$-norm\sjoerd{s} 
of the gradient of a network $\Phi$ with symmetric bias distributions
 at a fixed point $x \in \mathbb{R}^d$.
This directly leads to near-optimal lower bounds for the Lipschitz constant\sjoerd{s}
of $\Phi$ when $p \in [2,\infty]$.
In \Cref{sec:upper}, we prove an upper bound for the \sjoerd{$\ell^p$-}Lipschitz constants
for any $p \in [1,\infty]$ in the case of zero-bias networks.
In \Cref{sec:lower}, we establish a lower bound 
for the Lipschitz constant\sjoerd{s} of zero-bias networks
in the regime $p \in [1,\infty]$.
In \Cref{sec:bias}, we explain
how to generalize our results from zero-bias networks
to general networks with non-zero biases. \sjoerd{Throughout, we provide proof sketches of our results which contain the key ideas. The lengthier proofs with full details can be found in the Appendix.}

\subsection{Notation}

For a statement $A$, we set $\mathbbm{1}_A \defeq 1$ if $A$ is true and $\mathbbm{1}_A \defeq 0$ otherwise. 
We denote constants whose enlargement does not change a given statement using capital letters and constants whose shrinkage does not change a given statement using lowercase letters. 
Without loss of generality, we may assume that all constants of the first kind are greater than or equal to 1, and all constants of the second kind are less than or equal to 1. 
We write $\alpha \lesssim \beta$ if there exists an absolute constant $C > 0$ such that $\alpha \leq C \cdot \beta$, and we write $\alpha \asymp \beta$ if both $\alpha \lesssim \beta$ and $ \beta \lesssim \alpha$.
For $x \in \RR$, we define $\sgn_+(x) \defeq \mathbbm{1}_{x \geq 0} - \mathbbm{1}_{x < 0}$, 
$\sgn_-(x) \defeq \mathbbm{1}_{x > 0} - \mathbbm{1}_{x \leq 0}$, 
and $\sgn_0(x) \defeq \mathbbm{1}_{x > 0} - \mathbbm{1}_{x < 0}$.
The cardinality of a set $B$ is denoted by $\# B$.

Let $x, y \in \RR^n$. We denote their inner product by $\langle x,y \rangle \defeq \sum_{i=1}^n x_i y_i$. For $p \in [1,\infty)$, we denote the $\ell^p$-norm by $\Vert x \Vert_p \defeq \left(\sum_{i=1}^n \abs{x_i}^p\right)^{1/p}$ and the $\ell^\infty$-norm by $\mnorm{x}_\infty \defeq \underset{i=1,\dots,n}{\max} \abs{x_i}$.
We denote the open and closed Euclidean balls with center $x$ and radius $r > 0$ by
\[
    B_n(x,r) \defeq \left\{ y \in \RR^n: \ \mnorm{x-y}_2 < r\right\} \quad \text{and} \quad 
    \overline{B}_n(x,r) \defeq \left\{ y \in \RR^n: \ \mnorm{x-y}_2 \leq r\right\},
\]
and the unit sphere by $\mathbb{S}^{n-1} \defeq \left\{ x \in \RR^n: \ \Vert x \Vert_2 = 1\right\}$.
The $d$-dimensional Lebesgue measure is denoted by ${\metalambda}^d$.

Let $K \subseteq \RR^n$. The interior of $K$ (with respect to the standard topology on $\RR^n$) is denoted by $\inte(K)$. We define the cone generated by $K$ and the Minkowski difference of $K$ with itself by
\begin{equation*}
    \cone(K) \coloneqq \{ s x : \ x \in K, s > 0\}
    \quad\text{and}\quad
    K - K \coloneqq \{ x - y : \ x, y \in K \},
\end{equation*}
respectively. Finally, we denote the diameter of $\emptyset \neq K \subseteq \RR^n$ 
by $\diam(K) \defeq \underset{x,y \in K}{\sup} \Vert x - y \Vert_2 \in [0, \infty]$.

Let $A \in \RR^{m \times n}$. 
We denote its $i$-th row by $A_{i,:}$ and its $j$-th column by $A_{:,j}$. 
For $p,q \in [1,\infty]$, we define the operator norm by
\begin{equation*}
  \Vert A \Vert_{p \to q} \defeq \underset{\mnorm{x}_p \leq 1}{\underset{ x \in \RR^n}{\sup}} \Vert Ax \Vert_q.
\end{equation*}
By $\abs{A}$ we denote the matrix whose entries are the absolute values 
of the entries of $A$. 
We denote the trace by $\Tr(A) \defeq \sum_{i=1}^{\min\{m, n\}} A_{i,i}$. 
$I_{m \times n} \in \RR^{m \times n}$ denotes the matrix with 
$(I_{m \times n})_{i,j} = \mathbbm{1}_{i=j}$ for all $i \in \{1,\dots,m\}$ 
and $j \in \{1,\dots,n\}$, and we use the shorthand $I_n \defeq I_{n \times n}$. 
By $\diag \{0,1,-1\}^{m \times m}$ we denote the set of $m \times m$ diagonal matrices  with diagonal entries in $\{0,1,-1\}$. For $a_1, \dots, a_m \in \RR$ we denote by $\diag(a_1,\dots,a_m) \in \RR^{m \times m}$ the $m \times m$ diagonal matrix with entries $a_1,\dots,a_m$ on the diagonal. 
For a vector $x \in \RR^k$ we define the diagonal matrix $\Delta(x) \in \RR^{k \times k}$ via
\begin{equation}\label{eq:delta}
  \left(\Delta(x)\right)_{i,j}
  \defeq \begin{cases}
            \mathbbm{1}_{x_i > 0}&\text{if } i=j, \\
            0&\text{if }i \neq j .
          \end{cases}
\end{equation}

For $\ell_1, \ell_2 \in \ZZ$ with $\ell_1 \leq \ell_2$ and matrices $A_{\ell_1},\dots,A_{\ell_2}$ with suitable dimensions, we let 
\begin{equation}\label{eq:rev_order}
    \prod_{i=\ell_2}^{\ell_1} A_i \defeq A_{\ell_2} \cdots A_{\ell_1}
\end{equation}
denote their matrix product in reverse order. If $A_{\ell_1} \in \RR^{m \times n}$ and $\ell_2 < \ell_1$, define the empty product as
\[
    \prod_{i=\ell_2}^{\ell_1} A_i \defeq I_{n \times n}.
\]

We write $\mathcal{N}(0, \sigma^2)$ for the normal distribution with expectation $0$ and variance $\sigma^2$, and $\mathcal{N}(0, I_k)$ for the distribution of a $k$-dimensional random vector with independent standard Gaussian entries.
We use $\Unif(K)$ to denote a uniform distribution on a set $K$.

Let $X$ and $Y$ be random variables. We write $X \d Y$ if $X$ and $Y$ are identically distributed. If $X$ is real-valued, we denote the expectation of $X$ by $\EE[X]$, the variance by $\VV[X]$ and let $\mnorm{X}_{\psi_1}$ and $\mnorm{X}_{\psi_2}$ denote the sub-exponential and sub-gaussian norm of $X$, respectively; see \cite[Sections~2.5~\&~2.7]{vershynin_high-dimensional_2018}. 
For a subset $\emptyset \neq K \subseteq \RR^k$, we let
\begin{equation*}\label{eq:gaussian_width}
w(K) \defeq \underset{g \sim \mathcal{N}(0,I_k)}{\EE} \left[ \underset{x \in K}{\sup}\ \langle x, g\rangle\right]
\end{equation*}
denote the \emph{Gaussian width} of $K$; see \cite[Section~7.5]{vershynin_high-dimensional_2018}.

Let $f : \RR^n \to \RR^m$. We denote by 
$\fres{f}{K}$ the restriction of $f$ onto $K \subseteq \RR^n$. 
If $f$ is differentiable at $x_0 \in \RR^n$, we denote by 
$\mathrm{D} f(x_0) \in \RR^{m \times n}$ the Jacobian of $f$ at $x_0$. If $m = 1$, we let $\nabla f(x_0) \in \RR^n$ be the gradient of $f$ at $x_0$.

For given $K \subseteq \RR^k$ and $\eps>0$, 
we call a set $\mathcal{N} \subseteq K$ 
an $\eps$-net of $K$, if for every $x \in K$ there exists $x^\ast \in \mathcal{N}$
with $\mnorm{x - x^\ast}_2 \leq \eps$.
We denote the $\eps$-covering number (with respect to Euclidean distances) of $K$ by
\begin{equation*}\label{eq:cov_num}
\mathcal{N}(K, \eps) \defeq \min\left\{ \#\Lambda: \ \text{$\Lambda \subseteq K$ is an $\eps$-net of $K$}\right\},
\end{equation*}
that is, the smallest possible cardinality of an $\eps$-net of $K$.
Similarly, we call a set $\Lambda \subseteq K$ 
$\delta$-separated if $\mnorm{x-y}_2 > \delta$ for every $x,y\in \Lambda$ with $x \neq y$.
The $\delta$-packing number of $K$ is defined as 
\begin{equation*}\label{eq:pack_num}
\mathcal{P}(K,\delta) \defeq \sup \left\{\#\Lambda: \ \Lambda \subseteq K, \ \Lambda \text{ is }\delta\text{-separated}\right\}.
\end{equation*}
It is well-known that the covering numbers and packing numbers are equivalent in the sense that
\[
\mathcal{P}(K, 2\delta) \leq \mathcal{N}(K, \delta) \leq \mathcal{P}(K, \delta);
\]
see \cite[Lemma~4.2.8]{vershynin_high-dimensional_2018}.

%% file: related_work.tex
\subsection{Related work}

\begin{enumerate}
    \item \textbf{Connection of the Lipschitz constant of neural networks to adversarial examples, generative models, and generalization:}
    As previously mentioned, it was first shown in \cite{szegedy2013intriguing} that neural networks are susceptible to adversarial attacks.
    Subsequently, a large body of research developed increasingly sophisticated adversarial attacks
    (see, e.g., \cite{carlini2018audio,kurakin2018adversarial,moosavi2016deepfool,papernot2016limitations})
    and defense methods against these attacks
    (see, e.g., \cite{madry2018towards,schott2018towards,pang2019improving,dong2020adversarial,bui2021unified}).
    However,
    these advances are mostly empirically driven.
    In contrast, 
    a rigorous theoretical understanding of adversarial examples and defenses is still lacking.
    One line of research 
    shows the existence of adversarial examples in \textit{neural networks with random weights} \cite{shamir2019simple,daniely2020most,bubeck2021single,bartlett2021adversarial,montanari2023adversarial}.
    More precisely, these works establish that, for a fixed or random input, 
    there exists, with high probability, 
    a small adversarial perturbation of this input.
    Recent work also studies
    the existence of adversarial examples for fixed or random input in trained models \cite{vardi2022gradient,wang2022adversarial,frei2023double,melamed2023adversarial}
    or in Neural Tangent Kernel models \cite{bombari2023beyond} under certain assumptions.
    In all of these works, the key step in establishing the existence of an adversarial perturbation
    is to show that the gradient of the neural network evaluated at this fixed or random input has a large norm,
    which implies a large Lipschitz constant.

    This highlights the relevance of Lipschitz constants in the context of adversarial examples of neural networks.
    Indeed, neural networks with smaller Lipschitz constants are expected to be more robust against adversarial attacks.
    For this reason, 
    a number of adversarial defenses use techniques
    which are closely related to Lipschitz regularization,
    see,
    e.g.,
    \cite{lee2020lipschitz,leino2021globally}.
    For example, 
    when dealing with images,
    one is typically interested in the $\ell^\infty$-Lipschitz constant.
    The reason for this is that,
    in adversarial attacks, one can perturb each pixel
    by a small, bounded amount
    such that the change remains invisible to the human eye \cite{goodfellow_adversarial}.
    
    Lipschitz constants of neural networks are not only of interest in the context 
    of adversarial examples.
    For instance, several works also connected the Lipschitz constants of neural
    networks to the generalization properties \cite{bartlett2017spectrally}.
    Moreover, the Lipschitz constant is also relevant in the context 
    of solving inverse problems with generative neural models:
    In \cite{bora2017compressed}, theoretical guarantees for Compressed Sensing are derived 
    where the sample complexity depends on the $\ell^2$-Lipschitz constant
    of the generative neural network.
    In \cite{miyato2018spectral} a technique called Spectral normalization is used in GAN training to stabilize 
    the Lipschitz constant of the discriminator.

    \item \textbf{Fundamental limits (the law of robustness)}
    In \cite{bubeck2021law},
    it was conjectured that a neural network
    that interpolates training data 
    must possess a large $\ell^2$-Lipschitz constant,
    unless the network is overparameterized,
    i.e., the number of parameters is much larger than the number of training points.
    In follow-up work \cite{bubeck2021universal},
    the authors proved this conjecture for a broad class of parameterized models, going beyond neural networks.
    However,
    several open questions remain.
    Specifically, the result in \cite{bubeck2021universal} requires that the training data is fitted exactly.
    This suggests that the result might be circumvented
    when the data is only fitted approximately
    which is common practice in robust optimization.
    Moreover, the paper \cite{bubeck2021universal} is only concerned with the $\ell^2$-Lipschitz constant.
    This leaves the question regarding general $\ell^p$-Lipschitz constants
    for $p \ne 2$ open for future investigations.

    \item {\textbf{Computing and certifying Lipschitz constants of neural networks:}}
    Since, as discussed above,
    Lipschitz constants of neural networks are intimately
    linked to the robustness of neural networks,
    there has been significant work 
    on algorithms for computing or estimating these constants.
    However, computing 
    the Lipschitz constant was shown to be NP-hard \cite{virmaux2018,jordan2020}.
    Several algorithms were introduced to estimate 
    the Lipschitz constant from above.
    These include
    an approach
    based on sparse estimation \cite{Latorre2020Lipschitz}
    and an approach based on semialgebraic optimization \cite{chen2020lipschitz}.
    Methods for estimating the Lipschitz constant for generative models have also been studied \cite{jordan2021LipschitzCertification}.
    The method LipSDP by Fazlyab et al. \cite{fazlyab2019efficient}
    is an approach based on semidefinite programming (SDP) 
    to estimate the $\ell^2$-Lipschitz constant
    in polynomial time.
    This method was later extended to estimate the $\ell^\infty$-Lipschitz constant \cite{wang2022quantitative}
    and to more general activation functions \cite{pauli2024novel}.
    Since solving SDPs is inherently computationally expensive, Fazlyab et al. introduced in \cite{fazlyab2023certified} 
    a more scalable algorithm
    to estimate the Lipschitz constant using the idea of Loop Transformation applied to the network's nonlinearities.

\item \textbf{Random neural networks:}
Neural networks with random weights have been studied in various contexts.
First, they describe the initialization of neural networks before training.
Thus, to understand the training dynamics one must first understand 
the properties of random neural networks at the beginning.
For example, this is explicit in the Neural Tangent Kernel regime \cite{jacot2018neural},
where it is shown 
that the training weights remain close to their initial values during training.
Second, they provide a simple, tractable model for understanding certain properties of neural networks.
This includes, for example, the study of adversarial examples,
see, e.g. \cite{bartlett2021adversarial,montanari2023adversarial} as discussed above,
or the separation capacity of neural networks \cite{dirksen2022separation}.
Third, they are of intrinsic mathematical interest,
as they are high-dimensional non-linear compositions of random matrices.
For example,
the injectivity of random neural networks has been studied in \cite{Maillard_Injectivity}.

The gradients of these random neural networks have also been studied in several works.
For example, in \cite{hanin2018}
    it is characterized when these gradients 
    explode or vanish as a function of depth and width.
In \cite{hanin2019deep},
the distribution of the gradients of a random neural network
has been characterized in the asymptotic regime of infinite depth and width.

\item \textbf{Direct prior work on Lipschitz constants of random neural networks}
Upper bounds for Lipschitz constants of ReLU neural networks with random weights and biases 
were obtained in previous work \cite{ngyuen_montufar,buchanan2021deep}
in the context of studying the lazy training regime \cite{chizat_lazytraining}.
In \cite{ngyuen_montufar}, the authors derive an upper bound for the Lipschitz constant of the feature map of the neural network,
i.e., the linear last layer is not included.
In our notation, \cite[Theorem 6.2]{ngyuen_montufar} states
that,
with high probability, the Lipschitz constant of the feature map can be bounded by $ O(2^L \max \left\{ 1, \sqrt{d/N} \right\}) $, up to
logarithmic factors.
Since the final linear layer of the network has a Lipschitz bound of order $ \sqrt{N}$,
one can derive a bound of $ O( 2^L \max{\sqrt{N;d}} ) $ for the overall Lipschitz constant 
of the neural network.
This is  weaker than the bound of $ \sqrt{d \log(eL) }  $
derived in this paper.
Our paper also derives bounds for the Lipschitz constant of the feature map.
Namely we show that it is bounded by $O(1)$ for sufficiently large $N$; see \Cref{corr:advanced_lip}.
As discussed in the introduction,
in \cite{buchanan2021deep}, it was shown that
in the case of zero-bias networks
one has an upper bound of $\lip_2 (\Phi) \lesssim \sqrt{d \ln (N)} $
under the assumption $ N \gtrsim d^4 L \ln^4 (N)$.
In this paper, we improve this result by showing 
that $ \lip_2 (\Phi) \lesssim \sqrt{d \ln (\ee L)} $
under the assumption $N \gtrsim d L^4 \ln^2(\ee L)$. 

In our previous paper \cite{geuchen2024upper},
it was shown
that $\lip_2 (\Phi) \asymp \sqrt{d} $ for $L=1$ and arbitrary distributions of the biases
and 
$ \sqrt{d} 
\lesssim \lip_2 (\Phi) 
\lesssim C^L \sqrt{Ld\ln (N/d)}$
for $L \ge 2$ and \emph{symmetric} distributions of the biases.
In the case of $L \ge 2 $,
the proof of the upper bound is based on 
a covering number argument,
which involves counting the number of possible activation patterns
of the neural network \cite{hanin2019deep}.
This number is closely related to the number of linear regions 
of a neural network \cite{montufar2014number,zaslavsky1975}.
Following the argument in \cite{montufar2017notes},
it has then been shown in \cite{geuchen2024upper} that the number of activation patterns
can be bounded from above by $\left( \frac{eN}{d+1} \right)^{L(d+1)} $.
Using a covering number argument, in \cite{geuchen2024upper} it was then shown that
$ \lip_2 (\Phi) \lesssim C^L \sqrt{Ld\ln (N/d)}$.

However, to the best of our knowledge, 
no prior work provides bounds on $\lip_p(\Phi)$ of random ReLU neural networks for $p \neq 2$.
Thus, the results in this paper fill this gap 
and provide a comprehensive analysis of the Lipschitz constants of random neural networks 
for all $p \in [1, \infty]$.

The recent conference paper \cite{evangelidis2026width} establishes upper bounds for the 
$\ell^2$-Lipschitz constant of neural networks with Gaussian weights and arbitrary 
Lipschitz and \emph{smooth} activation functions. 
Notably, the ReLU activation considered in the present work does not satisfy the 
smoothness assumption imposed in \cite{evangelidis2026width}. The upper bound established in 
\cite{evangelidis2026width}
grows exponentially with the network depth and is independent of the network width.
The proof strategy in \cite{evangelidis2026width} is closely related to that employed 
in the present work, as it relies on Dudley's inequality and covering arguments. 
Furthermore, \cite{evangelidis2026width} extends the analysis of Lipschitz constants at 
initialization to trained networks in the setting of lazy training. 
Adapting these techniques to the setting considered in the present work
with the aim of establishing bounds on the Lipschitz constant of ReLU networks after training
is outside the scope of the present paper and left for future work.

\end{enumerate}

%% file: setting.tex
\section{Setting and preliminary results}\label{sec:prelim}
\subsection{Setting}\label{sec:setting}
We consider ReLU networks with $d$ input neurons, a single output neuron and $L$ hidden layers of width $N$. 
Formally, we consider maps 
\begin{equation} \label{eq:relu-network}
  \Phi: \ \RR^d \to \RR,
  \quad
  \Phi(x)
  \defeq \left( V^{(L)} \circ \relu \circ V^{(L-1)} \circ  \dots \circ \relu \circ V^{(0)}\right) (x).
\end{equation}
Here, $\relu$ denotes the function 
\begin{equation*}
  \relu (x) \defeq \max \{0,x\}, \quad x \in \RR,
\end{equation*} 
and its application in \eqref{eq:relu-network} is to be understood componentwise, i.e., 
\begin{equation*}
  \relu ((x_1, \dots, x_N)) = (\relu (x_1),\dots, \relu(x_N)). 
\end{equation*}
The maps $V^{(\ell)}$ for $0 \leq \ell \leq L$ are affine:
there exist \emph{weights} $W^{(0)} \in \RR^{N \times d}$, $W^{(1)}, \dots, W^{(L-1)} \in \RR^{N \times N}$
and $W^{(L)} \in \RR^{1 \times N}$ as well as \emph{biases} $b^{(0)}, \dots, b^{(L-1)} \in \RR^N$ and $b^{(L)} \in \RR$
such that
\begin{equation*}
  V^{(\ell)} (x) = W^{(\ell)}x + b^{(\ell)}
  \quad \text{for every} \quad 0 \leq \ell \leq L.
\end{equation*}
We refer to a network as \emph{zero-bias}, if the biases are all set to zero.

In the present work, we specifically analyze the properties of \emph{random} ReLU networks, i.e., ReLU networks where the weights and biases are drawn at random. 
We collect the different assumptions on this random initialization that are used over the course of the present paper in the following. 
\begin{assumption} \label{assum:1}
We assume that the weights and biases
are randomly
chosen in the following way:
For $0 \leq \ell < L$, we have 
\begin{equation*}
  \left( W^{(\ell)}\right)_{i,j} \overset{\mathrm{i.i.d.}}{\sim} \mathcal{N}(0,2/N)
  \quad\text{and}\quad
  \left( W^{(L)}\right)_{1,j} \overset{\mathrm{i.i.d.}}{\sim} \mathcal{N}(0,1),
\end{equation*}
where all entries are jointly independent.
The entries of the biases are drawn independently from (possibly different) real-valued random distributions and are jointly independent of the weights.
\end{assumption}

\begin{assumption}\label{assum:2}
    Let $\Phi: \RR^d \to \RR$ be a random ReLU network following \Cref{assum:1}.
    We say that the biases are \emph{symmetric} if
    the components of the bias vectors are drawn from symmetric distributions.
\end{assumption}

\begin{assumption}\label{assum:3}
Let $\Phi: \RR^d \to \RR$ be a random ReLU network following \Cref{assum:2}.
We say that the biases satisfy a \emph{small-ball property} with constant $C_\tau > 0$ if,
for every $t \in \RR$ and $\eps > 0$, 
\[
\PP\left(b^{(\ell)}_j \in (t-\eps, t +\eps)\right) \leq C_\tau \cdot \eps
\]
for any $j$ and $\ell$.
This is equivalent to every $b^{(\ell)}_j$ having an absolutely continuous probability 
distribution with probability density function bounded by $C_\tau/2$ almost everywhere; see \Cref{thm:abs_cont}.
\end{assumption}

%% file: preliminaries.tex
\subsection{The Lipschitz constant and the formal gradient of ReLU networks}
\label{subsec:lip}
\newcommand{\lipp}[2]{\lip_{#1 \rightarrow #2}}
For any $K \subseteq \RR^d$, any function $f : K \to \RR^k$ and any $p,q \in [1,\infty]$ we denote its \emph{Lipschitz constant} with respect to $\Vert \cdot \Vert_p$ and
$\Vert \cdot \Vert_q$ as
\begin{equation*}
  \lipp{p}{q} (f)
  \defeq  \underset{x \neq y}{\underset{x,y \in K}{\sup}} \frac{\mnorm{ f(x) - f(y)}_q }{\Vert x - y \Vert_p}.
\end{equation*}
We call $f$ \emph{Lipschitz continuous} if its Lipschitz constant is finite.\footnote{This is independent of $p$ and $q$ as all norms are equivalent on $\RR^n$.} Since we study ReLU networks $\Phi$ with an output dimension of $1$,
we have $\mnorm{\Phi(x) - \Phi(y)}_q = \abs{\Phi(x) - \Phi(y)}$ for any $q \in [1, \infty]$.
Therefore, the value of $\lipp{p}{q}(\Phi)$ is actually independent of $q$ and we simply write $\lip_p(\Phi)$ instead of $\lipp{p}{q}(\Phi)$.

The following well-known proposition establishes a relation between the Lipschitz constant
and the Jacobian of a function.
\begin{proposition}\label{prop:lipgrad}
  Let $K \subseteq \RR^d$ be a convex set with non-empty interior. 
  Let $f: \RR^d \to \RR^k$ be Lipschitz continuous and $M \subseteq \RR^d$ any measurable subset
  of $\RR^d$ with Lebesgue measure $\metalambda^d (\RR^d \setminus M) = 0$
  such that $f$ is differentiable in every $x \in M$. 
  Let $p,q \in [1,\infty]$.
  Then,
  \begin{equation*}
    \lipp{p}{q}\left(\fres{f}{K}\right) = \underset{x \in M\cap \inte(K)}{\sup} \Vert \mathrm{D} f (x) \Vert_{p \rightarrow q}.
  \end{equation*}
  In particular, if $k =1$, 
  \begin{equation*}
    \lip_p\left(\fres{f}{K}\right) = \underset{x \in M\cap \inte(K)}{\sup} \Vert \nabla f (x) \Vert_{p'},
  \end{equation*}
  where $p'\in [1,\infty]$ satisfies $1/p + 1/p' = 1$.
\end{proposition}
Note that the existence of a set $M$
with the required properties follows from the fact that $f$ is Lipschitz continuous;
this is known as Rademacher's theorem (cf.\ \cite[Section~3.1.2]{evans_measure_1992}). 
For the sake of completeness, we provide a proof of \Cref{prop:lipgrad} in \Cref{sec:prelim_proofs}.

Inspired by the above proposition, we base our estimation of the Lipschitz constant
of a ReLU network on the computation of its gradient.
Note, however, that a ReLU network is \emph{not} necessarily everywhere differentiable,
since the ReLU itself is not differentiable at $0$.
To overcome this problem, in the following, we introduce a computable proxy of the gradient which coincides almost everywhere with the actual gradient.

For this purpose, we introduce the following notation:
Using the same notation as in \Cref{sec:setting},
let $x =: \Phi^{(-1)}(x) \in \RR^d$ and define recursively 
\begin{align} \label{eq:d-matrices}
  D^{(\ell)}(x)
  &\defeq \Delta(W^{(\ell)}\Phi^{(\ell-1)}(x) + b^{(\ell)}), \nonumber\\
  \Phi^{(\ell)}(x)
  &\defeq D^{(\ell)}(x) \cdot \left(W^{(\ell)}\Phi^{(\ell-1)}(x) + b^{(\ell)}\right)
         = \relu(W^{(\ell)}\Phi^{(\ell-1)}(x) + b^{(\ell)}),
  \quad 0 \leq \ell < L.
\end{align}
Note that then $\Phi(x) = W^{(L)}\Phi^{(L-1)}(x) + b^{(L)}$.
For convenience, we further set $D^{(L)}(x) \defeq 1$ for every $x \in \RR^d$.

Using this notation, we define the formal gradient of the network, and similarly the formal Jacobian of the intermediate layers as follows.
\begin{definition}\label{def:formal}
Let $\Phi: \RR^d \to \RR$ be a ReLU network with $L$ hidden layers as in \Cref{sec:setting}, $x \in \RR^d$ and $\ell \in \{-1,\ldots,L-1\}$.
We define the \emph{formal Jacobian} of $\Phi^{(\ell)}$ at $x$ and the \emph{formal gradient} 
of $\Phi$ at $x$ by 
\begin{align}\label{eq:form_grad}
\jac \Phi^{(\ell)}(x) &\defeq
 \prod_{j=\ell}^0 D^{(j)}(x)W^{(j)}
 \quad \text{and} \quad
\nablaa \Phi (x) \defeq \left(W^{(L)}\left[\prod_{j=L-1}^0 D^{(j)}(x)W^{(j)}\right]\right)^T,
\end{align}
respectively. 
Here, the reverse order matrix product is as defined in \eqref{eq:rev_order}.
This in particular implies $\jac \Phi^{(-1)}(x) = I_{d \times d}$.
Furthermore, for $\ell_1 \in \{0,\ldots,L\}$ and $\ell_2 \in \{-1,\ldots,L-1\}$, we define %
\begin{equation*}
\jac \Phi^{(\ell_1) \to (\ell_2)} (x) \defeq \prod_{j= \ell_2}^{\ell_1} D^{(j)}(x)W^{(j)}, \quad x \in \RR^d. \tag*{\(\diamond\)}
\end{equation*}
\end{definition}
As already indicated above, the formal gradient and Jacobian serve as a proxy for their true counterparts. This is justified by the following proposition.
\begin{proposition}\label{prop:grad_relu}
Let $\Phi : \RR^d \to \RR$ be a ReLU neural network with $L$ hidden layers as in \Cref{sec:setting}
and fix $\ell \in \{-1,\ldots,L-1\}$. 
Then, for almost every $x \in \RR^d$, $\Phi^{(\ell)}$ is differentiable at $x$ with
\begin{equation*}
\mathrm{D} \Phi^{(\ell)}(x) = \jac \Phi^{(\ell)}(x).
\end{equation*}
Moreover, for almost every $x \in \RR^d$ the entire network $\Phi$ is differentiable at $x$ with  
\begin{equation*}
\nabla \Phi (x)  = \nablaa \Phi (x).
\end{equation*}
\end{proposition} 
The above is already well known to the community and the statement can (at least partially) be found in \cite[Theorem III.1]{berner2019towards}. In fact, the proof in \cite{berner2019towards} already includes the statement for the intermediate layers but it is not stated explicitly in the theorem. For this reason, and in order to clarify the proof itself, we decided to include a detailed argument in \Cref{sec:prelim_proofs}.

\begin{rem}
We emphasize that the formal gradient may in general \emph{not} be viewed as an extension of the true gradient. In fact, there might exist points of differentiability where $\act \Phi(x) \neq \nablaa \Phi(x)$.

As an example, consider the shallow zero-bias ReLU network $\Phi: \RR\to \RR$ with weight matrices 
\[
W^{(0)} = \left(\begin{matrix} 1 \\ -1\end{matrix}\right)
\quad \text{and} \quad W^{(1)} = \left(\begin{matrix} 1& -1\end{matrix}\right).
\]
A simple computation shows $\Phi(x) = x$ for every $x \in \RR$.
Hence, $\Phi$ is everywhere differentiable with $\act \Phi (x) = 1$.
On the other hand, $D^{(0)}(0) = \left(\begin{matrix} 0 & 0 \\ 0 & 0\end{matrix}\right)$ 
and therefore also $\nablaa \Phi(0) = 0$.
\end{rem}

At a fixed non-zero input, a random ReLU network is almost surely differentiable with the actual gradient coinciding with the
formal gradient; see \cite[Theorem~E.1]{geuchen2024upper}. Combining this fact with \Cref{prop:grad_relu,prop:lipgrad} we get the following result, which is the starting point in our analysis of the Lipschitz constant of random ReLU networks.
\begin{theorem}\label{thm:glob}\label{thm:up_low_bound}
Let $\Phi: \RR^d \to \RR$ be a random ReLU network with $L$ hidden layers satisfying \Cref{assum:1}, and let
$p,p',q \in [1,\infty]$ with $1/p + 1/p' = 1$.
Let 
\[
M_\Phi \defeq \{x \in \RR^d: \ \Phi \text{ differentiable at } x \text{ with } \nabla \Phi(x) = \nablaa \Phi(x)\}.
\]
Moreover, fix $x_0 \in \RR^d \setminus \{0\}$.
Then, almost surely with respect to the randomness in $\Phi$, 
\[
\mnorm{\nablaa \Phi (x_0)}_{p'} \leq \lip_p(\Phi)= \underset{x \in M_\Phi}{\sup}\mnorm{\nablaa \Phi (x)}_{p'}
 \leq \underset{x \in \RR^d}{\sup}\mnorm{\nablaa \Phi (x)}_{p'}
\]
and, in case $\Phi$ is zero-bias, even
\begin{equation*}
\mnorm{\nablaa \Phi (x_0)}_{p'} \leq \lip_p(\Phi)= \underset{x \in M_\Phi \cap \SS^{d-1}}{\sup}\mnorm{\nablaa \Phi (x)}_{p'}
 \leq \underset{x \in \SS^{d-1}}{\sup}\mnorm{\nablaa \Phi (x)}_{p'}.
\end{equation*}
Moreover, for every convex set $K \subseteq \RR^d$ with non-empty interior and $\ell \in \{-1, \dots, L-1\}$, 
\[
\lip_{p \to q}\left(\fres{\Phi^{(\ell)}}{K}\right) \leq \sup_{x \in \inte(K)} \mnorm{\jac \Phi^{(\ell)}(x)}_{p \to q}
\]
and, in case $\Phi$ is zero-bias, 
\[
\lip_{p \to q}\left(\Phi^{(\ell)}\right) \leq \sup_{x \in \SS^{d-1}} \mnorm{\jac \Phi^{(\ell)}(x)}_{p \to q}.
\]
\end{theorem}
Notably, in the special case of a zero-bias network, 
we may restrict ourselves to inputs from the unit sphere.
This is mainly due to the fact that in this case the formal gradient is invariant under scaling by positive constants; see \Cref{sec:prelim_proofs} for the details.
The same is not possible if one considers networks with non-zero biases.
The possibility to restrict to the sphere simplifies the analysis of the Lipschitz constant for zero-bias networks tremendously, since the compactness of the sphere allows for a covering-based approach. Hence, we consider the zero-bias case first in \Cref{sec:upper,sec:lower} and handle the case of networks with general biases by reducing it to the zero-bias setting in \Cref{sec:bias}.

\subsection{The randomized gradient}\label{sec:randgrad}
\newcommand{\D}{\overline{D}}
\newcommand{\DD}{\widehat{D}}
A major challenge in analyzing the gradient of a random ReLU network at a fixed input $x$ is that, while the weight matrices $W^{(\ell)}$
are stochastically independent, the matrices $D^{(\ell)}(x)$ are \emph{not} independent, since each depends on all the weights and biases
from the previous layers. 
In this section, we show that, \emph{for a fixed input} $x$, in the case of symmetric biases, on a high probability event and \emph{in distribution} we may effectively replace 
the $D^{(\ell)}(x)$ by matrices that are independent of the weights and biases, and of each other. 
This significantly simplifies the analysis of the gradient at a fixed input. 

The motivation behind the specific definition of the matrices $D^{(\ell)}(x)$ constructed in \eqref{eq:d-matrices} is that the ReLU is differentiable on $\RR \setminus \{0\}$
with $\relu' \equiv 1$ on $\RR_{> 0}$ and $\relu' \equiv 0$ on $\RR_{<0}$. 
However, it is not quite clear which value one should assign for a diagonal entry of a $D$-matrix, if the pre-activation, 
i.e., the input of the $\relu$, of the corresponding neuron is exactly zero.
Therefore, as was already noted in \cite{bartlett2021adversarial}, 
to facilitate the analysis of the gradient at a fixed point when considering random networks, it is convenient to introduce an 
additional source of randomness:
if the preactivation of a single neuron is exactly zero, 
then the corresponding entry on the diagonal of the $D$-matrix is set to either $0$ or $1$,
each with probability $1/2$. 
We formalize this idea in the following definition.
 \begin{definition}\label{def:dtilde}
Fix $x \in \RR^d$ and let $\Phi:\RR^d \to \RR$ be a random ReLU network with $L$ hidden layers of width $N$ as in \Cref{assum:1}. 
For $\ell \in \{0,\ldots,L-1\}$ we define the random matrices $\D^{(\ell)}(x)$ via
\begin{equation*}
\D^{(\ell)}(x) \defeq
D^{(\ell)}(x) + \mathrm{diag}\left(
\mathbbm{1}_{(W^{(\ell)}\Phi^{(\ell-1)}(x))_1 + b^{(\ell)}_1 = 0} 
\cdot \eps^{(\ell)}_{1}, \ldots, \mathbbm{1}_{(W^{(\ell)}\Phi^{(\ell-1)}(x))_N + b^{(\ell)}_N = 0} 
\cdot \eps^{(\ell)}_{N}\right),
\end{equation*}
where $\eps_i^{(\ell)} \iid \mathrm{Unif}(\{0,1\})$, for $\ell \in \{0,\ldots,L-1\}$ and $i \in \{1,\ldots,N\}$, 
are jointly independent of $(W^{(0)}, \ldots, W^{(L)}, b^{(0)}, \ldots, b^{(L)})$. 
\end{definition}
Notably, at a fixed input $x \in \RR^d \setminus \{0\}$, the matrices $\D^{(\ell)}(x)$ agree with the matrices $D^{(\ell)}(x)$ with high probability. 
This is shown in the following result, whose proof is deferred to \Cref{sec:prelim_proofs}.
\begin{proposition}\label{prop:randnorm}
Fix $x \in \RR^d \setminus \{0\}$ and let $\Phi: \RR^d \to \RR$ be a random ReLU network with symmetric biases and with $L$ hidden layers of width $N$ as in \Cref{assum:2}. 
Then, with probability at least $\left(1 - \frac{1}{2^N}\right)^L$, 
\[
D^{(\ell)}(x) = \D^{(\ell)}(x) \quad \text{for every} \quad \ell \in \{0,\ldots,L-1\}.
\]
\end{proposition}
Having introduced the randomized matrices $\D^{(\ell)}(x)$, one can show that \emph{in distribution}
 one can replace the matrices $\D^{(\ell)}(x)$ by \emph{independent} 
diagonal matrices $\DD^{(\ell)}$ with i.i.d. $\mathrm{Unif}\{0,1\}$-distributed entries on the diagonal. 
This distributional equivalence was originally stated and proven in \cite[Lemma~2.1]{bartlett2021adversarial}. 
In the following, we prove a generalization of that statement, where the arguments of the $\D$-matrices may vary from layer to layer; 
see below for a precise formulation. 
The proof can be found in \Cref{sec:distr_equiv}.
\begin{proposition}[{generalization~of~\cite[Lemma~2.1]{bartlett2021adversarial}}]\label{prop:randgrad}
Let $\Phi:\RR^d \to \RR$ be a random ReLU network with symmetric biases
and with $L$ hidden layers of width $N$ as in \Cref{assum:2}. 
Fix $\ell_1  \in \{0,\ldots,L-1\}$, $\ell_2 \in \{-1,\ldots,L-1\}$ and vectors $z_0,\ldots, z_{L-1} \in \RR^d$.
Set 
\[
\mu \defeq \begin{cases} d,& \text{if } \ell_1 = 0, \\ N,& \text{otherwise,}\end{cases}
\]
let $\nu \in \NN$ be arbitrary and let $A = A\left(W^{(0)}, \ldots, W^{(\ell_1 - 1)},b^{(0)},\ldots,b^{(\ell_1 -1)}\right) \in \RR^{\mu \times \nu}$ be a matrix which may solely depend on $W^{(0)}, \ldots, W^{(\ell_1 - 1)}$ and $b^{(0)},\ldots,b^{(\ell_1 -1)}$.
Then,
\begin{align*}
\op{\left[\prod_{i= \ell_2}^{\ell_1} \D^{(i)}(z_i) W^{(i)}\right]A} &\d \op{\left[\prod_{i= \ell_2}^{\ell_1} \DD^{(i)}W^{(i)}\right]A}
\quad \text{and} \\
W^{(L)}\left[\prod_{i= L-1}^{\ell_1} \D^{(i)}(z_i) W^{(i)}\right] &\d W^{(L)}\left[\prod_{i= L-1}^{\ell_1} \DD^{(i)}W^{(i)} \right]
\end{align*}
with 
\begin{equation*}
\DD^{(i)} \defeq\mathrm{diag} \left(\eps_1^{(i)}, \hdots, \eps_N^{(i)}\right) \quad \text{for every } i \in \{0,\ldots,L-1\}.
\end{equation*}
Here, the $\eps_j^{(i)}$ are i.i.d. $\mathrm{Unif}\{0,1\}$ and jointly independent of $(W^{(0)}, \dots, W^{(L)},b^{(0)}, \dots, b^{(L)})$.
\end{proposition}

\Cref{prop:randgrad,prop:randnorm} together show that, by passing to a high probability event, we may effectively replace the matrices $D^{(\ell)}$ by the randomized matrices $\D^{(\ell)}$ and afterwards replace these matrices by the stochastically independent diagonal matrices $\DD^{(\ell)}$.
We use this technique throughout this paper, especially in \Cref{sec:pw}.

\subsection{Random hyperplane tessellations} \label{sec:tess} 
It is of central importance in the present work to control the number of neurons at which the activation patterns differ for two inputs $x,y \in \RR^d$.
More precisely, one wants to control the expression $\Tr \abs{D^{(\ell)}(x) - D^{(\ell)}(y)}$.

To gain some intuition, let us for now focus on the case $\ell = 0$. 
Then the problem of estimating
\begin{equation*}
\Tr \abs{D^{(0)}(x) - D^{(0)}(y)}
= \# \left\{ i \in \{1,\ldots,N\}: \ \sgn\left((W^{(0)}x)_i + b^{(0)}_i\right) \neq \sgn\left((W^{(0)}y)_i + b^{(0)}_i\right)\right\}
\end{equation*}
is intimately related to \emph{random hyperplane tessellations} (cf. for instance
\cite{plan2014dimension,dirksen2021non}).
Indeed, each row of $W^{(0)}$ and the corresponding entry of the bias vector $b^{(0)}$ determine a random hyperplane in $\RR^d$ and $\Tr \abs{D^{(0)}(x) - D^{(0)}(y)}$ corresponds to the number of hyperplanes separating $x$ and $y$.

In this work, we make heavy use of the results presented in \cite{oymak2015near,dirksen2021non}. 
These provide \emph{local} uniform tessellation results for random hyperplane tessellations,
meaning that with high probability the number of separating hyperplanes is large for all vectors $x,y$ with $\mnorm{x-y}_2$ sufficiently large and small for all vectors $x,y$ with $\mnorm{x-y}_2$ sufficiently small. These results were obtained in \cite{oymak2015near} without biases and in \cite{dirksen2021non} with biases drawn from a uniform distribution.
In fact, for many results in \cite{dirksen2021non} one does not necessarily need to require the biases to be drawn from a uniform distribution; it suffices to assume 
that they satisfy a small-ball property, as we do in \Cref{corr:tess}.

For two \emph{fixed} inputs $x,y \in \SS^{d-1}$, we get the following well-known result. It states that the fraction
of separating random hyperplanes (with zero biases) between $x$ and $y$ is proportional to their Euclidean distance, with high
probability. 
For the sake of completeness, we provide a proof in \Cref{sec:rand_tess}.
\begin{proposition}\label{lem:recht_low_low}
    There exist absolute constants $C,c > 0$ such that the following holds. 
    For $m,n \in \NN$ and $\sigma > 0$, let $A \in \RR^{m \times n}$ 
    be a random matrix with entries $A_{i,j} \iid \mathcal{N}(0,\sigma^2)$, 
    let $x, y \in \SS^{n-1}$ and set $\delta \defeq \mnorm{x-y}_2$. 
    Let $\sgn \in \{\sgn_+, \sgn_-, \sgn_0\}$.
    Then, with probability at least $1-2\exp(-c \cdot \delta m)$,
    \[
    c \cdot \delta m \leq \# \big\{i \in \{1,\ldots,m\}: \ \sgn((Ax)_i) \neq \sgn((Ay)_i)\big\} \leq C \cdot \delta m.
    \]
\end{proposition}

Next, we move on to the localized results. 
The relevant result from \cite{oymak2015near} reads as follows. 
\begin{proposition}[{\cite[Theorem~3.2]{oymak2015near}}] \label{lem:recht_upper_original}
There exist absolute constants $C, c > 0$ such that the following holds. 
For $\sigma > 0$, let $A \in \RR^{m \times n}$ be a random matrix with entries $A_{i,j} \iid \mathcal{N}(0, \sigma^2)$, and let $K \subseteq \SS^{n-1}$ be non-empty and $\delta \in (0,\ee^{-1})$.
Suppose that $0 <\eps \leq c \cdot \delta \cdot \ln^{-1/2}(1/\delta)$ and
\[
m \geq C \cdot \max\big\{\delta^{-1} \ln(\ee \cdot \mathcal{N}(K, \eps/2)), \delta^{-3}w^2((K-K) \cap \B_n(0,\eps))\big\}.
\]
Let $\sgn \in \{\sgn_+, \sgn_-, \sgn_0\}$.
Then, with probability at least $1-\exp(-c \cdot \delta m)$,
\[
\# \big\{ i \in \{1,\ldots,m\} : \sgn((Ax)_i) \neq \sgn((Ay)_i)\big\} \leq \delta \cdot m
\quad \text{for all } x,y \in K \text{ with }\mnorm{x-y}_2 \leq \eps/2.
\]
\end{proposition}
Since \cite{oymak2015near} has to the best of our knowledge only appeared as a preprint so far and for the sake of completeness, we provide a proof for \Cref{lem:recht_upper_original} in \Cref{sec:rand_tess}. 

Moreover, for the case of networks with possibly non-zero biases, we make use of the following two results, which can be proven with similar techniques as in \cite{dirksen2021non}; see \Cref{sec:rand_tess} for the details.
The first one states that the bias terms do not play a significant role for inputs with large Euclidean norm. 
\begin{proposition}\label{prop:farout}
There exist absolute constants $C,c>0$ such that the following holds. Assume that ${A \in \RR^{m \times n}}$ is a random matrix with entries $A_{i,j} \iid \mathcal{N}(0, 1)$, and let $K \subseteq \SS^{n-1}$ be non-empty and $\delta \in (0,\ee^{-1})$. Suppose that $0 <\eps \leq c \cdot \delta \cdot \ln^{-1/2}(1/\delta)$
and
\[
m \geq C \cdot \max\big\{\delta^{-1} \cdot \ln(\ee \cdot \mathcal{N}(K, \eps)), \delta^{-3}\cdot w^2((K-K) \cap \B_n(0,\eps))\big\}.
\]
Let $\sgn \in \{\sgn_+, \sgn_-, \sgn_0\}$, $\lambda > 0$ and $\tau_1, \ldots, \tau_m \in \RR$ such that $\max_i \abs{\tau_i} \leq \lambda$.
Then, with probability at least ${1-\exp(-c \cdot \delta m)}$,
\[
    \# \left\{ i \in \{1,\ldots,m\} : \ \sgn((Ax)_i + \tau_i) \neq \sgn((Ax)_i)\right\} \leq \delta m \quad 
    \text{for all } x \in \cone( K) \text{ with } \mnorm{x}_2 \geq C \cdot\delta^{-1} \lambda. 
\]
\end{proposition}

We can further establish a bound similar to \Cref{lem:recht_upper_original} for networks with non-zero biases, which generalizes the upper bound from \cite[Theorem~2.9]{dirksen2021non}.
\begin{proposition}\label{corr:tess}
There exist absolute constants $C,c>0$ such that the following holds. 
Assume that $A \in \RR^{m \times n}$ is a random matrix with $A_{i,j} \iid \mathcal{N}(0, 1)$ and let $b \in \RR^m$ be a random vector with independent entries which is independent of $A$. Suppose that the entries of $b$ satisfy a small-ball property with constant $C_\tau > 0$, that is, for any $i \in \{1, \ldots, m\}$, and for all $t \in \RR$ and $\eps > 0$, 
\[
\PP(b_i \in (t-\eps, t+ \eps)) \leq C_\tau \cdot \eps.
\]
Let further $\delta \in (0,\ee^{-1})$ and $K \subseteq \RR^n$ non-empty. Suppose that $\eps \leq c \cdot C_\tau^{-1} \cdot \delta \cdot \ln^{-1/2}(1/\delta)$
and
\[
m \geq C \cdot \max\big\{\delta^{-1} \cdot \ln(\ee \cdot \mathcal{N}(K, \eps/2)), \delta^{-3}\cdot w^2((K - K) \cap \B_n(0,\eps)) \cdot C_\tau^2\big\}.
\]
Let $\sgn \in \{\sgn_+, \sgn_-, \sgn_0\}$.
Then, with probability at least $1-\exp(-c \cdot \delta m)$,
    \begin{align*}
    \# \big\{i \in \{1,\ldots,m\}: \ \sgn((Ax)_i + b_i) \neq \sgn((Ay)_i + b_i)\big\} \leq  \delta m
    \quad \text{for all } x,y \in K \text{ with } \mnorm{x-y}_2 \leq \eps/2.
    \end{align*}
\end{proposition}

%% file: pointwise.tex
\section{Pointwise bounds on the Jacobian and gradient}\label{sec:pw}

As a first step towards bounding the Lipschitz constant of random ReLU networks, 
we provide \emph{pointwise} estimates on the formal Jacobian and gradient of random ReLU networks, i.e., estimates that hold at fixed, non-zero inputs but not necessarily uniformly over the input space.  
The proofs in this section heavily rely on the randomized gradient introduced in \Cref{sec:randgrad} and in particular exploit the distributional equivalence from \Cref{prop:randgrad}.
This allows to condition on the $\DD$-matrices and to then only consider the randomness over the weight matrices $W^{(\ell)}$, so that the results can be shown using standard bounds for Gaussian matrices and vectors. 
Note that, in order to be able to apply the techniques from \Cref{sec:randgrad}, we need to assume that the biases are symmetric, which we do throughout this section. 
Note that this in particular includes the case of zero-bias networks. 
We refer to \Cref{sec:pw_proofs} for detailed proofs of what follows. 

In particular, in \Cref{thm:pw} we provide upper and lower bounds for the $\ell^q$-norm 
of the gradient at a fixed input point of a random ReLU network with symmetric biases 
as in \Cref{assum:2}.
As already discussed in \Cref{thm:glob}, these bounds immediately yield a corresponding lower bound for the Lipschitz constant $\lip_p(\Phi)$, where $1/p + 1/q = 1$.
Quite surprisingly, in the case $p \in [2,\infty]$ these pointwise bounds on the gradient are up to a logarithmic factor of the same order as the upper bounds for the Lipschitz constant 
established in \Cref{thm:main_upper},
where the supremum over the entire input space is taken. 
Roughly speaking, taking the supremum over the entire input space does not change anything (except for logarithmic factors) compared to considering just a single point. 
Note that this is not the case in the regime $p \in [1,2)$; see \Cref{sec:lower}.

Moreover, in \Cref{lem:special,lem:pw_2,lem:pw_3}, we state three auxiliary results which are frequently used over the course of the paper and which we expect to 
be of independent interest.

We start with the following lemma, which states that matrices of the form 
$\prod_{i = \ell_2}^{\ell_1} D^{(i)}(z_i)W^{(i)}$ satisfy a certain isometry property. 
\begin{lemma}\label{lem:special}
There exist positive absolute constants $C,c>0$ such that the following holds. Let
$\Phi: \RR^d \to \RR$ be a random ReLU network with symmetric biases, with $L$ hidden layers of width $N$ as in \Cref{assum:2} and with $N \geq C \cdot L^2 \cdot \ln(\ee L)$.
Let $\ell_1 \in \{0,\ldots,L-1\}$, $\ell_2 \in \{-1,\ldots,L-1\}$, 
fix deterministic $z_0,\ldots z_{L-1} \in \RR^d \setminus \{0\}$,
let
\[
\mu \defeq \begin{cases}d,& \text{if } \ell_1 =0, \\ N,& \text{otherwise,}\end{cases}
\]
and let $v = v(W^{(0)}, \ldots, W^{(\ell_1 - 1)}, b^{(0)}, \dots, b^{(\ell_1 - 1)}) \in \RR^\mu$ be a vector which may solely depend on $W^{(0)}, \ldots, W^{(\ell_1 - 1)}$ and $b^{(0)}, \dots, b^{(\ell_1 - 1)}$.
Then, with probability at least $1-\exp(-c \cdot N/L^2)$,
\begin{equation*}
\frac{1}{2} \cdot \Vert v \Vert_2 \leq \mnorm{\left[\prod_{i = \ell_2}^{\ell_1} D^{(i)}(z_i)W^{(i)}\right]v}_2 \leq \ee \cdot  \Vert v \Vert_2.
\end{equation*}
\end{lemma}
Note that in the special case of a \emph{zero-bias} ReLU network $\Phi: \RR^d \to \RR$, we have 
\[
\Phi^{(\ell)}(x) = \left[\prod_{i = \ell}^{0} D^{(i)}(x)W^{(i)}\right]x, \quad x \in \RR^d.
\]
Hence, \Cref{lem:special} in particular states that in the case of a zero-bias random ReLU network following \Cref{assum:1}, we have 
\[
\mnorm{\Phi^{(\ell)}(x)}_2 \asymp \mnorm{x}_2
\]
with high probability. 
Note that this statement already appears in a slightly stronger form in \cite[Lemma~7.1]{allenarxiv}.

We continue with the following statement, which essentially states 
for the quantity $\jac \Phi^{(\ell_1) \to (\ell_2)}(x)$ 
from \Cref{def:formal} that 
\[
\op{\jac\Phi^{(\ell_1) \to (\ell_2)}(x) A} \lesssim \op{A}
\]
for every matrix $A$ with $\rang(A) \leq N/L^2$.
We even state a slightly more general version, where the arguments of the $D$-matrices
may vary from layer to layer. 
\begin{lemma}\label{lem:pw_2}
There exist absolute constants $C ,c > 0$ such that the following holds.
Let $\Phi: \RR^d \to \RR$ be a random ReLU network with symmetric biases, with $L$ hidden layers of width $N$ as in \Cref{assum:2} with $N \geq C \cdot L^2 \cdot \ln(\ee L)$.
Moreover, fix $\ell_2 \in \{-1,\ldots,L-1\}$, $\ell_1 \in \{0, \dots, L\}$ and deterministic 
$z_0, \ldots, z_{L-1} \in \RR^d \setminus \{0\}$. 
Let 
\begin{equation*}
\mu \defeq \begin{cases} d  & \text{if }\ell_1 = 0, \\ N & \text{if }\ell_1 > 0,\end{cases}
\end{equation*}
pick $\nu \in \NN$
and let $A = A(W^{(0)}, \ldots, W^{(\ell_1 -1)}, b^{(0)}, \ldots, b^{(\ell_1 -1)})\in \RR^{\mu \times \nu}$ with $\rang(A) \leq N/L^2$.
Then,
\begin{equation*}
\op{\left[\prod_{i= \ell_2}^{\ell_1} D^{(i)}(z_i)W^{(i)}\right] A} \leq C \cdot \op{A}
\end{equation*}
with probability at least $1-\exp(-c \cdot N/L^2)$. 
\end{lemma}
Note that, by putting $\ell_1 = 0$, $A = I_{d \times d}$ and 
$z_0 = \hdots = z_{L-1} = x$, \Cref{lem:pw_2} provides a bound on the spectral norm of 
$\jac \Phi^{(\ell_2)}(x)$, under the assumption $N \gtrsim L^2(d + \ln(\ee L))$.

Next, the following lemma provides an estimate similar to the one from \Cref{lem:pw_2} incorporating also the final layer $W^{(L)}$.

\begin{lemma}\label{lem:pw_3}
There exist absolute constants $C>0$ and $c > 0$ such that the following holds.
Let $z_0,\ldots,z_{L-1} \in \RR^d \setminus \{0\}$ be fixed and let $\Phi: \RR^d \to \RR$ be a random ReLU network with symmetric biases, with $L$ hidden layers of width $N$ as in \Cref{assum:2}
with $N \geq C \cdot L^2 \cdot \ln(\ee L)$ 
and moreover, fix $\ell \in \{0,\ldots,L\}$. 
Let 
\begin{equation*}
\mu \defeq \begin{cases} d & \text{if }\ell = 0, \\ N & \text{if }\ell > 0,\end{cases}
\end{equation*}
let $\nu \in \NN$ and let $A \in \RR^{\mu \times \nu}$ with $\rang(A) \leq N/L^2$.
Then, for every $C \leq t \leq \sqrt{N}/L$,
\begin{equation*}
\op{W^{(L)} \cdot \left[\prod_{i= L-1}^\ell D^{(i)}(z_i)W^{(i)}\right] \cdot A} \leq C \cdot \op{A} \cdot  (\sqrt{\rang(A)} + t)
\end{equation*}
with probability at least $1-\exp(-c \cdot t^2)$. 
\end{lemma}
In particular, picking $\ell = 0$, $z_0 = \cdots = z_{L-1} = x$ and $A = I_{d\times d}$, we get a pointwise upper bound on 
the $\ell^2$-norm of the gradient $\nablaa \Phi(x)$ and hence in fact a special case of \Cref{thm:pw}.

Last but not least, we provide bounds on the $\ell^p$-norm of the gradient of 
a random ReLU network $\Phi: \RR^d \to \RR$ at a fixed non-zero input $x_0 \in \RR^d \setminus \{0\}$
as already discussed in the beginning of this section.
\begin{theorem}\label{thm:pw}
There exist absolute constants $C,c>0$ such that the following holds.
Let $\Phi: \RR^d \to \RR$ be a random ReLU network with symmetric biases, 
with $L$ hidden layers of width $N$ as in \Cref{assum:2} 
with $N \geq C \cdot L^2 \cdot \ln(\ee L)$ and $d \geq C$ 
and let $x_0 \in \RR^d \setminus \{0\}$ be fixed. 
Then the following hold.
\begin{enumerate}
{\item If $p \in [1,2]$, we have
\begin{equation*}
\PP \left(c \cdot d^{1/p} \leq \mnorm{\nablaa \Phi (x_0)}_p \leq C \cdot d^{1/p}\right) \geq 1 - \exp( -c \cdot \min\{d, N/L^2\} ).
\end{equation*}
}
\item{
If $p \in (2,c \cdot \ln(d))$, we have 
\begin{equation*}
\PP \left(c \cdot \sqrt{p} \cdot d^{1/p} \leq \mnorm{\nablaa \Phi (x_0)}_p \leq C \cdot \sqrt{p} \cdot d^{1/p}\right) \geq 1- \exp(-c \cdot \min\{p \cdot d^{2/p}, N/L^2\}).
\end{equation*}
}
\item{
If $p \geq \ln(d)$, we have
\begin{equation*}
\PP \left(c \cdot \sqrt{\ln (d)} \leq \mnorm{\nablaa \Phi (x_0)}_{p} \leq C \cdot \sqrt{\ln(d)}\right) \geq 1- \exp(-c \cdot \min\{\ln(d), N/L^2\}).
\end{equation*}
}
\end{enumerate}
\end{theorem}

\begin{rem}
We note that in \cite{Hanin2020}
the distribution of the gradient of neural networks
was characterized in the asymptotic limit $N,d \rightarrow \infty$
and shown to be a $\log$-normal distribution.
This is different to our work where we are interested in asymptotic high-probability bounds.
\end{rem}
The proof of \Cref{thm:pw} is based on a careful evaluation of the gradient $\nablaa \Phi(x_0)$ using the techniques from \Cref{sec:randgrad}, combined with the following result regarding the concentration
of the $\ell^p$-norm of a standard Gaussian vector.  

\begin{lemma}[{Consequence of \cite[Proposition~2.4,Theorem~4.11]{PAOURIS20173187}}]\label{thm:p_high_prob}
There exist absolute constants $C,c>0$ such that the following holds for $d \geq C$:
\begin{enumerate}
\item For every $p \in [1,2]$ we have
\begin{equation*}
\underset{X \sim \mathcal{N}(0, I_d)}{\PP} \left(c \cdot d^{1/p} \leq \Vert X \Vert_p \leq C \cdot d^{1/p}\right) \geq 1 - C \cdot \exp(-c \cdot d).
\end{equation*}
\item For every $p \in (2, c \cdot \ln(d))$ we have
\begin{equation*}
\underset{X \sim \mathcal{N}(0, I_d)}{\PP} \left(c \cdot \sqrt{p} \cdot d^{1/p} \leq \Vert X \Vert_p \leq C \cdot \sqrt{p}\cdot d^{1/p}\right) 
\geq 1 - C\cdot \exp(-c \cdot p \cdot d^{2/p}).
\end{equation*}
\item For every $p \geq \ln(d)$ we have
\begin{equation*}
\underset{X \sim \mathcal{N}(0, I_d)}{\PP} \left(c \cdot \sqrt{\ln(d)} \leq \Vert X \Vert_p \leq C \cdot \sqrt{\ln(d)}\right) \geq 1 - C \cdot \exp(-c\cdot \ln(d)).
\end{equation*}
\end{enumerate}
\end{lemma}
Below, we provide a proof sketch for \Cref{thm:pw}, illustrating the main ideas of the proof.
The proofs of \Cref{lem:special,lem:pw_2,lem:pw_3} are based on a similar approach.
\begin{proof}[Proof sketch of \Cref{thm:pw}]
    By definition of the formal gradient (see \Cref{def:formal}), 
    \[
    (\nablaa \Phi(x_0))^T = W^{(L)}\left[\prod_{\ell= L-1}^0 D^{(\ell)}(x_0)W^{(\ell)}\right].
    \]
    Using \Cref{prop:randnorm}, by passing to a high probability event and by using the distributional equivalence in \Cref{prop:randgrad}, it suffices
    to study 
    \[
     W^{(L)}\left[\prod_{\ell= L-1}^0 \D^{(\ell)}(x_0)W^{(\ell)}\right]
     \d W^{(L)}\left[\prod_{\ell= L-1}^0 \DD^{(\ell)}W^{(\ell)}\right],
    \]
    where the matrices $\D^{(\ell)}(x_0)$ and $\DD^{(\ell)}$ are defined as in 
    \Cref{sec:randgrad}.
    Since the matrices $\DD^{(\ell)}$ are independent of the weight matrices $W^{(\ell)}$, 
    we can condition on $\DD^{(0)}, \hdots, \DD^{(L-1)}$ and assume that 
    \begin{equation}\label{eq:d-matrices_con}
    \text{for all } \ell \in \{0,\ldots,L-1\} : 
    \quad \frac{N}{2}\left(1- \frac{1}{4L}\right) \leq \Tr(\DD^{(\ell)}) 
    \leq \frac{N}{2}\left(1+ \frac{1}{4L}\right),
    \end{equation}
    which occurs with probability at least $1- 2\exp(-c \cdot N/L^2)$. 
    Because of $\Tr(\DD^{(\ell)}) = \sum_{j=1}^N \eps_j^{(\ell)}$, this follows from 
    Hoeffding's inequality for bounded random variables 
    (see, e.g., \cite[Theorem~2.2.6]{vershynin_high-dimensional_2018}) 
    and the assumption $N \geq C \cdot L^2 \cdot \ln(\ee L)$.

    For fixed $\DD^{(\ell)}$-matrices, we set $\alpha_\ell \defeq \Tr(\DD^{(\ell)})$.
    We obtain 
    \[
    \mnorm{W^{(L)}\left[\prod_{\ell= L-1}^0 \DD^{(\ell)}W^{(\ell)}\right]}_p
    \d \mnorm{\prod_{\ell = L}^0 W_\dagger^{(\ell)}}_p,
    \]
    with random matrices 
    $W_\dagger^{(L)} \in \RR^{1 \times \alpha_{L-1}}$, 
    $W_\dagger^{(\ell)} \in \RR^{\alpha_\ell \times \alpha_{\ell-1}}$ for 
    $\ell \in \{1,\ldots,L-1\}$ and 
    $W_\dagger^{(0)} \in \RR^{\alpha_0 \times d}$ satisfying 
    $(W_\dagger^{(L)})_{1,j} \iid \mathcal{N}(0,1)$ and 
    $(W_\dagger^{(\ell)} )_{i,j} \iid \mathcal{N}(0,2/N)$ for $\ell \in \{0,\ldots,L-1\}$ and 
    for which the matrices 
    $W_\dagger^{(0)},\ldots, W_\dagger^{(L)}$ are jointly independent. 

    We now condition on the matrices $W_\dagger^{(1)},\ldots, W_\dagger^{(L)}$ and set 
    \[
    Z \defeq \prod_{\ell = L}^1 W_\dagger^{(\ell)} \in \RR^{1 \times \alpha_0}.
    \]
    Then,
    \[
    (Z \cdot W_\dagger^{(0)})^T \sim \mathcal{N}\left(0,\frac{2 \mnorm{Z}_2^2}{N} \cdot I_d\right)
    \]
    where the randomness is only with respect to $W_\dagger^{(0)}$.
    Hence, using \Cref{thm:p_high_prob}, we get for instance in the case $p \in [1,2]$
    that 
    \[
    \mnorm{Z \cdot W_\dagger^{(0)}}_p \asymp \frac{\mnorm{Z}_2}{\sqrt{N}} \cdot d^{1/p}
    \]
    with high probability. 
    Hence, it remains to show that $\mnorm{Z}_2 \asymp \sqrt{N}$ with high probability. 
    This is achieved via an inductive reasoning under an iterative application of \cite[Theorem~3.1.1]{vershynin_high-dimensional_2018} (concentration of the $\ell^2$-norm of a Gaussian random vector) and using \eqref{eq:d-matrices_con} with
    \[
    \frac{1}{2} \leq\left(1-\frac{1}{2L}\right)^L \leq \left(1+\frac{1}{2L}\right)^L \leq \ee;
    \]
    see \Cref{sec:pw_proofs} for the details. 
\end{proof}

The following corollary records the consequences of \Cref{thm:pw} for the Lipschitz constants of random ReLU neural networks. 
It follows by applying \Cref{thm:pw} with the setting $p = 1$, using the standard inequality $\mnorm{\cdot}_1  \leq d^{1-1/p} \cdot \mnorm{\cdot}_p$ for $p \in [1,\infty]$ and by then invoking \Cref{thm:glob}.

\begin{corollary}\label{corr:pw}
There exist absolute constants $C,c > 0$ such that the following holds. 
For a random ReLU network $\Phi: \RR^d \to \RR$ with symmetric biases, with $L$ hidden layers of width $N$ as in \Cref{assum:2} with $d \geq C$ and $N \geq C \cdot L^2 \cdot \ln(\ee L)$, 
\[
\lip_p(\Phi) \geq c \cdot d^{1-1/p} \quad \text{for all } p \in [2, \infty]
\]
with probability at least $1-\exp(-c \cdot \min\{N/L^2,d\})$.
\end{corollary}

%% file: upper_hom.tex
\section{The upper bound for zero-bias networks}\label{sec:upper}
In this section, we state the upper bound for the Lipschitz constant in the case of zero-bias networks and provide a proof sketch for it.
The main result reads as follows. 
\begin{theorem}\label{thm:main_upper}
There exist absolute constants $C,c > 0$ such that the following holds. 
Let $\Phi: \RR^d \to \RR$ be a random zero-bias ReLU network with $L$ hidden layers of width $N$ satisfying \Cref{assum:1}. We assume 
\[
N \geq C \cdot dL^4 \ln^2(\ee L)
\quad \text{and} \quad d \ln(\ee L) \geq C.
\]
Then, with probability at least $1- \exp(-c \cdot d \ln(\ee L))$,
\[
\lip_2(\Phi) \leq \underset{x \in \SS^{d-1}}{\sup} \mnorm{\nablaa \Phi(x)}_2 \leq C \cdot \sqrt{d \ln(\ee L)}.
\]
\end{theorem}
Using the standard inequalities
\begin{equation}\label{eq:norm_equiv}
    \mnorm{\,\cdot\,}_{2} \leq \mnorm{\,\cdot\,}_{p} \quad \text{for } p \in [1,2]
    \qquad\text{and}\qquad
    d^{1/p - 1/2} \mnorm{\,\cdot\,}_{2} \leq \mnorm{\,\cdot\,}_{p} \quad \text{for } p \in [2, \infty],
\end{equation}
we obtain the following corollary.
\begin{corollary}\label{corr:lp_upper}
For $p \in [1,\infty]$, we have on the same event as in \Cref{thm:main_upper},
\begin{equation*}
\lip_p(\Phi) \leq 
C \cdot \begin{cases} \sqrt{d \cdot \ln(\ee L)} & \text{if }p \in [1,2], \\
 d^{1-1/p} \cdot \sqrt{\ln \left(\ee L\right)} & \text{if }p \in [2,\infty].
\end{cases}
\end{equation*}
\end{corollary}
Notice that we handle the two regimes $p \in [1, 2]$ and $p \in [2, \infty]$ by application of \eqref{eq:norm_equiv}, which might appear to be loose at first sight. 
However,
the resulting bound in fact matches the lower bounds from \Cref{thm:main_lower}.

We now give a proof sketch for \Cref{thm:main_upper};
the detailed proof is given in \Cref{sec:upper_proof}.
\begin{proof}[Proof sketch of \Cref{thm:main_upper}]
The overall strategy
is to \emph{condition} on the random vector 
\[
\arrow{W} \defeq (W^{(0)}, \ldots, W^{(L-1)})
\]
 and to then apply Dudley's inequality
(see, e.g., \cite[Theorem~8.1.6, Ex.~8.1.8]{vershynin_high-dimensional_2018}) when considering the randomness over $W^{(L)}$.
To this end, in view of \Cref{thm:glob} and \eqref{eq:form_grad}, we rewrite 
\[
\underset{x \in \SS^{d-1}}{\sup} \mnorm{ \nablaa \Phi(x)}_2 
= \underset{x,z \in \SS^{d-1}}{\sup} \abs{\langle (W^{(L)})^T, \jac \Phi^{(L-1)}(x)z\rangle} 
= \underset{v \in \mathcal{L}_{\mathrm{up}}}{\sup} \ \langle (W^{(L)})^T, v \rangle,
\]
where 
\[
\mathcal{L}_{\mathrm{up}} = \mathcal{L}_{\mathrm{up}}(\arrow{W}) \defeq \left\{ \jac \Phi^{(L-1)}(x)z : \ x,z \in \SS^{d-1} \right\} \subseteq \RR^N.
\]
The discrete version of Dudley's inequality then yields 
\[
\underset{x \in \SS^{d-1}}{\sup} \mnorm{\nablaa \Phi(x)}_2
\lesssim  \sum_{k \in \ZZ} 2^{-k} \sqrt{\ln(\mathcal{N}(\mathcal{L}_{\mathrm{up}}, 2^{-k}))} 
\]
with high probability. 
Setting $\Lambda = \Lambda(\arrow{W}) \defeq \sup\limits_{x \in \SS^{d-1}} \op{\jac \Phi^{(L-1)}(x)}$, 
we note that $\mathcal{L}_{\mathrm{up}} \subseteq \overline{B}_N(0, \Lambda)$ 
and consequently,
\begin{align*}
&\norel\sum_{k \in \ZZ} 2^{-k} \sqrt{\ln(\mathcal{N}(\mathcal{L}_{\mathrm{up}}, 2^{-k}))} 
= \underset{2^{-k} \leq 2\Lambda}{\sum_{k \in \ZZ}} 2^{-k} \sqrt{\ln(\mathcal{N}(\mathcal{L}_{\mathrm{up}}, 2^{-k}))} \\
&= \underset{c \leq 2^{-k} \leq 2\Lambda}{\sum_{k \in \ZZ}} 2^{-k} \sqrt{\ln(\mathcal{N}(\mathcal{L}_{\mathrm{up}}, 2^{-k}))} +
\underset{\gamma_0 \leq 2^{-k} < c}{\sum_{k \in \ZZ}} 2^{-k} \sqrt{\ln(\mathcal{N}(\mathcal{L}_{\mathrm{up}}, 2^{-k}))}
+ \underset{2^{-k} < \gamma_0}{\sum_{k \in \ZZ}} 2^{-k} \sqrt{\ln(\mathcal{N}(\mathcal{L}_{\mathrm{up}}, 2^{-k}))} \\
&=: \mathrm{I} + \mathrm{II} + \mathrm{III}.
\end{align*}
Here, $c \in (0,1)$ is a suitable absolute constant and $\gamma_0 \defeq 1/\sqrt{L \ln(N/d)}$.

Let $\gamma \in [\gamma_0,c]$, $\delta \asymp \frac{\gamma^2}{L^2} \cdot \ln^{-1}(L/\gamma)$ 
and take $\eps \asymp \delta \cdot \ln^{-1/2}(1/\delta)$ in the spirit of 
\Cref{lem:recht_upper_original}.
We let $\neps \subseteq \SS^{d-1}$ be an $\eps$-net of $\SS^{d-1}$ 
satisfying $\ln (\#\neps) \lesssim d \cdot \ln(1/\eps) \asymp d \cdot \ln(L/\gamma)$ 
(existence follows from \cite[Corollary~4.2.13]{vershynin_high-dimensional_2018}).
We claim that 
\begin{equation} \label{eq:net_prop}
\neps' \defeq \left\{ \jac \Phi^{(L-1)}(x^\ast)z^\ast: \ x^\ast, z^\ast \in \neps\right\} \subseteq \mathcal{L}_{\mathrm{up}}
\end{equation}
is a $\gamma$-net of $\mathcal{L}_{\mathrm{up}}$.
To see this, for $x,z \in \SS^{d-1}$ we let $\pi(x), \pi(z) \in \neps$ be chosen such that
\[
\max\left\{ \mnorm{x - \pi(x)}_2, \mnorm{z -\pi(z)}_2\right\} \leq \eps.
\]
Using the triangle inequality and $\mnorm{z}_2 = 1$,
\begin{align*}
    &\norel\mnorm{\jac \Phi^{(L-1)}(x)z - \jac \Phi^{(L-1)}(\pi(x))\pi(z)}_2 \nonumber\\
    \label{eq:first_appr}
    &\leq \op{\jac \Phi^{(L-1)}(x) - \jac \Phi^{(L-1)}(\pi(x))} + \op{\jac \Phi^{(L-1)}(\pi(x))} \cdot \mnorm{z - \pi(z)}_2.
\end{align*}
Combining \Cref{lem:pw_2} and a union bound, we get 
\begin{equation}\label{eq:netbound}
\text{for all } x^\ast \in \neps: \quad \op{\jac \Phi^{(L-1)}(x^*)} \lesssim 1
\end{equation}
with high probability. 
This yields 
\begin{equation*}\label{eq:second_summand}
\op{\jac \Phi^{(L-1)}(\pi(x))} \cdot \mnorm{z - \pi(z)}_2 \lesssim \eps \leq \gamma
\end{equation*}
with high probability. 

It thus remains to establish a bound for 
\[
\op{\jac \Phi^{(L-1)}(x) - \jac \Phi^{(L-1)}(\pi(x))}.
\]
In order to do so, we use a decomposition of this difference which already appeared
 in \cite{bartlett2021adversarial}.
Namely, with $\jac \Phi^{(\ell_1) \to (\ell_2)}$ as in \Cref{def:formal}, 
\begin{align*}\label{eq:dec}
&\norel\jac \Phi^{(L-1)} (\pi(x)) -\jac \Phi^{(L-1)}(x) \nonumber\\  
&= \sum_{j= 0}^{L-1} \left(\jac\Phi^{(j+1) \to (L-1)}(\pi(x))\left(D^{(j)}(\pi(x)) - D^{(j)}(x)\right)W^{(j)} \jac \Phi^{(j-1)}(x) \right);
\end{align*}
see \Cref{lem:decomp}.
To get a bound on the
spectral norm of this difference, we use the triangle inequality and consider each summand individually.
Now, since $\mnorm{x - \pi(x)}_2 \leq \eps$, one can show using \Cref{lem:recht_upper_original} that with high probability, 
\begin{align*}
&\norel\sup_{x \in \SS^{d-1}, j \in \{0, \hdots, L-1\}}\Tr\abs{D^{(j)}(\pi(x)) - D^{(j)}(x)}\\
&=  \sup_{x \in \SS^{d-1}, j \in \{0, \hdots, L-1\}} 
\# \left\{ i \in \{1, \ldots, N\}:\ \sgn_-((W^{(j)} \Phi^{(j-1)}(\pi(x)))_i) \neq \sgn_-((W^{(j)} \Phi^{(j-1)}(x))_i)\right\} \\
&\lesssim \delta N.
\end{align*}
Using a technical argument, we show that this implies
\[
\sup_{x \in \SS^{d-1}}\op{\jac \Phi^{(L-1)}(x) - \jac \Phi^{(L-1)}(\pi(x))} \lesssim \gamma 
\]
with high probability; see \Cref{lem:impro} for the details. 

\newcommand{\lup}{\mathcal{L}_{\mathrm{up}}}
In view of \eqref{eq:net_prop}, we thus obtain for every fixed $\gamma \in [\gamma_0,c]$,
\begin{align}\label{eq:covb}
\ln(\mathcal{N}(\lup,\gamma)) \lesssim \ln(\# \neps) \lesssim d \cdot \ln(L/\gamma) 
\end{align}
with high probability. 
Moreover, by picking $\gamma \asymp 1$, we obtain 
\begin{align}
\Lambda &= \sup_{x \in \SS^{d-1}} \op{\jac \Phi^{(L-1)}(x)} \nonumber\\
\label{eq:lambdab}
&\leq \sup_{x^* \in \neps} \op{\jac \Phi^{(L-1)}(x^*)} + \sup_{x \in \SS^{d-1}}\op{\jac \Phi^{(L-1)}(x) - \jac \Phi^{(L-1)}(\pi(x))} \overset{\eqref{eq:netbound}}{\lesssim} 1
\end{align}
with high probability.

We can use \eqref{eq:covb} and \eqref{eq:lambdab} to bound $\mathrm{I} + \mathrm{II}$:
Note that 
\begin{align}
\mathrm{I} &= \underset{c \leq 2^{-k} \leq 2\Lambda}{\sum_{k \in \ZZ}} 2^{-k} \sqrt{\ln(\mathcal{N}(\mathcal{L}_{\mathrm{up}}, 2^{-k}))} \leq 
\sqrt{\ln(\mathcal{N}(\mathcal{L}_{\mathrm{up}}, c))}  \underset{c \leq 2^{-k} \leq 2\Lambda}{\sum_{k \in \ZZ}} 2^{-k}
\lesssim \sqrt{\ln(\mathcal{N}(\mathcal{L}_{\mathrm{up}}, c))} \cdot \Lambda \nonumber \\
\label{eq:I}
\overset{\eqref{eq:covb},\eqref{eq:lambdab}}&{\lesssim} \sqrt{d \cdot \ln(\ee L)}
\end{align}
with high probability. 
Moreover, 
\begin{align}
\mathrm{II} &= \underset{\gamma_0 \leq 2^{-k} < c}{\sum_{k \in \ZZ}} 2^{-k} \sqrt{\ln(\mathcal{N}(\mathcal{L}_{\mathrm{up}}, 2^{-k}))}
\overset{\eqref{eq:covb}}{\leq} \sqrt{d} \cdot \underset{\gamma_0 \leq 2^{-k} < c}{\sum_{k \in \ZZ}} 2^{-k} \cdot \sqrt{\ln(2^k L)} \nonumber \\
\label{eq:II}
&\lesssim \sqrt{d} \cdot \int_0^c \sqrt{\ln(L/\eps)} \ \dd \eps \overset{\mathrm{\cshref{lem:int_b}}}{\lesssim} \sqrt{d \cdot \ln(\ee L)}
\end{align}
with high probability.

As the last piece of the puzzle it remains to bound III. 
Using VC dimension arguments, one can show that 
\begin{equation}\label{eq:boundhelp}
\mathcal{N}(\mathcal{L}_\mathrm{up}, \gamma) \leq \left(\frac{\ee N}{d+1}\right)^{L(d+1)} \cdot \mathcal{N}(\SS^{d-1}, \gamma/\Lambda)
\leq \left(\frac{\ee N}{d+1}\right)^{L(d+1)} \cdot \left(\frac{3 \Lambda}{\gamma}\right)^d
\quad \text{for all } \gamma \in (0, \Lambda)
\end{equation}
holds deterministically; see \cite[Lemma~5.8]{geuchen2024upper}.
Since $\Lambda \lesssim 1$ with high probability (see \eqref{eq:lambdab}), this implies
\[
\sqrt{\ln(\mathcal{N}(\mathcal{L}_\mathrm{up}, \gamma))} \lesssim \sqrt{Ld\ln(N/d)} + \sqrt{d} \cdot \sqrt{\ln(1/\gamma)} \quad \text{for all } \gamma \in (0, \gamma_0)
\]
with high probability: For $\gamma \leq \Lambda$ this follows from 
\eqref{eq:boundhelp} and for $\Lambda < \gamma < \gamma_0$, this follows 
from the monotonicity of covering numbers and the estimate at $\Lambda$.
Recalling that $\gamma_0 = 1/\sqrt{L \ln(N/d)}$, we obtain
\begin{align}
\mathrm{III} &= 
\underset{2^{-k} < \gamma_0}{\sum_{k \in \ZZ}} 2^{-k} \sqrt{\ln(\mathcal{N}(\mathcal{L}_{\mathrm{up}}, 2^{-k}))} 
\lesssim \int_0^{\gamma_0} \sqrt{\ln(\mathcal{N}(\mathcal{L}_{\mathrm{up}}, \gamma))} \ \dd \gamma \nonumber \\
\label{eq:III}
&\lesssim \gamma_0 \cdot \sqrt{dL \ln(N/d)} + \sqrt{d} \cdot \int_0^{\gamma_0} \sqrt{\ln(1/\gamma)} \ \dd \gamma \lesssim \sqrt{d}
\end{align}
with high probability. 
We conclude the claim by combining \eqref{eq:I}, \eqref{eq:II} and \eqref{eq:III}. 
\end{proof}

%% file: lower_hom.tex
\section{The lower bound for zero-bias networks}\label{sec:lower}

In this section, we state the lower bound of the Lipschitz constant of zero-bias random ReLU networks
and explain the proof idea. 
The following is the main result of this section.
\begin{theorem}\label{thm:main_lower}
There exist absolute constants $C,c>0$ with the following property. 
Let $\Phi: \RR^d \to \RR$ be a random zero-bias ReLU network with $L$ hidden layers 
      of width $N$ following \Cref{assum:1} and satisfying $d \geq 3$,
      \[
      N \geq C \cdot dL^6 \ln(\ee L) \quad \text{and} \quad d\ln(\ee L) \geq C.
      \] 
      Then, with probability at least $1-\exp(-c \cdot d\ln(\ee L))$, 
      \[
      \lip_p(\Phi) \geq c \cdot \begin{cases}\sqrt{d \ln(\ee L)}& \text{if } p \in [1,2],\\
      d^{1- 1/p} \sqrt{\ln(\ee L)}& \text{if } p \in [2, \infty]\end{cases}
      \]
      is simultaneously satisfied for all $p \in [1,\infty]$.
\end{theorem}
We remark that the lower bounds stated in \Cref{thm:main_lower} precisely match 
the upper bounds stated in \Cref{corr:lp_upper}.

In the following, we explain the proof idea of \Cref{thm:main_lower}.
We start by focusing on the case of shallow networks (i.e., $L=1$), since in this setting, a comparably simple
argument is sufficient to establish the result. 
\begin{proof}[Proof sketch of \Cref{thm:main_lower} for shallow networks]
We first note that in the case of shallow networks, \Cref{thm:main_lower} states that if 
$N \gtrsim d$ and $d \gtrsim 1$ we have 
\[
\lip_p(\Phi) \gtrsim \begin{cases}\sqrt{d }& \text{if } p \in [1,2],\\
      d^{1- 1/p} & \text{if } p \in [2, \infty]\end{cases}
\]
with probability at least $1-\exp(-c \cdot d)$.
The case $p \in [2,\infty]$ is already governed by \Cref{corr:pw}.
Thus, since $\mnorm{\cdot}_1 \geq \mnorm{\cdot}_p$ for every $p \in [1,\infty]$, 
it suffices to prove the case $p = 1$. 
Note that by \Cref{prop:lipgrad}, we have 
\[
\lip_1(\Phi) = \underset{x \in M_\Phi \cap \SS^{d-1}}{\sup} \mnorm{\nablaa \Phi(x)}_\infty
\geq \underset{x \in M_\Phi \cap \SS^{d-1}}{\sup} \abs{\skpr{\nablaa \Phi(x)}{e_1}},
\]
where $e_1 \in \RR^d$ denotes the first standard basis vector. 

Similar to the proof of \Cref{thm:main_upper}, the strategy
is again to \emph{fix} (i.e., condition on) $W^{(0)}$
and to only consider the randomness over the last layer $W^{(1)}$.
In contrast to \Cref{thm:main_upper}, since we now seek to establish a lower bound, we no longer use Dudley's inequality but instead Sudakov's minoration inequality
(see e.g. \cite[Theorem~7.4.1]{vershynin_high-dimensional_2018}).
To this end, we rewrite 
\[
\underset{x \in M_\Phi \cap \SS^{d-1}}{\sup} \abs{\langle \nablaa \Phi(x), e_1\rangle} 
= \underset{x \in M_\Phi \cap \SS^{d-1}}{\sup} \abs{\langle (W^{(1)})^T, D^{(0)}(x)W^{(0)}e_1\rangle} 
= \underset{v \in \mathcal{L}_{\mathrm{low}}}{\sup} \ \langle (W^{(1)})^T, v \rangle,
\]
where 
\[
\mathcal{L}_{\mathrm{low}} =  \left\{ D^{(0)}(x)W^{(0)}e_1 : \ x \in M_\Phi \cap \SS^{d-1}\right\}.
\]
Sudakov's minoration inequality then yields 
\[
\underset{x \in M_\Phi \cap \SS^{d-1}}{\sup} \abs{\langle \nablaa \Phi(x), e_1\rangle}  
\gtrsim \underset{\delta > 0}{\sup} \ \delta \cdot \sqrt{\ln\left(\mathcal{P}(\mathcal{L}_{\mathrm{low}}, \delta)\right)}
\]
in expectation, where $\mathcal{P}(\mathcal{L}_{\mathrm{low}},\delta)$ denotes the $\delta$-packing number of $\mathcal{L}_{\mathrm{low}}$.
A high probability bound can be obtained using a standard result on Gaussian concentration, see \cite[Theorem~5.2.2]{vershynin_high-dimensional_2018}. 

The idea is now that a $\delta$-separated subset $P$ of the unit sphere $\SS^{d-1}$ directly yields a $(c \sqrt{\delta})$-separated subset $P'$ of the set $\mathcal{L}_{\mathrm{low}}$, where $c> 0$ is an absolute constant,
by letting 
\[
P' \defeq \left\{ D^{(0)}(x^\ast)W^{(0)}e_1: \ x^\ast \in P\right\}.
\]
To this end, one has to control expressions of the form 
\[
\mnorm{(D^{(0)}(x) - D^{(0)}(y))W^{(0)}e_1}_2
\]
for inputs $x,y \in \SS^{d-1}$ that have a Euclidean distance that is larger than $\delta$.
Applying \Cref{lem:recht_low_low}, we can show that with high probability there exist at least $c \cdot\delta N$ non-zero (diagonal) entries in $D^{(0)}(x) - D^{(0)}(y)$ for an absolute constant $c > 0$.
The main difficulty to overcome when considering the expression above is that $D^{(0)}(x) - D^{(0)}(y)$ and $W^{(0)}e_1$ are \emph{not} independent.
However, independence is trivial if we consider inputs from $\{0\} \times \SS^{d-2}$,
since $D^{(0)}(x) - D^{(0)}(y)$ does not depend on the first column of $W^{(0)}$ anymore. 
The result then follows from standard bounds on the norm of a Gaussian vector, see, e.g., \cite[Theorem~3.1.1]{vershynin_high-dimensional_2018}.
\end{proof}
Note that this trick to create independence does \emph{not} work anymore in the case of deep networks.
Therefore, to establish the bound in the general setting of deep networks, a different approach is required. 
More specifically, we do not study the formal gradient but instead directly prove a lower bound
for certain finite differences.
This result reads as follows. 
 \begin{proposition}\label{main:sharp:prop:direction}
 There exist positive absolute constants $C,c>0$ with the following property.
 Let $\Phi: \RR^d \to \RR$ be a random zero-bias ReLU network with $L$ 
 hidden layers of width $N$ following \Cref{assum:1}.
 Suppose that $d \geq 3$, $d \ln(\ee L) \geq C$ and
 $N\geq C \cdot d L^6\ln(\ee L)$. 
 Fix $\nu\in\SS^{d-1}$.
 Then, with probability at least
 $1-\exp(-c d\ln(\ee L))$,
 \begin{equation}\label{main:sharp:eq:direction-result}
  \sup_{x\in\SS^{d-1}\cap\nu^\perp, h>0}
  \frac{\Phi(x+h\nu)-\Phi(x-h\nu)}{2h}
  \geq c\sqrt{d\ln(\ee L)}.
 \end{equation}
 The constants and the probability estimate are independent of $\nu$.
 The statement concerns each fixed $\nu$; it does not assert a single
 event on which \eqref{main:sharp:eq:direction-result} holds for all $\nu$.
\end{proposition}
\Cref{main:sharp:prop:direction} directly implies \Cref{thm:main_lower} by considering 
\[
\nu = (1,0,\ldots) \in \RR^d \quad \text{and} \quad \nu = d^{-1/2}(1, \ldots, 1) \in \RR^d;
\]
for the details we refer to the beginning of \Cref{sec:lower_proof}.
In the following we give a proof sketch of \Cref{main:sharp:prop:direction} and note
again that the details can be found in \Cref{sec:lower_proof}.
\begin{proof}[Proof sketch of \Cref{main:sharp:prop:direction}]
    Fix $\nu \in \SS^{d-1}$ and $h \asymp 1/L$. 
    For $x \in \nu^\perp$, we write 
    \[
    V_x \defeq \frac{\Phi^{(L-1)}(x+h\nu)-\Phi^{(L-1)}(x-h\nu)}{2h} \in \RR^N
    \]
    and note that 
    \[
    \sup_{x\in\SS^{d-1}\cap\nu^\perp}
  \frac{\Phi(x+h\nu)-\Phi(x-h\nu)}{2h} = \sup_{x\in\SS^{d-1}\cap\nu^\perp} \skpr{(W^{(L)})^T}{V_x}.
    \]
    We set 
    \newcommand{\lfd}{\mathcal{L}_{\mathrm{FD}}}
    \[
    \lfd \defeq \{V_x: \ x \in \SS^{d-1}\cap\nu^\perp\} \subseteq \RR^N
    \]
    and note that by Sudakov's minoration inequality (see \cite[Theorem~7.4.1]{vershynin_high-dimensional_2018}),
    \[
    \sup_{x\in\SS^{d-1}\cap\nu^\perp}
  \frac{\Phi(x+h\nu)-\Phi(x-h\nu)}{2h} \gtrsim \sup_{\eps >0} \ \eps \cdot \sqrt{\ln(\mathcal{P}(\lfd, \eps))}
    \]
    in expectation. 

    Pick $\delta \asymp 1/L$ and let $\mathcal{P} \subseteq \SS^{d-1} \cap \nu^\perp$ be 
    a $\delta$-packing of $\SS^{d-1} \cap \nu^\perp$
    with $\ln(\# P ) \asymp d \ln(\ee L)$. We set 
    \[
    \mathcal{P}^\ast \defeq \{V_x: \ x \in P\} \subseteq \lfd
    \]
    and claim that $\mathcal{P}^\ast$ is a packing of $\lfd$ with constant resolution. 
    To see this, we pick $x,y \in \SS^{d-1} \cap \nu^\perp$ with $\mnorm{x-y}_2 \gtrsim \delta$
    and compute 
    \begin{align*}
        \mnorm{V_x - V_y }_2^2 = \skpr{V_x}{V_x} + \skpr{V_y}{V_y} - 2 \skpr{V_x}{V_y}.
    \end{align*}
    For each choice of (not necessarily distinct) $z_1, z_2 \in \{x,y\}$, we obtain 
    \begin{align*}
        \skpr{V_{z_1}}{V_{z_2}} = \frac{1}{4h^2} \cdot \sum_{\sigma, \tau \in \{-1,1\}} \sigma\tau \cdot \skpr{\Phi^{(L-1)}(z_1+\sigma h \nu)}{\Phi^{(L-1)}(z_2+\tau h \nu)}.
    \end{align*}
    Each of the inner products can be approximated using a deterministic kernel function; see \Cref{subsec:innerprod}.
    Via a technical analysis of this kernel function, we then show that 
    \[
    \mnorm{V_x - V_y }_2^2 \gtrsim 1
    \]
    with high probability; see \Cref{subsec:technical}.
    This implies the claim. 
\end{proof}

By combining \Cref{corr:lp_upper} and \Cref{thm:main_lower}, we 
get the following result.
For its formulation, we set
\begin{equation*}\label{sharp:eq:Ap}
 A_p(d):=
 \begin{cases}
  \sqrt d,&1\leq p\leq2,\\
  d^{1-1/p},&2\leq p\leq\infty,
 \end{cases}
\end{equation*}
for $p\in[1,\infty]$, 
where $1/\infty=0$. 
\begin{theorem}\label{sharp:thm:main}
 There exist absolute constants $C,c>0$ such that the following holds.
 Let $\Phi:\RR^d\to\RR$ be a zero-bias random ReLU network following \Cref{assum:1}, with
 $L$ hidden layers of width $N$. Assume that $d \geq 3$,
 \begin{equation*}\label{sharp:eq:main-width}
  d \ln(\ee L)\geq C\qquad\text{and}\qquad
  N\geq C \cdot d L^6\ln(\ee L).
 \end{equation*}
 Then, with probability at least $1-\exp(-c \cdot d\ln(\ee L))$, one has
 simultaneously for every $p\in[1,\infty]$,
 \begin{equation*}\label{sharp:eq:main-result}
  c \cdot A_p(d)\sqrt{\ln(\ee L)}
  \leq\lip_p(\Phi)
  \leq C \cdot A_p(d)\sqrt{\ln(\ee L)}.
 \end{equation*}
\end{theorem}

%% file: bias.tex
\section{Networks with biases}\label{sec:bias}
The results presented in \Cref{sec:upper,sec:lower} solely dealt with \emph{zero-bias} networks, i.e.,
networks where all the biases are set to zero. 
In this section, we discuss the case of networks with \emph{general} biases. 

A key observation is that for inputs with a sufficiently large Euclidean norm, 
the gradient of $\Phi$ is similar to the gradient of $\tilde{\Phi}$, the zero-bias network associated to $\Phi$ by setting all its biases to zero.
To make this intuition precise, we aim to bound
\[
\sup_{x \in \RR^d, \mnorm{x}_2 \geq R} \mnorm{\nablaa \Phi(x) - \nablaa \tilde{\Phi}(x)}_2
\]
from above using \Cref{prop:farout}, where $R > 0$ is chosen appropriately. 
This is discussed in more detail in \Cref{sec:hom_diff}.
After such a bound is established, and since the Lipschitz constants of zero-bias networks
can be upper bounded using \Cref{corr:lp_upper}, it essentially suffices to bound 
\[
\sup_{x \in B_d(0,R)} \mnorm{\nablaa \Phi(x)}_2,
\]
in order to obtain an upper bound for $\lip_2(\Phi)$; see \Cref{thm:main_upper_bias}.
As in \Cref{corr:lp_upper}, this then directly yields an upper bound for $\lip_p(\Phi)$
for all $p \in [1,\infty]$. 
A lower bound for $\lip_p(\Phi)$ can directly be observed from the fact that $\lip_p(\Phi) \geq \lip_p(\tilde{\Phi})$
for every $p \in [1,\infty]$; see \Cref{thm:main_lower_bias}.

\subsection{Bounding the difference to the gradient of the associated zero-bias network}\label{sec:hom_diff}
In this section, we discuss how the difference between the gradient of a network $\Phi$ with 
possibly non-zero, \emph{deterministic} biases can be related to the gradient of the associated homogeneous network $\tilde{\Phi}$.
For detailed proofs of the results presented in the present section, we refer to \Cref{sec:hom}.

Obtaining such a bound crucially relies on the following observation. 
\begin{lemma}\label{lem:hom_trace}
There exist absolute constants $C,c > 0$ such that the following holds.
    Let $\Phi: \RR^d \to \RR$ be a random ReLU network of width $N$ and depth $L$ as in \Cref{assum:1} with constant (deterministic) biases.
    Let $\lambda' > 0$ such that 
    \[
    \underset{ i \in \{1,\ldots, N\}}{\max_{\ell \in \{0,\ldots, L-1\}}} \abs{b^{(\ell)}_i} \leq \frac{\lambda' \sqrt{2}}{\sqrt{N}}.
    \]
    Assume $N \geq C \cdot d$ and $N \geq C \cdot L^3d\ln(N/d)$ and let $\delta \in (0,\ee^{-1})$ satisfy
    \[
    N \geq C \cdot dL^2 \cdot \ln(1/\delta), 
    \quad 
    \delta N \geq C \cdot d\cdot\max\{ \ln(1/\delta), L\ln^{-1}(1/\delta)\},
    \quad
    \text{and}
    \quad
     \delta \cdot \ln(1/\delta) \leq c /L^2.
    \] 
    Pick $j \in \{0,\ldots,L-1\}$.
    Then, with probability at least $1- \exp(-c \cdot \delta N)$,
    \begin{align*}
    &\norel \text{for all } x \in \RR^d \text{ with } \mnorm{x}_2 \geq C \cdot 3^L \cdot \lambda' \cdot \delta^{-1} \cdot \ln^{1/2}(1/\delta): \\
    &\# \left\{ i \in \{1,\ldots,N\}: \ \sgn_{-}\left(\left(W^{(j)} \Phi^{(j-1)}(x)\right)_i + b^{(j)}_i\right) \neq \sgn_{-}\left(\left(W^{(j)}\tilde{\Phi}^{(j-1)}(x)\right)_i\right)\right\} \leq \delta \cdot N.
    \end{align*}
\end{lemma}
\begin{proof}[Proof sketch]
    Using \cite[Corollary~7.3.3]{vershynin_high-dimensional_2018}, we have 
    \[
    \PP\left(\text{for all } \ell \in \{0,\ldots,L-1\}: \ \op{W^{(\ell)}} \leq 3\right) \geq 1 - \exp(-c \cdot N).
    \]
    On this event, we get for arbitrary $x \in \RR^d$ and $\ell \in \{0,\ldots,L-1\}$, using that the $\relu$ is $1$-Lipschitz,
    \begin{align*}
        \mnorm{\Phi^{(\ell)}(x) - \tilde{\Phi}^{(\ell)}(x)}_2
        &= \mnorm{\relu\left(W^{(\ell)}\Phi^{(\ell - 1)}(x) + b^{(\ell)}\right) - \relu\left(W^{(\ell)} \tilde{\Phi}^{(\ell - 1)}(x)\right)}_2\\
        &\leq \mnorm{W^{(\ell)}\Phi^{(\ell - 1)}(x) + b^{(\ell)} - W^{(\ell)} \tilde{\Phi}^{(\ell - 1)}(x)}_2 \\
        &\leq 3 \cdot \mnorm{\Phi^{(\ell - 1)}(x) - \tilde{\Phi}^{(\ell - 1)}(x)}_2 + \mnorm{b^{(\ell)}}_2 \\
        &\leq 3 \cdot \mnorm{\Phi^{(\ell - 1)}(x) - \tilde{\Phi}^{(\ell - 1)}(x)}_2 + \lambda'\sqrt{2} \\
        &\leq \lambda' \sqrt{2} \cdot \sum_{j=0}^\ell 3^{j} \leq 3^L \cdot \lambda',
    \end{align*}
    where the fifth inequality follows from induction. 

    Based on this observation, we can derive the claim: 
    Note that 
    \begin{align*}
        &\norel\# \left\{ i \in \{1,\ldots,N\}: \ \sgn_{-}\left(\left(W^{(j)} \Phi^{(j-1)}(x)\right)_i + b^{(j)}_i\right) \neq \sgn_{-}\left(\left(W^{(j)}\tilde{\Phi}^{(j-1)}(x)\right)_i\right)\right\} \\
        &\leq \# \left\{ i \in \{1,\ldots,N\}: \ \sgn_{-}\left(\left(W^{(j)} \Phi^{(j-1)}(x)\right)_i + b^{(j)}_i\right) \neq \sgn_{-}\left(\left(W^{(j)}\Phi^{(j-1)}(x)\right)_i\right)\right\} \\
        &\hspace{0.5cm}+\# \left\{ i \in \{1,\ldots,N\}: \ \sgn_{-}\left(\left(W^{(j)} \Phi^{(j-1)}(x)\right)_i \right) \neq \sgn_{-}\left(\left(W^{(j)}\tilde{\Phi}^{(j-1)}(x)\right)_i\right)\right\}.
    \end{align*}
    The first summand can be bounded by $\delta N/2$ using \Cref{prop:farout}, noting that 
    \begin{align*}
        \mnorm{\Phi^{(j-1)}(x)}_2 &\geq \mnorm{\tilde{\Phi}^{(j-1)}(x)}_2 - \mnorm{\Phi^{(j-1)}(x) - \tilde{\Phi}^{(j-1)}(x)}_2 \\
        \overset{\text{Cor. \ref{prop:isom}}}&{\gtrsim} 3^L \cdot \lambda' \cdot \delta^{-1} \cdot \ln^{1/2}(1/\delta) - 3^L \cdot \lambda' \\
        &\gtrsim \delta^{-1} \cdot \lambda'. 
    \end{align*}
    For the second summand, we apply \Cref{lem:recht_upper_original} since we have 
    \begin{align*}
        \mnorm{\frac{\Phi^{(j-1)}(x)}{\mnorm{\Phi^{(j-1)}(x)}_2}-\frac{\tilde{\Phi}^{(j-1)}(x)}{\mnorm{\tilde{\Phi}^{(j-1)}(x)}_2}}_2 
        \overset{\text{\cshref{lem:diff_bound}}}&{\lesssim} \frac{\mnorm{\Phi^{(j-1)}(x) - \tilde{\Phi}^{(j-1)}(x)}_2}{\mnorm{\tilde{\Phi}^{(j-1)}(x)}_2} \\
        \overset{\text{\cshref{prop:isom}}}&{\lesssim} \frac{3^L \cdot \lambda'}{\mnorm{x}_2} \\
        &\lesssim \frac{3^L \cdot \lambda'}{3^L \cdot \lambda' \cdot \delta^{-1} \cdot \ln^{1/2}(1/\delta)} = \delta \cdot \ln^{-1/2}(1/\delta). 
    \end{align*}
    Hence, we can also bound the second summand by $\delta N/2$, which implies the claim. 
\end{proof}
Based on \Cref{lem:hom_trace}, we can then prove the main theorem of this subsection.
\begin{theorem}\label{thm:grad_diff}
There exist absolute constants $C, c>0$ such that the following holds.
    Let $\Phi: \RR^d \to \RR$ be a random ReLU network with $L$ hidden layers of width $N$ following \Cref{assum:1} with fixed biases.
    Assume $N \geq Cd$ and $N \geq C \cdot L^3d\ln(N/d)$.
        Let $\delta \in (0,\ee^{-1})$ with 
        \[
        N \geq C \cdot dL^2 \cdot \ln(1/\delta), 
        \quad 
        \delta N \geq C \cdot d\cdot\max\{ \ln(1/\delta), L\ln^{-1}(1/\delta)\}
        \quad
        \text{and}
        \quad
         \delta \cdot \ln(1/\delta) \leq c \cdot 1/L^2.
        \] 
    Let $\lambda'$  be chosen as in \eqref{eq:lambda}.
    Then
    the event 
    \[
    \underset{\mnorm{x}_2 \geq C \cdot 3^L \cdot \lambda' \cdot \delta^{-1} \cdot \ln^{1/2}(1/\delta)}{\sup}
    \mnorm{\nablaa \Phi(x) - \nablaa \widetilde{\Phi}(x)}_2 \leq C \cdot L \cdot  \left(\sqrt{N}\cdot\delta\ln(1/\delta) + \sqrt{\delta \ln(1/\delta)Ld\ln(N/d)}\right)
    \]
    occurs with probability at least $1- \exp(-c \cdot \delta N)$.
\end{theorem}
\renewcommand{\DD}{\widetilde{D}}
\begin{proof}[Proof sketch]
    We write $\DD^{(\ell)}(x)$ for the $D$-matrices arising by passing an input $x$ through the homogeneous network $\tilde{\Phi}$.
    Moreover, we set $R \defeq C \cdot 3^L \cdot \lambda' \cdot \delta^{-1} \cdot \ln^{1/2}(1/\delta)$.

    Similar to the proof of \Cref{thm:main_upper}, we use a decomposition of the difference of the gradients, namely 
    \begin{align}
        &\norel \underset{\mnorm{x}_2 \geq R}{\sup}
        \mnorm{\nablaa \Phi(x) - \nablaa \widetilde{\Phi}(x)}_2 \nonumber\\ 
        &\leq \sum_{j=0}^{L-1}\left(\underset{\mnorm{x}_2 \geq R}{\sup}\op{W^{(L)}\jac \Phi^{(j+1) \to (L-1)}(x)\left(D^{(j)}(x) - \DD^{(j)}(x)\right)W^{(j)}\jac \tilde{\Phi}^{(j-1)}(x)} \right) \nonumber\\
        &\leq \sum_{j=0}^{L-1} \left(\underset{\mnorm{x}_2 \geq R}{\sup}
        \op{W^{(L)}\jac \Phi^{(j+1) \to (L-1)}(x)\left(D^{(j)}(x) - \DD^{(j)}(x)\right)}\right) \nonumber\\
        \label{eq:sketcheq}
        &\hspace{1.3cm}\cdot \left(\underset{\mnorm{x}_2 \geq R}{\sup}\op{\abs{D^{(j)}(x) - \DD^{(j)}(x)}W^{(j)}\jac \tilde{\Phi}^{(j-1)}(x)}\right).
    \end{align}

    We set $\mathcal{A} \defeq \{A \in \diag\{1,-1,0\}^{N \times N}: \ \Tr\abs{A} \leq \delta N\}$.
    Firstly, note that according to \Cref{lem:hom_trace}, we may assume that with high probability
    $D^{(j)}(x) - \DD^{(j)}(x) \in \mathcal{A}$ for every $j \in \{0,\ldots,L-1\}$.
    
    Each summand of the latter sum is then considered individually.
    We fix $j \in \{0,\ldots,L-1\}$ and get 
    \begin{align*}
        &\norel\underset{\mnorm{x}_2 \geq R}{\sup}
        \op{W^{(L)}\jac \Phi^{(j+1) \to (L-1)}(x)\left(D^{(j)}(x) - \DD^{(j)}(x)\right)} \\
        &\leq \underset{\mnorm{x}_2 \geq R, A \in \mathcal{A}}{\sup}
        \op{W^{(L)}\jac \Phi^{(j+1) \to (L-1)}(x)A} \\
        &\leq \underset{\mnorm{x}_2 \geq R, A \in \mathcal{A}}{\sup}
        \op{W^{(L)}\left[\jac \Phi^{(j+1) \to (L-1)}(x) - \jac \tilde{\Phi}^{(j+1) \to (L-1)}(x)\right]A} \\ 
        &\hspace{0.4cm}+\underset{\mnorm{x}_2 \geq R, A \in \mathcal{A}}{\sup}
        \op{W^{(L)}\jac \tilde{\Phi}^{(j+1) \to (L-1)}(x)A}.
    \end{align*}
    The first summand can be bounded using an inductive argument. 
    The second summand that only involves the homogeneous network can be bounded using a generalized version of \Cref{thm:main_upper}; see \Cref{prop:advanced_lip}. 
    In total we get a bound on the first factor in \eqref{eq:sketcheq}. 

    Note that the second factor in \eqref{eq:sketcheq} only involves the homogeneous network $\tilde{\Phi}$. 
    We may thus use 
    \[
    \underset{\mnorm{x}_2 \geq R}{\sup}\op{\left(D^{(j)}(x) - \DD^{(j)}(x)\right)W^{(j)}\jac \tilde{\Phi}^{(j-1)}(x)} 
    \leq \underset{\mnorm{x}_2 \geq R, A \in \mathcal{A}}{\sup}\op{AW^{(j)}\jac \tilde{\Phi}^{(j-1)}(x)} 
    \]
    and bound the latter expression analogously to the technique presented in the proof sketch of \Cref{thm:main_upper}. 
    Finally, we perform a union bound over $j \in \{0,\ldots,L-1\}$ to get the desired result. 
\end{proof}
\renewcommand{\DD}{\hat{D}}
\subsection{Bounds for the Lipschitz constant}
In this section, we discuss how \Cref{thm:grad_diff} may be used to establish upper bounds on the Lipschitz constant of random ReLU networks with non-zero biases. 
Note that we assumed the biases to be constant in \Cref{thm:grad_diff}.
When considering randomly drawn biases, the strategy is hence to \emph{condition} on the biases and to assume that they are upper bounded in absolute value by $\frac{\lambda' \sqrt{2}}{\sqrt{N}}$ for some $\lambda' \geq 0$. 
Moreover, we present a lower bound for $\lip_p(\Phi)$ in \Cref{thm:main_lower_bias}.
We refer to \Cref{sec:bias_bounds_proofs} for detailed proofs of all the results presented in this section. 

In order to establish an upper bound, 
we need to assume that the biases are drawn according to \Cref{assum:3}, i.e., we assume that the biases are drawn according to symmetric distributions and satisfy a small-ball property.
Note again that this is equivalent to all the entries of the bias vectors having
absolutely continuous distributions with density bounded almost everywhere; see \Cref{thm:abs_cont}.
\begin{theorem}\label{thm:main_upper_bias}
    There exist absolute constants $C,c> 0$
    with the following property:
    Let $\Phi: \RR^d \to \RR$ be a random ReLU network with symmetric biases that satisfy a small-ball condition with constant $C_\tau > 0$, and with $L$ hidden layers of width $N$ as in  \Cref{assum:3}.
    Moreover, let $\eta_\lambda \in [0,1]$ with 
    \[
    \PP\left(\underset{j \in \{1,\ldots,N\}}{\underset{\ell \in \{0,\ldots,L-1\}}{\sup}} \abs{b^{(\ell)}_j} \leq 
    \lambda\right) \geq 1 - \eta_\lambda
    \]
    for some $\lambda \geq C_\tau^{-1}$.
    We assume that 
    \begin{gather*}
    N \geq C \cdot dL^4\ln^2(\ee L) \quad \text{and} \quad  
    N \geq C \cdot dL^3\ln(\lambda C_\tau) \ln(N/d)\ln(L \ln(N/d)).
    \end{gather*}
    Then, with probability at least $1- \exp(-c \cdot d(L + \ln(\lambda C_\tau N/d))) - \eta_\lambda$,
    \[
    \lip_2(\Phi) \leq \underset{x \in \RR^d}{\sup} \mnorm{\nablaa \Phi(x)}_2 \leq C \cdot \sqrt{d(L + \ln(\lambda C_\tau N/d))}.
    \]
\end{theorem}
\begin{proof}[Proof sketch]
    We choose $\delta \asymp \frac{d}{N}(L + \ln(\lambda C_\tau N/d))$ and define $R >0$ 
    according to \Cref{thm:grad_diff}.
    We then use the estimate
    \begin{align}
\sup_{x \in \RR^d} \mnorm{\nablaa \Phi(x)}_2 
&\leq \sup_{x \in B_d(0,R)} \mnorm{\nablaa \Phi(x)}_2 + \sup_{x \in \RR^d, \mnorm{x}_2 \geq R} \mnorm{\nablaa \Phi(x) - \nablaa \tilde{\Phi}(x)}_2
+ \sup_{x \in \RR^d} \mnorm{\nablaa \widetilde{\Phi}(x)}_2 \nonumber\\
\label{eq:pat}
&= \sup_{x \in B_d(0,R)} \mnorm{\nablaa \Phi(x)}_2 + \sup_{x \in \RR^d, \mnorm{x}_2 \geq R} \mnorm{\nablaa \Phi(x) - \nablaa \tilde{\Phi}(x)}_2
+ \sup_{x \in \SS^{d-1}} \mnorm{\nablaa \widetilde{\Phi}(x)}_2,
\end{align}
where the last equality uses the fact that the formal gradient of the homogeneous network $\tilde{\Phi}$ is invariant with respect to scaling with positive scalars. 
The second summand can be controlled by conditioning on the biases and using \Cref{thm:grad_diff} and the third summand can be bounded 
using \Cref{thm:main_upper} (or, more precisely, \eqref{eq:u}).
The first summand $\sup_{x \in B_d(0,R)} \mnorm{\nablaa \Phi(x)}_2$
can be estimated similarly to the proof of \Cref{thm:main_upper} using \Cref{corr:tess} instead of \Cref{lem:recht_upper_original}. 
We refer to \Cref{sec:upper_proof_bias} for the full proof. 
\end{proof}
As in \Cref{sec:upper}, 
this result can be used to get an upper bound on $\lip_p(\Phi)$ for every $p \in [1,\infty]$. 
\begin{corollary}\label{corr:upper_bias}
    On the same event as in \Cref{thm:main_upper_bias}, for every $p \in [1,\infty]$,
    \[
    \lip_p(\Phi)  \leq C \cdot \begin{cases}\sqrt{d} \cdot \left(\sqrt{L} + \sqrt{\ln (\lambda C_\tau N /d)}\right)&\text{if } p \in [1,2], \\
    d^{1-1/p} \cdot \left(\sqrt{L} + \sqrt{\ln (\lambda C_\tau N /d)}\right)&\text{if } p \in [2,\infty].
    \end{cases}
    \]
\end{corollary}

We emphasize that, compared to the case of zero-bias networks discussed in \Cref{corr:lp_upper}, the upper bound derived in \Cref{corr:upper_bias} has a stronger dependence on the number of hidden layers $L$. 
However, the dependence is not exponential as in \cite[Theorem~3.3]{geuchen2024upper}, but only a factor of $\sqrt{L}$ appears.
This is due to the fact that
bounding the first summand in \eqref{eq:pat}
requires an $\eps$-net of $B_d(0,R)$ where $R$ grows exponentially in $L$, so that the size of the net grows exponentially in $L$, too.
Taking the square root of the logarithm then results in the factor $\sqrt{L}$.
Moreover, we obtain a logarithmic dependence on the network width. 

\begin{remark}\label{rem:sub_exponential}
\begin{enumerate}
\item{\label{item:remark_1}We note that the statement of \Cref{thm:main_upper_bias} is \emph{invariant} under scaling of the biases, i.e., under replacing $b^{(\ell)}$ by $\alpha \cdot b^{(\ell)}$
for each $\ell \in \{0,\ldots,L-1\}$ for a fixed scaling parameter $\alpha > 0$.
This is due to the fact that such a scaling changes the parameter $\lambda$ appearing in \Cref{thm:main_upper_bias} by a factor of $\alpha$ and changes the parameter $C_\tau$ by a factor of $1/\alpha$.}
\item{\label{item:remark_2}
Let $\mathcal{D}_\tau$ be a real-valued, symmetric, sub-exponential distribution
with small ball constant $C_\tau$.
Note that this in particular includes the case that $\mathcal{D}_\tau$ is sub-gaussian
and absolutely continuous with bounded density.
Assume that $b^{(\ell)}_i \iid \mathcal{D}_\tau$. 
For given $\eta \in (0,1)$, using \cite[Prop.~2.7.1]{vershynin_high-dimensional_2018}
followed by a union bound, this implies that 
\[
\PP \left(\sup_{j,\ell} \abs{b_j^{(\ell)}} \leq \lambda\right) \geq 1 -2NL\exp(-c \cdot \lambda/ \mnorm{\mathcal{D}_\tau}_{\psi_1}) \geq 1 - \eta
\]
is satisfied if $\lambda \gtrsim \mnorm{\mathcal{D}_\tau}_{\psi_1} \cdot \ln(2NL/\eta)$. 
Hence, for $\lambda \gtrsim \max\{C_\tau^{-1}, \mnorm{\mathcal{D}_\tau}_{\psi_1} \cdot \ln(2NL/\eta)\}$
and under the assumptions of \Cref{thm:main_upper_bias}, we observe that 
\[
\lip_2(\Phi) \lesssim \sqrt{d \cdot (L +\ln(\lambda C_\tau N/d))}
\]
with probability at least $1-\exp(-c \cdot d(L + \ln(\lambda C_\tau N/d))) - \eta$.}
\item{
Let $\mathcal{D}_\tau$ be a real-valued, symmetric, sub-exponential distribution
with small-ball constant $C_\tau$, and assume that the hidden-layer bias
entries are drawn independently from $\mathcal{D}_\tau$, as in \eqref{item:remark_2}.
If
\[
\begin{aligned}
    N&\gtrsim dL^4\ln^2(eL),\\
    N&\gtrsim dL^3
    \ln\!\bigl((1+\|\mathcal{D}_\tau\|_{\psi_1}C_\tau)N\bigr)
    \ln(N/d)\ln\!\bigl(L\ln(N/d)\bigr),
\end{aligned}
\]
then
\[
    \operatorname{Lip}_2(\Phi)
    \lesssim
    \sqrt{d\Bigl[L+\ln N+
        \ln\!\bigl(1+\|\mathcal{D}_\tau\|_{\psi_1}C_\tau\bigr)\Bigr]}
\]
with probability at least
\[
    1-\exp\!\left(
        -c \cdot d\Bigl[L+\ln N+
        \ln\!\bigl(1+\|\mathcal{D}_\tau\|_{\psi_1}C_\tau\bigr)\Bigr]
    \right).
\]

Indeed, choose
\[
\begin{aligned}
    \eta
    &=\exp\!\left(
        -d\Bigl[L+\ln N+
        \ln\!\bigl(1+\|\mathcal{D}_\tau\|_{\psi_1}C_\tau\bigr)\Bigr]
      \right),\\
    \lambda
    &\asymp dN\bigl(\|\mathcal{D}_\tau\|_{\psi_1}+C_\tau^{-1}\bigr).
\end{aligned}
\]
The second width condition then implies
\[
    d\ln\!\bigl(1+\|\mathcal{D}_\tau\|_{\psi_1}C_\tau\bigr)
    \lesssim
    \frac{N}{L^3\ln(N/d)\ln\!\bigl(L\ln(N/d)\bigr)}
    \lesssim N.
\]
Consequently,
\[
\begin{aligned}
    \|\mathcal{D}_\tau\|_{\psi_1}\ln(2NL/\eta)
    &=\|\mathcal{D}_\tau\|_{\psi_1}
      \Bigl[\ln(2NL)+d\bigl(L+\ln N+
      \ln(1+\|\mathcal{D}_\tau\|_{\psi_1}C_\tau)\bigr)\Bigr]\\
    &\lesssim \|\mathcal{D}_\tau\|_{\psi_1}dN.
\end{aligned}
\]
Thus the chosen $\lambda$ satisfies the tail requirement in \eqref{item:remark_2},
so that $\eta_\lambda\leq\eta$, and also satisfies
$\lambda\geq C_\tau^{-1}$.
Moreover,
\[
    \lambda C_\tau\asymp dN(1+\|\mathcal{D}_\tau\|_{\psi_1}C_\tau)
\]
and hence
\[
\begin{aligned}
    L+\ln(\lambda C_\tau N/d)
    &=L+ \ln(\lambda C_\tau) + \ln(N/d)\\
    &= L + \ln(dN) + \ln(1+ \mnorm{\mathcal{D}_\tau}_{\psi_1} C_\tau) \\
    &\asymp L+\ln N+
      \ln\!\bigl(1+\|\mathcal{D}_\tau\|_{\psi_1}C_\tau\bigr).
\end{aligned}
\]
Since $d\leq N$,
\[
    \ln(\lambda C_\tau)
    \lesssim
    \ln\!\bigl((1+\|\mathcal{D}_\tau\|_{\psi_1}C_\tau)N\bigr).
\]
Therefore the displayed width assumptions imply those of Theorem~6.3.
Applying that theorem and combining its exponential failure probability
with $\eta_\lambda\leq\eta$ proves the claim, after adjusting absolute
constants.
}
\end{enumerate}
\end{remark}

Fortunately, in the case of lower bounds we can generalize the results from \Cref{sec:lower} to the case of networks with general biases without \emph{any} special assumptions on the distribution of the biases. 
This is due to the fact that $\lip_p(\Phi) \geq \lip_p(\tilde{\Phi})$ holds deterministically,
where $\tilde{\Phi} : \RR^d \to \RR$ is the zero-bias network associated with $\Phi$. 
This leads to the following theorem. 
\begin{theorem}\label{thm:main_lower_bias}
There exist absolute constants $C, c > 0$ such that the following holds. Let $\Phi:\RR^d \to \RR$ be a random ReLU network with $L$ hidden layers of width $N$ as in \Cref{assum:1}.
Assume $d \geq 3$,
\begin{equation*} 
    N \geq C \cdot dL^6\ln(\ee L) \quad \text{and} \quad d\ln(\ee L) \geq C.
\end{equation*}
Then, with probability at least $1-\exp(-c \cdot d \ln(\ee L))$,
\[
\lip_p(\Phi) \geq c \cdot \begin{cases}\sqrt{d \ln(\ee L)}& \text{if } p \in [1,2],\\
      d^{1- 1/p} \sqrt{\ln(\ee L)}& \text{if } p \in [2, \infty]\end{cases}
\]
is simultaneously satisfied for all $p \in [1,\infty]$.
\end{theorem}
\begin{proof}[Proof sketch]
    Via induction over the number of hidden layers, one gets 
    \[
    \lim_{r \to \infty} \frac{\Phi(rx)}{r} = \tilde{\Phi}(x) \quad \text{for all } x \in \RR^d.
    \]
    This directly implies $\lip_p(\tilde{\Phi}) \leq \lip_p(\Phi)$, so that the claim follows 
    immediately from \Cref{thm:main_lower}.
For the full proof, we refer to \Cref{sec:lower_general_proofs}. 
\end{proof}

\subsection{Matching bounds for shallow networks}
From \cite[Theorems~3.1\&3.2]{geuchen2024upper}, we recall that for a \emph{shallow} random ReLU network following \Cref{assum:1}, 
\[
\lip_2(\Phi) \asymp \sqrt{d} 
\]
with high probability. 
Using \Cref{thm:main_lower_bias} and standard norm inequalities, 
we can generalize this estimate to arbitrary $\ell^p$-norms. 
\begin{corollary}\label{corr:shallow}
    There exist absolute constants $C,c>0$ with the following property. 
    Let $\Phi:\RR^d \to \RR$ be a shallow random ReLU network with $N$ neurons in the hidden layer as in \Cref{assum:1}. 
    We assume 
    \[
    N \geq C\cdot d \quad \text{and} \quad  d \geq C.
    \]
    Then, with probability at least $1- \exp(-c \cdot d)$, 
    \[
        c \cdot \sqrt{d} \leq \lip_p(\Phi) \leq C \cdot \sqrt{d}
        \quad \text{for } p \in [1,2]
    \]
    and 
    \[
    c \cdot d^{1-1/p} \leq \lip_p(\Phi) \leq C \cdot d^{1-1/p} \quad \text{for } p \in [2, \infty]
    \]
    holds simultaneously for all $p \in [1,\infty]$.
\end{corollary}
Hence, in the case of shallow networks, we obtain bounds that match up to an absolute multiplicative 
constant and no assumptions on the distribution of the biases are required.

%% file: prelim_proofs.tex
\section{\texorpdfstring{Proofs of preliminary results regarding the Jacobian of \\ReLU networks}{Proofs of preliminary results regarding the Jacobian of ReLU networks}} \label{sec:prelim_proofs}
In this section, we provide proofs for the statements from \Cref{sec:prelim}.
We start with \Cref{prop:grad_relu}, which states that every intermediate layer $\Phi^{(\ell)}$ of a ReLU network is almost everywhere differentiable and the true
Jacobian coincides almost everywhere with the formal Jacobian introduced in \Cref{def:formal}. 
This proof is essentially the same as the one in \cite[Theorem~III.1]{berner2019towards}.
\renewcommand*{\proofname}{Proof of \Cref{prop:grad_relu}}
\begin{proof}
The proof is by induction over $\ell \in \{-1,\ldots,L-1\}$.
For $\ell = -1$, there is nothing to show,
since $\Phi^{(-1)}(x)=x$ and $\jac \Phi^{(-1)}(x) = I_{d \times d}$; see \eqref{eq:d-matrices} and \Cref{def:formal}.

Next, let $\ell \in \{0,\ldots,L-1\}$ and assume that there exists a set $\tilde{N} \subseteq \RR^d$ with $\metalambda^d(\tilde{N}) = 0$
and such that $\Phi^{(\ell-1)}$ is differentiable on $\RR^d \setminus \tilde{N}$ with
\begin{equation*}
\mathrm{D} \Phi^{(\ell - 1)}(x) = \jac \Phi^{(\ell-1)}(x) \quad \text{for all } x \in \RR^d \setminus \tilde{N}.
\end{equation*}
We fix an index $j \in \{1,\ldots,N\}$ and define
\begin{equation*}
Z_j \defeq \left\{ x \in \RR^d: \ \Phi^{(\ell)}_j (x) = 0\right\}.
\end{equation*}
As the composition of affine maps and the componentwise application of the ReLU, $\Phi^{(\ell)}_j$ is Lipschitz continuous. 
Hence, we can apply \cite[Corollary~1~on~p.84]{evans_measure_1992} and conclude 
the existence of a set $N_j \subseteq Z_j$ with $\metalambda^d(N_j) = 0$ such that $\Phi^{(\ell)}_j$ is differentiable on $Z_j \setminus N_j$ with
\begin{equation*}
\left(\mathrm{D} \Phi^{(\ell)}_j\right) (z) = 0 =  \left(D^{(\ell)}(z) \cdot W^{(\ell)} \cdots D^{(0)}(z)\cdot W^{(0)}\right)_{j,:} = \left(\jac \Phi^{(\ell)}(x)\right)_{j,:},
\end{equation*}
where the second equality stems from $\left(D^{(\ell)}(z)\right)_{j,j} = 0$ for $z \in Z_j$.

It remains to analyze the derivative on $\RR^d \setminus Z_j$.
Let 
\[
\mu \defeq \begin{cases}d&\text{if } \ell = 0 \\ N&\text{if } \ell > 0.\end{cases}
\]
We decompose 
\begin{equation*}
\Phi^{(\ell)}_j = \varphi_j \circ \Phi^{(\ell - 1)} \quad \text{with} \quad \varphi_j: \RR^\mu \to \RR, \quad \varphi_j(z) = \relu(W_{j,:}^{(\ell)} \cdot z + b^{(\ell)}_j).
\end{equation*}
Note that $\tilde{Z_j} \defeq \{z \in \RR^\mu : \ \varphi_j(z) > 0\}\subseteq \RR^\mu$ is open 
with $\varphi_j(z)= W^{(\ell)}_{j,:} z + b_j^{(\ell)}$ for $z \in \tilde{Z_j}$.
Therefore, $\varphi_j$ is differentiable on $\tilde{Z_j}$
with
\begin{equation*}
\mathrm{D} \varphi_j(z) = W^{(\ell)}_{j,:} \quad \text{for } z \in \tilde{Z_j}.
\end{equation*}
Recall by the induction hypothesis that $\Phi^{(\ell - 1)}$ is differentiable on $\RR^d \setminus \tilde{N}$ with 
\begin{equation*}
\mathrm{D} \Phi^{(\ell - 1)}(x) = \jac \Phi^{(\ell-1)}(x), \quad x \in \RR^d \setminus \tilde{N}.
\end{equation*}
Note by definition that $x \in \RR^d \setminus Z_j$ implies $\Phi^{(\ell-1)}(x) \in \tilde{Z_j}$.
We employ the chain rule and get
\begin{align*}
\left(\mathrm{D} \Phi^{(\ell)}_j\right) (x) &= (\mathrm{D} \varphi_j)\big(\underbrace{\Phi^{(\ell-1)}(x)}_{\in \tilde{Z_j}}\big) \cdot \left(\mathrm{D} \Phi^{(\ell-1)} \right)(x) 
=W_{j,:}^{(\ell)} \cdot   \jac \Phi^{(\ell-1)}(x) \\
&= \left(D^{(\ell)}(x) \cdot W^{(\ell)}\cdot\jac \Phi^{(\ell-1)}(x)\right)_{j,:}
= \left(\jac \Phi^{(\ell)}(x)\right)_{j,:},
\end{align*}
for every $x \in \RR^d \setminus (\tilde{N} \cup Z_j)$, where the third equality follows from $\left(D^{(\ell)}(x)\right)_{j,j} = 1$ for $x \in \RR^d \setminus Z_j$.
Overall, we define $\N \defeq \tilde{N} \cup \bigcup_{j=1}^N N_j$, note $\metalambda^d(\N)=0$, and get
\begin{equation*}
\mathrm{D} \Phi^{(\ell)}(x) = \jac \Phi^{(\ell)}(x), \quad x \in \RR^d \setminus \N,
\end{equation*}
which concludes the induction step.

The conclusion for the final layer follows by noting that 
\begin{equation*}
\left(\mathrm{D} \Phi^{(L-1)}\right)(x) = D^{(L-1)}(x) W^{(L-1)} \cdots D^{(0)}(x)W^{(0)} \quad \text{for almost every } x \in \RR^d
\end{equation*}
and an application of the chain rule. 
\end{proof}
\renewcommand*{\proofname}{Proof}
Next, we aim to prove \Cref{prop:lipgrad} which is well-known and was for the 
case of functions $f: \RR^d \to \RR$ and $p=2$ already shown in \cite[Proposition~4.2]{geuchen2024upper}.
We include the proof here again to show that a similar reasoning also works in the case of functions $f: K \to \RR^k$, where $K \subseteq \RR^d$ is convex and has non-empty interior, and arbitrary $p,q \in [1,\infty]$.
\renewcommand*{\proofname}{Proof of \Cref{prop:lipgrad}}
\begin{proof}
By \cite[Theorem~6.3]{rockafellar2015convex}, we conclude 
\[
\overline{\inte(K)} = \overline{K}.
\]
Note here that the interior of $K$ coincides with its relative interior (as defined on \cite[p.~44]{rockafellar2015convex}) since we assume that $K$ has non-empty interior. 
Hence, since $f$ is continuous by assumption, we get 
\[
\lip_{p \to q}\left(\fres{f}{\inte(K)}\right) = 
\lip_{p \to q}\left(\fres{f}{\overline{\inte(K)}}\right)
= \lip_{p \to q}\left(\fres{f}{\overline{K}}\right)
= \lip_{p \to q}\left(\fres{f}{K}\right).
\]
Hence, from now on we may assume without loss of generality that $K$ is open. 

We first fix $i \in \{1,\ldots,k\}$. 
With $f_i$ denoting the $i$-th component of $f$, we have $f_i \in W^{1, \infty}_{\text{loc}}(\RR^d)$ by \cite[Section~4.2.3]{evans_measure_1992}, where
$W^{1, \infty}_{\text{loc}}(\RR^d)$ denotes the set of functions 
that are locally (on every open bounded set $U \subseteq \RR^d$) in the Sobolev space $W^{1, \infty}$. 
Then, \cite[Theorem~1~in~Section~6.2]{evans_measure_1992} implies that there 
exists a null set $N_i \subseteq \RR^d$ such that for every $x \in \RR^d \setminus N_i$, $f_i$ is differentiable at $x$, 
and the partial derivatives of $f_i$ agree with the weak partial derivatives of $f_i$ on $\RR^d \setminus N_i$. 
For $j \in \{1,\ldots,d\}$ let $\widetilde{\partial}_j f_i$ be an explicit weak 
partial derivative of $f_i$ in the direction of the $j$-th standard basis vector 
satisfying 
\begin{equation*}
\widetilde{\partial}_j f_i = \partial_j f_i \quad \text{on }\RR^d \setminus N_i,
\end{equation*}
where $\partial_j$ denotes the usual partial derivative in the direction of the $j$-th
standard basis vector. 
We further write
\begin{equation*}
\widetilde{\mathrm{D}} f \defeq \left(\widetilde{\partial}_j f_i \right)_{\substack{i = 1,\ldots,k \\ j = 1,\ldots,d}} \in \RR^{k \times d}.
\end{equation*}
With $N \defeq \cup_{i=1}^k N_i$, we note that $N \subseteq \RR^d$ is a null set and 
\begin{equation}\label{eq:weak_agree}
\mathrm{D} f = \widetilde{\mathrm{D}} f \quad \text{on } M \setminus N.
\end{equation}

Now, let $\varphi \in C_c^\infty (\RR^d)$ be a smooth function with compact support satisfying $\varphi \geq 0$ 
and furthermore $\int_{\RR^d} \varphi(x) \ \dd x = 1$;
see, e.g., \cite[Section~4.2.1]{evans_measure_1992} for an explicit example of such a function. 
For $\eps > 0$ let $\varphi_\eps (x) \defeq \eps^{-d} \varphi(x / \eps)$. 
We again fix $i \in \{1,\ldots,k\}$.
Since $f_i$ is continuous, the convolution $(f_i)_\eps \defeq f_i \ast \varphi_\eps$ converges pointwise (even locally uniformly) to $f_i$ as $\eps \to 0$; 
see \cite[Theorem~1~in~Section~4.2]{evans_measure_1992}. 
The same theorem also shows that $(f_i)_\eps \in C^\infty (\RR^d)$ with 
\begin{equation*}
\partial_j \left[(f_i)_\eps\right] = (\widetilde{\partial}_j f_i) \ast \varphi_\eps.
\end{equation*}
Letting $f_\eps \defeq \left((f_1)_\eps,\ldots,(f_k)_\eps\right)^T$, we hence get
\begin{equation*}
\mathrm{D} (f_\eps) = \left(\widetilde{\mathrm{D}} f\right) \ast \varphi_\eps \quad \text{on } \RR^d, 
\end{equation*}
where the convolution on the right hand side is to be understood in a componentwise sense, i.e., each entry of the matrix $\left(\widetilde{\mathrm{D}} f\right)$
is convolved with $\varphi_\eps$.

After these preparations, we fix $\delta > 0$ and set 
\[
K_\delta \defeq \{x \in K: \overline{B}_d(x, \delta) \subseteq K\}.
\]
Since $\varphi$ has compact support and by definition of $\varphi_\eps$, there exists $\eps' = \eps'(\delta)>0$ such that for all $\eps \in (0,\eps']$ 
we have $\supp(\varphi_\eps) \subseteq \overline{B}_d(0,\delta)$. 
By definition of $K_\delta$, this implies 
\begin{equation}\label{eq:y}
\varphi_\eps(x-y) = 0 \quad \text{for all }x \in K_\delta, \ y \in \RR^d \setminus K
\end{equation}
for $\eps \in (0, \eps']$.
Then, since $M \subseteq \RR^d$ is a set of full measure, and $N$ is a null set, we obtain for $x \in K_\delta$ and $\eps \in (0, \eps']$ that\footnote{
For the first inequality, let $A: \RR^d \to \RR^{k \times d}$ with $A_{i,j} \in L^1(\RR^d)$ for every $i \in \{1,\ldots,k\}$ and $j \in \{1,\ldots,d\}$ be given. 
We can then bound the $\Vert \cdot \Vert_{p \rightarrow q}$-norm of the entrywise integral of $A$ via 
\[
\left \Vert\int_{\RR^d} A(\xi) \ \dd \xi \right\Vert_{p \rightarrow q} 
= \underset{y \in \RR^k, \Vert y \Vert_{q'} = 1}{\underset{x \in \RR^d, \Vert x \Vert_p = 1}{\sup}} \left\langle \left(\int_{\RR^d} A(\xi) \ \dd \xi \right)\cdot x,y\right\rangle
= \underset{y \in \RR^k, \Vert y \Vert_{q'} = 1}{\underset{x \in \RR^d, \Vert x \Vert_p = 1}{\sup}} \int_{\RR^d} \langle A(\xi)\cdot x, y\rangle \ \dd \xi
\leq \int_{\RR^d} \Vert A(\xi) \Vert_{p \rightarrow q} \ \dd\xi,
\]
where $q' \in [1,\infty]$ is conjugate to $q$.
}
\begin{align*}
\Vert \mathrm{D} f_\eps (x) \Vert_{p \rightarrow q} &= \left\Vert \left[\left(\widetilde{\mathrm{D}} f\right) \ast \varphi_\eps\right] (x)\right\Vert_{p \rightarrow q}
 = \left\Vert  \int_{\RR^d} \varphi_\eps (x-y) \cdot\widetilde{\mathrm{D}} f (y) \ \dd y\right\Vert_{p \rightarrow q} \\
 &= \left\Vert  \int_{ M \setminus N} \varphi_\eps (x-y) \cdot\widetilde{\mathrm{D}} f (y) \ \dd y\right\Vert_{p \rightarrow q} 
 \overset{\eqref{eq:y}}{=}\left\Vert  \int_{(K \cap M) \setminus N} \varphi_\eps (x-y) \cdot\widetilde{\mathrm{D}} f (y) \ \dd y\right\Vert_{p \rightarrow q} \\
 &\leq  \int_{ (K \cap M) \setminus N} \left\Vert \varphi_\eps (x-y) \cdot\widetilde{\mathrm{D}} f (y)\right\Vert_{p \rightarrow q} \ \dd y 
 \leq \!\int_{(K \cap M) \setminus N} \!\!\!\!\!\!\!\!\!\abs{ \varphi_\eps (x-y)} \cdot \!\!\underset{\xi \in (K \cap M) \setminus N}{\sup}\left\Vert\widetilde{\mathrm{D}} f (\xi)\right\Vert_{p \rightarrow q} \ \!\!\!\! \!\dd y \\
 &\leq \underset{\xi \in (K \cap M) \setminus N}{\sup}\left\Vert\widetilde{\mathrm{D}} f (\xi)\right\Vert_{p \rightarrow q} \cdot \int_{\RR^d} \varphi_\eps(z) \ \dd z 
 = \underset{\xi \in (K \cap M) \setminus N}{\sup}\left\Vert\widetilde{\mathrm{D}} f (\xi)\right\Vert_{p \rightarrow q}\\
  \overset{\eqref{eq:weak_agree}}&{=} \underset{\xi \in (K \cap M) \setminus N}{\sup}\left\Vert\mathrm{D} f (\xi)\right\Vert_{p \rightarrow q}
 \leq \underset{\xi \in (K \cap M) }{\sup}\left\Vert\mathrm{D} f (\xi)\right\Vert_{p \rightarrow q}=: L.
\end{align*}
Fix $x,y \in K_\delta$ and $\eps > 0$. 
Let further $q' \in [1,\infty]$ be conjugate to $q$, i.e., $1/q + 1/q' = 1$.
By duality, we can pick $z \in \RR^k$ with 
\[
\mnorm{z}_{q'} = 1 \quad \text{and} \quad \left\langle z, f_\eps(y) - f_\eps(x)\right\rangle = \mnorm{f_\eps(y) - f_\eps(x)}_q.
\]
We then consider the function 
\[
g: \ \RR \to \RR, \quad g(t) = \langle z,f_\eps((1-t)x + ty) \rangle.
\]
The mean value theorem implies the existence of a number $\xi \in (0,1)$ with 
\begin{equation*}
g'(\xi) = \frac{g(1) - g(0)}{1 - 0}.
\end{equation*}
By plugging in the definition of $g$ and using the chain rule, we get 
\begin{equation*}
\mnorm{f_\eps(y) - f_\eps(x)}_q = \big\langle  z, \left[ (\mathrm{D} f_\eps)((1- \xi)x + \xi y)\right] (y-x) \big\rangle.
\end{equation*}
For every $t \in \RR^d$ with $\mnorm{t}_2 \leq \delta$, we have 
\[
(1- \xi)x + \xi y + t = (1- \xi)(\underbrace{x+t}_{\in K}) + \xi (\underbrace{y+t}_{\in K}) \in K,
\]
where the last step uses the convexity of $K$. 
Hence, by definition of $K_\delta$ and since $t$ was arbitrary, we have 
\[
(1- \xi)x + \xi y \in K_\delta.
\]
Using Hölder's inequality, this implies 
\begin{equation*}
\mnorm{ f_\eps(x) - f_\eps(y)}_{q} \leq \underbrace{\mnorm{z}_{q'}}_{=1} \cdot \mnorm{\left[ (\mathrm{D} f_\eps)((1- \xi)x + \xi y)\right] (y-x)}_q 
\leq L \cdot \Vert x - y \Vert_p
\end{equation*}
for all $x,y \in K_\delta$ and $\eps \in (0,\eps']$.
But then it also follows
\begin{equation*}
\mnorm {f(x) - f(y)}_{q} = \lim_{\eps \to 0} \ \mnorm{ f_\eps(x) -f_\eps(y) }_q \leq L \cdot \Vert x-y \Vert_p
\end{equation*}
and thus $\lipp{p}{q}(\fres{f}{K_\delta}) \leq L$.

Since $\delta$ was arbitrary and $K = \inte(K) = \bigcup_{\delta > 0} K_\delta$ (recall that $K$ is assumed to be open), 
we thus obtain $\lipp{p}{q}(\fres{f}{K}) \leq L$.
Hence, the inequality "$\leq$" follows. 

To prove the inequality in the other direction, let $x \in K \cap M$. 
Let $\nu \in \RR^d$ be arbitrary with $\Vert \nu \Vert_p = 1$.
Then it holds
\begin{align*}
\mnorm{\left[\mathrm{D} f(x) \right]\nu }_q&= \mnorm{\lim_{t \to 0} \frac{ f(x+t\nu) - f(x)}{ t}}_q
= \lim_{t \to 0} \frac{\mnorm{ f(x+t\nu) - f(x)}_q}{\vert t \vert} \\
&=  \lim_{t \to 0} \frac{\mnorm{ f(x+t\nu) - f(x)}_q}{\Vert (x + t\nu) - x \Vert_p} \leq \lipp{p}{q}\left(\fres{f}{K}\right),
\end{align*}
where the last inequality uses that $K$ is open, so that $x + t \nu \in K$ for $t$ small enough. 
The claim now follows, since by definition we have
\begin{equation*}
\mnorm{\mathrm{D} f(x)}_{p \to q} = \underset{\nu \in \RR^d, \mnorm{\nu}_p = 1}{\sup} \mnorm{\mathrm{D} f(x) \nu}_{q}. \qedhere
\end{equation*}
\end{proof}
Next, we provide the proof for \Cref{thm:up_low_bound}, which mainly relies on the fact that the formalized gradient of a zero-bias ReLU network is invariant under scaling 
by positive constants.
\begin{proof}[Proof of \Cref{thm:up_low_bound}]
The first and third part of the theorem follow directly from \cite[Theorem~E.1]{geuchen2024upper} and \Cref{prop:grad_relu,prop:lipgrad}. 

For the second part, we note that the lower bound again follows from \cite[Theorem~E.1]{geuchen2024upper}.
Further, according to \Cref{prop:grad_relu}, $M_\Phi$ is a set of full measure and so is
$M_\Phi \setminus \{0\}$. 
In view of \Cref{prop:lipgrad}, it hence suffices to show that
\[
\{\nablaa \Phi(x): \ x \in M_\Phi \setminus \{0\}\} = \left\{\nablaa \Phi(x): \ x \in M_\Phi \cap \SS^{d-1}\right\}.
\]
To this end, let $x \in M_\Phi \setminus \{0\}$ and define 
$z \defeq \frac{x}{\mnorm{x}_2} \in \SS^{d-1}$.
Since $\Phi$ is a zero-bias network,
\[
\nablaa \Phi(ay) = \nablaa \Phi(y) \quad \text{for every } a > 0 \text{ and } y \in \RR^d.
\]
Hence, we get 
\[
\nablaa \Phi(x)  = \nablaa \Phi \left(z\right).
\]
It remains to show that $z \in M_\Phi$.
To this end, note that we have 
\[
\Phi(z) = \frac{1}{\mnorm{x}_2} \cdot \Phi(z \cdot \mnorm{x}_2)
\]
since the ReLU is positively homogeneous. 
Since $\Phi$ is differentiable at $z \cdot \mnorm{x}_2 = x$, we may apply the chain rule and get 
that $\Phi$ is differentiable at $z$ with 
\[
\act \Phi(z) = \frac{1}{\mnorm{x}_2} \cdot \mnorm{x}_2 \cdot \act \Phi(x) \overset{x \in M_\Phi}{=} \nablaa \Phi (x) = \nablaa \Phi(z). 
\]
Therefore, we have $z \in M_\Phi$, which proves the second part.
The fourth part follows analogously. 
\end{proof}
\renewcommand*{\proofname}{Proof}
Next, we prove \Cref{prop:randnorm}.
It states that for a random ReLU network $\Phi$ with symmetric biases, at a fixed input $x \neq 0$, 
with high probability, we may replace each matrix $D^{(\ell)}(x)$ by its 
randomized counterpart $\D^{(\ell)}(x)$ as introduced in \Cref{def:dtilde}.
In fact, we show that, with high probability,  
all the preactivations occurring in $\Phi$ when $x$ is passed through the 
network are nonzero, which implies the desired statement. 
\begin{proof}[Proof of \Cref{prop:randnorm}]
We show via induction that for every $\ell \in \{-1,\ldots,L-1\}$, with probability $\left(1 - \frac{1}{2^N}\right)^{\ell+1}$, we have 
\[
\Phi^{(\ell)}(x) \neq 0 \quad \text{and} \quad \text{for all } \ell' \in \{0,\ldots,\ell\}, \ j \in \{1,\ldots,N\}: \quad  \quad (W^{(\ell')}\Phi^{(\ell'-1)}(x))_j + b^{(\ell')}_j \neq 0.
\]
This implies the claim by the definition of the matrices $D^{(i)}(x)$ and $\D^{(i)}(x)$. 

Note because of $x \neq 0$ and $\Phi^{(-1)}(x) = x$ that there is nothing to show in the case $\ell = -1$.
We now take $\ell \in \{0,\ldots,L-1\}$ and assume via induction that
\begin{equation}\label{eq:help}
\Phi^{(\ell - 1)}(x) \neq 0 \quad \text{and} \quad \text{for all } \ell' \in \{0,\ldots,\ell-1\}, \ j \in \{1,\ldots,N\}: \quad (W^{(\ell')}\Phi^{(\ell' -1)}(x))_j + b_j^{(\ell')}\neq 0
\end{equation}
occurs with probability at least $\left(1 - \frac{1}{2^N}\right)^{\ell}$.
Note that \eqref{eq:help} only depends on $W^{(0)}, \dots, W^{(\ell - 1)}$ and $b^{(0)}, \dots, b^{(\ell-1)}$.
We condition on $W^{(0)}, \ldots, W^{(\ell-1)}$ and $b^{(0)}, \dots, b^{(\ell-1)}$ and assume that \eqref{eq:help} is satisfied. 
Since $\Phi^{(\ell-1)}(x) \neq 0$, the components of $W^{(\ell)}\Phi^{(\ell-1)}(x) + b^{(\ell)}$
have independent, symmetric and absolutely continuous distributions, 
as follows from \cite[Proposition~9.1.6]{dudley2002real}.
Hence,
\[
\text{for all } j \in \{1,\ldots,N\}: \quad (W^{(\ell)}\Phi^{(\ell - 1)}(x))_j
+b^{(\ell)}_j\neq 0
\]
with probability 1. 
Moreover, we have 
\[
\Phi^{(\ell)}(x) = 0 \quad \Leftrightarrow \quad \text{for all } j \in \{1,\ldots,N\}: \quad (W^{(\ell)}\Phi^{(\ell-1)}(x))_j + b^{(\ell)}_j\leq 0.
\]
Since the components of $W^{(\ell)}\Phi^{(\ell-1)}(x) + b^{(\ell)}$ are independent with symmetric, continuous distributions, we conclude $\Phi^{(\ell)}(x) \neq 0$ with probability $1-\frac{1}{2^N}$.
Up to now, we have conditioned on $W^{(0)}, \ldots, W^{(\ell-1)}$ and $b^{(0)}, \dots, b^{(\ell-1)}$.
Reintroducing the randomness over these random matrices and vectors yields the claim.
\end{proof}

\renewcommand*{\proofname}{Proof}

%% file: distr_equiv.tex
\section{Proof of distributional equivalence} \label{sec:distr_equiv}
In this appendix, we prove the distributional equivalence stated in \Cref{prop:randgrad}. 
Let us recall that this proposition essentially states that in expressions such as 
\[
\mnorm{\prod_{i= \ell_2}^{\ell_1} \D^{(i)}(z_i)W^{(i)}}_{2 \to 2},
\]
where the $z_i$ are arbitrary inputs from $\RR^d$, 
we can --- in distribution --- replace the matrices $\D^{(i)}(z_i)$ by independent random diagonal matrices $\DD^{(i)}$,
where each entry on the diagonal of $\DD^{(i)}$ is either 1 or 0, each with probability $1/2$.
This heavily simplifies the analysis of such expressions. 

We start with an elementary lemma, which is probably well-known but for which we could not locate a convenient reference. 
\begin{lemma}\label{lem:distr}
Let $(\Omega, \mathscr{A}, \PP)$ be a probability space and 
\begin{equation*}
X: \quad (\Omega, \mathscr{A}, \PP) \to (\Omega_1, \mathscr{A}_1), \qquad Y: \quad (\Omega, \mathscr{A}, \PP) \to (\Omega_2, \mathscr{A}_2)
\end{equation*}
be two independent random variables taking values in the measurable spaces $(\Omega_1, \mathscr{A}_1)$ and $(\Omega_2, \mathscr{A}_2)$. Let 
\begin{equation*}
\varphi_1: \quad (\Omega_1 \times \Omega_2, \mathscr{A}_1  \otimes \mathscr{A}_2) \to (\Omega_3, \mathscr{A}_3) \qquad \text{and}
 \qquad \varphi_2: \quad (\Omega_1 \times \Omega_2, \mathscr{A}_1  \otimes \mathscr{A}_2) \to (\Omega_3, \mathscr{A}_3)
\end{equation*}
be measurable functions, taking values in some measurable space $(\Omega_3, \mathscr{A}_3)$. Assume that for every fixed $x \in \Omega_1$,
\begin{equation*}
\varphi_1(x,Y) \overset{d}{=} \varphi_2(x,Y).
\end{equation*}
Then,
\begin{equation*}
\varphi_1(X,Y) \overset{d}{=} \varphi_2(X,Y).
\end{equation*}
\end{lemma}
\begin{proof}
Let $A \in \mathscr{A}_3$ be arbitrary. 
Since $X$ and $Y$ are independent, we get
\begin{align*}
\PP \big(\varphi_1(X,Y) \in A \big) = \EE \left[ \mathbbm{1}_{\varphi_1(X,Y) \in A}\right] = 
\underset{x \sim X}{\EE} \ \underset{Y}{\EE} \left[ \mathbbm{1}_{\varphi_1(x,Y) \in A}\right] = 
\underset{x \sim X}{\EE} \ \underset{Y}{\EE} \left[ \mathbbm{1}_{\varphi_2(x,Y) \in A}\right] = 
\PP \big(\varphi_2(X,Y) \in A \big).
\end{align*}
Since $A \in \mathscr{A}_3$ was arbitrary, equality in distribution is shown.
\end{proof}
The following lemma provides a first statement in the direction of \Cref{prop:randgrad}.
\begin{lemma}\label{lem:smatrix}
Let $V \in \RR^{\ell \times N},W \in \RR^{N \times k}$ be random matrices and $b \in \RR^N$ a random vector, where all their entries are jointly independent and follow (possibly different) symmetric distributions.
Moreover, let $\eps_1 ,\ldots, \eps_N \iid \mathrm{Unif}\{0,1\}$ and assume that the random vector $(\eps_1, \ldots, \eps_N)$ is also independent of
$W$, $V$ and $b$.
For $y \in \RR^N$, define
\[
\tilde{\Delta}(y) \defeq \Delta(y) + \mathrm{diag}(\mathbbm{1}_{y_1 = 0} \cdot \eps_1,\ldots, \mathbbm{1}_{y_N = 0} \cdot \eps_N),
\]
with $\Delta(y)$ as defined in \eqref{eq:delta},
and let further
\[
\DD \defeq \mathrm{diag}(\eps_1, \ldots, \eps_N).
\]
Then, for any deterministic $x \in \RR^k$,
\[
V \tilde{\Delta}(Wx + b)W \d V \DD W.
\]
Moreover, for any deterministic matrix $Z \in \RR^{k \times m}$,
\[
\op{\tilde{\Delta}(Wx + b)WZ} \d \op{\DD WZ}.
\]
\end{lemma}
\begin{proof}
The argument for both statements is similar, but with a subtle difference.
Let us start with the first claim.
We introduce an additional source of randomness $S \in \RR^{N \times N}$, which is a diagonal matrix with independent $\Unif\{1,-1\}$-entries on the diagonal, independent of all the other occurring random variables. 
Since all the entries of $W$, $V$ and $b$ are independent and symmetrically distributed,\footnote{More precisely, for each fixed realization of $S$, the distributional equivalence is satisfied by symmetry
and independence of the entries of
 $W$, $V$ and $b$, and \Cref{lem:distr} then yields the distributional equivalence
when the randomness is with respect to $W$, $V$, $b$ and $S$.}
\[
(W,V, b,\DD) \d (SW, VS, Sb, \DD).
\]
Therefore, we get 
\[
V \tilde{\Delta}(Wx + b)W \d VS \tilde{\Delta}(SWx + Sb)SW.
\]
Since $\tilde{\Delta}(SWx + Sb)$ is a diagonal matrix, it commutes with the diagonal matrix $S$.
Thus, since $S^2 = I_{N \times N}$, we conclude
\[
VS \tilde{\Delta}(SWx + Sb)SW = V\tilde{\Delta}(SWx + Sb)W.
\]
Let $X \defeq (W,V,b)$ and $Y \defeq (S, \eps)$ with $\eps = (\eps_1, \ldots, \eps_N)$. 
Note that $X$ and $Y$ are independent. 
For a fixed realization of $X$, consider the matrix 
\[
T \defeq \tilde{\Delta}(SWx + Sb) = \Delta(SWx + Sb) + \mathrm{diag}(\mathbbm{1}_{(SWx + Sb)_1 = 0} \cdot \eps_1,\ldots, \mathbbm{1}_{(SWx+Sb)_N = 0} \cdot \eps_N).
\]
The diagonal entries of this matrix are independent, as follows from the independence of the components of $S$ and $\eps$.
Moreover, for fixed $i \in \{1,\ldots,N\}$, if $(Wx + b)_i = 0$ we have $T_{i,i} = \eps_i$ and if $(Wx + b)_i \lessgtr 0$ we have $T_{i,i} = \mathbbm{1}_{S_{i,i} \lessgtr 0}$.
Hence, we get $T_{i,i} \iid \mathrm{Unif}\{0,1\}$.
This shows 
\[
V\tilde{\Delta}(SWx + Sb)W \d V\DD W \quad \text{for fixed $V, W$ and $b$.}
\]
Applying \Cref{lem:distr} (where we note again that $X = (W,V,b)$ is independent of $Y = (S, \eps)$), we conclude that this distributional equivalence is even satisfied if the randomness with respect to $W$, $V$ and $b$ is also incorporated. 
Thus, the first claim is shown. 

For the second claim, we introduce the same diagonal matrix $S \in \RR^{N \times N}$. 
We further use the distributional equivalence $(W, b,\DD) \d (SW, Sb,\DD)$ and therefore get
\[
\op{\tilde{\Delta}(Wx + b)WZ} \d \op{\tilde{\Delta}(SWx + Sb)SWZ} = \op{S\tilde{\Delta}(SWx + Sb)SWZ},
\]
where for the last equality we used that the spectral norm of a matrix is not changed if multiplied with any realization of $S$.
Identically to the first part, one can now show 
\[
S\tilde{\Delta}(SWx + Sb)SWZ = \tilde{\Delta}(SWx + Sb)WZ \d \DD WZ,
\]
which gives us the claim.
\end{proof}
The proof of \Cref{prop:randgrad} now follows from an inductive argument, where we iteratively apply \cref{lem:smatrix}.
\renewcommand*{\proofname}{Proof of \Cref{prop:randgrad}}
\begin{proof}
Let $\nu \in \NN$ and let $A = A(W^{(0)}, \ldots, W^{(\ell_1 -1)}, b^{(0)},\ldots,b^{(\ell_1 -1)}) \in \RR^{\mu \times \nu}$ be any matrix that might possibly depend on $W^{(0)}, \ldots, W^{(\ell_1 -1)}$ and $b^{(0)},\ldots,b^{(\ell_1 -1)}$.

\textbf{Auxiliary claim: }We start by proving the auxiliary claim that for every $\ell_1\in \{0,\ldots,L-1\}$ and $\ell_2 \in \{0,\ldots,L\}$,
\[
W^{(\ell_2)}\left[\prod_{i= \ell_2-1}^{\ell_1} \D^{(i)}(z_i) W^{(i)}\right] A \d W^{(\ell_2)}\left[\prod_{i= \ell_2 -1}^{\ell_1} \DD^{(i)}W^{(i)} \right]A.
\]

We fix $\ell_1 \in \{0,\ldots, L-1\}$ and prove the claim via induction over $\ell_2 \in \{0, \ldots, L\}$.
Note that in the case $\ell_2 \leq\ell_1$ there is nothing to show, since then, by definition of the reverse matrix product in \eqref{eq:rev_order}, 
\[
\prod_{i= \ell_2-1}^{\ell_1} \D^{(i)}(z_i) W^{(i)} =\prod_{i= \ell_2 -1}^{\ell_1} \DD^{(i)}W^{(i)} = I_{\mu \times \mu}.
\]
In particular, this establishes the base case $\ell_2 = 0$ of the induction. 

For the induction step, let $\ell_2 \in \{\ell_1 + 1,\ldots, L\}$ and assume that the claim is already satisfied for $\ell_2 - 1$.
We then let $\arrow{W} \defeq (W^{(0)},\ldots, W^{(\ell_2 - 2)})$, $\arrow{b} \defeq (b^{(0)},\ldots, b^{(\ell_2 - 2)})$, $\arrow{W'} \defeq (W^{(\ell_2 - 1)}, W^{(\ell_2)})$, 
and $\arrow{\eps} \defeq (\eps^{(0)}, \ldots, \eps^{(\ell_2 - 2)})$.
Consider the two independent random variables $X \defeq (\arrow{W}, \arrow{b}, \arrow{\eps})$ and $Y \defeq (\arrow{W'}, b^{(\ell_2 - 1)}, \eps^{(\ell_2 - 1)})$ and the two measurable maps
\begin{align*}
\varphi_1(X,Y) &\defeq  W^{(\ell_2)}\D^{(\ell_2 - 1)}(z_{\ell_2 - 1}) W^{(\ell_2 - 1)}\left[\prod_{i= \ell_2-2}^{\ell_1} \D^{(i)}(z_i) W^{(i)}\right]A \quad \text{and} \\
\varphi_2(X,Y) &\defeq  W^{(\ell_2)}\DD^{(\ell_2 - 1)} W^{(\ell_2 - 1)}\left[\prod_{i= \ell_2-2}^{\ell_1} \D^{(i)}(z_i) W^{(i)} \right]A .
\end{align*}
We apply \Cref{lem:smatrix} (with $\eps = \eps^{(\ell_2 -1)}$) to conclude that for each fixed realization of $X$ (such that $\Phi^{(\ell_2 - 2)}(z_{\ell_2 - 1})$ and $A$ are fixed), we get
\begin{align*}
W^{(\ell_2)}\D^{(\ell_2 - 1)}(z_{\ell_2 - 1}) W^{(\ell_2 - 1)}
 &=  W^{(\ell_2)}\tilde{\Delta}\left(W^{(\ell_2 - 1)}\Phi^{(\ell_2 - 2)}(z_{\ell_2 - 1})+ b^{(\ell_2 - 1)}\right) W^{(\ell_2 - 1)} \\
 &\d W^{(\ell_2)}\DD^{(\ell_2 - 1)} W^{(\ell_2 - 1)},
\end{align*}
whence for each fixed realization $x_0$ of $X$ we have
\[
\varphi_1(x_0,Y) \d \varphi_2(x_0,Y).
\]
From \Cref{lem:distr} we then conclude 
\begin{align}
&\norel W^{(\ell_2)}\D^{(\ell_2 - 1)}(z_{\ell_2 - 1}) W^{(\ell_2 - 1)}\left[\prod_{i= \ell_2-2}^{\ell_1} \D^{(i)}(z_i) W^{(i)}\right]A \nonumber\\
\label{eq:ind_1}
&\d 
W^{(\ell_2)}\DD^{(\ell_2 - 1)} W^{(\ell_2 - 1)}\left[\prod_{i= \ell_2-2}^{\ell_1} \D^{(i)}(z_i) W^{(i)}\right]A,
\end{align}
where the randomness is taken with respect to all weights $W^{(i)}$, biases $b^{(i)}$, and all Bernoulli variables $\eps^{(i)}$.
We then define the random variables 
\[
X' \defeq (\eps^{(\ell_2 - 1)}, W^{(\ell_2)})\quad \text{and} \quad Y' \defeq (W^{(0)}, \ldots, W^{(\ell_2 - 1)}, b^{(0)},\ldots, b^{(\ell_2 -2)},\eps^{(0)}, \ldots, \eps^{(\ell_2 - 2)}).
\]
Moreover, we let
\begin{align*}
\varphi_1' (X',Y') &\defeq W^{(\ell_2)}\DD^{(\ell_2 - 1)} W^{(\ell_2 - 1)}\left[\prod_{i= \ell_2-2}^{\ell_1} \D^{(i)}(z_i) W^{(i)}\right]A , \\
\varphi_2' (X',Y') &\defeq W^{(\ell_2)}\DD^{(\ell_2 - 1)} W^{(\ell_2 - 1)}\left[\prod_{i= \ell_2-2}^{\ell_1} \DD^{(i)} W^{(i)}\right]A.
\end{align*}
From the induction hypothesis, we infer that for each fixed realization $x'$ of $X'$, we have
\[
\varphi_1' (x', Y') \d \varphi_2' (x', Y').
\]
Since $X'$ and $Y'$ are independent, \Cref{lem:distr} yields 
\begin{equation}
    \label{eq:ind_2}
W^{(\ell_2)}\DD^{(\ell_2 - 1)} W^{(\ell_2 - 1)}\left[\prod_{i= \ell_2-2}^{\ell_1} \D^{(i)}(z_i) W^{(i)}\right]A \d 
W^{(\ell_2)}\DD^{(\ell_2 - 1)} W^{(\ell_2 - 1)}\left[\prod_{i= \ell_2-2}^{\ell_1} \DD^{(i)} W^{(i)}\right]A.
\end{equation}
Combining \eqref{eq:ind_1} and \eqref{eq:ind_2} completes the induction step. 

\textbf{Final conclusion: }
We start with the first claim of \Cref{prop:randgrad}. 
It is trivial if $\ell_2 < \ell_1$, so we may assume that $\ell_2 \geq \ell_1$, which in particular yields $\ell_2 \geq 0$.
Note that
\[
\op{\left[\prod_{i= \ell_2}^{\ell_1} \D^{(i)}(z_i) W^{(i)}\right]A}
= \op{\D^{(\ell_2)}(z_{\ell_2})W^{(\ell_2)}\left[\prod_{i = \ell_2-1}^{\ell_1} \D^{(i)}(z_i) W^{(i)} \right]A}.
\]
Let $\arrow{W} \defeq (W^{(0)}, \ldots, W^{(\ell_2 - 1)})$, $\arrow{b} \defeq (b^{(0)},\ldots,b^{(\ell_2 -1)})$ and $\arrow{\eps} \defeq (\eps^{(0)}, \ldots, \eps^{(\ell_2 - 1)})$.
We define the independent random variables $X \defeq (\arrow{W}, \arrow{b},\arrow{\eps})$ and $Y \defeq (W^{(\ell_2)}, b^{(\ell_2)},\eps^{(\ell_2)})$.
We now assume that the realization of the random variable $X$ is fixed and, for convenience, let $Z \defeq \left[\prod_{i = \ell_2-1}^{\ell_1} \D^{(i)}(z_i) W^{(i)}\right] A$,
which is thus also fixed. 
From the second part of \Cref{lem:smatrix} (with $\eps = \eps^{(\ell_2)}$), we infer 
\begin{align*}
\op{\D^{(\ell_2)}(z_{\ell_2})W^{(\ell_2)}Z} &= \op{\tilde{\Delta}\left(W^{(\ell_2)}\Phi^{(\ell_2 - 1)}(z_{\ell_2})+ b^{(\ell_2)}\right)W^{(\ell_2)}Z} 
\d \op{\DD^{(\ell_2)}W^{(\ell_2)}Z} \\
&= \op{\DD^{(\ell_2)}W^{(\ell_2)}\left[\prod_{i = \ell_2-1}^{\ell_1} \D^{(i)}(z_i) W^{(i)}\right] A}
\end{align*}
for every fixed realization of $X$.
From \Cref{lem:distr}, we then obtain 
\[
\op{\D^{(\ell_2)}(z_{\ell_2})W^{(\ell_2)}\left[\prod_{i = \ell_2-1}^{\ell_1} \D^{(i)}(z_i) W^{(i)} \right]A}
\d \op{\DD^{(\ell_2)}W^{(\ell_2)}\left[\prod_{i = \ell_2-1}^{\ell_1} \D^{(i)}(z_i) W^{(i)}\right] A},
\]
where the randomness is with respect to all the occurring random variables.
From the auxiliary claim and the independence of the occurring random variables, we then get 
\[
\op{\DD^{(\ell_2)}W^{(\ell_2)}\left[\prod_{i = \ell_2-1}^{\ell_1} \D^{(i)}(z_i) W^{(i)}\right] A}
\d \op{\DD^{(\ell_2)}W^{(\ell_2)}\left[\prod_{i = \ell_2-1}^{\ell_1} \DD^{(i)} W^{(i)}\right] A},
\] 
as claimed. 

The second claim of \Cref{prop:randgrad} follows directly from the auxiliary claim, by picking $A$ as the identity matrix and $\ell_2 = L$. 
\end{proof}

\renewcommand*{\proofname}{Proof}

%% file: tesselations.tex
\section{Proofs for the results on random hyperplane tessellations}\label{sec:rand_tess}
In this appendix, we prove the results on random hyperplane tessellations stated in \Cref{sec:tess} and prove special cases that will be used in the present paper. 

We start by defining the angular distance between two points on the sphere, 
which turns out to be the correct quantity for describing
the probability that a Gaussian hyperplane with zero bias separates these two points. 
\begin{definition}\label{def:angle}
    For two points $x,y \in \SS^{n-1}$ we define the angular distance of these points via
\begin{equation*}
\ang(x,y) \defeq \frac{1}{\pi} \cdot \arccos \left(\langle x,y \rangle\right) \in [0,1].
\end{equation*}
\end{definition}

The following folklore lemma establishes that the angular distance is equivalent to the Euclidean distance. 
\begin{lemma}\label{lem:angdist}
For $x,y \in \SS^{n-1}$ we have
\begin{equation*}
 \frac{1}{\pi} \cdot \mnorm{x-y}_2 \leq \ang(x,y) \leq  \mnorm{x-y}_2.
\end{equation*}
\end{lemma}
\begin{proof}
We start by showing that for all $z \in [0,1]$ we have 
\[
z \leq \sqrt{2} \cdot \sqrt{1- \cos(\pi z)}\leq \pi \cdot z.
\]
Here, we use the identity $1- \cos (\pi z) = 2 \cdot \sin^2(\pi z / 2)$ and hence get that the above 
is equivalent to 
\[
z \leq 2 \cdot \sin (\pi z /2) \leq \pi \cdot z. 
\]
The upper bound follows directly from the well-known fact that $\sin(t) \leq t$ for $t \geq 0$. 
For the lower bound, we define $f(z) \defeq 2 \cdot \sin(\pi z/2) - z$ and note that $f(0)= 0$ and $f(1)= 1$. 
Moreover, since $f''(z) = - \frac{\pi^2}{2} \cdot \sin(\pi z /2) < 0$ for $z \in (0,1)$, we conclude that $f$ is concave on $(0,1)$, which implies $f(z) \geq 0$ for every $z \in [0,1]$. 

We now substitute $z = \ang(x,y)$ and first obtain 
\[
\sqrt{2}\cdot \sqrt{1- \cos(\pi\cdot\ang(x,y))}
= \sqrt{2} \cdot \sqrt{1- \langle x,y \rangle} 
= \sqrt{\mnorm{x}_2^2 + \mnorm{y}_2^2 - 2 \langle x, y \rangle } = \mnorm{x-y}_2.
\]
Hence, we have 
\[
\ang(x,y) \leq \mnorm{x-y}_2 \leq \pi \cdot \ang(x,y),
\]
which gives us the claim. 
\end{proof}

The central property of the angular distance related to random hyperplane tessellations is 
stated in the following proposition, cited from \cite{plan2014dimension}.
\begin{proposition}[cf. {\cite[Equation~(5.1)]{plan2014dimension}}]\label{prop:angdisttess}
    Let $x,y \in \SS^{d-1}$, $W \sim \mathcal{N}(0,I_d)$ and, moreover, let $\sgn \in \{\sgn_+, \sgn_-, \sgn_0\}$. 
    Then,
    \[
    \PP\big(\sgn(\langle W,x \rangle) \neq \sgn(\langle W,y \rangle)\big) = \ang(x,y).
    \]
\end{proposition}
\showdiamondfalse
\begin{rem}
\begin{enumerate}
\item We remark that \cite{plan2014dimension} only considers $\sgn_+$. 
However, from
\[
\underset{W \sim \mathcal{N}(0, I_d)}{\PP}(\langle W,x\rangle = 0) = 0 \quad \text{for every}
\quad  x \in \RR^d \setminus \{0\} 
\]
it follows that \cite[Equation~(5.1)]{plan2014dimension} remains true for every choice of 
$\sgn \in \{\sgn_+, \sgn_-, \sgn_0\}$.
\item Moreover, we remark that \cite{plan2014dimension} only speaks of the "normalized geodesic distance
on $\SS^{d-1}$". 
That this distance in fact agrees with $\ang(x,y)$ follows, e.g., from \cite[Example~10.21]{boumal2023optimization}. \hfill $\diamond$
\end{enumerate}
\end{rem}
\showdiamondtrue

Using Chernoff's inequality \cite[Theorem~2.3.1]{vershynin_high-dimensional_2018} we can now prove \Cref{lem:recht_low_low}. 
\begin{proof}[Proof of \Cref{lem:recht_low_low}]
    We may assume $\sigma^2 =1$ and $x \neq y$. 
    For $i \in \{1,\ldots,m\}$, set 
    \[
    X_i \defeq \mathbbm{1}_{\sgn((Ax)_i) \neq \sgn((Ay)_i)} \quad \text{and} 
    \quad S_m \defeq \sum_{i=1}^m X_i.
    \]
    By \Cref{prop:angdisttess} we have 
    \[
    \EE[S_m]= m \cdot \ang(x,y).
    \]
    Hence, applying \cite[Theorem~2.3.1]{vershynin_high-dimensional_2018} with $t \defeq \ee m \ang(x,y)$ we get 
    \[
    \PP\big(S_m \geq \ee m \ang(x,y)\big) \leq \exp(-m \cdot \ang(x,y))\cdot \left(\frac{\ee m \ang(x,y)}{\ee m \ang(x,y)}\right)^t = \exp(-m \cdot \ang(x,y)).
    \]
    Conversely, applying \cite[Ex.~2.3.2]{vershynin_high-dimensional_2018} with $t \defeq \frac{\ee}{3}\cdot m \cdot \ang(x,y)$, we get 
    \begin{align*}
       \PP\left(S_m \leq \frac{\ee}{3} m \ang(x,y)\right)
       &\leq \exp(-m \cdot \ang(x,y)) \cdot 3^{\frac{\ee}{3} \cdot m \cdot \ang(x,y)} \\
       &= \exp\left(\ln(3)\cdot \frac{\ee}{3} \cdot m \cdot \ang(x,y) - m \cdot \ang(x,y)\right) \\
       &\leq \exp(-c \cdot m \cdot \ang(x,y)),
    \end{align*}
    where $c \defeq 1 - \ln(3) \cdot \frac{\ee}{3} > 0$.
    In combination with \Cref{lem:angdist} we obtain the claim. 
\end{proof}

For the proofs of \Cref{lem:recht_upper_original,prop:farout,corr:tess} we need the following 
auxiliary statement. 
\begin{lemma}\label{lem:unif_tool}
    There exist absolute constants $C,c > 0$ such that the following holds.
    Let $A \in \RR^{m \times n}$ be a random matrix with entries 
    $A_{i,j} \iid \mathcal{N}(0,1)$, let $K \subseteq \SS^{n-1}$, 
    $\alpha \in \left(m^{-1}, \ee^{-1}\right)$
    and $\beta > 0$. 
    Suppose that $0 <\eps \leq c \cdot \beta \cdot \ln^{-1/2}(1/\alpha)$ and 
    \[
    m \geq C \cdot w^2\left((K-K) \cap \B_n(0,\eps)\right) \cdot \alpha^{-1} \cdot \beta^{-2}.
    \]
    Then 
    \[
    \underset{w \in (K-K) \cap \B_n(0,\eps)}{\sup} \# \big\{ i \in \{1,\ldots,m\}: \ \abs{(Aw)_i} \geq \beta \big\} \leq \alpha \cdot m
    \]
    with probability at least $1- 2\exp(-c \cdot \alpha m)$. 
\end{lemma}
\begin{proof}
    We set $K_\eps \defeq (K-K) \cap \B_n(0,\eps)$ and $M \defeq \lfloor \alpha \cdot m\rfloor\geq \frac{\alpha m}{2}$.
    By \cite[Theorem~SM1.1]{dirksen2022sharp}, with the setting $u=1$,
    \begin{equation}\label{eq:first-eq}
    \sup_{w \in K_\eps}\ \max_{I \subseteq \{1,\ldots,m\}, \abs{I} \leq M}
    \left(\sum_{i \in I} (Aw)_i^2\right)^{1/2} \leq C_1 \cdot \left(w(K_\eps) + \eps \cdot \sqrt{M\ln(\ee m / M)}\right)
    \end{equation}
    with probability at least $1- 2\exp(-c_1 \cdot M\ln(\ee m / M)) \geq 1 -2\exp(-c_2 \cdot \alpha m)$ with absolute constants $C_1, c_1,c_2 > 0$.
    Here, we also used that $- K_\eps = K_\eps$, whence we may replace the Gaussian complexity (which appears in the formulation of \cite[Theorem~SM1.1]{dirksen2022sharp}) by the Gaussian width.
    By the assumption on $m$ and by taking $C > 0$ large enough, we may assume 
    \begin{equation}\label{eq:second-eq}
    C_1 \cdot w(K_\eps) \leq \sqrt{\alpha m /2} \cdot \frac{\beta}{4} \leq \sqrt{M} \cdot \frac{\beta}{4}
    \end{equation}
    since $m \geq C \cdot w^2(K_\eps) \cdot \alpha^{-1} \cdot \beta^{-2}$.
    Further, since $\eps \leq c \cdot \beta \cdot \ln^{-1/2}(1/\alpha)$, we get 
    \begin{align}
    C_1 \cdot \eps \cdot \sqrt{M \cdot \ln(\ee m /M)} &= 
    C_1 \cdot \eps \cdot \sqrt{M} \cdot \ln^{1/2}\Big(\frac{\ee}{M} \cdot \underbrace{\frac{\alpha m}{2}}_{\leq M} \cdot \frac{2}{\alpha}\Big) 
    \leq C_1 \cdot \eps \cdot \sqrt{M \cdot \ln(2\ee /\alpha)} \nonumber\\
    \label{eq:third-eq}
    \overset{\text{\cshref{lem:log_b}}}&{\leq} \sqrt{2\ee} \cdot C_1 \cdot \eps \cdot \sqrt{M\cdot \ln(1/\alpha)} \leq \sqrt{M} \cdot \frac{\beta}{4}
    \end{align}
    if $c$ is small enough. 
    Combining \eqref{eq:first-eq}, \eqref{eq:second-eq} and \eqref{eq:third-eq} yields 
    \begin{equation}\label{eq:contraax}
    \sup_{w \in K_\eps}\ \max_{I \subseteq \{1,\ldots,m\}, \abs{I} \leq M}
    \left(\frac{1}{M} \cdot \sum_{i \in I} (Aw)_i^2\right)^{1/2} \leq  \frac{\beta}{2}
    \end{equation}
    with probability at least $1-2\exp(-c_2 \cdot \alpha m)$.
    Now note that if we had
    \[
    \sup_{w \in K_\eps} \# \big\{i \in \{1,\ldots,m\}: \ \abs{(Aw)_i} \geq \beta \big\} > \alpha \cdot m \geq M
    \]
    then there would exist $w \in K_\eps$ and a set $I \subseteq \{1,\ldots,m\}$ with $\abs{I} = M$ and $\abs{(Aw)_i} \geq \beta$ for every $i \in I$.
    In particular, 
    \[
    \left(\frac{1}{M} \cdot \sum_{i \in I} (Aw)_i^2\right)^{1/2} \geq \left(\frac{1}{M} \cdot \sum_{i \in I} \beta^2\right)^{1/2}
    = \left(\frac{1}{M} \cdot M \cdot \beta^2\right)^{1/2}
     = \beta > \frac{\beta}{2}
     \]
    in contradiction to \eqref{eq:contraax}.
    This completes the proof. 
\end{proof}
For the proofs of \Cref{lem:recht_upper_original,prop:farout} we additionally need the following statement, which is based on an application of Chernoff's inequality. 
\begin{lemma}\label{lem:chernoff}
There exists an absolute constant $c > 0$ such that the following holds.
For arbitrary $\alpha \in (0,1]$, $x \in \SS^{n-1}$ and a random matrix $A \in \RR^{m \times n}$ with $A_{i,j} \iid \mathcal{N}(0,1)$, with probability at least 
$1- \exp(-c \cdot \alpha m)$, 
\[
\# \big\{ i \in \{1, \dots, m\}: \ \abs{(Ax)_i} \leq \alpha\big\} \leq 2\ee \cdot \alpha m.
\]
\end{lemma}
\begin{proof}
    Fix $i \in \{1,\dots,m\}$.
    Then $(Ax)_i \sim \mathcal{N}(0,1)$. 
    Moreover, for $w \in [-1,1]$ we have 
    \[
    c_1/2 \defeq \frac{1}{\sqrt{2\pi \ee }} \leq \frac{1}{\sqrt{2\pi}} \exp(-w^2/2) \leq \frac{1}{\sqrt{2\pi}} \leq 1,
    \]
    which implies because of $\alpha \in (0,1]$ that
    \[
    c_1 \cdot \alpha \leq \PP(\abs{(Ax)_i} \leq \alpha)\leq 2 \cdot \alpha .
    \]
    For $i \in \{1,\ldots,m\}$, set $X_i \defeq \mathbbm{1}_{\abs{(Ax)_i} \leq \alpha}$
    and further 
    \[
    S_m \defeq \sum_{i=1}^m X_i = \# \big\{i \in \{1,\dots, m\}: \ \abs{(Ax)_i} \leq \alpha \big\}.
    \]
    Note that the $X_i$ are independent Bernoulli variables and 
    \[
    c_1 \cdot \alpha m\leq\EE [S_m] \leq 2 \cdot \alpha m.
    \]
    By Chernoff's bound (see, e.g., \cite[Theorem~2.3.1]{vershynin_high-dimensional_2018})
    with $t \defeq 2\ee \cdot \alpha m$,
    \[
    \PP \left(S_m \geq 2\ee \cdot \alpha m\right)
    \leq \exp(- \EE [S_m]) \cdot \left(\frac{\ee \cdot \EE [S_m]}{t}\right)^t
    \leq \exp(- c_1 \cdot \alpha m) \cdot \left(\frac{2\ee \cdot \alpha m }{2\ee \cdot \alpha m}\right)^t = \exp(- c_1 \cdot \alpha m). \qedhere
    \]
\end{proof}
As the last tool which is required for proving \Cref{lem:recht_upper_original,prop:farout} we state and
prove the following elementary but important observation. 
\begin{lemma}\label{lem:elem}
    Let $\sgn \in \{\sgn_+, \sgn_-, \sgn-0\}$.
    Then, given $x,y,x^\ast \in \RR$ with $\abs{x^\ast} > \tau$ and 
    with $\abs{x-x^\ast}, \abs{y-x^\ast} \leq \tau$ 
for some $\tau > 0$, we have 
\[
\sgn(x) = \sgn(x^\ast) = \sgn(y).
\]
\end{lemma}
\begin{proof}
    By symmetry, it suffices to show $\sgn(x) = \sgn(x^\ast)$.
    Moreover, we may without loss of generality assume that $x^\ast > 0$.
    We then get
    \[
    x = x^\ast - (x^\ast - x) \geq x^\ast - \abs{x^\ast - x} > \tau - \tau = 0
    \]
    and thus, $\sgn(x) = \sgn(x^\ast)$.
\end{proof}
We now have all the tools together to prove \Cref{lem:recht_upper_original}.
\begin{proof}[Proof of \Cref{lem:recht_upper_original}]
    We may assume $\sigma^2 = 1$.

    Let $K_\eps \defeq (K-K) \cap \B_n(0,\eps)$.
    Let $C_1, c_1 > 0$ be the constants from \Cref{lem:unif_tool}. 
    Set $\alpha \defeq \delta/3$ and $\beta \defeq \frac{\delta}{6\ee}$.
    Note that $\delta m \geq C$ by assumption.
    Since we can choose $C>3$, this yields $1/m < \alpha < \delta < \ee^{-1}$.
    Moreover, taking $c > 0$ small enough,  
    \[
    \eps \leq c \cdot \delta \cdot \ln^{-1/2}(1/\delta) = 6\ee c \cdot \beta \cdot \ln^{-1/2}(1/(3\alpha)) \overset{\text{\cshref{lem:log_b}}}{\leq} 6\ee\sqrt{3} \cdot c\cdot\beta \cdot \ln^{-1/2}(1/\alpha) \leq c_1 \cdot \beta \cdot \ln^{-1/2}(1/\alpha).
    \]
    Furthermore, by possibly enlarging $C > 0$, we can ensure that
    \[
    C_1 \cdot w^2(K_\eps) \cdot \alpha^{-1} \cdot \beta^{-2} 
    = 3(6\ee)^2 \cdot C_1 \cdot w^2(K_\eps) \cdot \delta^{-3} \leq m.
    \]
    The conditions of \Cref{lem:unif_tool} are thus satisfied and we get  
    \[
    \sup_{w \in K_\eps} \# \big\{i \in \{1,\ldots,m\}: \ \abs{(Aw)_i} \geq \delta/(6\ee) \big\} \leq\frac{\delta m }{3}
    \]
    with probability at least $1-2\exp(-c_1 \cdot \alpha m) = 1-2\exp(-c_2 \cdot \delta m)$ with $c_2 \defeq c_1/3$ and we call the event defined by that property $\E_1$.

    \medskip
    \newcommand{\nepsstrich}{\mathcal{N}_{\eps '}}
    Finally, let $\neps$ be an $\eps/2$-net of $K$ with $\abs{\neps} = \mathcal{N}(K,\eps/2)$. 
    Using \Cref{lem:chernoff} (with $\alpha =\delta/(6\ee)$) and a union bound, we obtain 
    the existence of an absolute constant $c_3 > 0$ such that
    \[
    \text{for all } x^\ast \in \neps: \quad \# \{i \in \{1,\ldots,m\}: \ \abs{(Ax^\ast)_i} \leq \delta/(6\ee)\} \leq m \cdot \frac{\delta}{3}
    \]
    with probability at least $1-\exp(\ln(\mathcal{N}(K,\eps/2)) - c_3 \cdot \delta m)
    \geq 1-\exp(-c_4 \cdot \delta m)$ by letting $c_4 \defeq c_3 / 2$ and by using that
    \[
    \ln(\mathcal{N}(K,\eps/2)) \leq \ln(\ee \cdot \mathcal{N}(K,\eps/2)) \leq \frac{\delta m}{C} \leq \frac{c_3}{2} \cdot \delta m,
    \]
    which holds for $C$ large enough. 
    We call the event defined by that property $\E_2$.

   We now show that on $\E_1 \cap \E_2$, the estimate 
   \[
   \#\big\{i \in \{1, \dots, m\}: \ \sgn((Ax)_i) = \sgn((Ay)_i)\big\} \leq \delta m
   \]
   holds. Indeed, assume that the defining properties of $\E_1$ and $\E_2$ are satisfied.
    Let $x,y\in K$ with $\mnorm{x-y}_2 \leq \eps/2$ be arbitrary and pick $x^\ast  \in \neps$ that satisfies $\mnorm{x-x^\ast}_2 \leq \eps /2$. 
    Clearly,
    \[
    \mnorm{x^\ast - y}_2 \leq \mnorm{x^\ast - x}_2 + \mnorm{x - y}_2
    \leq \eps
    \]
    and hence $x^\ast - x, x^\ast - y \in K_\eps$.
    We then get 
    \begin{align*}
        &\norel\# \big\{i \in \{1,\ldots,m\}: \ \sgn((Ax)_i) \neq \sgn((Ay)_i)\big\} \\
        \overset{\mathrm{\cshref{lem:elem}}}&{\leq} \# \big\{i \in \{1,\ldots,m\}: \ \abs{(Ax^\ast)_i}\leq \delta/(6 \ee )\big\} + \# \big\{ i \in \{1,\ldots,m\}: \ \abs{(A(x^\ast - x))_i} \geq \delta/(6\ee)\big\} \\
        &\norel {}+ \# \big\{ i \in \{1,\ldots,m\}: \ \abs{(A(x^\ast - y))_i} \geq \delta/(6\ee)\big\} \leq  \delta m.
    \end{align*}
    Moreover, letting $c_5 \defeq \min\{c_2,c_4\}$, we have  
    \begin{align*}
        \PP(\E_1 \cap \E_2) \geq 1-3\exp(-c_5 \cdot \delta m) \geq 1-\exp(-c \cdot \delta m)
    \end{align*}
    by picking $c \leq c_5 /2$, since, by using $\delta m \geq C$ and taking $C> 0$ large enough, we have $\ln(3) \leq \frac{c_5}{2} \cdot \delta m$.
    This yields the claim. 
\end{proof}

We prove a special case of \Cref{lem:recht_upper_original}, where the considered inputs are taken from a set $f(\SS^{n-1}) \subseteq \SS^{\ell-1}$, where
$f : \SS^{n-1} \to \SS^{\ell-1}$ is Lipschitz. 
Crucially, the obtained bound does not depend on $\ell$, but only on the initial dimension $n$.
Hence, in cases where $\ell \gg n$, we can use this approach to preserve the underlying $n$-dimensionality of the inputs.  
\begin{lemma}\label{lem:recht_upper_h}
There exist absolute constants $C,c>0$ with the following property: 
Given natural numbers $\ell, m , n \in \NN$ and a random matrix $A \in \RR^{m \times \ell}$ with $A_{i,j} \iid \mathcal{N}(0, \sigma^2)$ for some fixed variance
$\sigma^2 > 0$, a Lipschitz-continuous map $f: \SS^{n-1} \to \SS^{\ell - 1}$, $M \geq \max\{1,\lip_{2 \to 2}(f)\}$ and $\delta \in (0, \ee^{-1})$ with 
\begin{equation*}
m \geq C  \cdot n \cdot \delta^{-1} \cdot \ln(M/\delta),
\end{equation*}
the event defined by
\begin{align*}
\text{for all}\quad  x,y \in f(\SS^{n - 1}) \quad \text{with} \quad \mnorm{x-y}_2 \leq c \cdot \delta \cdot \ln^{-1/2}(1/\delta): \\
\# \big\{ i \in \{1,\ldots,m\}: \ \sgn((Ax)_i) \neq \sgn((Ay)_i)\big\} \leq \delta \cdot m
\end{align*}
occurs with probability at least $1-\exp(-c \cdot \delta m)$.
 \end{lemma}
 \begin{proof}
 Let $\tilde{C}, \tilde{c}>0$ be the constants provided by \Cref{lem:recht_upper_original}.
Let $\delta \in (0, \ee^{-1})$ be arbitrary and set $\eps \defeq \tilde{c} \cdot \delta \cdot \ln^{-1/2}(1/\delta)$.
Moreover, let $K \defeq f(\SS^{n - 1})$.
The goal is now to check that the condition from \Cref{lem:recht_upper_original} is satisfied, i.e., that
\begin{equation}\label{eq:tbs}
m \geq \tilde{C} \cdot 
\max\left\{\delta^{-3} w^2((K-K)\cap \B_\ell(0,\eps), \delta^{-1} \cdot \ln(\ee \cdot \mathcal{N}(K,\eps/2))\right\}.
\end{equation}
This then proves the claim of the lemma upon choosing $c \leq \tilde{c}/2$.

To this end, we take an arbitrary $\theta \in (0,1)$. 
Let $\mathcal{N} \subseteq \SS^{n-1}$ be a $\frac{\theta}{M}$-net of $\SS^{n-1}$ with
\begin{equation*}
\#N \leq \left(\frac{3M}{\theta}\right)^n,
\end{equation*}
which exists according to \cite[Corollary~4.2.13]{vershynin_high-dimensional_2018} and noting $\theta / M \leq 1$.
Then, since $\lip_{2 \to 2}(f) \leq M$, 
one easily sees that $f(\mathcal{N})$ is a $\theta$-net of $K$, whence
\begin{equation}\label{eq:netlipbound}
\mathcal{N}(K,\theta) \leq \left(\frac{3M}{\theta}\right)^n \quad \text{for all } \theta \in (0,1).
\end{equation}
Moreover, by \cite[Exercise~4.2.10]{vershynin_high-dimensional_2018},
\begin{equation*}
\mathcal{N}((K-K) \cap \B_{\ell}(0,\eps), \Vert \cdot \Vert_2, \theta) \leq \mathcal{N}((K-K) , \Vert \cdot \Vert_2, \theta / 2)
\leq \mathcal{N}(K , \Vert \cdot \Vert_2, \theta / 4)^2,
\end{equation*}
since taking the Minkowski difference of a $\frac{\theta}{4}$-net of $K$ with itself canonically yields a $\frac{\theta}{2}$-net
of $K-K$. Therefore, since $\diam((K-K) \cap \B_{\ell}(0,\eps))\leq 2\eps$, Dudley's inequality (see, e.g., \cite[Theorem~8.1.10]{vershynin_high-dimensional_2018}) yields
\begin{align*}
w((K-K) \cap \B_{\ell}(0,\eps)) &\leq C_1 \cdot \int_0^{2\eps} 
\ln^{1/2}\left(\mathcal{N}((K-K) \cap \B_{\ell}(0,\eps), \Vert \cdot \Vert_2, \theta)\right) \ \dd \theta \\
&\leq \sqrt{2} \cdot C_1 \cdot \int_0^{2\eps} \ln^{1/2}(\mathcal{N}(K, \Vert \cdot \Vert_2, \theta / 4)) \ \dd \theta
\end{align*}
with an absolute constant $C_1 > 0$.
Since $\frac{\eps}{2} = \frac{1}{2} \cdot \tilde{c} \cdot \delta \cdot \ln^{-1/2}(1/\delta)
\leq 1$ (since we may assume $\tilde{c} \leq 1$), we can use \eqref{eq:netlipbound} to further bound
\begin{align*}
w((K-K) \cap \B_{\ell}(0,\eps)) \leq \sqrt{2} \cdot C_1 \cdot \sqrt{n} \cdot \int_0^{2\eps} \ln^{1/2}\left(\frac{12M}{ \theta}\right) \ \dd\theta. 
\end{align*}
Using \Cref{lem:int_b}, we get 
\begin{align*}
\int_0^{2\eps} \ln^{1/2}\left(\frac{12M}{ \theta}\right) \ \dd\theta 
\leq 4\eps \cdot \ln^{1/2}(6M/\eps).
\end{align*}
Note that 
\begin{align}
    \ln^{1/2}(6M/\eps) &= \ln^{1/2}\left(\frac{6M \cdot \ln^{1/2}(1/\delta)}{\widetilde{c}\cdot\delta}\right)
    \overset{\ln^{1/2}(1/\delta) \leq 1/\delta}{\leq}
    \ln^{1/2}\left(\frac{6M}{\widetilde{c}\cdot\delta^2}\right)
    \leq \sqrt{2} \cdot \ln^{1/2}\left(\frac{6M}{\tilde{c} \delta}\right) \nonumber\\
    \label{eq:6M}
    \overset{\text{\cshref{lem:log_b}}}&{\leq} \frac{2\sqrt{3}}{\sqrt{\tilde{c}}} \cdot \ln^{1/2}(M/\delta)
    \leq \frac{2\sqrt{3}}{\tilde{c}} \cdot \ln^{1/2}(M/\delta),
\end{align}
whence
\[
4\eps \cdot \ln^{1/2}(6M/\eps) 
\leq  4\tilde{c} \cdot \delta \cdot \ln^{-1/2}(1/\delta) \cdot \frac{2\sqrt{3}}{\tilde{c}} \cdot \ln^{1/2}(M/\delta)
\leq 8\sqrt{3} \cdot \delta \cdot \ln^{1/2}(M/\delta).
\]
Overall, we get 
\begin{equation*}
w^2((K-K) \cap \B_{\ell}(0,\eps)) \leq C_2 \cdot n  \cdot \delta^2 \cdot \ln(M/\delta)
\end{equation*}
with a suitable absolute constant $C_2 > 0$.
According to \eqref{eq:tbs}, we want to ensure 
\begin{equation*}
w^2\left((K-K)\cap \B_\ell(0,\eps)\right) \cdot \delta^{-3}\leq C_2 \cdot n \cdot \delta^{-1} \cdot \ln(M/\delta) \overset{!}{\leq} \tilde{C}^{-1} \cdot m,
\end{equation*}
which is satisfied taking $C \geq \tilde{C} \cdot C_2$ in our assumptions. 

Moreover, we consider
\begin{align*}
\delta^{-1} \cdot \ln(\ee \cdot\mathcal{N}(K, \Vert \cdot \Vert_2, \eps/2)) \overset{\eqref{eq:netlipbound}}&{\leq}
\delta^{-1} \cdot n \cdot \ln \left(\frac{6\ee M}{\eps}\right) 
\overset{\text{\cshref{lem:log_b}}}{\leq} \ee \cdot \delta^{-1} \cdot n \cdot \ln \left(\frac{6M}{\eps}\right) \\
\overset{\eqref{eq:6M}}&{\leq} \delta^{-1} \cdot n \cdot \frac{12\ee}{(\widetilde{c})^2} \cdot \ln(M/\delta) 
\leq C_3 \cdot n  \cdot \delta^{-1} \cdot \ln(M/\delta) 
\leq \widetilde{C}^{-1} \cdot m
\end{align*}
for an absolute constant $C_3>0$, 
where the last inequality is satisfied if we take $C \geq \widetilde{C} \cdot C_3$
in our assumptions. 
The lemma is proven. 
 \end{proof}

\renewcommand*{\proofname}{Proof}

Let us now turn towards the non-homogeneous setting, i.e., the setting of non-zero biases. 
We first prove \Cref{prop:farout}, which states that for inputs with large Euclidean norm the bias term may be neglected. 

\begin{proof}[Proof of \Cref{prop:farout}]
    Set $K_\eps \defeq (K-K) \cap \B_n(0,\eps)$.
    Let $C_1, c_1 > 0$ be the constants from \Cref{lem:unif_tool}. 
    Set $\alpha \defeq \delta/2$ and $\beta \defeq \frac{\delta}{8\ee}$.
    Note that $\delta m \geq C$ by assumption; by choosing $C>2$, this yields $m^{-1} < \alpha < \delta < \ee^{-1}$.
    Moreover, taking $c > 0$ small enough,  
    \[
    \eps \leq c \cdot \delta \cdot \ln^{-1/2}(1/\delta) = 8\ee \cdot c \cdot \beta \cdot \ln^{-1/2}(1/(2\alpha)) \overset{\text{\cshref{lem:log_b}}}{\leq} 8\ee\sqrt{2} \cdot c\cdot\beta \cdot \ln^{-1/2}(1/\alpha) \leq c_1 \cdot \beta \cdot \ln^{-1/2}(1/\alpha).
    \]
    Furthermore, by possibly enlarging $C > 0$, 
    \[
    C_1 \cdot w^2(K_\eps) \cdot \alpha^{-1} \cdot \beta^{-2} 
    = 2(8\ee)^2 \cdot C_1 \cdot w^2(K_\eps) \cdot \delta^{-3} \leq m. 
    \]
    The conditions of \Cref{lem:unif_tool} are thus satisfied and we get  
    \[
    \sup_{w \in K_\eps} \# \big\{i \in \{1,\ldots,m\}: \ \abs{(Aw)_i} \geq \delta/(8\ee) \big\} \leq\frac{\delta m }{2}
    \]
    with probability at least $1-2\exp(-c_1 \cdot \alpha m) = 1-2\exp(-c_2 \cdot \delta m)$, where $c_2 \defeq c_1/2$. We call the event defined by that property $\E_1$.

    \medskip
    
    Finally, let $\neps$ be an $\eps$-net of $K$ with $\#\neps = \mathcal{N}(K,\eps)$. 
    Using \Cref{lem:chernoff} (with $\alpha = \delta/(4\ee)$) and a union bound, we obtain
    the existence of an absolute constant $c_3 > 0$ such that
    \begin{equation}\label{eq:unbound}
    \text{for all } x^\ast \in \neps: \quad \# \big\{i \in \{1,\ldots,m\}: \ \abs{(Ax^\ast)_i} \leq \delta/(4\ee)\big\} \leq m \cdot \frac{\delta}{2}
    \end{equation}
    with probability at least $1-\exp(\ln(\mathcal{N}(K,\eps)) - c_3 \cdot \delta m)
    \geq 1-\exp(-c_4 \cdot \delta m)$ by letting $c_4 \defeq c_3 / 2$ and by choosing $C>0$ so large that
    \[
    \ln(\mathcal{N}(K,\eps)) \leq \ln(\ee \cdot \mathcal{N}(K,\eps))\leq C^{-1} \cdot \delta m \leq \frac{c_3}{2} \cdot \delta m.
    \]
    We write $\E_2$ for the event defined in \eqref{eq:unbound}.
    On $\E_1\cap \E_2$, we observe 
    \begin{align*}
        &\norel \sup_{x \in \cone(K), \mnorm{x}_2 \geq 8\ee\delta^{-1}\lambda} \# \big\{ i \in \{1,\ldots,m\} : \ \sgn((Ax)_i + \tau_i) \neq \sgn((Ax)_i)\big\} \\
        &\leq 
        \sup_{x \in \cone(K), \mnorm{x}_2 \geq 8\ee\delta^{-1}\lambda} \# \big\{ i \in \{1,\ldots,m\} : \ \abs{(Ax)_i} \leq \abs{\tau_i}\big\} \\
        &\leq 
        \sup_{x \in \cone(K), \mnorm{x}_2 \geq 8\ee\delta^{-1}\lambda} \# \big\{ i \in \{1,\ldots,m\} : \ \abs{(Ax)_i} \leq \lambda\big\} \\
        &= \sup_{x \in \cone(K), \mnorm{x}_2 \geq 8\ee\delta^{-1}\lambda} \# \left\{ i \in \{1,\ldots,m\} : \ \abs{\left(A\frac{x}{\mnorm{x}_2}\right)_i} \leq \frac{\lambda}{\mnorm{x}_2}\right\} \\
        &\leq \sup_{z \in  K} \# \big\{ i \in \{1,\ldots,m\} : \ \abs{\left(Az\right)_i} \leq \delta/(8\ee)\big\} \\
        &\leq \sup_{x^\ast \in  \neps} \# \big\{ i \in \{1,\ldots,m\} : \ \abs{\left(Ax^\ast\right)_i} \leq \delta/(4\ee)\big\} 
        + \sup_{w \in K_\eps} \# \big\{ i \in \{1,\ldots,m\} : \ \abs{\left(Aw\right)_i} \geq \delta/(8\ee)\big\} \leq \delta m.
    \end{align*}
    Moreover, letting $c_5 \defeq \min\{c_2,c_4\}$, we have  
    \begin{align*}
        \PP(\E_1 \cap \E_2) \geq 1-3\exp(-c_5 \cdot \delta m) \geq 1-\exp(-c \cdot \delta m)
    \end{align*}
    by picking $c \leq c_5 /2$, since we may assume because of $\delta m \geq C$ that $\ln(3) \leq \frac{c_5}{2} \cdot \delta m$.
    This yields the claim. 
\end{proof}

We continue by proving \Cref{corr:tess} which may be seen as a version of \Cref{lem:recht_upper_original} for the case with non-zero biases. 

\begin{proof}[Proof of \Cref{corr:tess}]
    \textbf{Pointwise estimate:}
    Let $z \in K$, let $W \in \RR^n$ be a random vector that is independent of $b$, 
    and with $W \sim \mathcal{N}(0,I_n)$.
    By conditioning on $W$ and using the small-ball property of the distribution of $b_i$, 
    we obtain for $i \in \{1, \dots, m\}$,
    \[
    \PP\left(\abs{\langle W, z \rangle + b_i } \leq \frac{\delta}{3\ee \cdot C_\tau}\right)\leq \frac{\delta}{3 \ee}.
    \]
    For $i \in \{1,\ldots,m\}$, set $X_i \defeq \mathbbm{1}_{\abs{(Az)_i + b_i} \leq \delta/(3\ee C_\tau)}$.
    Note by what we have just shown that we have $\PP(X_i = 1) \leq \frac{\delta}{3\ee}$ for $i \in \{1,\ldots,m\}$. 
    Hence\footnote{
    To see this, for $i \in \{1, \dots, m\}$, let $p_i \in [0,1)$ be such that $X_i \sim \mathrm{Ber}(p_i)$ and set $q_i \defeq \frac{\frac{\delta}{3 \ee} - p_i}{1- p_i} \in (0,1)$. 
    Let $Z_1, \dots, Z_m$ be jointly independent and independent of the $X_i$ with $Z_i \sim \mathrm{Ber}(q_i)$ for $i \in \{1, \dots, m\}$.
    Finally, set $Y_i \defeq \begin{cases}1,& X_i = 1, \\ Z_i,& X_i = 0.\end{cases}$
    Clearly, the $Y_i$ are independent Bernoulli variables and satisfy $X_i \leq Y_i$.
    Moreover, 
    \[
    \PP(Y_i = 1) = \PP(X_i = 1) + \PP(X_i = 0)\PP(Z_i=1) = p_i + (1-p_i)q_i = \frac{\delta}{3\ee}.
    \]
    }, we can pick $Y_1, \ldots, Y_m \iid \mathrm{Ber}(\delta/(3\ee))$ with
    $X_i \leq Y_i$ for every $i \in \{1,\ldots,m\}$.
    We set $S_m \defeq \sum_{i=1}^m Y_i$ and note that 
    \[
    \EE[S_m] = m \cdot \frac{\delta}{3\ee} =: \mu.
    \]
    We may hence apply \cite[Theorem~2.3.1]{vershynin_high-dimensional_2018}
    with $t \defeq m \cdot \frac{\delta}{3}$ and get 
    \[
    \PP\left(\sum_{i=1}^m X_i \geq m \cdot \frac{\delta}{3}\right)\leq\PP \left(S_m \geq m \cdot \frac{\delta}{3}\right)
    \leq \exp(-\mu)\cdot \left(\frac{\ee \mu}{t}\right)^t 
    = \exp(-c_1 \cdot \delta m)
    \]
    with an absolute constant $c_1 > 0$.
    In other words, we have with probability at least $1-\exp(-c_1 \cdot \delta m)$,
    \begin{equation}\label{eq:pointtt}
    \# \big\{i \in \{1,\ldots,m\}: \ \abs{(Az)_i + b_i} \leq \delta/(3\ee C_\tau)\big\} \leq m \cdot \frac{\delta}{3}.
    \end{equation}

    \medskip

    \textbf{Uniformization:}
    Let $C_2, c_2 > 0$ be the constants from \Cref{lem:unif_tool}. 
    Set $\alpha \defeq \delta/3$ and $\beta \defeq \frac{\delta}{3\ee C_\tau}$.
    Note that $\delta m \geq C$ by the assumption on $m$ in \Cref{corr:tess}; 
    by choosing $C > 3$ this yields $m^{-1} < \alpha < \delta < \ee^{-1}$.
    Moreover, taking $c > 0$ small enough,  
    \begin{align*}
    \eps &\leq c \cdot C_\tau^{-1} \cdot \delta \cdot \ln^{-1/2}(1/\delta) = c \cdot C_\tau^{-1} \cdot 3\ee \cdot C_\tau  \cdot \beta \cdot \ln^{-1/2}(1/(3\alpha)) \\ \overset{\text{\cshref{lem:log_b}}}&{\leq} 3\ee\sqrt{3} \cdot c\cdot\beta \cdot \ln^{-1/2}(1/\alpha) \leq c_2 \cdot \beta \cdot \ln^{-1/2}(1/\alpha).
    \end{align*}
    Furthermore, by possibly enlarging $C > 0$, we see for $K_\eps \defeq (K-K) \cap \B_n(0,\eps)$ that
    \[
    C_2 \cdot w^2(K_\eps) \cdot \alpha^{-1} \cdot \beta^{-2} 
    = 3(3\ee)^2 \cdot C_\tau^2 \cdot C_2 \cdot w^2(K_\eps) \cdot \delta^{-3}\leq m. 
    \]
    The conditions of \Cref{lem:unif_tool} are thus satisfied and we get  
    \[
    \sup_{w \in K_\eps} \# \big\{i \in \{1,\ldots,m\}: \ \abs{(Aw)_i} \geq \delta/(3\ee C_\tau)\big\} \leq\frac{\delta m }{3}
    \]
    with probability at least $1-2\exp(-c_2 \cdot \alpha m) = 1-2\exp(-c_3 \cdot \delta m)$ where $c_3 \defeq c_2/3$. We call the event defined by that property $\E_1$.

    \medskip
    
    \textbf{Final proof:}
    Let $\neps$ be an $\eps/2$-net of $K$ with $\abs{\neps} = \mathcal{N}(K,\eps/2)$. 
    Using \eqref{eq:pointtt} and a union bound, we obtain 
    \[
    \text{for all } x^\ast \in \mathcal{N}_{\eps}: 
    \quad \# \big\{i \in \{1,\ldots,m\}: \ \abs{(Ax^\ast)_i + b_i} \leq \delta/(3\ee C_\tau)\big\} \leq m \cdot \frac{\delta}{3}
    \]
    with probability at least $1-\exp(\ln(\mathcal{N}(K,\eps/2)) - c_1 \cdot \delta m)
    \geq 1-\exp(-c_4 \cdot \delta m)$ by letting $c_4 \defeq c_1 / 2$ and by choosing $C>0$ so large that
    \[
    \ln(\mathcal{N}(K,\eps/2)) \leq \ln(\ee \cdot \mathcal{N}(K,\eps/2))\leq \frac{\delta m}{C} \leq \frac{c_1}{2} \cdot \delta m.
    \]
    We call the event defined by that property $\E_2$.

    We now show that on $\E_1 \cap \E_2$, we have for all $x,y \in K$ with $\mnorm{x-y}_2 \leq \eps/2$
    that 
    \[
    \# \big\{ i \in \{1, \dots, m\}: \ \sgn((Ax)_i + b_i) \neq \sgn((Ay)_i + b_i)\big\} \leq \delta m.
    \]
    Indeed, assume that the defining properties of $\E_1$ and $\E_2$ are satisfied.
    Let then $x,y\in K$ with $\mnorm{x-y}_2 \leq \eps/2$ be arbitrary and pick $x^\ast  \in \neps$ that satisfies $\mnorm{x-x^\ast}_2 \leq \eps /2$. 
    This implies 
    \[
    \mnorm{x^\ast - y}_2 \leq \mnorm{x^\ast - x}_2 + \mnorm{x - y}_2
    \leq \eps
    \]
    and hence $x^\ast - x, x^\ast - y \in K_\eps$.
    We hence get 
    \begin{align*}
        &\norel\# \big\{i \in \{1,\ldots,m\}: \ \sgn((Ax)_i + b_i) \neq \sgn((Ay)_i + b_i)\big\} \\
        \overset{\mathrm{\cshref{lem:elem}}}&{\leq} \! \! \! \!\!\! \!\# \big\{i \in \{1,\ldots,m\}: \ \abs{(Ax^\ast)_i + b_i}\leq \delta/(3 \ee C_\tau)\big\} \! + \! \# \big\{ i \in \{1,\ldots,m\}: \ \abs{(A(x^\ast - x))_i} \geq \delta/(3\ee C_\tau)\big\} \\
        &\norel + \# \big\{ i \in \{1,\ldots,m\}: \ \abs{(A(x^\ast - y))_i} \geq \delta/(3\ee C_\tau)\big\} \leq  \delta m,
    \end{align*}
    uniformly over the choice of $x$ and $y$.
    Moreover, letting $c_5 \defeq \min\{c_3,c_4\}$, we have  
    \begin{align*}
        \PP(\E_1 \cap \E_2) \geq 1-3\exp(-c_5 \cdot \delta m) \geq 1-\exp(-c \cdot \delta m)
    \end{align*}
    by picking $c \leq c_5 /2$, since we may assume because of $\delta m \geq C$ that $\ln(3) \leq \frac{c_5}{2} \cdot \delta m$.
    This yields the claim. 
\end{proof}
Analogous to \Cref{lem:recht_upper_h}, we prove a version of \Cref{corr:tess} for the case where the considered set $K$ is the image of a Euclidean ball under a Lipschitz mapping. 
\begin{lemma}\label{lem:tess_bias_lip}
There exist absolute constants $C,c>0$ with the following property: 
Let $\ell, m , n \in \NN$ and let $A \in \RR^{m \times \ell}$ be a random matrix with $A_{i,j} \iid \mathcal{N}(0, 1)$.
Let $b \in \RR^m$ be a random vector that is independent of $A$ and with independent entries that satisfy 
\[
\sup_{t \in \RR, \eps > 0, i \in \{1,\ldots,m\}} \eps^{-1} \cdot \PP \big(b_i \in (t-\eps, t + \eps)\big) \leq C_\tau 
\]
for a constant $C_\tau > 0$.
Let $R >0$ with $RC_\tau \geq 1$ and let $f: B_n(0,R) \to \RR^\ell$ be Lipschitz continuous, $M \geq \max\{1,\lip_{2 \to 2}(f)\}$ and $\delta \in (0, \ee^{-1})$. 
Further, we assume 
\[
m \geq C \cdot \delta^{-1} \cdot n \cdot \ln(MRC_\tau / \delta).
\]
Then the event defined by
\begin{align*}
\text{for all}\quad x,y \in f(B_n(0,R)) \quad \text{with} \quad \mnorm{x-y}_2 \leq c \cdot C_\tau^{-1} \cdot \delta \cdot \ln^{-1/2}(1/\delta): \\
\# \big\{ i \in \{1,\ldots,m\}: \ \sgn((Ax)_i + b_i) \neq \sgn((Ay)_i + b_i)\big\} \leq \delta \cdot m
\end{align*}
occurs with probability at least $1-\exp(-c \cdot \delta m)$.
 \end{lemma}
 \begin{proof}
 Let $C', c'>0$ be the constants provided by \Cref{corr:tess}.
Let $\delta \in (0, \ee^{-1})$ be arbitrary and set $\eps \defeq c' \cdot C_\tau^{-1} \cdot \delta \cdot \ln^{-1/2}(1/\delta)$. 
Moreover, let $K \defeq f(B_n(0,R))$.
The goal is now to check that the condition from \Cref{corr:tess} is satisfied, i.e., that
\begin{equation}\label{eq:to_verify}
m \geq C' \cdot 
\max\big\{\delta^{-3} w^2((K-K)\cap \B_\ell(0,\eps) \cdot C_\tau^2, \delta^{-1} \cdot \ln(\ee \cdot \mathcal{N}(K,\eps/2))\big\}.
\end{equation}
This will prove the claim of the lemma upon choosing $c \leq c'/2$.

To verify \eqref{eq:to_verify} we take an arbitrary $\theta \in (0,MR]$. 
Let $\mathcal{N} \subseteq B_n(0,R)$ be a $\frac{\theta}{M}$-net of $B_n(0,R)$ with
\begin{equation*}
\abs{\mathcal{N}} \leq \left(\frac{3MR}{\theta}\right)^n,
\end{equation*}
which exists according to \cite[Corollary~4.2.13]{vershynin_high-dimensional_2018} and noting $\theta / M \leq R$.
Then, since $\lip_{2 \to 2}(f) \leq M$, one easily sees that $f(\mathcal{N})$ is a $\theta$-net of $K$, whence
\begin{equation}\label{eq:cov_b}
\mathcal{N}(K, \theta) \leq \left(\frac{3MR}{\theta}\right)^n.
\end{equation}
Moreover, 
\begin{equation*}
\mathcal{N}\big((K-K) \cap \B_{\ell}(0,\eps), \Vert \cdot \Vert_2, \theta\big) \leq \mathcal{N}(K-K, \Vert \cdot \Vert_2, \theta / 2)
\leq \mathcal{N}(K , \Vert \cdot \Vert_2, \theta / 4)^2,
\end{equation*}
since taking the Minkowski difference of a $\frac{\theta}{4}$-net of $K$ with itself canonically yields a $\frac{\theta}{2}$-net
of $K-K$ and by \cite[Exercise~4.2.10]{vershynin_high-dimensional_2018}. Therefore, since $\diam((K-K) \cap \B_{\ell}(0,\eps))\leq 2\eps$, Dudley's inequality (cf. \cite[Theorem~8.1.10]{vershynin_high-dimensional_2018}) gives us
\begin{align*}
&\norel w\big((K-K) \cap \B_{\ell}(0,\eps)\big) \\
&\leq C_1 \cdot \int_0^{2\eps} 
\ln^{1/2}\left(\mathcal{N}((K-K) \cap \B_{\ell}(0,\eps), \Vert \cdot \Vert_2, \theta)\right) \ \dd \theta \\
&\leq \sqrt{2} \cdot C_1 \cdot \int_0^{2\eps} \ln^{1/2}(\mathcal{N}(K, \Vert \cdot \Vert_2, \theta / 4)) \ \dd \theta 
\end{align*}
with an absolute constant $C_1 > 0$.

By using the bound established in \eqref{eq:cov_b} and since\footnote{Here, the first inequality uses
that we can without loss of generality assume that $c' \leq \ee^{-1}$.}
\[
\frac{\eps}{2} \leq \ee \cdot \eps \leq C_\tau^{-1} \overset{RC_\tau \geq 1}{\leq} R \leq MR,
\]
we obtain 
\begin{align*}
    \int_0^{2\eps} \ln^{1/2}(\mathcal{N}(K, \Vert \cdot \Vert_2, \theta / 4)) \ \dd \theta
    &\leq 
    \sqrt{n} \cdot \int_0^{2\eps} \ln^{1/2}\left(12MR/\theta\right) \ \dd \theta 
    \overset{\text{\cshref{lem:int_b}}}&{\leq} 4\cdot \eps \cdot \sqrt{n} \cdot \ln^{1/2}(6MR/\eps) \\
    \overset{\text{\cshref{lem:log_b}}}&{\leq} 4\sqrt{6}  \cdot \eps \cdot \sqrt{n} \cdot \ln^{1/2}(MR/\eps).
\end{align*}
Hence, in total we obtain the bound 
\[
w^2((K-K) \cap \B_{\ell}(0,\eps)) \leq C_2 \cdot n \cdot 
\eps^2 \cdot \ln(MR/\eps)
\]
with an absolute constant $C_2 > 0$.
We compute 
\begin{align}
\ln(MR/\eps) &= \ln \left(\frac{MRC_\tau \cdot \ln^{1/2}(1/\delta)}{c' \cdot \delta}\right) 
\overset{\ln^{1/2}(1/\delta) \leq 1/\delta}{\leq}\ln\left(\frac{MR C_\tau}{c' \cdot \delta^2}\right) 
\overset{MRC_\tau / c' \geq 1}{\leq}\ln\left(\left(\frac{MR C_\tau}{c' \cdot \delta}\right)^2\right) \nonumber\\
\label{eq:boun}
&= 2 \cdot \ln\left(\frac{MR C_\tau}{c' \cdot \delta}\right)
\overset{\text{\cshref{lem:log_b}}}{\leq} 2 \cdot (c')^{-1} \cdot \ln\left(\frac{MR C_\tau}{\delta}\right).
\end{align}
This gives us 
\begin{align*}
\delta^{-3} \cdot w^2((K-K)\cap \B_\ell(0,\eps) \cdot C_\tau^2 
&\leq 
C_2  \cdot n \cdot 
\eps^2 \cdot \delta^{-3}\cdot \ln(MR/\eps) \cdot C_\tau^2 \\
&\leq  C_2 \cdot (c')^2 \cdot n \cdot \delta^2 \cdot \ln^{-1}(1/\delta) \cdot C_\tau^{-2} \cdot \delta^{-3}\cdot \ln(MR/\eps) \cdot C_\tau^2 \\
\overset{\eqref{eq:boun}}&{\leq}
2C_2 \cdot c'  \cdot n \cdot \delta^{-1} \cdot \ln^{-1}(1/\delta) \cdot \ln(MRC_\tau/\delta) \\
&\leq C_3  \cdot n \cdot \delta^{-1} \cdot \ln(MRC_\tau/\delta)
\end{align*}
with $C_3 \defeq 2 \cdot C_2 \cdot c'$.

Moreover, we consider
\begin{align*}
\delta^{-1} \cdot \ln(\ee \cdot \mathcal{N}(K, \eps/2)) & \overset{\eqref{eq:cov_b}}{\leq}
\delta^{-1} \cdot n \cdot \ln \left(\frac{6\ee MR}{\eps}\right) 
\overset{\text{\cshref{lem:log_b}}}{\leq} 
6\ee  \cdot \delta^{-1} \cdot n \cdot \ln(MR/\eps)  \\
\overset{\eqref{eq:boun}}&{\leq}
12 \cdot \ee \cdot (c')^{-1} \cdot \delta^{-1} \cdot n\cdot\ln(MRC_\tau/\delta) 
= C_4  \cdot \delta^{-1} \cdot n\cdot\ln(MRC_\tau/\delta)
\end{align*}
with $C_4 \defeq 12 \ee \cdot (c')^{-1}$.
Hence, we have 
\begin{align*}
    &\norel C' \cdot 
\max\big\{\delta^{-3} w^2((K-K)\cap \B_\ell(0,\eps) \cdot C_\tau^2, \ \delta^{-1} \cdot \ln(\ee \cdot \mathcal{N}(K, \eps/2))\big\} \\
&\leq C' \cdot \max\{C_3,C_4\} \cdot n \cdot \delta^{-1} \cdot \ln(MRC_\tau/\delta) \leq m
\end{align*}
by assumption, by picking $C \geq C' \cdot \max\{C_3, C_4\}$.
This proves the lemma. 
 \end{proof}

%% file: pw_proof.tex
\section{Proof of pointwise bounds}\label{sec:pw_proofs}
In this appendix, we provide proofs for the statements from \Cref{sec:pw}.
All the proofs rely on the techniques discussed in \Cref{sec:randgrad}, i.e., 
the fact that the $D$-matrices in the definition of the formal Jacobian and gradient can effectively be replaced
by independent random diagonal matrices $\DD$ with either $1$ or $0$ on the diagonal, 
each with probability $1/2$. 

We start with the proof of \Cref{lem:special}.
\begin{proof}[Proof of \Cref{lem:special}]
Since the claim is trivial in the case $\ell_2 < \ell_1$, we may assume $\ell_1 \leq \ell_2$ from now on. 
Thus, let $\ell_1, \ell_2 \in \{0,\ldots,L-1\}$ with $\ell_1 \leq \ell_2$ be fixed. 
Note that
\begin{equation*}
\left\Vert \left[\prod_{i = \ell_2}^{\ell_1} \D^{(i)}(z_i)W^{(i)}\right]v\right\Vert_2 \d \left\Vert \left[\prod_{i = \ell_2}^{\ell_1} \DD^{(i)}W^{(i)}\right]v \right\Vert_2
\end{equation*}
according to \Cref{prop:randgrad}. 
Regarding the matrices $\DD^{(\ell_1)}, \ldots, \DD^{(\ell_2)}$, we condition on the event  
\begin{equation}\label{eq:dbound}
\text{for all } \ell' \in \{\ell_1,\ldots,\ell_2\}: \quad \left(1 - \frac{1}{4L} \right) \cdot \frac{N}{2}\leq \Tr(\DD^{(\ell')}) \leq \left(1 + \frac{1}{4L}\right) \cdot \frac{N}{2},
\end{equation}
which occurs with probability at least $(1- 2\exp(-c_1 \cdot N/L^2))^{\ell_2 - \ell_1 + 1} \geq 
1-2L\exp(-c_1 \cdot N/L^2)$, where $c_1>0$ is an absolute constant.
Here, we used that $\Tr(\DD^{(\ell')})- \frac{N}{2} = \sum_{j=1}^N \left(\eps_j^{(\ell')} - \frac{1}{2}\right)$ and applied Bernoulli's inequality and Hoeffding's inequality for bounded random variables; see, e.g., \cite[Theorem~2.2.6]{vershynin_high-dimensional_2018}.
Since we condition on the matrices $\DD^{(\ell')}$ for $\ell' \in  \{\ell_1,\ldots, \ell_2\}$, we can from now on assume that they are fixed.
Note that then
\begin{equation}\label{eq:felix_sugg}
\left\Vert  \left[\prod_{i = \ell_2}^{\ell_1} \DD^{(i)}W^{(i)}\right]v \right\Vert_2 \d \left\Vert \left[\prod_{i = \ell_2}^{\ell_1} W_\dagger^{(i)}\right]v \right\Vert_2
\end{equation}
with $W_\dagger^{(\ell_1)} \in \RR^{\Tr(\DD^{(\ell_1)}) \times \mu}$ and 
$W_\dagger^{(i)} \in \RR^{\Tr(\DD^{(i)}) \times \Tr(\DD^{(i-1)})}$ for every $i \in \{\ell_1 + 1,\ldots,\ell_2\}$.
Here, the random matrices $W_\dagger^{(i)}$ are independent, have independent $\mathcal{N}(0,2/N)$ entries, and are independent of $v$.
We now show via induction over $\ell' \in \{\ell_1 -1 , \ldots, \ell_2\}$ that 
\begin{equation}\label{eq:indts}
 \left(1 - \frac{1}{2L}\right)^{\ell' - \ell_1 + 1} \cdot \mnorm{v}_2 \leq \left\Vert \left[\prod_{i = \ell'}^{\ell_1} W_\dagger^{(i)}\right]v \right\Vert_2 \leq \left(1 + \frac{1}{2L}\right)^{\ell' - \ell_1 + 1} \cdot \mnorm{v}_2
\end{equation}
with probability at least $(1-2\exp(-c_2 \cdot N/L^2))^{\ell' - \ell_1 + 1}$ 
with an absolute constant $c_2 > 0$ that will be specified later (in fact, we will choose 
$c_2 = c_3 /32$, where $c_3>0$ is an absolute constant appearing in 
\cite[Theorem~3.1.1]{vershynin_high-dimensional_2018} regarding the concentration of the norm of 
a Gaussian random vector). 

Note that the case $\ell' = \ell_1 - 1$ is trivial, so we can move to the induction step. 
We hence fix $\ell' \in \{\ell_1, \ldots, \ell_2\}$ and assume via induction that 
\begin{equation}\label{eq:induc}
 \left(1 - \frac{1}{2L}\right)^{\ell' - \ell_1 } \cdot \mnorm{v}_2 \leq \left\Vert \left[\prod_{i = \ell' - 1}^{\ell_1} W_\dagger^{(i)}\right]v \right\Vert_2 \leq \left(1 + \frac{1}{2L}\right)^{\ell' - \ell_1 } \cdot \mnorm{v}_2
\end{equation}
with probability at least $(1-2\exp(-c_2 \cdot N/L^2))^{\ell' - \ell_1}$.
We condition on $W_\dagger^{(\ell_1)}, \ldots, W_\dagger^{(\ell'-1)}$ and $v$ and assume that \eqref{eq:induc} is satisfied. 
For brevity, write 
\[
v' \defeq \left[\prod_{i = \ell' - 1}^{\ell_1} W_\dagger^{(i)}\right]v,
\]
which is now a fixed vector. 
Note that $W_\dagger^{(\ell')}v' \in \RR^{\Tr(\DD^{(\ell')})}$ is a random vector with independent $\mathcal{N}(0, 2\mnorm{v'}_2^2 / N)$-entries. 
Hence, \cite[Theorem~3.1.1]{vershynin_high-dimensional_2018} yields for arbitrary $t > 0$ that 
\begin{equation*}
\frac{\sqrt{2}}{\sqrt{N}}\mnorm{v'}_2 \cdot \left(\sqrt{\Tr(\DD^{(\ell')})} - t\right) \leq \mnorm{W_\dagger^{(\ell')}v'}_2 \!\leq \! 
\frac{\sqrt{2}}{\sqrt{N}}\mnorm{v'}_2 \cdot \left(\sqrt{\Tr(\DD^{(\ell')})} + t\right)
\end{equation*}
with probability at least $1-2\exp(- c_3 \cdot t^2)$, with an absolute constant $c_3 > 0$.
We set $t \defeq \frac{\sqrt{N}}{4\sqrt{2}\cdot L}$ and obtain, using \eqref{eq:dbound},
\begin{align*}
\frac{\sqrt{2}}{\sqrt{N}} \cdot \left(\sqrt{\Tr(\DD^{(\ell')})} + t\right)
&\leq \frac{\sqrt{2}}{\sqrt{N}} \cdot \left(\frac{\sqrt{N}}{\sqrt{2}}\cdot\sqrt{1 + \frac{1}{4L}} + \frac{\sqrt{N}}{4\sqrt{2} \cdot L}\right)  \\
\overset{\sqrt{x} \leq x \text{ for } x \geq 1}&{\leq}  1 + \frac{1}{2L}
\end{align*}
and analogously 
\[
\frac{\sqrt{2}}{\sqrt{N}} \cdot \left(\sqrt{\Tr(\DD^{(\ell')})} - t\right)
\geq 1- \frac{1}{2L}.
\]
Hence, we observe, using \eqref{eq:induc}, Bernoulli's inequality and the well-known
estimate $\left(1 + \frac{1}{k}\right)^k \leq \ee$,
\[
\left(1 - \frac{1}{2L}\right)^{\ell' - \ell_1 + 1} \cdot \mnorm{v}_2 \leq \left\Vert \left[\prod_{i = \ell'}^{\ell_1} W_\dagger^{(i)}\right]v \right\Vert_2 = \mnorm{W_\dagger^{(\ell')}v'}_2\leq \left(1 + \frac{1}{2L}\right)^{\ell' - \ell_1 + 1} \cdot \mnorm{v}_2
\]
with probability at least $1-2\exp(-c_2 \cdot N/L^2)$, where $c_2 \defeq c_3 / 32$ and the randomness is solely with respect to $W_\dagger^{(\ell')}$.
Reintroducing the randomness over $W_\dagger^{(\ell_1)}, \ldots, W_\dagger^{(\ell'-1)}$ and $v$,
we obtain \eqref{eq:indts}. 

In the end, picking $\ell' = \ell_2$ gives us
\begin{equation*}
\frac{1}{2} \cdot \mnorm{v}_2 \leq\left(1 - \frac{1}{2L}\right)^{\ell_2 - \ell_1 + 1} \cdot \mnorm{v}_2\leq 
\left\Vert \left[\prod_{i = \ell_2}^{\ell_1} W_\dagger^{(i)}\right]v \right\Vert_2 \leq \left(1 + \frac{1}{2L}\right)^{\ell_2 - \ell_1 + 1} \cdot \mnorm{v}_2
 \leq \ee \cdot \mnorm{v}_2
\end{equation*}
with probability at least 
\[
(1-2\exp(-c_2 \cdot N/L^2))^{\ell_2 -  \ell_1 + 1} 
\geq (1-2\exp(- c_2 \cdot N/L^2))^{L} \geq 1 - 2\exp\big(\ln(L)- c_2 \cdot  N/L^2\big),
\]
where we again used Bernoulli's inequality. 

Keeping in mind that we conditioned on the $\DD$-matrices (see \eqref{eq:dbound} and recalling \eqref{eq:felix_sugg}), we get
\begin{equation*}
\frac{1}{2} \cdot \Vert v \Vert_2 \leq \mnorm{\left[\prod_{i = \ell_2}^{\ell_1} \D^{(i)}(z_i)W^{(i)}\right]v}_2 \leq \ee \cdot  \Vert v \Vert_2
\end{equation*}
with probability at least $1-4\exp(\ln(L)-c_4 \cdot N/L^2)$, where $c_4 \defeq \min\{c_1, c_2, \ln(2)\}$.
Finally, we use \Cref{prop:randnorm} combined with a union bound over $z_0, \dots, z_{L-1}$ to conclude 
\[
\frac{1}{2} \cdot \Vert v \Vert_2 \leq \mnorm{\left[\prod_{i = \ell_2}^{\ell_1} D^{(i)}(z_i)W^{(i)}\right]v}_2 \leq \ee \cdot  \Vert v \Vert_2
\]
with probability at least 
\begin{align*}
1-4\exp\big(\ln(L)-c_4 \cdot N/L^2\big) - \frac{L^2}{2^N} &= 1-4\exp\big(\ln(L)-c_4 \cdot N/L^2\big) - \exp\big(\ln(L^2) - \ln(2)\cdot N\big)
\\ &\geq 1 - 5 \exp\big(2\ln(L)- c_4 \cdot N/L^2\big) \\
&= 1 - \exp\big(\ln(5) + 2\ln(L) -c_4 \cdot N/L^2\big),
\end{align*}
where we used $c_4 \leq \ln(2)$.

Since by assumption $N \geq C \cdot L^2 \cdot \ln(\ee \cdot L)$, we see by taking $C > 0$ large 
enough that
\[
\ln(5) + 2\ln(L) \leq 3 \cdot \ln(5L) \overset{\text{\cshref{lem:log_b}}}{\leq}
\frac{15}{\ee} \cdot \ln(\ee \cdot L) \leq \frac{c_4}{2} \cdot \frac{N}{L^2},
\]
so the claim follows letting $c \defeq c_4 /2$.
\end{proof}

We continue by proving \Cref{lem:pw_2}.

\begin{proof}[Proof of \Cref{lem:pw_2}]
We may assume $\ell_1 \leq \ell_2$, since in the case $\ell_1 > \ell_2$ we have 
\[
\op{\left[ \prod_{i = \ell_2}^{\ell_1} D^{(i)}(z_i) W^{(i)}\right] A} = \op{A}.
\]

Firstly, note that according to \Cref{prop:randgrad},
\begin{equation*}
\op{\left[ \prod_{i = \ell_2}^{\ell_1} \D^{(i)}(z_i) W^{(i)}\right] A} \d
\op{\left[ \prod_{i = \ell_2}^{\ell_1} \DD^{(i)} W^{(i)}\right] A}.
\end{equation*}
We condition on the $\DD$-matrices and assume 
\begin{equation*}
\text{for all } i \in \{\ell_1, \ldots, \ell_2\}: \quad \Tr (\DD^{(i)}) \leq \frac{N}{2}\left(1 + \frac{1}{L}\right),
\end{equation*}
which holds with probability at least $(1-\exp(-c_1 \cdot N/L^2))^{\ell_2 - \ell_1 + 1} \geq 1 - L \cdot \exp(-c_1 \cdot N/L^2)$
for some absolute constant $c_1>0$, which follows from Bernoulli's inequality and
 Hoeffding's inequality for bounded random variables \cite[Theorem~2.2.6]{vershynin_high-dimensional_2018}
after noting that $\Tr(\DD^{(i)})- \frac{N}{2} = \sum_{j=1}^N \left(\eps_j^{(i)} - \frac{1}{2}\right)$.
Fixing the $\DD$-matrices, we observe
\begin{equation}\label{eq:felix_sugg_2}
\op{\left[ \prod_{i = \ell_2}^{\ell_1} \DD^{(i)} W^{(i)}\right] A} \d 
\op{\left[\prod_{i = \ell_2}^{\ell_1} W_\dagger^{(i)}\right]A},
\end{equation}
where $W_\dagger^{(i)} \in \RR^{\Tr (\DD^{(i)}) \times \Tr (\DD^{(i-1)})}$ for $i \in \{\ell_1 + 1,\ldots, \ell_2\}$
and $W_\dagger^{(\ell_1)} \in \RR^{\Tr (\DD^{(\ell_1)}) \times \mu}$.
Here, the random matrices $W_\dagger^{(i)}$ are independent and have i.i.d. $\mathcal{N}(0,2/N)$ entries. 
Hence, the goal is to bound
\begin{equation*}
\op{\left[\prod_{i = \ell_2}^{\ell_1} W_\dagger^{(i)}\right]A}.
\end{equation*}
To this end, we show via induction over $\ell' \in \{\ell_1 -1,\ldots, \ell_2\}$ that 
\begin{equation}\label{eq:showy}
\op{\left[\prod_{i = \ell'}^{\ell_1} W_\dagger^{(i)}\right]A} \leq \left(1 + \frac{5}{L}\right)^{\ell' - \ell_1 + 1} \cdot \op{A}
\end{equation}
with probability at least $(1-\exp(-c_2 \cdot N/L^2))^{\ell' - \ell_1 + 1}$ for every 
$\ell' \in \{\ell_1 -1,\ldots, \ell_2\}$, where $c_2 > 0$ is the absolute constant occurring in the formulation of \Cref{lem:pw_4}.

Note that the case $\ell' = \ell_1-1$ is trivial. 
We hence choose $\ell' \in \{\ell_1, \ldots, \ell_2\}$ and assume via induction that 
\begin{equation}\label{eq:inde}
\op{\left[\prod_{i = \ell'-1}^{\ell_1} W_\dagger^{(i)}\right]A} \leq \left(1 + \frac{5}{L}\right)^{\ell' - \ell_1 } \cdot \op{A}
\end{equation}
with probability at least $(1-\exp(-c_2 \cdot N/L^2))^{\ell' - \ell_1}$. 
We condition on $W_\dagger^{(\ell_1)}, \ldots, W_\dagger^{(\ell'-1)}$ and 
assume that \eqref{eq:inde} is satisfied. 
We then set $V' \defeq \left[\prod_{i = \ell'-1}^{\ell_1} W_\dagger^{(i)}\right]A$, which is now a fixed matrix. 
By \Cref{lem:pw_4}, there exists an absolute constant $C_2 > 0$ such that for any $t \geq C_2$,
\[
\op{W_\dagger^{(\ell')} V'} \leq \sqrt{2/N} \cdot \op{V'} \cdot \left(\sqrt{\Tr(\DD^{(\ell')})} + \sqrt{\rang(A)} + t\right)
\]
with probability at least $1- \exp(-c_2 \cdot t^2)$.
We set $t \defeq \sqrt{N}/L$, noting that $t \geq \sqrt{C} \geq C_2$ for $C$ large enough.
Using our assumptions, we then get
\begin{align*}
    \sqrt{\Tr(\DD^{(\ell')})} \! + \! \sqrt{\rang(A)} \! + \! t
    \leq \sqrt{N/2} \cdot \sqrt{1 + \frac{1}{L}} + \frac{2\sqrt{N}}{L} \! \leq\!  \sqrt{N/2} \cdot \left(1 + \frac{1 + 2\sqrt{2}}{L}\right) \leq \sqrt{\frac{N}{2}} \cdot \left(1 + \frac{5}{L}\right).
\end{align*}
Overall, we thus get 
\[
\op{\left[\prod_{i = \ell'}^{\ell_1} W_\dagger^{(i)}\right]A} \leq \left(1 + \frac{5}{L}\right)^{\ell' - \ell_1 + 1} \cdot \op{A}
\]
with probability at least $1- \exp(- c_2 \cdot N/L^2)$. 
We reintroduce the randomness over $W_\dagger^{(\ell_1)}, \ldots, W_\dagger^{(\ell'-1)}$ and $A$ and 
thus obtain \eqref{eq:showy}. 
Via induction, \eqref{eq:showy} is then shown for every $\ell' \in \{\ell_1 - 1, \ldots, \ell_2\}$. 

In particular, letting $\ell' = \ell_2$ and recalling \eqref{eq:felix_sugg_2}, we get 
\[
\op{\left[ \prod_{i = \ell_2}^{\ell_1} \DD^{(i)} W^{(i)}\right] A}
\leq \left( 1 + \frac{5}{L}\right)^{\ell_2 - \ell_1 + 1} \cdot \op{A} \leq \ee^5 \cdot \op{A}
\]
with probability at least 
\[
(1-\exp(-c_2 \cdot N/L^2))^{\ell_2 - \ell_1 + 1} 
\geq 1-L\exp(-c_2 \cdot N/L^2).
\]
Recall that we have conditioned on the matrices $\DD^{(i)}$ so far. 
Lifting the conditioning on these matrices, we get 
\[
\op{\left[ \prod_{i = \ell_2}^{\ell_1} \D^{(i)} W^{(i)}\right] A} \leq \ee^5 \cdot \op{A}
\]
with probability at least $1-2L \cdot \exp(-c_3 \cdot N/L^2)$ with $c_3 \defeq \min\{c_1,c_2, \ln(2)\}$.

In the end, we use \Cref{prop:randnorm} and a union bound over $z_0, \dots, z_{L-1}$ to argue that
with probability at least 
\[
 1 - \frac{L^2}{2^N} = 1-L^2\exp(-\ln(2)\cdot N) \overset{c_3 \leq \ln(2)}{\geq} 
 1-L^2\exp(-c_3\cdot N/L^2)
\]
 we have
$D^{(i)}(z_i) = \D^{(i)}(z_i)$ for every $i \in \{0,\ldots,L-1\}$, which gives us
\begin{equation*}
\op{\left[ \prod_{i = \ell_2}^{\ell_1} D^{(i)}(z_i) W^{(i)}\right] A} \leq \ee^5 \cdot \op{A}
\end{equation*}
with probability at least $1-3L^2\exp(-c_3 \cdot N/L^2) = 1-\exp(\ln(3L^2) -c_3 \cdot N/L^2)$.
Since $N \geq C \cdot L^2 \cdot \ln(\ee L)$, we have for $C > 0$ large enough that 
\[
\ln(3L^2) = 2 \cdot \ln(3L)\overset{\text{\cshref{lem:log_b}}}{\leq} \frac{6}{\ee} \cdot \ln(\ee L) \leq \frac{3}{C} \cdot \frac{N}{L^2} \leq \frac{c_3}{2} \cdot \frac{N}{L^2},
\]
so the claim follows by taking $c \defeq c_3 /2$.
\end{proof}

\Cref{lem:pw_3} is now an immediate consequence of \Cref{lem:pw_2}.

\begin{proof}[Proof of \Cref{lem:pw_3}]
We write 
\begin{equation*}
V' \defeq \left[\prod_{i = L - 1}^{\ell} D^{(i)}(z_i) W^{(i)}\right]A \in \RR^{N \times \nu},
\end{equation*}
noting that $V'$ only depends on $W^{(0)}, \ldots, W^{(L-1)}, b^{(0)}, \ldots, b^{(L-1)}$
and that $\op{V'} \leq C_1 \cdot \op{A}$ 
with probability at least $1-\exp(-c_1 \cdot N/L^2)$ 
with absolute constants $C_1, c_1 > 0$ according to \Cref{lem:pw_2}.
In the following, we 
condition on the matrices $W^{(0)},\ldots,W^{(L-1)}$ and the biases $b^{(0)}, \hdots, b^{(L-1)}$ and
assume that $\op{V'} \leq C_1 \cdot \op{A}$ is satisfied.
Note that $W^{(L)}$ is independent of $W^{(0)},\ldots, W^{(L-1)}$ and $b^{(0)},\ldots, b^{(L-1)}$.

Note that we may assume $A \neq 0$ without loss of generality. 
By \Cref{lem:pw_4} and recalling that $W^{(L)} \in \RR^{1 \times N}$, we obtain absolute constants $C_2, c_2 > 0$ such that for every $t \geq C_2$, 
\begin{align*}
\op{W^{(L)}V'} &\leq \op{V'} \cdot \left(1 + \sqrt{\rang(V')} + t\right)
\leq C_1 \cdot \op{A} \cdot \left(1 + \sqrt{\rang(A)} + t\right) \\
&\leq 2C_1 \cdot \op{A} \cdot \left(\sqrt{\rang(A)} + t\right)
\end{align*}
with probability at least $1- \exp(-c_2 \cdot t^2)$.

Up to now, we have conditioned on $W^{(0)},\ldots,W^{(L-1)}$ and $b^{(0)}, \hdots, b^{(L-1)}$.
Reintroducing the randomness over these random variables yields 
\[
\op{W^{(L)}V'} \leq 2C_1 \cdot \op{A} \cdot \left(\sqrt{\rang(A)} + t\right)
\]
with probability at least 
\[
1-\exp(-c_1 \cdot N/L^2) - \exp(-c_2 \cdot t^2) \geq 1 - 2\exp(-c_3 \cdot t^2) \geq 1 - \exp(-c_4 \cdot t^2)
\]
with $c_3 \defeq \min\{c_1,c_2\}$ and $c_4 \defeq c_3/2$, using that $C \leq t \leq \sqrt{N}/L$ by assumption. 
\end{proof}

Next, we prove \Cref{thm:pw}.
\begin{proof}[Proof of \Cref{thm:pw}]
According to \Cref{prop:randgrad}, 
\begin{align*}
\mnorm{W^{(L)} \prod_{i= L-1}^0 \D^{(i)}(x_0)W^{(i)}}_{p} &\d 
\mnorm{W^{(L)} \prod_{i= L-1}^0 \DD^{(i)}W^{(i)}}_{p}.
\end{align*}
We now condition on the $\DD$-matrices and assume that
\[
\frac{N}{2} \cdot \left(1 - \frac{1}{4L}\right) \leq \Tr (\DD^{(i)}) \leq \frac{N}{2} \cdot \left(1 + \frac{1}{4L}\right) \quad \text{for all } i \in \{0,\ldots,L-1\}.
\]
Due to Hoeffding's inequality for bounded random variables \cite[Theorem~2.2.6]{vershynin_high-dimensional_2018}
and Bernoulli's inequality, this occurs with probability at least $(1 - 2\exp(- c_1 \cdot N/L^2))^L \geq 1-2L\exp(-c_1\cdot N/L^2)$
with an absolute constant $c_1>0$.
For fixed $\DD$-matrices,
\[
\mnorm{W^{(L)}\left[\prod_{i= L-1}^{0} \DD^{(i)}W^{(i)}\right]}_{p} \d \mnorm{\prod_{i= L}^{0} W_\dagger^{(i)}}_{p},
\]
with $W_\dagger^{(0)} \in \RR^{ \Tr(\DD^{(0)}) \times d}$, $W_\dagger^{(i)} \in \RR^{\Tr(\DD^{(i)}) \times \Tr(\DD^{(i-1)})}$ for $i \in \{1,\ldots,L-1\}$
and $W_\dagger^{(L)} \in \RR^{1 \times \Tr(\DD^{(L-1)})}$.
Here, the matrices $W_\dagger^{(i)}$ are independent with entries satisfying $(W_\dagger^{(i)})_{j,k} \iid \mathcal{N}(0,2/N)$ for every $i \in \{0,\ldots,L-1\}$ and $(W_\dagger^{(L)})_{j} \iid \mathcal{N}(0,1)$.
We first show via induction over $\ell' \in \{0,\ldots,L-1\}$ that 
\[
\sqrt{\frac{N}{2}} \cdot \left(1 - \frac{1}{2L}\right)^{\ell'+1}\leq \mnorm{\prod_{i= L}^{L-\ell'} W_\dagger^{(i)}}_2 \leq \sqrt{\frac{N}{2}} \cdot \left(1 + \frac{1}{2L}\right)^{\ell'+1}
\]
with probability at least $(1-2\exp(-c_2 \cdot N /L^2))^{\ell'+1}$ with an absolute constant $c_2 > 0$.
To this end, we begin with the case $\ell' = 0$.
From \cite[Theorem~3.1.1]{vershynin_high-dimensional_2018} we obtain
\[
\sqrt{\Tr(\DD^{(L-1)})} - \frac{\sqrt{N}}{4\sqrt{2}L} \leq \mnorm{W^{(L)}_\dagger}_2 \leq \sqrt{\Tr(\DD^{(L-1)})} + \frac{\sqrt{N}}{4\sqrt{2}L}
\]
with probability at least $1- 2\exp( -c_2 \cdot N/L^2)$.
From our assumption on $\DD^{(L-1)}$ we can further bound
\[
 \sqrt{\Tr(\DD^{(L-1)})} + \frac{\sqrt{N}}{4\sqrt{2}L} \leq  \sqrt{\frac{N}{2}\left(1 + \frac{1}{4L}\right)} + \frac{\sqrt{N}}{4\sqrt{2}L}
 \leq \sqrt{\frac{N}{2}} \left(1 + \frac{1}{4L} + \frac{1}{4L}\right) = \sqrt{\frac{N}{2}} \left(1 + \frac{1}{2L}\right)
\]
and similarly
\[
\sqrt{\Tr(\DD^{(L-1)})} - \frac{\sqrt{N}}{4\sqrt{2}L} \geq \sqrt{\frac{N}{2}}\left(1 - \frac{1}{2L}\right).
\]
Hence, the case $\ell'=0$ is shown. 

We now let $\ell' \in \{1,\ldots, L-1\}$ and assume that the claim is satisfied by $\ell'-1$.
We then condition on the matrices $W^{(L-\ell'+1)}_\dagger, \ldots, W^{(L)}_\dagger$ 
and let $Z \defeq \prod_{i=L}^{L-\ell'+1} W_\dagger^{(i)} \in \RR^{1 \times \Tr(\DD^{(L-1)})}$.
From the rotation invariance of the Gaussian distribution we see
\[
(Z W^{(L-\ell')}_\dagger)^T  \sim \mathcal{N}\left(0, \frac{2}{N}\cdot \mnorm{Z}_2^2 \cdot I_{\Tr(\DD^{(L-\ell'-1)})}\right).
\]
We apply \cite[Theorem~3.1.1]{vershynin_high-dimensional_2018} and get
\[
\sqrt{\frac{2}{N}} \cdot \mnorm{Z}_2 \cdot \left(\sqrt{\Tr(\DD^{(L-\ell'-1)})} - \frac{\sqrt{N}}{4\sqrt{2}L}\right)
\leq \mnorm{ ZW^{(L-\ell')}_\dagger}_2 
\leq \sqrt{\frac{2}{N}} \cdot \mnorm{Z}_2 \cdot \left(\sqrt{\Tr(\DD^{(L-\ell'-1)})} + \frac{\sqrt{N}}{4\sqrt{2}L}\right)
\]
with probability at least $1-2\exp(-c_2 \cdot N/L^2)$.
Identically to above, we bound
\[
\sqrt{\Tr(\DD^{(L-\ell'-1)})} + \frac{\sqrt{N}}{4\sqrt{2}L} \leq \sqrt{\frac{N}{2}} \left(1 + \frac{1}{2L}\right)
\]
and
\[
\sqrt{\Tr(\DD^{(L-\ell'-1)})} - \frac{\sqrt{N}}{4\sqrt{2}L} \geq \sqrt{\frac{N}{2}}\left(1 - \frac{1}{2L}\right).
\]
We now inductively employ the bound on $\mnorm{Z}_2$ and get the desired result using the independence of the random matrices $W^{(i)}_\dagger$.

We have shown that 
\[
\frac{1}{2} \cdot \sqrt{N/2} \leq \sqrt{\frac{N}{2}} \cdot \left(1 - \frac{1}{2L}\right)^{L}\leq \mnorm{\prod_{i= L}^{1} W_\dagger^{(i)}}_2 
\leq \sqrt{\frac{N}{2}} \cdot \left(1 + \frac{1}{2L}\right)^{L}\leq \sqrt{N/2} \cdot \ee
\]
with probability at least $(1-2\exp(-c_2 \cdot  N /L^2))^{L} \geq 1 - 2L\exp(-c_2 \cdot N/L^2)$. 
For the final step, we condition on the matrices $W_\dagger^{(1)}, \ldots,W_\dagger^{(L)}$, 
let $Z \defeq \prod_{i= L}^{1} W_\dagger^{(i)} \in \RR^{1 \times \Tr(\DD^{(0)})}$ and assume 
\[
\frac{1}{2} \cdot \sqrt{\frac{N}{2}} \leq \mnorm{Z}_2 
\leq \sqrt{\frac{N}{2}} \cdot \ee.
\]
Since $(ZW^{(0)}_\dagger)^T \sim \mathcal{N}\left(0, \frac{2}{N} \cdot \mnorm{Z}_2^2 \cdot I_d\right)$, \Cref{thm:p_high_prob} provides absolute
constants $C', c'>0$ such that:
\begin{enumerate}
{\item If $p \in [1,2]$ and $d \geq C'$, we have
\begin{equation*}
\PP \left(c' \cdot \frac{\sqrt{2} \cdot \mnorm{Z}_2}{\sqrt{N}} \cdot d^{1/p} \leq \mnorm{Z W^{(0)}_\dagger }_p 
\leq C' \cdot \frac{\sqrt{2} \cdot \mnorm{Z}_2}{\sqrt{N}}\cdot d^{1/p}\right) \geq 1- C'\exp(-c' \cdot d).
\end{equation*}}
\item{
For $p \in (2,c' \cdot \ln(d))$ and $d \geq C'$ we have 
\begin{align*}
&\norel\PP \left(c' \cdot \frac{\sqrt{2} \cdot \mnorm{Z}_2}{\sqrt{N}} \cdot \sqrt{p} \cdot d^{1/p}
 \leq \mnorm{ZW^{(0)}_\dagger }_p \leq C' \cdot \frac{\sqrt{2} \cdot \mnorm{Z}_2}{\sqrt{N}} \cdot \sqrt{p} \cdot d^{1/p}\right) \\
&\geq 1-C'\exp(-c' \cdot p \cdot d^{2/p}).
\end{align*}
}
\item{
If $p \geq \ln(d)$ and $d \geq C'$, we have
\begin{align*}
&\norel\PP \left(c' \cdot \frac{\sqrt{2} \cdot \mnorm{Z}_2}{\sqrt{N}} \cdot \sqrt{\ln (d)} \leq \mnorm{ZW^{(0)}_\dagger }_{p} 
\leq C' \cdot \frac{\sqrt{2} \cdot \mnorm{Z}_2}{\sqrt{N}} \cdot \sqrt{\ln(d)}\right) \\
&\geq 1-C'\exp(-c' \cdot \ln(d)).
\end{align*}
}
\end{enumerate}
Up to now we conditioned on $W_\dagger^{(1)}, \ldots,W_\dagger^{(L)}$
and on $\DD^{(0)}, \ldots, \DD^{(L-1)}$.
Reintroducing the randomness over $W_\dagger^{(1)}, \ldots,W_\dagger^{(L)}$, we find
\begin{enumerate}
{\item If $p \in [1,2]$ and $d \geq C'$, we have
\begin{equation*}
\PP \left(\frac{c'}{2} \cdot d^{1/p} \leq \mnorm{\prod_{i= L}^{0} W_\dagger^{(i)}}_{p}
\leq C' \cdot \ee \cdot d^{1/p}\right) \geq 1- C'\exp(-c' \cdot d)- 2L\exp(-c_2 \cdot N/L^2).
\end{equation*}}
\item{
For $p \in (2,c' \cdot \ln(d))$ and $d \geq C'$ we have 
\begin{align*}
&\norel\PP \left(\frac{c'}{2}\cdot \sqrt{p} \cdot d^{ 1/p}
 \leq \mnorm{\prod_{i= L}^{0} W_\dagger^{(i)}}_{p}\leq C' \cdot \ee \cdot \sqrt{p} \cdot d^{1/p}\right) \\
&\geq 1-C'\exp(-c' \cdot p \cdot d^{2/p})- 2L\exp(-c_2 \cdot N/L^2).
\end{align*}
}
\item{
If $p \geq \ln(d)$ and $d \geq C'$, we have
\begin{align*}
&\norel\PP \left(\frac{c'}{2} \cdot \sqrt{\ln (d)} \leq \mnorm{\prod_{i= L}^{0} W_\dagger^{(i)}}_{p} 
\leq C' \cdot \ee \cdot \sqrt{\ln(d)}\right) \\
&\geq 1-C'\exp(-c' \cdot \ln(d))- 2L\exp(-c_2 \cdot N/L^2).
\end{align*}
}
\end{enumerate}
Finally, we reintroduce the randomness over the $\DD$-matrices and apply \Cref{prop:randnorm}, which yields the following with $c_3 \defeq \min\{c_1, c_2,\ln(2)\}$:
\begin{enumerate}
{\item If $p \in [1,2]$ and $d \geq C'$, we have
\begin{align*}
&\norel \PP \left(\frac{c'}{2} \cdot d^{1/p} \leq \mnorm{\nablaa \Phi (x_0)}_p 
\leq C' \cdot \ee \cdot d^{1/p}\right) \\
&\geq 1- C'\exp(-c' \cdot d)- 4L\exp(-c_3 \cdot N/L^2)- \frac{L}{2^N}.
\end{align*}}
\item{
For $p \in (2,c' \cdot \ln(d))$ and $d \geq C'$ we have 
\begin{align*}
&\norel \PP \left(\frac{c'}{2}\cdot \sqrt{p} \cdot d^{ 1/p} \leq \mnorm{\nablaa \Phi (x_0)}_p\leq C' \cdot \ee \cdot \sqrt{p} \cdot d^{1/p}\right)  \\
&\geq 1-C'\exp(-c' \cdot p \cdot d^{2/p})- 4L\exp(-c_3 \cdot N/L^2) - \frac{L}{2^N}.
\end{align*}
}
\item{
If $p \geq \ln(d)$ and $d \geq C'$, we have
\begin{align*}
&\norel\PP \left(\frac{c'}{2} \cdot \sqrt{\ln (d)} \leq \mnorm{\nablaa \Phi (x_0)}_p
\leq C' \cdot \ee \cdot \sqrt{\ln(d)}\right) \\
&\geq 1-C'\exp(-c' \cdot \ln(d))- 4L\exp(-c_3 \cdot N/L^2) - \frac{L}{2^N}.
\end{align*}
}
\end{enumerate}
Since $c_3 \leq \ln(2)$, 
we get
\begin{align*}
4L\exp(-c_3 \cdot N/L^2) + \frac{L}{2^N}
&= \exp(\ln(4L) - c_3 \cdot N/L^2) + \exp(\underbrace{\ln(L)}_{\leq \ln(4L)} - \underbrace{\ln(2) \cdot N}_{\geq c_3 \cdot N/L^2})\\
&\leq \exp(\underbrace{\ln(2)}_{\leq \ln(4L)} + \ln(4L) - c_3 \cdot N/L^2) \\
&\leq \exp(2\ln(4L) - c_3 \cdot N/L^2) \\
&\leq \exp(-c_4 \cdot N/L^2)
\end{align*}
with $c_4 \defeq c_3 /2$ and since $N \geq C \cdot L^2\ln(\ee L)$, which ensures for $C$ large enough that
\[
2\ln(4L) \overset{\text{\cshref{lem:log_b}}}{\leq} \frac{8}{\ee} \cdot \ln(\ee L) \leq \frac{c_3}{2} \cdot \frac{N}{L^2}.
\]

To conclude the proof, we again distinguish the three cases.
\begin{enumerate}
\item \label{item:case1}Let us first consider the case $p \in [1,2]$.
Note that $d \geq C$ and hence $\ln(C') \leq \frac{c'}{2} \cdot d$, for $C$ large enough.
We then get, with $c_5 \defeq c'/2$,
\begin{align*}
    1- C'\exp(-c' \cdot d)- 4L\exp(-c_3 \cdot N/L^2)- \frac{L}{2^N}
    &\geq 1-C' \exp(-c' \cdot d) - \exp(-c_4 \cdot N/L^2) \\
    &\geq 1- \exp(-c_5 \cdot d) - \exp(-c_4 \cdot N/L^2) \\
    &\geq 1-\exp(-c_6 \cdot \min\{d, N/L^2\})
\end{align*}
with $c_6 \defeq \min\{c_4,c_5\}/2$ and using that $\min\{d, N/L^2\} \geq C$ by the 
assumptions of \Cref{thm:pw}. 

\item We now consider the case $p \in (2,c' \cdot \ln(d))$.
Using the well-known estimate $\exp(2t) \geq 1+2t \geq t$ for $t > 0$, 
we find 
\[
\exp\left(\frac{2}{p}\cdot \ln(d)\right) \geq \frac{\ln(d)}{p} \quad \text{for all }d \in \NN,\ p \in [1,\infty),
\]
which is equivalent to 
\begin{equation}\label{eq:thanksto}
p \cdot d^{2/p} \geq \ln(d) \quad \text{for all }d \in \NN,\ p \in [1,\infty).
\end{equation}
We choose $C$ large enough, such that for $d \geq C$, 
\[
\ln(C') \leq  \frac{c'}{2} \cdot \ln(d) \leq \frac{c'}{2} \cdot p \cdot d^{2/p}.
\]
We then get with $c_5 = c'/2$ and $c_6 = \min\{c_4, c_6\}/2$ as before
\begin{align*}
 &\norel 1- C' \exp(-c' \cdot  p \cdot d^{2/p}) - 4L\exp(-c_3 \cdot N/L^2) - \frac{L}{2^N} \\
 &\geq  1- C' \exp(- c' \cdot p \cdot d^{2/p}) - \exp \left(- c_4 \cdot N/L^2\right) \\
  &\geq  1-  \exp(- c_5 \cdot p \cdot d^{2/p}) - \exp (- c_4 \cdot N/L^2) \\
 &\geq 1-\exp(-c_6 \cdot \min\{p \cdot d^{2/p}, N/L^2\}),
\end{align*}
using that 
\[
\min\{p \cdot d^{2/p}, N/L^2\} \geq \min\{\ln(d), C\} \geq \min\{\ln(C),C\},
\]
with $C$ large.

\item Lastly, if $p \geq \ln(d)$ we note as in Case \eqref{item:case1} that we may assume $\ln(C') \leq \frac{c'}{2} \cdot \ln(d)$.
This gives us
\begin{align*}
 &\norel 1- C' \exp(-c' \cdot  \ln(d)) - 4L\exp(-c_3 \cdot N/L^2) - \frac{L}{2^N}  \\
 &\geq  1- C' \exp(- c' \ln(d)) - \exp \left(- c_4 \cdot N/L^2\right) \\
  &\geq  1-  \exp(- c_5 \cdot \ln(d)) - \exp (- c_4 \cdot N/L^2) \\
 &\geq 1-\exp(-c_6 \cdot \min\{\ln(d), N/L^2\}),
\end{align*}
where in the last step we again used that $\min\{\ln(d), N/L^2\}\geq \ln(C)$ by assumption.
\end{enumerate}
Overall, the claim is shown. 
\end{proof}

We conclude this appendix with a proof of \Cref{thm:p_high_prob}, 
which combines two results from \cite{PAOURIS20173187}.
\begin{proof}[Proof of \Cref{thm:p_high_prob}]
According to \cite[Proposition~2.4]{PAOURIS20173187}, there exist absolute constants $C_1, c_1>0$, such that for every $d \in \NN$ and $p \in [1,\infty)$ with $p < \ln(d)$,
\[
c_1 \cdot \sqrt{p} \cdot d^{1/p} \leq \underset{X \sim \mathcal{N}(0,I_d)}{\EE} \left[\mnorm{X}_p\right] \leq C_1 \cdot \sqrt{p} \cdot d^{1/p},
\]
whereas for $d \in \NN$ and $p \in [1, \infty]$ with $p \geq \ln(d)$, one has
\[
c_1 \cdot \sqrt{\ln(d)} \leq \underset{X \sim \mathcal{N}(0,I_d)}{\EE} \left[\mnorm{X}_p\right] \leq C_1 \cdot \sqrt{\ln(d)}.
\]

An application of \cite[Theorem~4.11]{PAOURIS20173187} (with $\eps = 1/2$) yields the existence of absolute constants $C_2,c_2,c_3 > 0$ such that for every $d \geq C_2$ and 
$p \in [1,\infty]$ one has
\[
\underset{X \sim \mathcal{N}(0,I_d)}{\PP} \left(\frac{1}{2} \cdot \EE \left[\mnorm{X}_p\right] \leq \mnorm{X}_p \leq \frac{3}{2} \cdot \EE \left[\mnorm{X}_p\right]\right)
\geq 1 - C_2 \exp(-c_2 \cdot \beta(d,p,1/2)),
\]
where 
\[
\beta(d,p,1/2) \defeq \begin{cases}
\frac{d}{4} & \text{if } 1 \leq p \leq 2, \\
\max \left\{\min\left\{p^2\cdot2^{-p}\cdot \frac{d}{4}, \left(\frac{d}{2}\right)^{2/p}\right\},\frac{1}{2} \cdot p \cdot d^{2/p}\right\} & \text{if } 2 < p \leq c_3 \cdot \ln(d), \\
\ln(d) \cdot \frac{1}{2} & \text{if } p > c_3 \cdot \ln(d).
\end{cases}
\]

If $p \in [1,2]$ and $d \geq \max\{C_2,8\}$, we get $\ln(d) \geq \ln(8) >2 \geq p$ and hence 
\[
c_1 \cdot d^{1/p}\leq c_1 \cdot \sqrt{p} \cdot d^{1/p} \leq \underset{X \sim \mathcal{N}(0,I_d)}{\EE} \left[\mnorm{X}_p\right] 
\leq C_1 \cdot \sqrt{p} \cdot d^{1/p} \leq C_1 \cdot \sqrt{2} \cdot d^{1/p}.
\]
We thus get
\begin{align*}
\underset{X \sim \mathcal{N}(0,I_d)}{\PP} \left(\frac{c_1}{2} \cdot d^{1/p} \leq \mnorm{X}_p \leq \frac{3}{\sqrt{2}} \cdot C_1 \cdot d^{1/p}\right)
&\geq \underset{X \sim \mathcal{N}(0,I_d)}{\PP} \left(\frac{1}{2} \cdot \EE \left[\mnorm{X}_p\right] \leq \mnorm{X}_p \leq \frac{3}{2} \cdot \EE \left[\mnorm{X}_p\right]\right) \\
&\geq 1-C_2\exp\left(-\frac{c_2}{4} \cdot d\right)
\end{align*}
in the case $p \in [1,2]$ and $d \geq \max\{C_2, 8\}$.

In the case $p \in (2,c_3 \cdot \ln(d))$ we have 
$\ln(d)> p$, and therefore
\[
c_1 \cdot \sqrt{p} \cdot d^{1/p} \leq \underset{X \sim \mathcal{N}(0,I_d)}{\EE} \left[\mnorm{X}_p\right] \leq C_1 \cdot \sqrt{p} \cdot d^{1/p}.
\]
Moreover, we get 
\begin{align*}
\beta(d,p,1/2)&= \max \left\{\min\left\{p^2\cdot2^{-p}\cdot \frac{d}{4}, \left(\frac{d}{2}\right)^{2/p}\right\},\frac{1}{2} \cdot p \cdot d^{2/p}\right\} 
\geq \frac{1}{2} \cdot p \cdot d^{2/p}.
\end{align*}
Therefore, for $d \geq C_2$,
\begin{align*}
&\norel\underset{X \sim \mathcal{N}(0,I_d)}{\PP} \left(\frac{c_1}{2} \cdot \sqrt{p} \cdot d^{1/p} \leq \mnorm{X}_p \leq \frac{3}{2} \cdot C_1 \cdot \sqrt{p}\cdot d^{1/p}\right) \\
&\geq \underset{X \sim \mathcal{N}(0,I_d)}{\PP} \left(\frac{1}{2} \cdot \EE \left[\mnorm{X}_p\right] \leq \mnorm{X}_p \leq \frac{3}{2} \cdot \EE \left[\mnorm{X}_p\right]\right) \\
&\geq 1-C_2\exp\left(- \frac{c_2}{2} \cdot p \cdot d^{2/p}\right).
\end{align*}

Lastly, in the case $p \geq \ln(d)\geq \ln(8) > 2$, we have
\[
    c_1 \cdot \sqrt{\ln(d)} \leq \EE \left[\mnorm{X}_p\right] \leq C_1 \cdot \sqrt{\ln(d)}
\]
and
\[
\beta(d,p, 1/2) \geq \min\left\{\ln(d) \cdot \frac{1}{2}, \frac{1}{2} \cdot p \cdot d^{2/p}\right\}= \ln(d) \cdot \frac{1}{2},
\]
thanks to \eqref{eq:thanksto}.
Hence, we get 
\begin{align*}
&\norel\underset{X \sim \mathcal{N}(0,I_d)}{\PP} \left(\frac{c_1}{2} \cdot \sqrt{\ln(d)} \leq \mnorm{X}_p \leq \frac{3}{2} \cdot C_1 \cdot \sqrt{\ln(d)}\right) \\
&\geq \underset{X \sim \mathcal{N}(0,I_d)}{\PP} \left(\frac{1}{2} \cdot \EE \left[\mnorm{X}_p\right] 
\leq \mnorm{X}_p \leq \frac{3}{2} \cdot \EE \left[\mnorm{X}_p\right]\right) \\
&\geq 1-C_2\exp\left(-\frac{c_2}{2} \cdot \ln(d)\right)
\end{align*}
for every $d \geq \max\{C_2,8\}$.

In the end, the claim follows letting $C \defeq \max\{\frac{3}{\sqrt{2}} \cdot C_1, C_2,8\}$ and 
$c \defeq \min\{\frac{c_1}{2}, \frac{c_2}{4},c_3\}$.
\end{proof}

%% file: upper_proof.tex
\section{\texorpdfstring{Proof of \Cref{thm:main_upper}}{Proof of Theorem 4.1}}\label{sec:upper_proof}
In this appendix, we prove \Cref{thm:main_upper}, i.e., the upper bound for the Lipschitz constant of random \emph{zero-bias} ReLU networks $\Phi: \RR^d \to \RR$ following \Cref{assum:1}.
In fact, we even aim to show a more general statement that will be used later in \Cref{sec:hom_diff} to tackle the case of random networks with general biases. 
To be precise, using the notation $\jac \Phi^{(\ell_1) \to (\ell_2)}$ introduced in \Cref{def:formal}, we aim to bound expressions of the form 
\[
\underset{x \in \SS^{d-1}, A }{\sup} \mnorm{\jac \Phi^{(\ell_1) \to (\ell_2)}(x)A}_{2 \to 2 }
\quad 
\text{and}
\quad 
\underset{x \in \SS^{d-1}, A}{\sup} \mnorm{W^{(L)}\jac \Phi^{(\ell) \to (L-1)}(x)A}_{2 \to 2 }
\]
for suitable indices $\ell,\ell_1, \ell_2 \in \{0,\dots,L-1\}$, 
where $A$ runs over all diagonal matrices where the diagonal entries 
belong to $\{0,-1,1\}$ and that satisfy a certain sparsity condition; 
see \Cref{eq:aj} for a precise definition. 
We note that the special case of $\ell = 0$ and $A = I_{d \times d}$ 
in the second expression yields a bound for 
\[
\sup_{x \in \SS^{d-1}} \mnorm{\nablaa \Phi(x)}_2,
\]
which implies a bound for $\lip_2(\Phi)$ according to \Cref{thm:up_low_bound}.

First, we provide a decomposition of the difference between the formal gradients of the 
network at two input points; we remark that a similar decomposition has been used in \cite{bartlett2021adversarial}.
\begin{lemma}\label{lem:decomp}
Let $x,y \in \RR^d$ and $\Phi: \RR^d \to \RR$ be a ReLU network with $L$ hidden layers of width $N$
as in \eqref{eq:relu-network}.
Then, for every $\ell_1 \in \{0,\dots,L-1\}, \ell_2 \in \{-1,\dots,L-1\}$
and with $\jac \Phi^{(\ell_1)\to (\ell_2)}$ as in \Cref{def:formal},
\begin{align*}
&\norel \left(\jac \Phi^{(\ell_1)\to (\ell_2)} \right)(x) -\left(\jac \Phi^{(\ell_1) \to (\ell_2)} \right)(y)  \\
&= \sum_{j= \ell_1}^{\ell_2} \left(\jac \Phi^{(j+1) \to (\ell_2)}(x)\left(D^{(j)}(x) - D^{(j)}(y)\right)W^{(j)} \jac \Phi^{(\ell_1) \to (j-1)}(y) \right).
\end{align*}
\end{lemma}
\begin{proof}
Fix $\ell_1 \in \{0,\dots,L-1\}$. The proof is via induction over $\ell_2 \in \{-1,\dots, L-1\}$, 
where the case $\ell_2 < \ell_1$ (and thus particularly the case $\ell_2 = -1$) is trivial,
since then 
\[
\jac \Phi^{(\ell_1)\to (\ell_2)}(x) = I_{k \times k} = \jac \Phi^{(\ell_1)\to (\ell_2)}(y)
\]
with $k \in \{d,N\}$ depending on $\ell_1$.
Therefore, we take $\ell_2 \in \{\ell_1,\dots,L-1\}$ and assume that $\ell_2 - 1$ satisfies the claim.
We then get 
\begin{align*}
&\norel \left(\jac \Phi^{(\ell_1) \to (\ell_2)} \right)(x) -\left(\jac \Phi^{(\ell_1) \to (\ell_2)} \right)(y) \\ 
&= D^{(\ell_2)}(x)W^{(\ell_2)}  \left(\jac \Phi^{(\ell_1) \to (\ell_2 - 1)} \right)(x) - D^{(\ell_2)}(y)W^{(\ell_2)}  \left(\jac \Phi^{(\ell_1) \to (\ell_2 - 1)} \right)(y) \\
&= D^{(\ell_2)}(x)W^{(\ell_2)}\left(\left(\jac \Phi^{(\ell_1)\to (\ell_2 -1)} \right)(x) - \left(\jac \Phi^{(\ell_1)\to(\ell_2-1)} \right)(y)\right)
 \\
 &\hspace{0.5cm}+ \left(D^{(\ell_2)}(x) - D^{(\ell_2)}(y)\right)W^{(\ell_2)}\left(\jac \Phi^{(\ell_1)\to(\ell_2-1)} \right)(y) \\
 \overset{\text{IH}}&{=} D^{(\ell_2)}(x)W^{(\ell_2)}\left(\sum_{j= \ell_1}^{\ell_2-1} \left(\jac \Phi^{(j+1) \to (\ell_2 - 1)}(x)\left(D^{(j)}(x) - D^{(j)}(y)\right)W^{(j)} \jac \Phi^{(\ell_1) \to (j-1)}(y) \right)\right) \\
& \hspace{0.5cm} + \left(D^{(\ell_2)}(x) - D^{(\ell_2)}(y)\right)W^{(\ell_2)}\left(\jac \Phi^{(\ell_1)\to(\ell_2-1)} \right)(y) \\
&= \sum_{j= \ell_1}^{\ell_2-1} \left(\jac \Phi^{(j+1) \to (\ell_2)}(x)\left(D^{(j)}(x) - D^{(j)}(y)\right)W^{(j)} \jac \Phi^{(\ell_1) \to (j-1)}(y) \right)\\
&\hspace{0.5cm} + \underbrace{\jac \Phi^{(\ell_2 + 1) \to (\ell_2)}(x)}_{= I_{N \times N}}(D^{(\ell_2)}(x) - D^{(\ell_2)}(y))W^{(\ell_2)}\left(\jac \Phi^{(\ell_1)\to(\ell_2-1)} \right)(y) \\
&= \sum_{j= \ell_1}^{\ell_2} \left(\jac \Phi^{(j+1) \to (\ell_2)}(x)\left(D^{(j)}(x) - D^{(j)}(y)\right)W^{(j)} \jac \Phi^{(\ell_1) \to (j-1)}(y) \right).
\end{align*}
By induction, the claim is shown. 
\end{proof}

For given $\delta > 0$ and $\ell \in \NN$, we let
\begin{equation}\label{eq:aj}
\mathcal{A}^{(0)}_\delta \defeq \{I_{d \times d}\} 
\quad 
\text{and}
\quad
\mathcal{A}^{(\ell)}_\delta \defeq \left\{ A \in \diag\{0,1,-1\}^{N \times N} : \ \Tr \abs{A} \leq  \delta N\right\}.
\end{equation}
We note for later use that if $A \in \mathcal{A}^{(\ell)}_\delta$, 
then also $\abs{A} \in \mathcal{A}^{(\ell)}_\delta$.
We continue by proving the following auxiliary statement. 
 \begin{lemma}\label{lem:auxil}
    There exist absolute constants $C >0$ and $c >0$ such that the following holds. 
    Let $\Phi: \RR^d \to \RR$ be a random \emph{zero-bias} ReLU network satisfying \Cref{assum:1} with $L$ hidden layers of width $N$
    satisfying $N \geq C \cdot L^3d\ln(N/d)$.
    Moreover, let $\delta \in (0,\ee^{-1})$ with 
    \[
    N \geq C \cdot dL^2 \cdot \ln(1/\delta), 
    \quad 
    \delta N \geq C \cdot d\cdot \ln(1/\delta)
    \quad
    \text{and}
    \quad
     \delta \cdot \ln(1/\delta) \leq c \cdot 1/L^2
    \]
    and let $\eps \leq c \cdot \delta \cdot \ln^{-1/2}(1/\delta)$.
    Then, for every $\ell_2 \in \{-1,\dots,L-1\}$ we have with probability at least $1- 4(\ell_2+1)\exp(-c \cdot \delta N)$
    that all of the following statements are satisfied simultaneously:
    \begin{enumerate}
    \item{
    \label{item:1}
    $\displaystyle \sup_{x,y \in \SS^{d-1}, \mnorm{x-y}_2 \leq \eps} \Tr \abs{D^{(j)}(x) - D^{(j)}(y)} \leq \delta N \quad \text{for all } j \in \{0, \dots, \ell_2\}$,
    }
    \item{
\label{item:2}
    $\displaystyle \underset{A_2 \in A_\delta^{(1)}}{\sup_{x \in \SS^{d-1}, A_1 \in \mathcal{A}_\delta^{(\ell_1)}}} \op{A_2 W^{(j)} \jac \Phi^{(\ell_1) \to (j-1)}(x)A_1} \leq C \cdot \left(\sqrt{\delta \ln(1/\delta)} + \sqrt{\frac{Ld\ln(N/d)}{N}}\right) \quad \\ \text{for all } \ell_1 \in \{0,\dots, L-1\} \text{ and }j \in \{0, \dots, \ell_2\}$,
    }
       \item{ \label{item:3}$\displaystyle
         \underset{A \in \mathcal{A}_{\delta}^{(\ell_1)}}{\underset{x \in \SS^{d-1}}{\sup}} \mnorm{\jac \Phi^{(\ell_1) \to (j)}(x)A}_{2 \to 2 }
    \leq C
         \quad \text{for all } \ell_1 \in \{0, \dots, L-1\} \text{ and } j \in \{-1, \dots, \ell_2\}$,}
    \item{\label{item:4}
    $\displaystyle \mnorm{\Phi^{(j)}(x)}_2 \geq \frac{\mnorm{x}_2}{4} \quad \text{for every } x \in \RR^d \text{ and } j \in \{0, \dots, \ell_2\}$. 
    }
    \end{enumerate}
    \end{lemma}

    \begin{proof}
    We let $C_1 \geq 1$ and $c_1 > 0$ be the absolute constants from \Cref{lem:pw_2} and 
    $C' > 0,c' \in (0,1)$ be the absolute constants from \Cref{lem:recht_upper_h} and 
    set $\eps \defeq \frac{c'}{16C_1} \cdot \delta \cdot \ln^{-1/2}(1/\delta)$.
    Note that the statement of the lemma follows if we can prove 
    the required properties for this specific choice of $\eps$ if we take $c \leq \frac{c'}{16C_1}$.
    
        The proof is via induction over $\ell_2$.
        In the case $\ell_2 = -1$, only \eqref{item:3} needs to be considered. 
        However, here we have, since necessarily $j = -1$,
        \[
        \underset{A \in \mathcal{A}_{\delta}^{(\ell_1)}}{\underset{x \in \SS^{d-1}}{\sup}} \op{\jac \Phi^{(\ell_1) \to (j)}(x)A} 
        =\sup_{A \in \mathcal{A}_{\delta}^{(\ell_1)}} \op{A} \leq  1 \leq 2C_1,
        \]
        for every $\ell_1 \in \{0, \dots, L-1\}$.
        
        Let us move to the induction step. 
        Pick $\ell_2 \in \{0,\dots,L-1\}$ and assume that with probability at least $1- 4\ell_2\exp(-c \cdot \delta N)$, all of the following are satisfied:
        \begin{enumerate}[label=(\alph*)]
        \item{
        \label{item:11}
        $\displaystyle \sup_{x,y \in \SS^{d-1}, \mnorm{x-y}_2 \leq \eps} \Tr \abs{D^{(j)}(x) - D^{(j)}(y)} \leq \delta N \quad \text{for all } j \in \{0, \dots, \ell_2-1\}$,
        }
        \item{
        \label{item:22}
        $\displaystyle \underset{A_2 \in A_\delta^{(1)}}{\sup_{x \in \SS^{d-1}, A_1 \in \mathcal{A}_\delta^{(\ell_1)}}} \op{A_2 W^{(j)} \jac \Phi^{(\ell_1) \to (j-1)}(x)A_1} \leq C_2 \cdot \left(\sqrt{\delta \ln(1/\delta)} + \sqrt{\frac{Ld\ln(N/d)}{N}}\right) \quad  \\
        \text{for all } \ell_1 \in \{0,\dots, L-1\} \text{ and } j \in \{0, \dots, \ell_2-1\}$,
        }
           \item{ \label{item:33}$\displaystyle
             \underset{A \in \mathcal{A}_{\delta}^{(\ell_1)}}{\underset{x \in \SS^{d-1}}{\sup}} \mnorm{\jac \Phi^{(\ell_1) \to (j)}(x)A}_{2 \to 2 }
        \leq 2C_1
             \quad \text{for all }  \ell_1 \in \{0, \dots, L-1\} \text{ and }j \in \{-1, \dots, \ell_2-1\}$,}
        \item{\label{item:44}
    $\displaystyle \mnorm{\Phi^{(j)}(x)}_2 \geq \frac{\mnorm{x}_2}{4} \quad \text{for every } x \in \RR^d \text{ and } j \in \{0, \dots, \ell_2-1\}$,
    }
    \end{enumerate}
        where $C_2, c > 0$
        are appropriate constants to be exactly determined later 
        (which will be constant throughout the induction)
        and we recall that $C_1 \geq 1$ is the absolute constant from \Cref{lem:pw_2}.
        
        We now condition on the weights $W^{(0)}, \dots, W^{(\ell_2 -1)}$ and assume that properties \ref{item:11}-\ref{item:44} are satisfied. 
        Note that, by picking $\ell_1 = 0$ in \ref{item:33} and by \Cref{thm:glob}, this in particular implies 
        \begin{equation}\label{eq:lip_help}
        \lip_{2 \to 2}(\Phi^{(j)}) \leq \sup_{x \in \SS^{d-1}} \op{\jac \Phi^{(j)}(x)} \leq 2C_1 \quad \text{for all } j \in \{0, \dots, \ell_2 -1\}. 
        \end{equation}
        
     We set 
        \[
        \mu \defeq \begin{cases}d& \text{if }\ell_2 =0, \\ N& \text{if }\ell_2>0.\end{cases}
        \]
     Let
        \[
        f: \quad \SS^{d-1} \to \SS^{\mu - 1}, \quad f(x) = \frac{\Phi^{(\ell_2-1)}(x)}{\mnorm{\Phi^{(\ell_2-1)}(x)}_2},
        \]
        which is well-defined according to \ref{item:44}.
        For $x,y \in \SS^{d-1}$ we have 
        \begin{align*}
            \mnorm{f(x) - f(y)}_2 \overset{\text{\cshref{lem:diff_bound}}}&{\leq}
            2 \cdot \frac{\mnorm{\Phi^{(\ell_2-1)}(x) - \Phi^{(\ell_2-1)}(y)}_2}{\mnorm{\Phi^{(\ell_2-1)}(x)}_2}
            \overset{\eqref{eq:lip_help}, \ref{item:44}}{\leq} 16C_1 \cdot \mnorm{x-y}_2.
        \end{align*}
        Thus, $f$ is $(16C_1)$-Lipschitz. 
        Note that for $x, y \in \SS^{d-1}$ with $\mnorm{x-y}_2\leq \eps$ we have 
        \[
        \mnorm{f(x)-f(y)}_2 \leq 16C_1 \cdot \eps = c' \cdot \delta \cdot \ln^{-1/2}(1/\delta).
        \]
        Hence, we may apply \Cref{lem:recht_upper_h} by using that if $C$ is chosen large enough,
        \[
        N \geq C \cdot d \cdot \delta^{-1} \cdot \ln(1/\delta) \geq C' \cdot d \cdot \delta^{-1} \cdot \ln \left(16C_1 / \delta\right)
        \]
        and get 
        \begin{align}
        &\norel\underset{x,y \in \SS^{d-1}, \mnorm{x - y}_2 \leq \eps}{\sup}\Tr \abs{D^{(\ell_2)}(y) - D^{(\ell_2)}(x)} \nonumber\\
        \label{eq:ee_1}
        &\leq \underset{\mnorm{z_1-z_2}_2 \leq c' \cdot \delta \cdot \ln^{-1/2}(1/\delta)}{\sup_{z_1,z_2 \in f(\SS^{d-1}),}}
        \#\left\{i \in \{1,\dots,N\}: \ \sgn_-((W^{(\ell_2)}z_1)_i) \neq \sgn_-((W^{(\ell_2)}z_2)_i)\right\} \leq \delta N
        \end{align}
        with probability at least $1- \exp(-c' \cdot \delta N)$ over the randomness in $W^{(\ell_2)}$. 
        We denote the event on which \eqref{eq:ee_1} holds $\E_1$.

        Further, 
        fix $\ell_1 \in \{0,\dots, L-1\}$ and set 
        \[
        \mathcal{V} \defeq \left\{ \jac \Phi^{(\ell_1) \to (\ell_2-1)}(x)A_1: \ x \in \SS^{d-1}, \ A_1 \in \mathcal{A}_\delta^{(\ell_1)}\right\}.
        \]
        Note that we have $\op{M} \leq 2C_1$ and $\rang(M) \leq \delta N$ for every $M \in \mathcal{V}$ by \ref{item:33} and by definition of $\mathcal{A}_\delta^{(\ell_1)}$. 
        Here, we used that 
        $d \leq C \cdot d \cdot \ln(1/\delta) \leq \delta N$ to handle the case $\ell_1 = 0$.
        Hence, for fixed $x \in \SS^{d-1}, A_1 \in \mathcal{A}_\delta^{(\ell_1)}$, and $A_2 \in \mathcal{A}_\delta^{(1)}$, \Cref{lem:pw_4} shows for absolute constants $C_3, c_1 > 0$ 
        that for every $ t \geq C_3$, the condition
        \begin{align*}
            \op{A_2W^{(\ell_2)}\jac \Phi^{(\ell_1) \to (\ell_2-1)}(x)A_1}
            &\leq \sqrt{2} \cdot 2C_1 \cdot \frac{\sqrt{\delta N} + \sqrt{\delta N}+ t}{\sqrt{N}}
        \end{align*}
        is true with probability at least $1- \exp(-c_1 \cdot t^2)$.
        Moreover, combining \Cref{lem:card_b} and \cite[Lemma~5.7]{geuchen2024upper}, we get 
        \begin{align}
        \ln(\# V) + \ln(\#\mathcal{A}_\delta^{(1)})
        &\leq L(d+1) \cdot \ln\left(\frac{\ee N}{d+1}\right) +  2\delta N \cdot \ln(4\ee /\delta)\nonumber\\
        \overset{\text{\cshref{lem:log_b}}}&{\leq} 2\ee L d \cdot \ln(N/d) + 8\ee \cdot \delta N \cdot \ln(1/\delta) \nonumber\\
        \label{eq:carddd}
        &\leq 8\ee  (Ld \ln(N/d) + \delta N \cdot \ln(1/\delta)).
        \end{align}
        Hence, via a union bound, we observe for every $ t\geq C_3$ that the condition  
        \[
        \underset{A_2 \in \mathcal{A}_\delta^{(1)}}{\underset{x \in \SS^{d-1},A_1 \in \mathcal{A}_\delta^{(\ell_1)}}{\sup}} \op{A_2W^{(\ell_2)}\jac \Phi^{(\ell_1) \to (\ell_2-1)}(x)A_1}
        \leq \sqrt{2} \cdot 2C_1 \cdot \frac{\sqrt{\delta N} + \sqrt{\delta N}+ t}{\sqrt{N}}
        \]
        is satisfied with probability at least 
        \[
        1- \exp\big(8\ee  ( Ld\ln(N/d) +  \delta N \cdot\ln(1/\delta))-c_1t^2\big).
        \]
        We then explicitly pick 
        \[
        t \defeq \sqrt{c_1^{-1} \cdot 9\ee \cdot (Ld \ln(N/d) + \delta N \cdot \ln(1/\delta))}, 
        \]
        which yields the existence of an absolute constant $C_2 > 0$, such that 
        \begin{align*}
        &\norel\underset{A_2 \in \mathcal{A}_\delta^{(1)}}{\underset{x \in \SS^{d-1},A_1 \in \mathcal{A}_\delta^{(\ell_1)}}{\sup}} \op{A_2W^{(\ell_2)}\jac \Phi^{(\ell_1) \to (\ell_2-1)}(x)A_1}\\
        &\leq C_2  \cdot \frac{\sqrt{\delta N \cdot \ln(1/\delta)} + \sqrt{Ld\ln(N/d)}}{\sqrt{N}}
        \end{align*}
        with probability at least 
        \[
        1-\exp\big(-\ee \cdot ( Ld\ln(N/d) + \delta N \cdot \ln(1/\delta))\big)
        \]
        over the randomness in $W^{(\ell_2)}$.
        
        Until now, $\ell_1 \in \{0, \dots, L-1\}$ was fixed. 
        By performing a union bound over $\ell_1$, we note that 
        \begin{align}
        &\norel\underset{A_2 \in \mathcal{A}_\delta^{(1)}}{\underset{x \in \SS^{d-1},A_1 \in \mathcal{A}_\delta^{(\ell_1)}}{\sup}} \op{A_2W^{(\ell_2)}\jac \Phi^{(\ell_1) \to (\ell_2-1)}(x)A_1} \nonumber\\
        \label{eq:ee_2}
        &\leq C_2 \cdot \left(\sqrt{\delta \ln(1/\delta)} + \sqrt{\frac{Ld\ln(N/d)}{N}}\right) \quad \text{for all }
        \ell_1 \in \{0, \dots, L-1\}
        \end{align}
        with probability at least 
        \[
        1-\exp\big(\ln(L)-\ee \cdot ( Ld\ln(N/d) + \delta N \cdot \ln(1/\delta))\big) 
        \overset{\ln(L) \leq L}{\geq} 1 - \exp(- \delta N)
        \]
        over the randomness in $W^{(\ell_2)}$.
        We denote the event defined by property \eqref{eq:ee_2} as $\E_2$.

        Next, 
        let $\neps\subseteq \SS^{d-1}$ be an $\eps$-net of $\SS^{d-1}$ with 
        \begin{align}
        \ln(\#\neps) &\leq d \cdot \ln(3/\eps) 
        = d \cdot \ln\left(\frac{48C_1\cdot \ln^{1/2}(1/\delta)}{c' \cdot \delta}\right)
        \overset{\ln^{1/2}(1/\delta) \leq 1/\delta}{\leq} d \cdot \ln\left(\frac{48C_1}{c' \cdot \delta^2}\right) \nonumber\\
        \label{eq:neps_b}
        &\leq 2 \cdot d \cdot \ln(48C_1/(c'\delta)) \overset{\text{\cshref{lem:log_b}}}{\leq} 
        C_4 \cdot d \cdot \ln(1/\delta)
        \end{align}
        with an absolute constant $C_4 > 0$, where the existence of $\neps$
        is ensured by \cite[Corollary~4.2.13]{vershynin_high-dimensional_2018}. 
        For $x \in \SS^{d-1}$ we denote by $\pi(x) \in \neps$ a net point with $\mnorm{x- \pi(x)}_2 \leq \eps$.
        On $\E_1 \cap \E_2$, by \Cref{lem:decomp}, for arbitrary $\ell_1 \in \{0, \dots, \ell_2\}$,
        \begin{align*}
            &\norel\underset{x \in \SS^{d-1}, A \in \mathcal{A}_\delta^{(\ell_1)}}{\sup}
            \op{\jac \Phi^{(\ell_1) \to (\ell_2)}(\pi(x))A - 
            \jac \Phi^{(\ell_1) \to (\ell_2)}(x)A} \nonumber\\
            &\leq {\underset{x \in \SS^{d-1}, A \in \mathcal{A}_\delta^{(\ell_1)}}{\sup}} \sum_{j= \ell_1}^{\ell_2} 
            \op{\jac \Phi^{(j+1) \to (\ell_2)}(\pi(x))\left(D^{(j)}(\pi(x)) - D^{(j)}(x)\right)
            W^{(j)}\jac \Phi^{(\ell_1) \to (j-1)}(x)A} \nonumber\\
            \overset{(*)}&{\leq} 
            {\underset{x \in \SS^{d-1}, A \in \mathcal{A}_\delta^{(\ell_1)}}{\sup}} \Bigg( 
            \sum_{j= \ell_1}^{\ell_2} \op{\jac \Phi^{(j+1) \to (\ell_2)}(\pi(x))\left(D^{(j)}(\pi(x)) - D^{(j)}(x)\right)} \nonumber\\
            & \hspace{4.5cm}\cdot 
            \op{\abs{D^{(j)}(\pi(x)) - D^{(j)}(x)}W^{(j)}\jac \Phi^{(\ell_1) \to (j-1)}(x)A} \Bigg)\nonumber\\
            &\leq \sum_{j= \ell_1}^{\ell_2} \left( \underset{x \in \SS^{d-1}}{\sup} \op{\jac \Phi^{(j+1) \to (\ell_2)}(\pi(x))\left(D^{(j)}(\pi(x)) - D^{(j)}(x)\right)} \right) \nonumber\\
            \label{eq:tb}
            & \hspace{1.2cm}\cdot 
            \left( {\underset{x \in \SS^{d-1}, A \in \mathcal{A}_\delta^{(\ell_1)}}{\sup}} \op{\abs{D^{(j)}(\pi(x)) - D^{(j)}(x)}W^{(j)}\jac \Phi^{(\ell_1) \to (j-1)}(x)A}\right) \\
            \overset{\E_1, \ref{item:11}}&{\leq} \sum_{j= \ell_1}^{\ell_2} \left(\sup_{x^\ast \in \neps, A_2 \in \mathcal{A}_\delta^{(1)}} \op{\jac \Phi^{(j+1) \to (\ell_2)}(x^\ast)A_2}\right) \\
            &\hspace{1.2cm}\cdot \left(\underset{A_2 \in \mathcal{A}_\delta^{(1)}}{\sup_{x \in \SS^{d-1}, A_1 \in \mathcal{A}_\delta^{(\ell_1)}}} \op{A_2 W^{(j)} \jac \Phi^{(\ell_1) \to (j-1)}(x)A_1}\right) \\
            \overset{\E_2, \ref{item:22}}&{\leq} C_2 \cdot \left(\sqrt{\delta \ln(1/\delta)} + \sqrt{\frac{Ld\ln(N/d)}{N}}\right) \cdot \sum_{j= \ell_1}^{\ell_2} \left(\sup_{x^\ast \in \neps, A_2 \in \mathcal{A}_\delta^{(1)}} \op{\jac \Phi^{(j+1) \to (\ell_2)}(x^\ast)A_2}\right).
        \end{align*}
        Here, the step marked with ($\ast$) used that $D = D^{(j)}(\pi(x)) - D^{(j)}(x) \in \diag\{0,1,-1\}^{N \times N}$
        so that $D = D \cdot \abs{D}$, where $\abs{D}$ is obtained by applying $\abs{\cdot}$ to each matrix entry.
        Since $\delta \ln(1/\delta) \leq c/L^2$ and $N \geq C \cdot L^3 d \ln(N/d)$, 
        we can ensure by choosing $C$ large and $c$ small enough that
        \[
        C_2 \cdot \left(\sqrt{\delta \ln(1/\delta)} + \sqrt{\frac{Ld\ln(N/d)}{N}}\right) \leq \frac{1}{L}.
        \]
        Hence, on $\E_1 \cap \E_2$, for arbitrary $\ell_1 \in \{0, \ldots, \ell_2\}$,
        \begin{align*}
        &\norel \underset{x \in \SS^{d-1}, A \in \mathcal{A}_\delta^{(\ell_1)}}{\sup}
            \op{\jac \Phi^{(\ell_1) \to (\ell_2)}(\pi(x))A - 
            \jac \Phi^{(\ell_1) \to (\ell_2)}(x)A} \\
            &\leq \frac{1}{L} 
            \cdot \sum_{j= \ell_1}^{\ell_2} \left(\sup_{x^\ast \in \neps, A_2 \in \mathcal{A}_\delta^{(1)}} \op{\jac \Phi^{(j+1) \to (\ell_2)}(x^\ast)A_2}\right).
        \end{align*}
        In case of $\ell_1 \in \{0, \ldots, L-1\}$ with $\ell_1 > \ell_2$, the same estimate is
        trivially satisfied, since then $\jac \Phi^{(\ell_1) \to (\ell_2)}(z) = I_{N \times N}$
        for any $z \in \RR^d$.
        
        Recall that up to now we have conditioned on $W^{(0)}, \dots, W^{(\ell_2 -1)}$. 
        Reintroducing the randomness over these random matrices, we note that with
        probability at least 
        \[
        1- 4\ell_2\exp(-c \cdot \delta N) - \exp(-c' \cdot \delta N) - \exp(- \delta N),
        \] 
        all of the following properties are simultaneously satisfied:
        \begin{enumerate}[label=(\roman*)]
            \item{
        \label{item:111}
        $\displaystyle \sup_{x,y \in \SS^{d-1}, \mnorm{x-y}_2 \leq \eps} \Tr \abs{D^{(j)}(x) - D^{(j)}(y)} \leq \delta N \quad \text{for all } j \in \{0, \dots, \ell_2\}$,
        }
        \item{
        \label{item:222}
        $\displaystyle \underset{A_2 \in A_\delta^{(1)}}{\sup_{x \in \SS^{d-1}, A_1 \in \mathcal{A}_\delta^{(\ell_1)}}} \op{A_2 W^{(j)} \jac \Phi^{(\ell_1) \to (j-1)}(x)A_1} \leq C_2 \cdot \left(\sqrt{\delta \ln(1/\delta)} + \sqrt{\frac{Ld\ln(N/d)}{N}}\right) \quad  \\
        \text{for all } \ell_1 \in \{0,\dots, L-1\} \text{ and } j \in \{0, \dots, \ell_2\}$,
        }
        \item{\label{item:333}
            $\displaystyle \underset{x \in \SS^{d-1}, A \in \mathcal{A}_\delta^{(\ell_1)}}{\sup}
            \op{\jac \Phi^{(\ell_1) \to (\ell_2)}(\pi(x))A - 
            \jac \Phi^{(\ell_1) \to (\ell_2)}(x)A} \\ 
            \hspace{1cm}\leq \frac{1}{L} 
            \cdot \sum_{j= \ell_1}^{\ell_2} \left(\sup_{x^\ast \in \neps, A_2 \in \mathcal{A}_\delta^{(1)}} \op{\jac \Phi^{(j+1) \to (\ell_2)}(x^\ast)A_2}\right) \quad \text{for all } \ell_1 \in \{0, \dots, L-1\},$
        }
           \item{ \label{item:444}$\displaystyle
             \underset{A \in \mathcal{A}_{\delta}^{(\ell_1)}}{\underset{x \in \SS^{d-1}}{\sup}} \mnorm{\jac \Phi^{(\ell_1) \to (j)}(x)A}_{2 \to 2 }
        \leq 2C_1
             \quad \text{for all }  \ell_1 \in \{0, \dots, L-1\} \text{ and }j \in \{-1, \dots, \ell_2-1\}$,}
        \item{\label{item:555}
    $\displaystyle \mnorm{\Phi^{(j)}(x)}_2 \geq \frac{\mnorm{x}_2}{4} \quad \text{for every } x \in \RR^d \text{ and } j \in \{0, \dots, \ell_2-1\}$.
    }
        \end{enumerate}
        We call the event defined by properties \ref{item:111}-\ref{item:555} $\E_3$, noting that the randomness is now with respect to $W^{(0)}, \dots, W^{(\ell_2)}$.

        For fixed $x^\ast \in \neps$, $j \in \{0, \dots, L\}$ and $A_2 \in \mathcal{A}_\delta^{(j)}$, using \Cref{lem:pw_2} we get
        \[
        \op{\jac \Phi^{(j) \to (\ell_2)}(x^\ast)A_2} \leq C_1
        \]
        with probability at least $1-\exp(-c_1 \cdot N/L^2)$.
        Note here that $d\leq C \cdot d \cdot \ln(1/\delta) \leq \delta N \leq N/L^2$, since $\delta \leq \delta \ln(1/\delta) \leq c/L^2$.
        Since 
        \[
        \ln\big(\#\mathcal{A}_\delta^{(j)}\big) \overset{\text{\cshref{lem:card_b}}}{\leq} \delta N \cdot \ln(4\ee / \delta) \overset{\text{\cshref{lem:log_b}}}{\leq} 4\ee \cdot \delta N \cdot \ln(1/\delta) \quad \text{and} \quad 
        \ln(\#\neps) \overset{\eqref{eq:neps_b}}{\leq} C_4 \cdot d \cdot \ln(1/\delta),
        \]
        a union bound yields that 
        \begin{equation}\label{eq:e4}
           \op{\jac \Phi^{(j) \to (\ell_2)}(x^\ast)A_2} \leq C_1 \quad \text{for all } x^\ast \in \neps, j \in \{0, \dots, L\} \text{ and } A_2 \in \mathcal{A}_\delta^{(j)},
        \end{equation}
        with probability at least 
        \[
        1- \exp\big(\underbrace{\ln(L+1)}_{\leq N/(C \cdot L^2)} + 4\ee \cdot \underbrace{\delta N \cdot \ln(1/\delta)}_{\leq c \cdot N/L^2} + C_4 \cdot \underbrace{d \cdot \ln(1/\delta)}_{\leq N/(CL^2)} - c_1 \cdot N/L^2 \big) \geq 1 - \exp(-c_2 \cdot N/L^2).
        \]
        Here, for the absolute constants $C \geq 1$ sufficiently large and $c>0$
        sufficiently small, 
        we used the estimate $N \geq C \cdot L^3 d \ln(N/d) \geq C \cdot L^2 \ln(\ee L) \geq C \cdot L^2\ln(L+1)$,
         as well as $\delta \ln(1/\delta) \leq c/L^2$ and $N \geq C \cdot L^2 d \ln(1/\delta)$ and set $c_2 \defeq c_1 / 2$. 
        We call the event defined by \eqref{eq:e4} $\E_4$.
        On $\E_5 \defeq \E_3 \cap \E_4$, by \ref{item:333}, 
        \[
        \underset{x \in \SS^{d-1}, A \in \mathcal{A}_\delta^{(\ell_1)}}{\sup}
            \op{\jac \Phi^{(\ell_1) \to (\ell_2)}(\pi(x))A - 
            \jac \Phi^{(\ell_1) \to (\ell_2)}(x)A} \leq C_1
        \]
        for every $\ell_1 \in \{0, \dots, L-1\}$
        and therefore, again by \eqref{eq:e4},
        \begin{align*}
        &\norel \underset{A \in \mathcal{A}_{\delta}^{(\ell_1)}}{\underset{x \in \SS^{d-1}}{\sup}} \mnorm{\jac \Phi^{(\ell_1) \to (\ell_2)}(x)A}_{2 \to 2 } \\
        &\leq \underset{A \in \mathcal{A}_{\delta}^{(\ell_1)}}{\underset{x^\ast \in \neps}{\sup}} \mnorm{\jac \Phi^{(\ell_1) \to (\ell_2)}(x^\ast)A}_{2 \to 2 } \!\!+\!\!\underset{x \in \SS^{d-1}, A \in \mathcal{A}_\delta^{(\ell_1)}}{\sup}
            \op{\jac \Phi^{(\ell_1) \to (\ell_2)}(\pi(x))A - 
            \jac \Phi^{(\ell_1) \to (\ell_2)}(x)A} \leq 2C_1
        \end{align*}
        for every $\ell_1 \in \{0, \dots, L-1\}$.

        All together, we get that on $\E_5$, all of the following properties are simultaneously 
        satisfied:
        \begin{enumerate}[label=(\Alph*)]
        \item{
        \label{item:1111}
        $\displaystyle \sup_{x,y \in \SS^{d-1}, \mnorm{x-y}_2 \leq \eps} \Tr \abs{D^{(j)}(x) - D^{(j)}(y)} \leq \delta N \quad \text{for all } j \in \{0, \dots, \ell_2\}$,
        }
        \item{
        \label{item:2222}
        $\displaystyle \underset{A_2 \in A_\delta^{(1)}}{\sup_{x \in \SS^{d-1}, A_1 \in \mathcal{A}_\delta^{(\ell_1)}}} \op{A_2 W^{(j)} \jac \Phi^{(\ell_1) \to (j-1)}(x)A_1} \leq C_2 \cdot \left(\sqrt{\delta \ln(1/\delta)} + \sqrt{\frac{Ld\ln(N/d)}{N}}\right) \quad  \\
        \text{for all } \ell_1 \in \{0,\dots, L-1\} \text{ and } j \in \{0, \dots, \ell_2\}$,
        }
           \item{ \label{item:3333}$\displaystyle
             \underset{A \in \mathcal{A}_{\delta}^{(\ell_1)}}{\underset{x \in \SS^{d-1}}{\sup}} \mnorm{\jac \Phi^{(\ell_1) \to (j)}(x)A}_{2 \to 2 }
        \leq 2C_1
             \quad \text{for all }  \ell_1 \in \{0, \dots, L-1\} \text{ and }j \in \{-1, \dots, \ell_2\}$,}
        \item{\label{item:4444}
    $\displaystyle \mnorm{\Phi^{(j)}(x)}_2 \geq \frac{\mnorm{x}_2}{4} \quad \text{for every } x \in \RR^d \text{ and } j \in \{0, \dots, \ell_2-1\}$.
    }
    \end{enumerate}
        It thus remains to prove \ref{item:4444} for the case $j = \ell_2$. 
        By \Cref{lem:special} and a union bound (using \eqref{eq:neps_b}), 
        \[
        \mnorm{\Phi^{(\ell_2)}(x^\ast)}_2 \geq 1/2 \quad \text{for all }x^\ast \in \neps
        \]
        with probability at least 
        \[
        1-\exp\big(C_4 \cdot d \cdot \ln(1/\delta) - c_3 \cdot N/L^2\big) \geq 1- \exp(-c_4 \cdot N/L^2).
        \]
        Here, $c_3 > 0$ is an absolute constant, $c_4 \defeq c_3 /2$, and we used $N \geq C \cdot L^2d \ln(1/\delta)$.
        We call the event defined by that property $\E_6$. 
        Note that on $\E_5$ we have $\lip_{2 \to 2}(\Phi^{(\ell_2)}) \leq 2C_1$, 
        as follows from \ref{item:3333} and \Cref{thm:glob}.
        On $\E_7 \defeq \E_5 \cap \E_6$, we thus get for arbitrary $x \in \SS^{d-1}$ that 
        \[
        \mnorm{\Phi^{(\ell_2)}(x)}_2 \geq \mnorm{\Phi^{(\ell_2)}(\pi(x))}_2 - 
        \mnorm{\Phi^{(\ell_2)}(\pi(x)) - \Phi^{(\ell_2)}(x)}_2 \geq \frac{1}{2} - 2C_1 \eps \geq \frac{1}{4}.
        \]
        Since a zero-bias ReLU network is positively homogeneous, we get 
        \[
        \mnorm{\Phi^{(\ell_2)}(x)}_2 \geq \mnorm{x}_2/4 \quad \text{for every } x \in \RR^d,
        \]
        as desired. 
        
        Hence, on $\E_7$ all the desired properties for the induction step are satisfied. 
        We get 
        \begin{align*}
            \PP(\E_7) &= \PP(\E_5 \cap \E_6) = \PP(\E_3 \cap \E_4 \cap \E_6) \\
            &\geq 1- 4\ell_2\exp(-c \cdot \delta N) - \exp(-c' \cdot \delta N) - \exp(- \delta N) - \exp(-c_2 \cdot N/L^2) - \exp(-c_4 \cdot N/L^2) \\
            &\geq 1 - 4(\ell_2 + 1) \exp(-c \cdot \delta N)
        \end{align*}
        by taking $c \defeq \min \{c', c_2, c_4,1\}$ and using $\delta \leq \delta \cdot \ln(1/\delta)\leq c/L^2 \leq 1/L^2$. 
\end{proof}
As a direct consequence of \Cref{lem:auxil} we explicitly state the following corollary
 which we will use in \Cref{sec:hom}.
\begin{corollary}\label{prop:isom}
    There exist absolute constants $C >0$ and $c >0$ such that the following holds. 
    Let $\Phi: \RR^d \to \RR$ be a random \emph{zero-bias} ReLU network satisfying \Cref{assum:1} with $L$ hidden layers of width $N$
    satisfying $N \geq C \cdot L^3d\ln(N/d)$.
    Moreover, let $\delta \in (0,\ee^{-1})$ with 
    \[
    N \geq C \cdot dL^2 \cdot \ln(1/\delta), 
    \quad 
    \delta N \geq C \cdot d  \cdot \ln(1/\delta),
    \quad
    \text{and}
    \quad
     \delta \cdot \ln(1/\delta) \leq c /L^2.
    \]
    Moreover, let $\ell \in \{-1,\dots,L-1\}$.
    Then with probability at least $1 - \exp(-c \cdot \delta N)$,
    \[
    \text{for all } x \in \RR^d: \quad \mnorm{\Phi^{(\ell)}(x)}_2 \geq \frac{\mnorm{x}_2}{4}.
    \]
\end{corollary}

Having shown the auxiliary statement in \Cref{lem:auxil}, 
we can now include the last layer $W^{(L)}$.
\begin{theorem}\label{prop:advanced_lip} 
    There exist absolute constants $C >0$ and $c >0$ such that the following holds. 
    Let $\Phi: \RR^d \to \RR$ be a random \emph{zero-bias} ReLU network satisfying \Cref{assum:1} with $L$ hidden layers of width $N$
    satisfying $N \geq C \cdot L^3d\ln(N/d)$.
    Moreover, let $\delta \in (0,\ee^{-1})$ with 
    \[
    N \geq C \cdot dL^2 \cdot \ln(1/\delta), 
    \quad 
    \delta N \geq C \cdot d  \cdot \ln(1/\delta)
    \quad
    \text{and}
    \quad
     \delta \cdot \ln(1/\delta) \leq c /L^2.
    \]
    Then with probability at least $1 -\exp(-c \cdot \delta N)$
    and with $\mathcal{A}_\delta^{(\ell)}$ as in \eqref{eq:aj},
        \[
         \underset{A \in \mathcal{A}_{\delta}^{(\ell_1)}}{\underset{x \in \SS^{d-1}}{\sup}} \mnorm{\jac \Phi^{(\ell_1) \to (\ell_2)}(x)A}_{2 \to 2 } 
    \leq C
    \quad \text{for all }\ell_1 \in \{0,\dots,L-1\} \text{ and }  \ell_2 \in \{-1,\dots,L-1\}.
        \]
    Moreover, with probability at least $1 -\exp(-c \cdot \delta N)$,
        \[
         \underset{A \in \mathcal{A}_{\delta}^{(\ell)}}{\underset{x \in \SS^{d-1}}{\sup}} \mnorm{W^{(L)}\jac \Phi^{(\ell) \to (L-1)}(x)A}_{2 \to 2 }
    \leq C \cdot \sqrt{N \cdot \delta\ln(1/\delta)} \quad \text{for all $\ell \in \{0, \dots, L\}$}.
        \]
\end{theorem}
\begin{proof}
    We first note that the first claim is an immediate consequence of \Cref{lem:auxil} 
    using that
    \[
    \delta N \geq C \cdot d \cdot \ln(1/\delta) 
    \overset{1/\delta \geq c^{-1} \cdot \ln(1/\delta) \cdot L^2 \geq \ee L}{\geq} C \cdot \ln(\ee L),
    \]
    if $C > 0$ is chosen sufficiently large and $c>0$ sufficiently small. 
    Therefore, let us move to the second claim. 
    
    \textbf{Step 1 (Pointwise estimate):} We start with a pointwise estimate:
        Let $x \in \SS^{d-1}$ be fixed, $j \in \{0, \dots, L\}$ and 
        $A \in \mathcal{A}_\delta^{(j)}$. 
        Using \Cref{lem:pw_3}, we get the existence of absolute constants $C_1,c_1 > 0$ such that for every $C_1 \leq t \leq \sqrt{N}/L$, the estimate
        \begin{align*}
            \mnorm{W^{(L)}\jac \Phi^{(j) \to (L-1)}(x)A}_{2 \to 2 } \leq C_1 \cdot (\sqrt{\Tr \abs{A}} + t)
        \end{align*}
        is true with probability at least $1-\exp(-c_1 \cdot t^2)$.
        Note that we can ensure
        \[
        d \leq C \cdot d \cdot \ln(1/\delta) \leq \delta N \leq  \delta \cdot \ln(1/\delta) \cdot N \leq  c \cdot N/L^2 \leq N/L^2,
        \]
        by taking $C \geq 1$ and $c \leq 1$ sufficiently small, which implies 
        $\rang(A) = \Tr \abs{A} \leq \delta N  \leq N/L^2$ for every 
        $A \in \mathcal{A}_\delta^{(j)}$, as required in \Cref{lem:pw_3}.

        \medskip
        \textbf{Step 2 (Union bound):} We continue by taking a union bound:
        Let $C_4, c' > 0$ be the constants provided by \Cref{lem:auxil} and choose 
        $\eps \defeq c' \cdot \delta \cdot \ln^{-1/2}(1/\delta)$.
        Let $\neps \subseteq \SS^{d-1}$
        be an $\eps$-net for $\SS^{d-1}$ with 
        \begin{equation}\label{eq:neps_b_2}
        \ln(\#\neps) \leq
        C_2 \cdot d \cdot \ln(1/\delta)
        \end{equation}
        with an absolute constant $C_2 > 0$ similar to \eqref{eq:neps_b}.
        Let $j \in \{0, \dots, L\}$. 
        Then,
        \begin{equation*}
        \ln\big(\#\mathcal{A}_\delta^{(j)}\big) \leq 4\ee \cdot \delta N \cdot \ln(1/\delta)
        \end{equation*}
        according to \Cref{lem:card_b,lem:log_b}.
        Since by assumption $\delta N \geq C \cdot d \cdot \ln(1/\delta) \geq C \cdot d$, we get (after possibly enlarging $C_2$)
        \[
        \ln(\#\neps) + \ln\big(\#\mathcal{A}_\delta^{(j)}\big)  \leq C_2 \cdot \delta N \cdot \ln(1/\delta).
        \]
        Hence, taking a union bound, 
        we get for every $C_1 \leq t \leq \sqrt{N}/L$ that the estimate
        \[
        \underset{A \in \mathcal{A}_\delta^{(j)}}{\underset{x^\ast \in \neps}{\sup}}
        \mnorm{W^{(L)}\jac \Phi^{(j) \to (L-1)}(x^\ast)A}_{2 \to 2 } \leq C_1 \cdot (\sqrt{\delta N } + t)
        \]
        is satisfied with probability at least $1- \exp(C_2 \cdot \delta N \cdot \ln(1/\delta)-c_1 t^2)$.
        We explicitly pick 
        \[
        t \defeq \sqrt{(1+C_2) c_1^{-1}\cdot \delta N \cdot \ln(1/\delta)}
        \]
        and hence get for any $j \in \{0, \ldots, L\}$,
        \begin{equation}\label{eq:unif_b_2}
        \underset{A \in \mathcal{A}_\delta^{(j)}}{\underset{x^\ast \in \neps}{\sup}}
        \mnorm{W^{(L)}\jac \Phi^{(j) \to (L-1)}(x^\ast)A}_{2 \to 2 } \leq C_3 \cdot \sqrt{\delta N \cdot\ln(1/\delta)}
        \end{equation}
        with probability at least 
        \[
        1- \exp(-  \delta N \cdot \ln(1/\delta)) \geq 1 - \exp(-\delta N)
        \]
        with an absolute constant $C_3 >0$.
        Note that $ t \leq \frac{\sqrt{N}}{L}$ can be ensured by taking $c$ small enough, since we have
        \[
        \delta \ln(1/\delta) \leq c /L^2,
        \]
        and $t \geq C_1$ is trivially satisfied (for $C$ large enough), 
        because of $\delta N \geq C \cdot d \cdot \ln(1/\delta) \geq C \cdot d \geq C$. 
        
        \medskip

        \textbf{Step 3 (Controlling the deviation):}
        
         It remains to control the deviation when moving from an arbitrary point $x \in \SS^{d-1}$ to a net point $\pi(x) \in \neps$ with $\mnorm{x-\pi(x)}_2 \leq \eps$. 

        Recall that $C_4, c' > 0$ are the absolute constants provided by \Cref{lem:auxil}.
        We let $\E_1$ be the event defined by the properties 
        \begin{enumerate}
        \item{
    \label{item:1_1}
    $\displaystyle \sup_{x,y \in \SS^{d-1}, \mnorm{x-y}_2 \leq \eps} \Tr \abs{D^{(j)}(x) - D^{(j)}(y)} \leq \delta N \quad \text{for all } j \in \{0, \dots, L-1\}$,
    }
    \item{
\label{item:2_2}
    $\displaystyle \underset{A_2 \in A_\delta^{(1)}}{\sup_{x \in \SS^{d-1}, A_1 \in \mathcal{A}_\delta^{(\ell)}}} \op{A_2 W^{(j)} \jac \Phi^{(\ell) \to (j-1)}(x)A_1} \leq C_4 \cdot \left(\sqrt{\delta \ln(1/\delta)} + \sqrt{\frac{Ld\ln(N/d)}{N}}\right) \quad \\ \text{for all }\ell,j \in \{0, \dots, L-1\}$.
    }
    \end{enumerate}
    By \Cref{lem:auxil} we have $\PP(\E_1) \geq 1 - 4L\exp(-c' \cdot \delta N)$.
        Using \Cref{lem:decomp} and using that we have 
        $D = D^{(j)}(\pi(x)) - D^{(j)}(x) \in \diag\{0,1,-1\}^{N \times N}$ and thus $D = D \cdot \abs{D}$,
        on $\E_1$ we get 
        \begin{align*}
            &\norel{\underset{x \in \SS^{d-1}, A \in \mathcal{A}_\delta^{(\ell)}}{\sup}}
            \op{W^{(L)}\jac \Phi^{(\ell) \to (L-1)}(\pi(x))A - 
            W^{(L)}\jac \Phi^{(\ell) \to (L-1)}(x)A} \nonumber\\
            &\leq \sum_{j= \ell}^{L-1} \left( \underset{x \in \SS^{d-1}}{\sup} \op{W^{(L)}\jac \Phi^{(j+1) \to (L-1)}(\pi(x))\left(D^{(j)}(\pi(x)) - D^{(j)}(x)\right)} \right) \nonumber\\
            & \hspace{1.2cm}\cdot 
            \left( {\underset{x \in \SS^{d-1}, A \in \mathcal{A}_\delta^{(\ell)}}{\sup}} \op{\abs{D^{(j)}(\pi(x)) - D^{(j)}(x)}W^{(j)}\jac \Phi^{(\ell) \to (j-1)}(x)A}\right) \\
            &\leq \sum_{j= \ell}^{L-1} \left( \underset{x^\ast \in \neps, A_2 \in \mathcal{A}_\delta^{(1)}}{\sup} \op{W^{(L)}\jac \Phi^{(j+1) \to (L-1)}(x^\ast)A_2} \right) \nonumber\\
            & \hspace{1.2cm}\cdot 
            \left( \underset{A_2 \in \mathcal{A}_\delta^{(1)}}{{\underset{x \in \SS^{d-1}, A_1 \in \mathcal{A}_\delta^{(\ell)},}{\sup}}} \op{A_2W^{(j)}\jac \Phi^{(\ell) \to (j-1)}(x)A_1}\right) \\
            \overset{(\ref{item:2_2})}&{\leq} C_4 \cdot \left(\sqrt{\delta \ln(1/\delta)} + \sqrt{\frac{Ld\ln(N/d)}{N}}\right) \cdot \sum_{j= \ell}^{L-1} \left( \underset{x^\ast \in \neps, A_2 \in \mathcal{A}_\delta^{(1)}}{\sup} \op{W^{(L)}\jac \Phi^{(j+1) \to (L-1)}(x^\ast)A_2} \right) 
        \end{align*}
        for arbitrary $\ell \in \{0, \dots, L-1\}$.

        We let $\E_2$ be the event defined by 
        \[
        \underset{A \in \mathcal{A}_\delta^{(j)}}{\underset{x^\ast \in \neps}{\sup}}
        \mnorm{W^{(L)}\jac \Phi^{(j) \to (L-1)}(x^\ast)A}_{2 \to 2 } \leq C_3 \cdot \sqrt{N \cdot \delta\ln(1/\delta)} \quad \text{for all } j \in \{0, \dots , L\}.
        \]
        Using \eqref{eq:unif_b_2} and a union bound,
        \[
        \PP(\E_2) \geq 1- \exp(\ln(L+1) - \delta N) \geq 1 - \exp(-c_2 \cdot \delta N)
        \]
        with an absolute constant $c_2 > 0$, since $\delta N \geq C \cdot \ln(\ee L)$, 
        as shown at the start of the proof.
        On $\E_1 \cap \E_2$, 
        \begin{align}
            &\norel{\underset{x \in \SS^{d-1}, A \in \mathcal{A}_\delta^{(\ell)}}{\sup}}
            \op{W^{(L)}\jac \Phi^{(\ell) \to (L-1)}(\pi(x))A - 
            W^{(L)}\jac \Phi^{(\ell) \to (L-1)}(x)A} \nonumber \\
            \label{eq:e1cape2}
            &\leq C_3C_4 \cdot L \cdot \left(\sqrt{N} \cdot \delta\ln(1/\delta) + \sqrt{Ld\ln(N/d) \cdot \delta\ln(1/\delta)}\right)
        \end{align}
        for every $\ell \in \{0,\dots, L-1\}$.
        Since by assumption $\delta \ln(1/\delta) \leq c/L^2$ and $N \geq C \cdot L^3 d \ln(N/d)$, we get 
        \begin{align*}
            &\norel{\underset{x \in \SS^{d-1}, A \in \mathcal{A}_\delta^{(\ell)}}{\sup}}
            \op{W^{(L)}\jac \Phi^{(\ell) \to (L-1)}(\pi(x))A - 
            W^{(L)}\jac \Phi^{(\ell) \to (L-1)}(x)A}  \\
            &\leq \sqrt{N \cdot \delta\ln(1/\delta)}
        \end{align*}
        for every $\ell \in \{0, \dots, L-1\}$ on $\E_1 \cap \E_2$. 

        \medskip
        \textbf{Step 4 (Concluding the proof):} For $\ell \in \{0, \dots, L-1\}$, we see that
        \begin{align*}
            &\norel\underset{A \in \mathcal{A}_\delta^{(\ell)}}{\underset{x \in \SS^{d-1}}{\sup}}
        \mnorm{W^{(L)}\jac \Phi^{(\ell) \to (L-1)}(x)A}_{2 \to 2 } \\
        &\leq \underset{A \in \mathcal{A}_\delta^{(\ell)}}{\underset{x^\ast \in \neps}{\sup}}
        \mnorm{W^{(L)}\jac \Phi^{(\ell) \to (L-1)}(x^\ast)A}_{2 \to 2}
        \\
        &\hspace{1cm}+ {\underset{x \in \SS^{d-1}, A \in \mathcal{A}_\delta^{(\ell)}}{\sup}}
            \op{W^{(L)}\jac \Phi^{(\ell) \to (L-1)}(\pi(x))A - 
            W^{(L)}\jac \Phi^{(\ell) \to (L-1)}(x)A} \\
            &\leq (C_3 + 1) \cdot \sqrt{N \cdot \delta \ln(1/\delta)}
        \end{align*}
        on $\E_1 \cap \E_2$ and for $c_3 \defeq \min\{c',c_2\} / 2$ we see
        \[
        \PP(\E_1 \cap \E_2) \geq 1 - 4L\exp(-c' \cdot \delta N) - \exp(-c_2 \cdot \delta N) \geq 1 - \exp(-c_3 \cdot \delta N),
        \]
        since $\delta N \geq C \cdot \ln(\ee L)$, as shown at the start of the proof.
        This proves the second part of the theorem. 
\end{proof}
By taking $\delta \asymp \frac{d}{N} \cdot \ln(N/d)$, we obtain the 
following special case of \Cref{prop:advanced_lip}.
\begin{corollary}\label{corr:advanced_lip}
    There exist absolute constants $C,c > 0$ such that the following holds. 
    Let $\Phi: \RR^d \to \RR$ be a random \emph{zero-bias} ReLU network satisfying \Cref{assum:1} with $L$ hidden layers of width $N$ satisfying 
    $N \geq C \cdot d$ and 
    \[
    N \geq C \cdot L^2d\ln(N/d) (L + \ln(N/d)).
    \]
    Then, with probability at least $1- \exp(-c \cdot d\ln(N/d))$, 
    \[
    \lip_{2 \to 2}(\Phi^{(\ell)}) \leq \underset{x \in \SS^{d-1}}{\sup}
    \op{\jac \Phi^{(\ell)}(x)} \leq C \quad \text{for all } \ell \in \{0, \dots, L-1\}
    \]
    and 
    \[
    \lip_2(\Phi) \leq \underset{x \in \SS^{d-1}}{\sup} \mnorm{\nablaa \Phi(x)}_2 \leq C \cdot \sqrt{d} \cdot \ln(N/d).
    \]
\end{corollary}
Notably, with this approach we already obtain an upper bound for $\lip_2(\Phi)$. 
Using a different approach, it is, however, 
possible to reduce the factor $\ln(N/d)$ to $\sqrt{\ln(\ee L)}$ and to in 
particular entirely remove the dependence on the width $N$.
To this end, we rewrite 
\[
\underset{x \in \SS^{d-1}}{\sup} \mnorm{\nablaa \Phi(x)}_2 = \underset{x,z \in \SS^{d-1}}{\sup} \langle \nablaa \Phi(x),z\rangle = \underset{x,z \in \SS^{d-1}}{\sup} \left\langle \left(W^{(L)}\jac \Phi^{(L-1)}(x)\right)^T,z \right\rangle 
=\underset{v \in \mathcal{L}_{\mathrm{up}}}{\sup} \ \left\langle (W^{(L)})^T,v\right\rangle,
\]
where we define 
\[
\mathcal{L}_{\mathrm{up}} \defeq \left\{ \jac \Phi^{(L-1)}(x)z : \ x,z \in \SS^{d-1}\right\} \subseteq \RR^N.
\]
Conditioning on $W^{(0)}, \dots, W^{(L-1)}$ and only considering the randomness over 
$W^{(L)}$, we can use the discrete version of Dudley's inequality 
(see \cite[Theorem~8.1.6~\&~Exercise~8.1.8]{vershynin_high-dimensional_2018}) and get 
\begin{equation}\label{eq:dud}
 \underset{v \in \mathcal{L}_{\mathrm{up}}}{\sup} \left\langle (W^{(L)})^T,v \right\rangle\lesssim \underset{k \in \ZZ : \ 2^{-k} \leq \diam(\mathcal{L}_{\mathrm{up}})}{\sum} 2^{-k} \cdot \sqrt{\ln(\mathcal{N}(\mathcal{L}_{\mathrm{up}}, 2^{-k}))}
\end{equation}
with high probability.
Using \Cref{lem:auxil}, one can show that for arbitrary $0 < \gamma \ll 1$, 
\begin{equation*}
\ln(\mathcal{N}(\mathcal{L}_{\mathrm{up}}, \gamma)) \lesssim d \cdot \ln(L/\gamma)
\end{equation*}
with high probability; see \Cref{lem:impro}.
Plugging this bound into \eqref{eq:dud} yields the desired upper bound of $\sqrt{d \ln(\ee L)}$.

In the following, we formalize this approach. 
As just described, we start by proving an upper bound for $\ln(\mathcal{N}(\mathcal{L}_{\mathrm{up}}, \gamma))$.
\begin{lemma}\label{lem:impro}
    There exist absolute constants $C,c>0$ with the following property. 
    Let $\Phi: \RR^d  \to \RR$ be a random zero-bias ReLU network with $L$ hidden layers of width $N$ following \Cref{assum:1} satisfying $N \geq C \cdot d$.
    Set 
    \[
    \mathcal{L}_{\mathrm{up}} \defeq \left\{ \jac \Phi^{(L-1)}(x)z : \ x,z \in \SS^{d-1}\right\} \subseteq \RR^N.
    \]
    Let $\gamma \in (0,c)$ and assume that 
    \begin{align*}
     N \geq C \cdot dL^3\gamma^{-2}\ln(N/d) \quad \text{and} \quad  
    N \geq C\cdot dL^2 \gamma^{-2}\ln^2(L/\gamma).
    \end{align*}
    Then, with probability at least $1 - \exp(-c \cdot N \gamma^2 L^{-2} \ln^{-1}(L/\gamma))$,
    \[
    \underset{x \in \SS^{d-1}}{\sup} \op{\jac \Phi^{(L-1)}(x)} \leq C \quad \text{and} \quad \ln\left(\mathcal{N}(\mathcal{L}_{\mathrm{up}}, \gamma)\right) \leq C \cdot d \cdot \ln(L/\gamma).
    \]
\end{lemma}
\begin{proof}
    Let $c_1>0$ be an absolute constant to be specified later. Define
    $\delta \defeq c_1 \cdot \gamma^2 L^{-2} \ln^{-1}(L/\gamma)$.
    Let $C_2, c_2 > 0$ be the absolute constants provided by \Cref{lem:auxil} and let $\eps \defeq c_2 \cdot \delta \cdot \ln^{-1/2}(1/\delta)$ as in \Cref{lem:auxil}.
    Let $\neps$ be an $\eps$-net of $\SS^{d-1}$ with 
    \[
    \ln (\# \neps) \leq  C_3 \cdot d \cdot \ln(1/\delta)
    \]
    where $C_3 > 0$ is an absolute constant. 
    $\neps$ exists according to \cite[Corollary~4.2.13]{vershynin_high-dimensional_2018};
    here we use that 
    \[
    d \cdot \ln(1/\eps) = d \cdot\big(\ln(1/c_2) + \ln(1/\delta) + \ln(\ln^{1/2}(1/\delta))\big)
    \lesssim \cdot d \cdot \ln(1/\delta)
    \]
    with an absolute implied constant.
    If $C>0$ is chosen sufficiently large and $c>0$ sufficiently small, $\delta$ satisfies the conditions of \Cref{lem:auxil}.
    Indeed, since 
    \begin{align*}
        \ln(1/\delta) \leq \ln\Big(c_1^{-1} \cdot L^2 \gamma^{-2} \cdot \underbrace{\ln(L/\gamma)}_{\leq L/\gamma}\Big) 
        \overset{\text{\cshref{lem:log_b}}}{\leq} 3c_1^{-1} \cdot \ln(L/\gamma), 
    \end{align*}
    we have 
    \begin{align*}
        C_2 \cdot dL^2 \cdot \ln(1/\delta) 
        \leq 3C_2c_1^{-1} \cdot dL^2 \cdot \ln(L/\gamma)
        \leq 3C_2c_1^{-1} \cdot dL^2 \gamma^{-2}\ln^2(L/\gamma) 
        \overset{C \geq 3C_2c_1^{-1}}{\leq} N
    \end{align*}
    and 
    \begin{align*}
        C_2 \cdot d \cdot \ln(1/\delta) \cdot \delta^{-1}
        \leq 3C_2c_1^{-2} \cdot d \cdot \ln(L/\gamma) \cdot \frac{L^2}{\gamma^2} \cdot \ln(L/\gamma) \overset{C \geq 3C_2c_1^{-2}}{\leq} N,
    \end{align*}
    as well as 
    \begin{align*}
        \delta \cdot \ln(1/\delta) \leq 3 \cdot \frac{\gamma^2}{L^2} \cdot \ln^{-1}(L/\gamma) \cdot \ln(L/\gamma) = 3 \cdot \frac{\gamma^2}{L^2}
        \leq 3 \cdot \frac{c^2}{L^2} \overset{c \leq \sqrt{c_2/3}}{\leq} \frac{c_2}{L^2}.
    \end{align*}
    Let $\E_1$ be the event defined by the properties 
    \begin{enumerate}
        \item{
    \label{item:uno}
    $\displaystyle \sup_{x,y \in \SS^{d-1}, \mnorm{x-y}_2 \leq \eps} \Tr \abs{D^{(j)}(x) - D^{(j)}(y)} \leq \delta N \quad \text{for all } j \in \{0, \dots, L-1\}$,
    }
    \item{
\label{item:due}
    $\displaystyle \underset{A_2 \in A_\delta^{(1)}}{\sup_{x \in \SS^{d-1}}} \op{A_2 W^{(j)} \jac \Phi^{(j-1)}(x)} \leq C_2 \cdot \left(\sqrt{\delta \ln(1/\delta)} + \sqrt{\frac{Ld\ln(N/d)}{N}}\right) \quad \\ \text{for all } j \in \{0, \dots, L-1\}$,
    }
       \item{ \label{item:tres}$\displaystyle
         {\underset{x \in \SS^{d-1}}{\sup}} \mnorm{\jac \Phi^{(j)}(x)}_{2 \to 2 }
    \leq C_2
         \quad \text{for all } j \in \{-1, \dots, L-1\}$.}
    \end{enumerate}
    By \Cref{lem:auxil}, $\PP(\E_1)\geq 1 - 4L\exp(-c_2 \cdot \delta N)$.
    For $x \in \SS^{d-1}$, we denote by $\pi(x) \in \neps$ a net point with $\mnorm{x - \pi(x)}_2 \leq \eps$.
    On $\E_1$, using \Cref{lem:decomp} and the fact that we have 
    \[
    D = D^{(j)}(\pi(x)) - D^{(j)}(x) \in \diag\{0,1,-1\}^{N \times N}
    \] 
    and thus 
    $D = D \cdot \abs{D}$, we obtain 
    \begin{align*}
            &\norel\underset{x \in \SS^{d-1}}{\sup}
            \op{\jac \Phi^{ (L-1)}(\pi(x)) - 
            \jac \Phi^{(L-1)}(x)} \nonumber\\
            &\leq {\underset{x \in \SS^{d-1}}{\sup}} \sum_{j=0}^{L-1} 
            \op{\jac \Phi^{(j+1) \to (L-1)}(\pi(x))\left(D^{(j)}(\pi(x)) - D^{(j)}(x)\right)
            W^{(j)}\jac \Phi^{(j-1)}(x)} \nonumber\\
            &\leq 
            {\underset{x \in \SS^{d-1}}{\sup}} \Bigg( 
            \sum_{j= 0}^{L-1} \op{\jac \Phi^{(j+1) \to (L-1)}(\pi(x))\left(D^{(j)}(\pi(x)) - D^{(j)}(x)\right)} \nonumber\\
            & \hspace{4.5cm}\cdot 
            \op{\abs{D^{(j)}(\pi(x)) - D^{(j)}(x)}W^{(j)}\jac \Phi^{ (j-1)}(x)} \Bigg)\nonumber\\
            &\leq \sum_{j= 0}^{L-1} \left( \underset{x \in \SS^{d-1}}{\sup} \op{\jac \Phi^{(j+1) \to (L-1)}(\pi(x))\left(D^{(j)}(\pi(x)) - D^{(j)}(x)\right)} \right) \nonumber\\
             & \hspace{1.2cm}\cdot 
            \left( {\underset{x \in \SS^{d-1}}{\sup}} \op{\abs{D^{(j)}(\pi(x)) - D^{(j)}(x)}W^{(j)}\jac \Phi^{(j-1)}(x)}\right) \\
            \overset{\E_1, (\ref{item:uno})}&{\leq} \sum_{j= 0}^{L-1} \left(\sup_{x^\ast \in \neps, A \in \mathcal{A}_\delta^{(1)}} \op{\jac \Phi^{(j+1) \to (L-1)}(x^\ast)A}\right)\cdot \left(\underset{A \in \mathcal{A}_\delta^{(1)}}{\sup_{x \in \SS^{d-1}}} \op{A W^{(j)} \jac \Phi^{(j-1)}(x)}\right) \\
            \overset{\E_1, (\ref{item:due})}&{\leq} C_2 \cdot \left(\sqrt{\delta \ln(1/\delta)} + \sqrt{\frac{Ld\ln(N/d)}{N}}\right) \cdot \sum_{j= 0}^{L-1} \left(\sup_{x^\ast \in \neps, A \in \mathcal{A}_\delta^{(1)}} \op{\jac \Phi^{(j+1) \to (L-1)}(x^\ast)A}\right).
        \end{align*}

    We let $C_4 > 0$ be the absolute constant provided by \Cref{lem:pw_2} and let $\E_2$ be the event defined by 
    \[
    \underset{j \in \{0, \dots, L-1\}}{\sup_{x^\ast \in \neps, A \in \mathcal{A}_\delta^{(1)}}} \op{\jac \Phi^{(j+1) \to (L-1)}(x^\ast)A} \leq C_4.
    \]
    Recall from \Cref{lem:card_b,lem:log_b} that 
    \[
    \ln\big(\# \mathcal{A}_\delta^{(1)}\big) \leq 4\ee \cdot N \cdot \delta\ln(1/\delta).
    \]
    Since $\rang(A) \leq \delta N \overset{\delta \leq 1/L^2}{\leq} N/L^2$ for every 
    $A \in \mathcal{A}_\delta^{(1)}$ 
    and since 
    \[
    N \geq C \cdot dL^3 \gamma^{-2}\ln(N/d) \geq C \cdot L^3 \geq C_4 \cdot L^2 \ln(\ee L),
    \]
    we can combine \Cref{lem:pw_2} with a union bound and obtain 
    \[
    \PP(\E_2) \geq 1 - \exp \big(\ln(L) + C_3 \cdot d \cdot \ln(1/\delta) + 4\ee \cdot N \cdot \delta \ln(1/\delta) - c_4 \cdot N/L^2\big) \geq 1- \exp(-c_5 \cdot N/L^2),
    \]
    with an absolute constant $c_4 > 0$, $c_5 \defeq c_4 /2$ and where the last inequality follows from our assumptions
    for $c$ (and thus $\gamma$) sufficiently small, depending on $c_4$.
    On $\E_1 \cap \E_2$ we have 
    \[
    \underset{x \in \SS^{d-1}}{\sup}
                \op{\jac \Phi^{ (L-1)}(\pi(x)) - 
                \jac \Phi^{(L-1)}(x)} \leq C_2C_4 \cdot L \cdot \left(\sqrt{\delta \ln(1/\delta)} + \sqrt{\frac{Ld\ln(N/d)}{N}}\right).
    \]
    Note that 
    \begin{align*}
    \sqrt{\delta \ln(1/\delta)} &= \sqrt{c_1 \cdot \frac{\gamma^2}{L^2} \cdot \ln^{-1}(L/\gamma) \cdot \ln(c_1^{-1} \cdot L^2 \gamma^{-2} \cdot \underbrace{\ln(L/\gamma)}_{\leq L/\gamma})} \\
    &\leq \sqrt{c_1 \cdot \frac{\gamma^2}{L^2} \cdot \ln^{-1}(L/\gamma) \cdot 2 \cdot \frac{1}{2} \cdot \ln(c_1^{-1} \cdot L^3 \gamma^{-3})} \\
    &= \sqrt{2c_1 \cdot \frac{\gamma^2}{L^2} \cdot \ln^{-1}(L/\gamma) \cdot  \ln\left(c_1^{-1/2} \cdot L^{3/2} \gamma^{-3/2}\right)} \\
    \overset{\text{\cshref{lem:log_b}}}&{\leq}\sqrt{3c_1^{1/2} \cdot \frac{\gamma^2}{L^2} \cdot \ln^{-1}(L/\gamma) \cdot \ln\left(L/\gamma\right)} \\
    &=\sqrt{3c_1^{1/2} \cdot \frac{\gamma^2}{L^2}} \\
    &= \sqrt{3} \cdot c_1^{1/4} \cdot \frac{\gamma}{L} \leq \frac{\gamma}{4C_2C_4 \cdot L}
    \end{align*}
    if $c_1$ is chosen sufficiently small.
    Moreover, since $N \geq C \cdot dL^3\gamma^{-2}\ln(N/d)$ and if $C > 0$ is chosen sufficiently large,  
    \[
    \sqrt{\frac{Ld\ln(N/d)}{N}} \leq \frac{\gamma}{4C_2C_4 \cdot L}.
    \]
    On $\E_1 \cap \E_2$, we thus observe 
    \[
    \underset{x \in \SS^{d-1}}{\sup}
                \op{\jac \Phi^{ (L-1)}(\pi(x)) - 
                \jac \Phi^{(L-1)}(x)} \leq \frac{\gamma}{2}.
    \]
    In the end, we define the set 
    \[
    \mathcal{L}_{\mathrm{up}}^\ast \defeq \left\{ \jac \Phi^{(L-1)}(x^\ast)z^\ast : \ x^\ast,z^\ast \in \neps\right\}.
    \]
    For arbitrary $x,z \in \SS^{d-1}$, we then get on $\E_1 \cap \E_2$,
    \begin{align*}
        &\norel\mnorm{\jac \Phi^{(L-1)}(x)z - \jac \Phi^{(L-1)}(\pi(x))\pi(z)}_2 \\
        &\leq \mnorm{\jac \Phi^{(L-1)}(x)z - \jac \Phi^{(L-1)}(\pi(x))z}_2 + \mnorm{\jac \Phi^{(L-1)}(\pi(x))z - \jac \Phi^{(L-1)}(\pi(x))\pi(z)}_2 \\
        &\leq \op{\jac \Phi^{(L-1)}(x) - \jac \Phi^{(L-1)}(\pi(x))} + \op{\jac \Phi^{(L-1)}(\pi(x))} \cdot \mnorm{\pi(z)-z}_2 \\
        \overset{(\ref{item:tres})}&{\leq} \frac{\gamma}{2} + C_2 \cdot \eps 
        \leq \gamma,
    \end{align*}
    where the last inequality follows from 
    \[
    \eps \leq \delta = c_1 \cdot \gamma^2L^{-2}\ln^{-1}(L/\gamma)\leq c_1 \cdot \gamma \leq \frac{\gamma}{2C_1}
    \]
    if $c_1$ is chosen small enough.
    Consequently, we obtain
    \[
    \ln(\mathcal{N}(\mathcal{L}_{\mathrm{up}},\gamma)) \leq \ln\left((\# \neps)^2\right) \leq 6C_3c_1^{-1} \cdot d \cdot \ln(L/\gamma).
    \]
    Finally, it remains to estimate the probability of $\E_1 \cap \E_2$. 
    Note that 
    \begin{align*}
        \PP(\E_1) &\geq 1-4L\exp(-c_2 \cdot \delta N) 
        = 1 - \exp\big(\ln(4L) - c_2c_1 \cdot N\gamma^2L^{-2}\ln^{-1}(L/\gamma) \big) \\
        &\geq 1 -\exp(-c_6 \cdot N\gamma^2L^{-2}\ln^{-1}(L/\gamma))
    \end{align*}
    where $c_6 > 0$ is an absolute constant, since by assumption 
    \[
    N \geq C \cdot L^2 \gamma^{-2}\ln^2(L/\gamma)\overset{L/\gamma \geq \ee L}{\geq} C \cdot L^2\cdot \gamma^{-2} \ln(L/\gamma)\ln(\ee L).
    \]
    Since clearly $N/L^2 \geq N\gamma^2 L^{-2} \ln^{-1}(L/\gamma)$, we get 
    \begin{align*}
        \PP(\E_1 \cap \E_2) &\geq 1 - \exp(-c_6 \cdot N\gamma^2 L^{-2} \ln^{-1}(L/\gamma)) - \exp(-c_5 \cdot N/L^2) \\
        &\geq 1 - 2\exp(-c_7 \cdot N\gamma^2 L^{-2} \ln^{-1}(L/\gamma))
    \end{align*}
    with $c_7 \defeq \min\{c_5,c_6\}$. 
    This implies the claim since $N \geq C \cdot L^2 \gamma^{-2} \cdot \ln(L/\gamma)$ by assumption.
\end{proof}

While \Cref{lem:impro} provides a bound for $\ln(\mathcal{N}(\mathcal{L}_{\mathrm{up}},\gamma))$
for each \emph{fixed} $\gamma > 0$, it cannot be used to bound 
$\mathcal{N}(\mathcal{L}_{\mathrm{up}}, 2^{-k})$ for \emph{infinitely}
many values of $k \in \ZZ$, since a union bound is not possible due to 
\[
\exp(-c \cdot N \gamma^2 L^{-2} \ln^{-1}(L/\gamma)) \to 1 \quad (\gamma \to 0).
\]
Therefore, we additionally need the following statement to control $\mathcal{N}(\mathcal{L}_{\mathrm{up}},2^{-k})$ for large values of $k$.
\begin{lemma}[{cf. \cite[Lemma~5.8]{geuchen2024upper}}]\label{lem:old_bound}
    There exists a constant $C > 0$ with the following property. 
    Let $\Phi: \RR^d \to \RR$ be a deterministic ReLU network with $L$ 
    hidden layers of width $N$ with $N \geq C \cdot d$. 
    Set 
    \[
    \mathcal{L}_{\mathrm{up}} \defeq \left\{ \jac \Phi^{(L-1)}(x)z : \ x ,z\in \SS^{d-1}\right\} \subseteq \RR^N \quad \text{and} \quad \Lambda \defeq \underset{x \in \SS^{d-1}}{\sup} \op{\jac \Phi^{(L-1)}(x)}.
    \]
    Then, for $\Lambda \geq \gamma > 0$,
    \[
    \ln(\mathcal{N}(\mathcal{L}_{\mathrm{up}}, \gamma)) \leq C \cdot \big(d \cdot \ln(\Lambda/\gamma) + Ld \ln(N/d)\big).
    \]
\end{lemma}
Notably, \Cref{lem:old_bound} is not a high-probability statement but holds even deterministically. 
Moreover, we note that the statement in \cite{geuchen2024upper} is slightly different to the one presented here, in particular since in \cite{geuchen2024upper} we have 
$\Lambda = \op{W^{(L-1)}} \cdots \op{W^{(0)}}$. 
However, a simple modification of the proof in \cite{geuchen2024upper} shows that the estimate remains true if $\Lambda$ is as defined in \Cref{lem:old_bound}.

Now we have all the ingredients together to prove the desired bound on the Lipschitz constant. 
\begin{proof}[Proof of \Cref{thm:main_upper}]
    By the discrete version of Dudley's inequality 
    (see \cite[Theorem~8.1.6~\&~Exercise~8.1.8]{vershynin_high-dimensional_2018}), 
    there exists an absolute constant $C_1 > 0$ such that for every $u \geq 0$, 
    conditioned on $W^{(0)}, \ldots, W^{(L-1)}$,
    \[
    \underset{x \in \SS^{d-1}}{\sup}\mnorm{\nablaa \Phi(x)}_2
    \leq C_1 \cdot \left(\sum_{k \in \ZZ} 2^{-k} \cdot \sqrt{\ln(\mathcal{N}(\mathcal{L}_{\mathrm{up}}, 2^{-k}))} + u \cdot \diam(\mathcal{L}_{\mathrm{up}})\right)
    \]
    with probability at least $1-2\exp(-u^2)$.
    We set $\gamma_0 \defeq  \frac{1}{\sqrt{L \ln(N/d)}}$.
    Let $C_2,c_2 > 0$ be the constants given by \Cref{lem:impro}.
    Since $N \geq C \cdot d$ and after possibly enlarging $C$, 
    we can ensure that $\gamma_0 \in (0, c_2/2)$.
    Moreover, we set
    \[
    \Lambda \defeq \underset{x \in \SS^{d-1}}{\sup} \op{\jac \Phi^{(L-1)}(x)}.
    \]
    Since $\mathcal{L}_{\mathrm{up}} \subseteq \B_N(0, \Lambda)$
    we have $\mathcal{N}(\mathcal{L}_{\mathrm{up}}, \eps) = 1$ for $\eps \geq 2\Lambda$.
    Hence, 
    \begin{align*}
    &\norel\sum_{k \in \ZZ} 2^{-k} \cdot \sqrt{\ln(\mathcal{N}(\mathcal{L}_{\mathrm{up}}, 2^{-k}))} \\
    &\leq \underset{\frac{c_2}{2}\leq 2^{-k} \leq 2\Lambda}{\sum_{k \in \ZZ:}} 2^{-k} \cdot \sqrt{\ln(\mathcal{N}(\mathcal{L}_{\mathrm{up}}, 2^{-k}))}
    +\underset{\gamma_0 \leq 2^{-k} < \frac{c_2}{2}}{\sum_{k \in \ZZ}} 2^{-k} \cdot \sqrt{\ln(\mathcal{N}(\mathcal{L}_{\mathrm{up}}, 2^{-k}))} \\
    & \hspace{1cm}+ \underset{2^{-k} < \min\{\gamma_0, 2\Lambda\}}{\sum_{k \in \ZZ:}} 2^{-k} \cdot \sqrt{\ln(\mathcal{N}(\mathcal{L}_{\mathrm{up}}, 2^{-k}))}.
    \end{align*}
    We study all three summands appearing in that decomposition individually,
    where we now remove the conditioning on $W^{(0)}, \ldots, W^{(L-1)}$.

    Note that 
    \begin{align*}
        \underset{\frac{c_2}{2}\leq 2^{-k} \leq 2\Lambda}{\sum_{k \in \ZZ:}} 2^{-k} \cdot \sqrt{\ln(\mathcal{N}(\mathcal{L}_{\mathrm{up}}, 2^{-k}))} \leq \sqrt{\ln(\mathcal{N}(\mathcal{L}_{\mathrm{up}}, c_2/2))} \cdot \underset{\frac{c_2}{2}\leq 2^{-k} \leq 2\Lambda}{\sum_{k \in \ZZ:}} 2^{-k}.
    \end{align*}
    \Cref{lem:impro} with $\gamma = c_2/2$ guarantees 
    the existence of constants $C_3, c_3 > 0$ such that with probability at least $1- \exp(-c_3 \cdot NL^{-2}\ln^{-1}(\ee L))$, 
    \begin{equation}\label{eq:help_2}
    2\Lambda = 2\cdot \underset{x \in \SS^{d-1}}{\sup} \op{\jac \Phi^{(L-1)}(x)} \leq C_3 \quad \text{and} \quad 
    \ln(\mathcal{N}(\mathcal{L}_{\mathrm{up}}, c_2 /2)) \leq C_3 \cdot d \cdot \ln(\ee L).
    \end{equation}
    On this event, we thus get 
    \[
    \sqrt{\ln(\mathcal{N}(\mathcal{L}_{\mathrm{up}}, c_2/2))} \cdot \underset{\frac{c_2}{2}\leq 2^{-k} \leq 2\Lambda}{\sum_{k \in \ZZ:}} 2^{-k} \leq \sqrt{C_3 \cdot d \cdot \ln(\ee L)} \cdot \underset{\frac{c_2}{2}\leq 2^{-k} \leq C_3}{\sum_{k \in \ZZ:}} 2^{-k} \leq C_4 \cdot \sqrt{d\ln(\ee L)},
    \]
    where $C_4 > 0$ is a suitable absolute constant. 

    Let us move to the second summand and fix $k \in \ZZ$ with 
    $\gamma_0 \leq 2^{-k} < \frac{c_2}{2}$.
    Note that this implies $k \in \NN$ since we can assume $c_2 \leq 1$.
    The goal is now to apply \Cref{lem:impro} with $\gamma = 2^{-k}$.
    To this end, we first verify that the conditions of \Cref{lem:impro} are satisfied. 
    Let $K >0$ be the large constant from \Cref{lem:impro}. 
    Note that 
    \[
    \ln(C \cdot L^4\ln^2(\ee L)) = \ln(C) + 4 \ln(L) + \ln(\ln^2(\ee L)) \leq (\ln(C)+5)\ln(\ee L).
    \]
    Moreover, $x \mapsto x/\ln^2(x)$ is nondecreasing on $[e^2, \infty)$ since 
    \[
    \frac{\dd}{\dd x}\left[\frac{x}{\ln^2(x)}\right] = \frac{\ln^2(x)-2\ln(x)}{\ln^4(x)}.
    \]
    Therefore, 
    \begin{equation}\label{eq:loghelper}
    \frac{N/d}{\ln^2(N/d)} \geq \frac{C \cdot L^4 \ln^2(\ee L)}{\ln^2(C \cdot L^4\underbrace{\ln^2(\ee L)}_{\leq L^2})}
    \geq \frac{C \cdot L^4 \ln^2(\ee L)}{(\ln(C) + 5)^2 \ln^2(\ee L)} = \frac{C}{(\ln(C) + 5)^2} L^4 \geq K \cdot L^4
    \end{equation}
    by choosing $C>0$ sufficiently large. 
    Regarding the first condition of \Cref{lem:impro}, we observe
    \[
    K \cdot dL^3\gamma^{-2}\ln(N/d) \leq K \cdot dL^3\gamma_0^{-2}\ln(N/d) = K \cdot dL^4 \ln^2(N/d) \leq N.
    \]
    Regarding the second one, we observe 
    \begin{align*}
    K \cdot dL^2 \gamma^{-2} \ln(L/\gamma) &\leq K \cdot dL^2 \gamma_0^{-2} \ln(L/\gamma_0)
    = K \cdot dL^3\ln(N/d)\ln\left(L \sqrt{L \ln(N/d)}\right) \\
    &\leq K \cdot dL^4\ln^2(N/d) \leq N.
    \end{align*}
    Consequently, the conditions of \Cref{lem:impro} are satisfied and we conclude
    the existence of absolute constants $C_5, c_5 > 0$ 
    such that with probability at least 
    $1 - \exp(-c_5 \cdot N\gamma_0^2L^{-2}\ln^{-1}(L/\gamma_0))$,
    \[
    \ln(\mathcal{N}(\mathcal{L}_{\mathrm{up}}, 2^{-k})) \leq C_5 \cdot d \cdot \ln(2^k L).
    \]
    Here we used that 
    \[
        \frac{\dd}{\dd\gamma}\left(\gamma^2 \ln^{-1}(L/\gamma)\right)
        =
        \frac{\gamma\left(2\ln(L/\gamma)+1\right)}
        {\ln^2(L/\gamma)}
        \geq 0 \quad \text{for } \gamma \in (0,1),
    \]
    so that 
    \[
    \gamma_0^2 \ln^{-1}(L/\gamma_0) \leq (2^{-k})^2 \ln^{-1}(L/2^{-k}).
    \]
    Note that 
    \[
    \#\{k \in \NN: \ \gamma_0 \leq 2^{-k} < c_2 /2\} \leq \#\{k \in \NN: \ \gamma_0 \leq 2^{-k}\}
    \leq \#\left\{k \in \NN: \ \frac{\ln(1/\gamma_0)}{\ln(2)} \geq k\right\} \leq 2 \cdot \ln(1/\gamma_0).
    \]
    Via a union bound, we get that with probability at least $1-\exp(2\ln(1/\gamma_0) -c_5 \cdot N\gamma_0^2L^{-2}\ln^{-1}(L/\gamma_0))$, 
    \[
    \ln(\mathcal{N}(\mathcal{L}_{\mathrm{up}}, 2^{-k})) \leq C_5 \cdot d \cdot \ln(2^k L) \quad 
    \text{for all } k \in \NN \text{ with } \gamma_0 \leq 2^{-k} < c_2/2.
    \]
    On this event, we obtain 
    \begin{align*}
        \underset{\gamma_0 \leq 2^{-k} < \frac{c_2}{2}}{\sum_{k \in \ZZ}} 2^{-k} \cdot \sqrt{\ln(\mathcal{N}(\mathcal{L}_{\mathrm{up}}, 2^{-k}))}
        \leq \sqrt{C_5} \cdot \sqrt{d} \cdot \underset{\gamma_0 \leq 2^{-k} < \frac{c_2}{2}}{\sum_{k \in \ZZ}} 2^{-k} \cdot \sqrt{\ln(2^k L)}.
    \end{align*}
    Since $2^{-k} = 2 \int_{2^{-k-1}}^{2^{-k}} \dd \eps$,
    \begin{align*}
        \underset{\gamma_0 \leq 2^{-k} < \frac{c_2}{2}}{\sum_{k \in \ZZ}} 2^{-k} \cdot \sqrt{\ln(2^k L)} &= 2 \cdot \underset{\gamma_0 \leq 2^{-k} < \frac{c_2}{2}}{\sum_{k \in \ZZ}} \int_{2^{-k-1}}^{2^{-k}} \sqrt{\ln(2^k L)} \ \dd \eps 
        \leq 2 \cdot \underset{\gamma_0 \leq 2^{-k} < \frac{c_2}{2}}{\sum_{k \in \ZZ}} \int_{2^{-k-1}}^{2^{-k}} \sqrt{\ln(L/\eps)} \ \dd \eps \\
        &\leq 2 \cdot \int_{\gamma_0/2}^{c_2/2} \sqrt{\ln(L/\eps)} \ \dd \eps \leq  2 \cdot \int_{0}^{c_2/2} \sqrt{\ln(L/\eps)} \ \dd \eps\\
        \overset{\text{\cshref{lem:int_b}}}&{\leq} 2 \cdot c_2 \cdot \sqrt{\ln(2L/c_2)} \overset{\text{\cshref{lem:log_b}}}{\leq} C_6 \cdot \sqrt{\ln(\ee L)},
    \end{align*}
    where $C_6 > 0$ is a suitable absolute constant.
    Overall, we get
    \[
    \underset{\gamma_0 \leq 2^{-k} < \frac{c_2}{2}}{\sum_{k \in \ZZ}} 2^{-k} \cdot \sqrt{\ln(\mathcal{N}(\mathcal{L}_{\mathrm{up}}, 2^{-k}))} \leq C_7 \cdot \sqrt{d \ln(\ee L)}
    \]
    with probability at least 
    \[
    1-\exp\big(2\ln(1/\gamma_0) -c_5 \cdot N\gamma_0^2L^{-2}\ln^{-1}(L/\gamma_0)\big)
    \]
    with an absolute constant $C_7 > 0$.
    However, since by assumption 
    \[
    N \geq C \cdot L^3 \ln(N/d)\ln^2(L \cdot \ln(N/d)),
    \]
    we have by our choice of $\gamma_0$,
    \[
    2 \ln(1/\gamma_0) \leq \frac{c_5}{2} \cdot N\gamma_0^2 L^{-2} \ln^{-1}(L/\gamma_0)
    \]
    if $C > 0$ is sufficiently large and therefore, 
    \[
    1-\exp\big(2\ln(1/\gamma_0) -c_5 \cdot N\gamma_0^2L^{-2}\ln^{-1}(L/\gamma_0)\big)
    \geq 1-\exp\big(-c_6 \cdot N\gamma_0^2L^{-2}\ln^{-1}(L/\gamma_0)\big)
    \]
    with $c_6 \defeq c_5 /2$.
    Recalling our choice of $\gamma_0$, we get 
    \[
    1-\exp\big(-c_6 \cdot N\gamma_0^2L^{-2}\ln^{-1}(L/\gamma_0)\big) 
    \geq 1 - \exp\big(-c_7 \cdot N L^{-3}\ln^{-1}(N/d)\ln^{-1}(L \ln(N/d))\big)
    \]
    with an absolute constant $c_7 > 0$.

    It remains to study the third summand. 
    Firstly, we can again pass to an integral and get
    \begin{align*}
        \underset{2^{-k} < \min\{\gamma_0, 2\Lambda\}}{\sum_{k \in \ZZ:}} 2^{-k} \cdot \sqrt{\ln(\mathcal{N}(\mathcal{L}_{\mathrm{up}}, 2^{-k}))} \leq 2 \cdot \int_0^{\min\{\gamma_0, 2\Lambda\}} \sqrt{\ln(\mathcal{N}(\mathcal{L}_{\mathrm{up}}, \eps))} \ \dd \eps.
    \end{align*}

    By \Cref{lem:old_bound}, we get, 
    for $0 < \eps \leq \Lambda$,
    \[
    \ln\left(\mathcal{N}(\mathcal{L}_{\mathrm{up}}, \eps)\right) \leq  C_8 \cdot d \cdot \ln(\Lambda / \eps) + C_8 \cdot Ld\ln(N/d)
    \]
    with an absolute constant $C_8 > 0$.
    If $\gamma_0 \leq \Lambda$, we obtain
    \begin{align*}
        &\norel\int_{0}^{\min\{\gamma_0, 2\Lambda\}} \sqrt{\ln(\mathcal{N}(\mathcal{L}_{\mathrm{up}}, \eps))} \ \dd \eps \\
        &=\int_{0}^{\gamma_0} \sqrt{\ln(\mathcal{N}(\mathcal{L}_{\mathrm{up}}, \eps))} \ \dd \eps \\
        &\leq \sqrt{C_8} \cdot \gamma_0 \cdot \sqrt{Ld\ln(N/d)} + \sqrt{C_8} \cdot \sqrt{d} \cdot \int_{0}^{\gamma_0} \sqrt{\ln(\ee\Lambda / \eps)} \ \dd \eps \\
        \overset{\text{\cshref{lem:int_b}}}&{\leq } \sqrt{C_8} \cdot \gamma_0 \cdot \sqrt{Ld\ln(N/d)} + 2\sqrt{C_8} \cdot \sqrt{d} \cdot \gamma_0 \cdot \sqrt{\ln(\ee\Lambda/\gamma_0)} \\
        &\leq \sqrt{C_8} \cdot \sqrt{d} + 2\sqrt{C_8} \cdot \frac{\sqrt{d}}{\sqrt{L \ln(N/d)}} \cdot \sqrt{\ln(\Lambda L \ln(N/d))}
        \leq C_9 \cdot \sqrt{d}
    \end{align*}
    with an absolute constant $C_9 > 0$.
    Note that the last inequality holds with probability at least $1- \exp(-c_3 \cdot NL^{-2}\ln^{-1}(\ee L))$
    according to \eqref{eq:help_2}.
    If $\gamma_0 > \Lambda >0$, we obtain 
    \begin{align*}
        &\norel\int_{0}^{\min\{\gamma_0, 2\Lambda\}} \sqrt{\ln(\mathcal{N}(\mathcal{L}_{\mathrm{up}}, \eps))} \ \dd \eps \\
        &\leq \int_{0}^{\Lambda} \sqrt{\ln(\mathcal{N}(\mathcal{L}_{\mathrm{up}}, \eps))} \ \dd \eps + \int_{\Lambda}^{\gamma_0} \sqrt{\ln(\mathcal{N}(\mathcal{L}_{\mathrm{up}}, \eps))} \ \dd \eps\\
        &\leq \sqrt{C_8} \cdot \Lambda \cdot \sqrt{Ld\ln(N/d)} + \sqrt{C_8} \cdot \sqrt{d} \cdot \int_{0}^{\Lambda} \sqrt{\ln(\Lambda / \eps)} \ \dd \eps + \gamma_0 \cdot \sqrt{\ln(\mathcal{N}(\mathcal{L}_{\mathrm{up}}, \Lambda))}\\
        &\leq \sqrt{C_8} \cdot \gamma_0 \cdot \sqrt{Ld\ln(N/d)} + \sqrt{C_8} \cdot \sqrt{d} \cdot \int_{0}^{\Lambda} \sqrt{\ln(\ee\Lambda / \eps)} \ \dd \eps + \sqrt{C_8} \cdot \gamma_0 \cdot \sqrt{Ld\ln(N/d)}\\
        \overset{\text{\cshref{lem:int_b}}}&{\leq } \sqrt{C_8} \cdot \sqrt{d} + 2\sqrt{C_8} \cdot \sqrt{d} + \sqrt{C_8} \cdot \sqrt{d} \\
        &\leq C_9 \cdot \sqrt{d}
    \end{align*}
    after possibly increasing the constant $C_9$.
    Lastly, if $\Lambda = 0$ then $\mathcal{L}_{\mathrm{up}}= \{0\}$
    and, thus, $\mathcal{\ln(\mathcal{N}(\mathcal{L}_{\mathrm{up}}, \eps))} = 0$ for all $\eps >0$.

    Overall, we thus observe that there exists an absolute constant $C_{10}>0$ such that 
    for every $u \geq 0$,
    \begin{equation}\label{eq:u}
    \underset{x \in \SS^{d-1}}{\sup} \mnorm{\nablaa \Phi(x)}_2
    \leq C_{10} \cdot (\sqrt{d \ln(\ee L)} + u)
    \end{equation}
    on an event of probability at least 
    \begin{align*}
        &\norel 1-2\exp(-u^2)-\exp(-c_3 \cdot NL^{-2}\ln^{-1}(\ee L)) - \exp\big(-c_7 \cdot N L^{-3}\ln^{-1}(N/d)\ln^{-1}(L \ln(N/d))\big) \\
        &\geq 1-2\exp(-u^2)-2\exp\big(-c_7 \cdot N L^{-3}\ln^{-1}(N/d)\ln^{-1}(L \ln(N/d))\big).
    \end{align*}
    We plug in $u = \sqrt{d \ln(\ee L)}$ 
    and use that our assumptions imply
    \[
    N \geq C \cdot dL^3 \ln(N/d)\ln(L \ln(N/d)) \ln(\ee L) \quad \text{and} \quad d\ln(\ee L) \geq C
    \]
    to obtain the claim. 
\end{proof}

%% file: sharp_lipschitz_section.tex
\section{The lower bound for zero-bias networks}
\label{sec:lower_proof}

In this section, we prove \Cref{thm:main_lower}. 
More precisely, we establish a lower bound that matches the one from \Cref{corr:lp_upper}
under a
polynomial condition on the width. Our argument uses directional finite
differences of the network instead of its gradients. After conditioning
on the hidden-layer weights, these finite differences form a Gaussian
process. We compare its covariance, on a suitable finite set of inputs,
with an explicitly computable iterated ReLU kernel and apply Sudakov's
minoration inequality.

The proof of the lower bound is based on the following directional
statement. The supremum in this statement is over a deterministic
equator of the unit sphere and over $h >0$; in particular, it does not require selecting
points of differentiability of the network.
We note that this is a restatement of \Cref{main:sharp:prop:direction}.

\begin{proposition}\label{sharp:prop:direction}
 There exist absolute constants $C,c>0$ such that the following holds.
 Let $\Phi: \RR^d \to \RR$ be a random zero-bias ReLU network with $L$ 
 hidden layers of width $N$ following \Cref{assum:1}.
 Suppose that $d \geq 3$, $d \ln(\ee L) \geq C$ and
 $N\geq C \cdot d L^6\ln(\ee L)$. 
 Fix $\nu\in\SS^{d-1}$.
 Then, with probability at least
 $1-\exp(-c d\ln(\ee L))$,
 \begin{equation}\label{sharp:eq:direction-result}
  \sup_{x\in\SS^{d-1}\cap\nu^\perp, h>0}
  \frac{\Phi(x+h\nu)-\Phi(x-h\nu)}{2h}
  \geq c\sqrt{d\ln(\ee L)}.
 \end{equation}
 The constants and the probability estimate are independent of $\nu$.
 The statement concerns each fixed $\nu$; it does not assert a single
 event on which \eqref{sharp:eq:direction-result} holds for all $\nu$.
\end{proposition}

Indeed, by considering two specific choices of $\nu$, 
\Cref{sharp:prop:direction} directly implies \Cref{thm:main_lower}.
\begin{proof}[Proof of \Cref{thm:main_lower}]
For every deterministic unit vector
 $\nu \in \SS^{d-1}$, every $h>0$, every $x\in\RR^d$, and every $p\in[1,\infty]$,
 the definition of the Lipschitz constant gives
 \begin{equation}\label{sharp:eq:deterministic-secant-bound}
  \lip_p(\Phi)
  \geq\frac{|\Phi(x+h\nu)-\Phi(x-h\nu)|}{2h\mnorm{\nu}_p}
  \geq\frac{\Phi(x+h\nu)-\Phi(x-h\nu)}{2h\mnorm{\nu}_p}.
 \end{equation}
 Apply \Cref{thm:direction} to the two fixed unit vectors
 \[
  \nu^{(1)}:=e_1 \in \RR^d,\qquad
  \nu^{(2)}:=d^{-1/2}(1,\ldots,1) \in \RR^d.
 \]
 A union bound gives a common event of probability at least
 $1-\exp(-c \cdot d\ln(\ee L))$ for suitable constants $C,c >0$ on which the directional lower bound holds
 for both vectors.
 For $1\leq p\leq2$, use $\mnorm{\nu^{(1)}}_p=1$ in
 \eqref{sharp:eq:deterministic-secant-bound}; this gives
 $\lip_p(\Phi)\geq c \cdot \sqrt{d\ln(\ee L)}$.
 For $2\leq p\leq\infty$, use
 $\mnorm{\nu^{(2)}}_p=d^{1/p-1/2}$ to obtain
 \[
  \lip_p(\Phi)
  \geq c \cdot d^{1/2-1/p}\sqrt{d\ln(\ee L)}
  =c \cdot d^{1-1/p}\sqrt{\ln(\ee L)}.
 \]
 These deductions are deterministic once the two directional events
 hold. Thus the lower bounds hold simultaneously for the uncountable
 family of exponents $p\in[1,\infty]$, without any further union bound.
\end{proof}

Consequently, the rest of this section is devoted to proving \Cref{sharp:prop:direction}.
To this end, we note that for fixed $h >0$ and $\nu \in \SS^{d-1}$, we have 
\[
\sup_{x\in\SS^{d-1}\cap\nu^\perp}
  \frac{\Phi(x+h\nu)-\Phi(x-h\nu)}{2h} = 
  \sup_{x\in\SS^{d-1}\cap\nu^\perp}\skpr{(W^{(L)})^T}{\frac{\Phi^{(L-1)}(x+h\nu)-\Phi^{(L-1)}(x-h\nu)}{2h}}.
\]
If only the randomness over $W^{(L)}$ is considered, the latter is the supremum of a Gaussian
process and can be lower bounded using Sudakov's minoration inequality (see \cite[Theorem~7.4.1]{vershynin_high-dimensional_2018}).
Since the application of Sudakov's minoration inequality requires lower bounding packing numbers,
we seek to lower bound expressions of the form 
\[
\mnorm{\frac{\Phi^{(L-1)}(x+h\nu)-\Phi^{(L-1)}(x-h\nu)}{2h} - \frac{\Phi^{(L-1)}(y+h\nu)-\Phi^{(L-1)}(y-h\nu)}{2h}}_2^2.
\]
In order to control such expressions, we aim to approximate feature inner products, i.e., 
expressions of the form $\skpr{\Phi^{(\ell)}(u)}{\Phi^{(\ell)}(v)}$ by deterministic expressions,
which is the content of the following subsection. 
\subsection{Control of feature inner products}
\label{subsec:innerprod}
As already explained before, this subsection is devoted to establishing a deterministic
expression to approximate feature inner products of the form $\skpr{\Phi^{(\ell)}(x)}{\Phi^{(\ell)}(y)}$
for $x,y \in \RR^d \setminus \{0\}$. 
To this end, we first define the following auxiliary function. 
\begin{lemma}\label{lem:kappa}
      The function
            \begin{equation}\label{sharp:eq:kappa}
            \kappa: \quad [-1,1] \to [0,1], \quad
            \kappa(t)=\frac{\sqrt{1-t^2}+(\pi-\arccos (t))t}{\pi}
            \end{equation}
      is well-defined, increasing, 1-Lipschitz on $[-1,1]$ and satisfies 
      \begin{equation}\label{sharp:eq:kappa-derivatives}
      \kappa'(t)= 1 - \frac{\arccos(t)}{\pi} \quad \text{for } -1 < t < 1.
      \end{equation}
      Moreover, 
      \begin{equation}\label{sharp:eq:kappa-increment}
      \kappa(t) \geq t \quad \text{for } -1 \leq t \leq 1.
      \end{equation}
\end{lemma}
\begin{proof}
      Differentiating \eqref{sharp:eq:kappa} proves
      \eqref{sharp:eq:kappa-derivatives}. 
      In particular, since $\kappa'(t)>0$ for all $t \in (-1,1)$, we conclude that 
      $\kappa$ is increasing on $[-1,1]$. 
      Consequently, we have 
      \[
      0 = \kappa(-1) \leq \kappa(t) \leq \kappa(1) = 1 \quad \text{for all } t \in [-1,1],
      \]
      proving that $\kappa$ is well-defined. 
      Moreover, $\kappa$ is 1-Lipschitz since
      $0 < \kappa'(t) < 1$ for all $t \in (-1,1)$.

      Finally, for $t = \cos(\theta)$ with $\theta \in [0, \pi]$, we get 
      \begin{equation}\label{eq:integral}
            \kappa(\cos(\theta)) - \cos(\theta) = \frac{\sin(\theta) - \theta \cos(\theta)}{\pi}= \frac{1}{\pi} \cdot \left[\sin(s) - s \cos(s)\right]_0^\theta
            = \frac{1}{\pi}\int_0^\theta s \sin(s) \ \dd s \geq 0. \qedhere
      \end{equation}
\end{proof}
The relevance of $\kappa$ in the context of feature inner products of random ReLU neural networks
stems from the following well-known lemma. 
We remark that the claimed identity \eqref{sharp:eq:gaussian-relu} already 
appears (in a more general form) in the literature; see for instance \cite{NIPS2009_5751ec3e} and 
\cite[Equation~(11)]{lee2018deep}.
\begin{lemma}\label{lem:sharp:prod}
      Let $(G,H)$ be a centered jointly Gaussian pair with
      $\EE G^2=\EE H^2=1$. Then,
      \begin{equation}\label{sharp:eq:gaussian-relu}
      2\EE[\relu(G)\relu(H)]=\kappa(\EE[GH]),
      \end{equation}
      where $\kappa$ is as defined in \eqref{sharp:eq:kappa}.
\end{lemma}
\begin{proof}
      We first note that by Cauchy-Schwarz, 
      \[
      \abs{\EE[GH]} \leq \sqrt{\EE[G^2]}\cdot \sqrt{\EE[H^2]} = 1,
      \]
      so that $\kappa(\EE[GH])$ is well-defined. 
      We write $\EE[GH]=\cos\theta$, where $\theta\in[0,\pi]$. 
      If $X,Y$ are independent
      standard Gaussian random variables, then $(G,H)$ has the same
      distribution as $(X,X\cos\theta+Y\sin\theta)$.\footnote{To see this, we first note that $Z \defeq (X,X\cos\theta+Y\sin\theta)$
      is a centered jointly Gaussian pair. Moreover, we compute $\EE[X^2] = 1$, 
      $\EE[(X\cos\theta+Y\sin\theta)^2]= \cos^2(\theta) + \sin^2(\theta)=1$
      and 
      \[
      \EE[X(X\cos\theta+Y\sin\theta)] = \cos(\theta)\EE[X^2]+\sin(\theta)\EE[XY]= \cos(\theta).
      \]
      Consequently, $(G,H)$ and $Z$ have the same covariance matrix and
      by \cite[Remark~7.1.11]{vershynin_high-dimensional_2018}, they are equal in distribution.
      } 
      In polar coordinates,
      write $(X,Y)=(R\cos\alpha,R\sin\alpha)$, where $\alpha \sim \Unif([-\pi, \pi])$, 
      $R$ and $\alpha$ are independent, and
      $\EE R^2=2$ \cite[Ex.~3.3.7]{vershynin_high-dimensional_2018}.
      We then have 
      \begin{align*}
            (G,H) &\d (X,X\cos\theta+Y\sin\theta) 
            = (R\cos(\alpha), R \cos(\alpha)\cos(\theta) + R \sin(\alpha)\sin(\theta)) \\
            &= (R\cos(\alpha), R\cos(\alpha - \theta)),
      \end{align*}
      where the last equality uses the cosine difference identity. 
      Consequently, we obtain 
      \begin{align}
      \EE[\relu(G)\relu(H)] &= \EE[R^2 \relu(\cos(\alpha))\relu(\cos(\alpha - \theta))] \nonumber\\
      \label{eq:R^2} 
      &=  2 \cdot \EE[\relu(\cos(\alpha))\relu(\cos(\alpha - \theta))],
      \end{align}
      where the last equality uses the independence of $R$ and $\alpha$ and $\EE[R^2] = 2$.
      Moreover, if $\theta \leq \pi/2$, then
      \begin{align*}
      \relu(\cos(\alpha))\relu(\cos(\alpha - \theta)) &= 
      \mathbbm{1}_{\alpha \in [-\pi/2, \pi/2]} \cdot \cos(\alpha) \cdot \mathbbm{1}_{\alpha \in [-\pi/2+\theta, \pi/2 + \theta]} \cdot \cos(\alpha - \theta) \\
      &=\mathbbm{1}_{\alpha \in [\theta - \pi/2, \pi/2]} \cdot \cos(\alpha) \cdot \cos(\alpha - \theta)
      \end{align*}
      and if $\theta > \pi/2$, then 
      \begin{align*}
      \relu(\cos(\alpha))\relu(\cos(\alpha - \theta)) &= 
      \mathbbm{1}_{\alpha \in [-\pi/2, \pi/2]} \cdot \cos(\alpha) \cdot \mathbbm{1}_{\alpha \in [-\pi, -3\pi/2 + \theta]\cup [-\pi/2 + \theta, \pi ]} \cdot \cos(\alpha - \theta) \\
      &=\mathbbm{1}_{\alpha \in [\theta - \pi/2, \pi/2]} \cdot \cos(\alpha) \cdot \cos(\alpha - \theta).
      \end{align*}
      We thus obtain 
      \begin{align*}
            \EE[\relu(\cos(\alpha))\relu(\cos(\alpha - \theta))] 
            &= \frac{1}{2\pi} \cdot \int_{\theta - \pi/2}^{\pi/2} \cos(\alpha)\cos(\alpha - \theta)\ \dd \alpha \\
            &= \frac{1}{2\pi} \cdot \int_{\theta - \pi/2}^{\pi/2} \frac{1}{2}\left(\cos(\theta)+ \cos(2\alpha - \theta)\right)\ \dd \alpha \\
            &=\frac{1}{2\pi} \cdot\left[ \frac{\alpha}{2}\cos(\theta) + \frac{1}{4}\sin(2\alpha - \theta)\right]_{\theta - \pi/2}^{\pi/2} \\
            &= \frac{1}{4\pi} \!\cdot \!\left(\frac{\pi}{2}\cos(\theta)\!+\!\frac{1}{2}\sin(\pi \!-\! \theta)\!-\! (\theta \!-\! \pi/2)\cos(\theta) \!-\! \frac{1}{2}\sin(\theta \!-\! \pi)\right)\\
            &= \frac{1}{4\pi}\cdot \left((\pi - \theta)\cos(\theta) + \sin(\pi - \theta)\right) \\
            &= \frac{(\pi - \theta)\cos(\theta) + \sin(\theta)}{4\pi}.
      \end{align*}
      Combining this with \eqref{eq:R^2}, we obtain 
      \[
          2\cdot \EE[\relu(G)\relu(H)] =   4 \cdot \EE[\relu(\cos(\alpha))\relu(\cos(\alpha - \theta))] 
          = \kappa(\cos(\theta)) = \kappa(\EE[GH]). \qedhere
      \]
\end{proof}
We can now establish the desired approximation
of the feature inner product. 
To this end, for $t \in [-1,1]$, we set 
\begin{equation}\label{eq:qj}
q_{-1}(t):=t \quad \text{and} \quad q_\ell:=\kappa\circ q_{\ell-1} \quad \text{for} \quad  \ell\geq0.
\end{equation}
For nonzero $x,y\in\RR^d \setminus \{0\}$, define
\begin{equation}\label{sharp:eq:Kj}
 K_\ell(x,y):=\mnorm{x}_2\mnorm{y}_2\,
 q_\ell\left(\frac{\skpr{x}{y}}{\mnorm{x}_2\mnorm{y}_2}\right),
 \qquad \ell\geq -1.
\end{equation}
We then get the following result. 
\begin{proposition}\label{prop:kl}
      There exists an absolute constant $c>0$ with the following property. 
      Let $\Phi: \RR^d \to \RR$ be a random zero-bias ReLU network with $L$ hidden layers of width $N$
      following \Cref{assum:1}, let $x,y \in \RR^d \setminus \{0\}$ and let $\eps \in (0, \ln(2)/(2L))$.
      Then, with probability at least $1-8L\exp(-c \cdot \eps^2 N)$,
      \[
            \abs{\skpr{\Phi^{(\ell)}(x)}{\Phi^{(\ell)}(y)} - K_\ell(x,y)} \leq 8(\ell+1) \eps \cdot \mnorm{x}_2 \mnorm{y}_2
      \]
      holds simultaneously for all $\ell \in \{-1, \ldots, L-1\}$.
\end{proposition}
\begin{proof}
      We first introduce the following useful notation that will be used throughout the proof: 
      For $\ell \in \{-1, \ldots, L-1\}$, we set 
      \[
            r_\ell(x,y) \defeq \begin{cases}
            \frac{\skpr{\Phi^{(\ell)}(x)}{\Phi^{(\ell)}(y)}}{\mnorm{\Phi^{(\ell)}(x)}_2\mnorm{\Phi^{(\ell)}(y)}_2}&\text{if } \Phi^{(\ell)}(x) \neq 0 \neq \Phi^{(\ell)}(y), \\
            0&\text{otherwise.}
            \end{cases}
      \]
      Moreover, for $\ell \in \{0, \ldots, L-1\}$ we define
      \[
            Q_\ell(x) \defeq \begin{cases}
                  \frac{\mnorm{\Phi^{(\ell)}(x)}_2^2}{\mnorm{\Phi^{(\ell-1)}(x)}_2^2}&\text{if } \Phi^{(\ell-1)}(x) \neq 0, \\
                  0& \text{otherwise}
            \end{cases}
      \]
      and $Q_\ell(y)$ analogously. 

      We now aim to show the following statement:
      \textit{There exists an absolute constant $c>0$ such that
      for every $\ell \in \{-1, \ldots, L-1\}$, with probability at least 
      $1-8(\ell+1)\exp(-c \cdot \eps^2 N)$, the following five properties are simultaneously satisfied for all $i \in \{-1, \ldots, \ell\}$
      and $j \in \{0, \ldots, \ell\}$:
      \begin{enumerate}
            \item $\Phi^{(i)}(x), \Phi^{(i)}(y) \neq 0$,
            \item $\max\{\abs{Q_j(x) - 1},\abs{Q_j(y) - 1} \} \leq \eps$,
            \item $\abs{r_j(x,y) - \kappa \left(r_{j-1}(x,y)\right)} \leq 4\eps,$
            \item $\abs{r_i(x,y) -q_i\left(r_{-1}(x,y)\right)} \leq 4(i+1) \eps,$
            \item \label{item:sharp:5}$\abs{K_i(x,y) - \skpr{\Phi^{(i)}(x)}{\Phi^{(i)}(y)}} \leq 8(i+1)\eps \cdot \mnorm{x}_2\mnorm{y}_2.$
      \end{enumerate} 
      }
            The proof is via induction over $\ell$. There is nothing to show in the case $\ell = -1$. 
            For the induction step, we fix $\ell \in \{0, \ldots, L-1\}$ and assume that the claim is 
            already satisfied by $\ell-1$.
            We condition on $W^{(0)}, \ldots, W^{(\ell-1)}$ and assume that for all $i \in \{-1, \ldots, \ell-1\}$
            and $j \in \{0, \ldots, \ell-1\}$ the five properties
            \begin{enumerate}[label=(\alph*)]
            \item $\Phi^{(i)}(x), \Phi^{(i)}(y) \neq 0$,
            \item \label{item:sharp:b}$\max\{\abs{Q_j(x) - 1},\abs{Q_j(y) - 1} \} \leq \eps$,
            \item $\abs{r_j(x,y) - \kappa(r_{j-1}(x,y))} \leq 4\eps$,
            \item \label{sharp:item:d}$\abs{r_i(x,y) -q_i\left(r_{-1}(x,y)\right)} \leq 4(i+1) \eps,$
            \item $\abs{K_i(x,y) - \skpr{\Phi^{(i)}(x)}{\Phi^{(i)}(y)}} \leq 8(i+1)\eps \cdot \mnorm{x}_2 \mnorm{y}_2$
            \end{enumerate}
            are simultaneously satisfied. 
            According to the induction hypothesis, this happens with probability at least 
            $1-8\ell\exp(-c \cdot \eps^2N)$, where the absolute constant $c>0$ is exactly determined later.

            Firstly, since 
            \[
                  \Phi^{(\ell)}(x) = 0 \quad \Leftrightarrow \quad \left(W^{(\ell)}\Phi^{(\ell-1)}(x)\right)_i \leq 0 \quad \text{for all } i \in \{1, \dots, N\}
            \]
            and since $\left(W^{(\ell)}\Phi^{(\ell-1)}(x)\right)_i \iid \mathcal{N}(0, 2 \mnorm{\Phi^{(\ell-1)}(x)}_2^2/N)$, we get $\Phi^{(\ell)}(x) \neq 0$
            with probability at least $1- \frac{1}{2^N}$. 
            Via a union bound, we obtain that with probability at least $1- \frac{1}{2^{N-1}}$, 
            \begin{equation}\label{eq:xyneq}
                 \Phi^{(\ell)}(x)\neq 0 \neq  \Phi^{(\ell)}(y) 
            \end{equation}
            and call the event on which \eqref{eq:xyneq} is satisfied $\E_1$.

            For $k \in \{1, \ldots, N\}$, we set 
            \[
                  G_k(x)\defeq \sqrt{\frac{N}{2}} \cdot \frac{\left(W^{(\ell)}\Phi^{(\ell-1)}(x)\right)_k}{\mnorm{\Phi^{(\ell-1)}(x)}_2}
                  \quad \text{and} \quad
                  G_k(y)\defeq \sqrt{\frac{N}{2}} \cdot \frac{\left(W^{(\ell)}\Phi^{(\ell-1)}(y)\right)_k}{\mnorm{\Phi^{(\ell-1)}(y)}_2}.
            \]
            By construction, $(G_k(x), G_k(y))$ is a jointly Gaussian pair with unit marginal variances and
            \[
            \EE[G_k(x)G_k(y)] = \frac{\skpr{\Phi^{(\ell-1)}(x)}{\Phi^{(\ell-1)}(y)}}{\mnorm{\Phi^{(\ell-1)}(x)}_2 \mnorm{\Phi^{(\ell-1)}(y)}_2} = r_{\ell-1}(x,y).
            \]
            By \Cref{lem:sharp:prod}, 
            \begin{equation}\label{eq:reluprod}
                  2\EE\left[\relu(G_k(x))\relu(G_k(y))\right]= \kappa(r_{\ell-1}(x,y)).
            \end{equation}
            We now set 
            \[
                  B(x,y) \defeq \frac{\skpr{ \Phi^{(\ell)}(x)}{ \Phi^{(\ell)}(y)}}{\mnorm{\Phi^{(\ell-1)}(x)}_2\mnorm{\Phi^{(\ell-1)}(y)}_2} = \frac{2}{N} \sum_{k=1}^N\relu(G_k(x))\relu(G_k(y))
            \]
            and observe that we can rewrite 
            \begin{align*}
                  &\norel\abs{B(x,y) - \kappa(r_{\ell-1}(x,y))} \\
                  &= \frac{1}{N} \cdot \abs{\sum_{k=1}^N \Big(2\relu(G_k(x))\relu(G_k(y)) - \kappa(r_{\ell-1}(x,y))\Big)} \\
                  \overset{\eqref{eq:reluprod}}&{=}\frac{1}{N} \!\cdot\! \abs{\sum_{k=1}^N \Big(2\relu(G_k(x))\relu(G_k(y)) \!- \!\EE[2\relu(G_k(x))\relu(G_k(y))]\Big)}.
            \end{align*}
            Since both $G_k(x)$ and $G_k(y)$ are standard Gaussian and since $\relu(t) \leq \abs{t}$ for all $t \in \RR$,
            we obtain 
            \begin{align*}
            \mnorm{2\relu(G_k(x))\relu(G_k(y))}_{\psi_1} \overset{\text{\cite[Lemma~2.7.7]{vershynin_high-dimensional_2018}}}&{\leq} 2 \cdot \mnorm{\relu(G_k(x))}_{\psi_2}\mnorm{\relu(G_k(y))}_{\psi_2} \\
            &\leq 2 \cdot \mnorm{G_k(x)}_{\psi_2}\mnorm{G_k(y)}_{\psi_2} \leq C_1,
            \end{align*}
            where $C_1 > 0$ is an absolute constant. 
            The centering lemma for sub-exponential random variables \cite[Ex.~2.7.10]{vershynin_high-dimensional_2018}
            establishes that 
            \[
                 \mnorm{2\relu(G_k(x))\relu(G_k(y)) - \EE[2\relu(G_k(x))\relu(G_k(y))]}_{\psi_1} \leq C_2 
            \]
            with an absolute constant $C_2 > 0$.
            We can thus use Bernstein's inequality for sums of sub-exponential, independent and centered random variables
            \cite[Theorem~2.8.1]{vershynin_high-dimensional_2018}
            and obtain the existence of an absolute constant $c_1 > 0$ with the property that
            for every $t \geq 0$ with probability at least $1-2\exp(-c_1 \cdot \min(t^2 /N, t))$, 
            \[
                  \abs{B(x,y) - \kappa(r_{\ell-1}(x,y))} \leq t/N.
            \]
            Hence, by explicitly picking $t = \eps N$, we get 
            \begin{equation}\label{eq:bern}
                  \abs{B(x,y) - \kappa(r_{\ell-1}(x,y))} \leq \eps
            \end{equation}
            with probability at least $1-2\exp(-c_1 \cdot \eps^2 N)$.
            We call the event on which \eqref{eq:bern} is satisfied $\E_2$.

            Moreover, note that 
            \begin{align*}
                  \abs{Q_\ell(x) - 1} &= \frac{1}{N} \cdot \abs{\sum_{k=1}^N \left(2 (\relu(G_k(x)))^2 - 1\right)} \\
                  &= \frac{1}{N} \cdot \abs{\sum_{k=1}^N \Big(2 (\relu(G_k(x)))^2 - \EE[2 (\relu(G_k(x)))^2]\Big)},
            \end{align*}
            where we used $\EE[2 (\relu(G_k(x)))^2]= 1$ for every $k \in \{1, \ldots, N\}$, as follows from the symmetry of $G_k(x)$.
            By an argument analogous to the one above, we can use Bernstein's inequality and get that with probability 
            at least $1-2\exp(-c_1 \cdot \eps^2 N)$,
            \[
            \abs{Q_\ell(x) - 1} \leq \eps.
            \]
            Using a union bound we obtain that with probability at least $1-4\exp(-c_1 \cdot \eps^2 N)$,
            \begin{equation}\label{eq:max}
                  \max\{\abs{Q_\ell(x) - 1}, \abs{Q_\ell(y) - 1}\} \leq \eps
            \end{equation}
            and define $\E_3$ to be the event on which \eqref{eq:max} is satisfied. 

            We now claim that on the intersection $\E_1 \cap \E_2 \cap \E_3$ the five properties
            \begin{enumerate}[label=(\roman*)]
                  \item \label{item:sharp:i}$\Phi^{(\ell)}(x), \Phi^{(\ell)}(y) \neq 0$,
                  \item \label{item:sharp:ii}$\max\{\abs{Q_\ell(x) - 1},\abs{Q_\ell(y) - 1} \} \leq \eps$,
                  \item \label{item:sharp:iii}$\abs{r_\ell(x,y) - \kappa(r_{\ell-1}(x,y))} \leq 4\eps$,
                  \item \label{item:sharp:iv}$\abs{r_\ell(x,y) -q_\ell\left(r_{-1}(x,y)\right)} \leq 4(\ell+1) \eps,$
                  \item \label{item:sharp:v}$\abs{K_\ell(x,y) - \skpr{\Phi^{(\ell)}(x)}{\Phi^{(\ell)}(y)}} \leq 8(\ell+1)\eps \cdot \mnorm{x}_2\mnorm{y}_2$
            \end{enumerate}
            are satisfied. 
            
            Firstly, note that \ref{item:sharp:i} and \ref{item:sharp:ii} are fulfilled by the definition of $\E_1$ and $\E_3$, respectively. 
            To see \ref{item:sharp:iii}, note that
            \begin{align*}
                  r_\ell(x,y)=\frac{\skpr{\Phi^{(\ell)}(x)}{\Phi^{(\ell)}(y)}}{\mnorm{\Phi^{(\ell)}(x)}_2\mnorm{\Phi^{(\ell)}(y)}_2} = \frac{B(x,y)}{\sqrt{Q_\ell(x)Q_\ell(y)}}.
            \end{align*}
            By \ref{item:sharp:ii}, we have $\sqrt{Q_\ell(x)Q_\ell(y)} \in [1-\eps, 1+\eps]$.
            Consequently, we obtain 
            \begin{align*}
                  \abs{r_{\ell}(x,y) - \kappa(r_{\ell-1}(x,y))}
                  &= \abs{\frac{B(x,y) - \sqrt{Q_\ell(x)Q_\ell(y)} \kappa(r_{\ell-1}(x,y))}{\sqrt{Q_\ell(x)Q_\ell(y)}}} \\
                  &\leq \frac{\abs{B(x,y) - \kappa(r_{\ell-1}(x,y))}}{\sqrt{Q_\ell(x)Q_\ell(y)}} + \abs{\frac{1}{\sqrt{Q_\ell(x)Q_\ell(y)}}-1} \abs{\kappa(r_{\ell-1}(x,y))} \\
                  \overset{\E_2 \cap \E_3}&{\leq} \frac{\eps}{1- \eps} + \frac{\eps}{1-\eps} \leq 4\eps.
            \end{align*}
            Here, the penultimate step used that $\kappa(t) \leq 1$ for every $t \in [-1,1]$ (see \Cref{lem:kappa}) and 
            the last step used $\eps \leq 1/2$.
            
            For \ref{item:sharp:iv}, we compute 
            \begin{align*}
            \abs{r_{\ell}(x,y) -q_\ell\left(r_{-1}(x,y)\right)}
            &\leq \abs{r_{\ell}(x,y) - \kappa \left(r_{\ell-1}(x,y)\right)} + \abs{\kappa \left(r_{\ell-1}(x,y)\right) - q_\ell\left(r_{-1}(x,y)\right)} \\
            &\leq 4\eps + \abs{\kappa \left(r_{\ell-1}(x,y)\right) - \kappa\left(q_{\ell-1}\left(r_{-1}(x,y)\right)\right)}.
            \end{align*}
            Since $\kappa$ is 1-Lipschitz as shown in \Cref{lem:kappa}, we get 
            \[
                  \abs{\kappa \left(r_{\ell-1}(x,y)\right) - \kappa\left(q_{\ell-1}\left(r_{-1}(x,y)\right)\right)}
                  \leq \abs{r_{\ell-1}(x,y) - q_{\ell-1}\left(r_{-1}(x,y)\right)} \leq 4\ell\eps,
            \]
            where the last inequality follows from \ref{sharp:item:d}.
            This proves \ref{item:sharp:iv}.

            It thus remains to show \ref{item:sharp:v}.
            Note that 
            \begin{align*}
                  \frac{\mnorm{\Phi^{(\ell)}(x)}_2\mnorm{\Phi^{(\ell)}(y)}_2}{\mnorm{x}_2\mnorm{y}_2} = \prod_{i=0}^{\ell}\sqrt{Q_i(x)Q_i(y)}.
            \end{align*}
            We recall from \ref{item:sharp:b} and \ref{item:sharp:ii} that $\sqrt{Q_i(x)Q_i(y)} \in [1-\eps, 1+\eps]$ for every $i \in \{0, \ldots, \ell\}$.
            We thus obtain
            \begin{align*}
                  \abs{\ln\left(\frac{\mnorm{\Phi^{(\ell)}(x)}_2\mnorm{\Phi^{(\ell)}(y)}_2}{\mnorm{x}_2\mnorm{y}_2}\right)}
                  &\leq \sum_{i=0}^\ell \abs{\ln(\sqrt{Q_i(x)Q_i(y)})} = \sum_{i=0}^\ell \abs{\ln(1+\sqrt{Q_i(x)Q_i(y)}-1)} \\
                  &\leq 2 \cdot \sum_{i=0}^\ell \abs{\sqrt{Q_i(x)Q_i(y)}-1} \leq 2(\ell + 1)\eps.
            \end{align*}
            In the third step, we used that $\eps \leq \ln(2)/(2L) \leq 1/2$ and $\abs{\ln(1+z)} \leq 2\abs{z}$ for $\abs{z} \leq 1/2$.\footnote{
                  To see this, we note that the case $z =0$ is trivial. For $z \neq 0$, note that 
                  \[
                        \int_0^1\frac{1}{1+tz} \ \dd t = \left[\frac{1}{z}\ln(1+tz)\right]_0^1 = \frac{\ln(1+z)}{z}.
                  \]
                  Hence, we get 
                  \[
                        \abs{\frac{\ln(1+z)}{z}} \leq \int_0^1 \frac{1}{\abs{1+tz}} \ \dd t \leq 
                        \int_0^1 \frac{1}{1 - t\abs{z}} \ \dd t \leq \int_0^1 2 \ \dd t = 2,
                  \]
                  where the last inequality uses $t \in [0,1]$ and $\abs{z} \leq 1/2$.
            } 
            We obtain 
            \begin{align*}
                  \frac{\mnorm{\Phi^{(\ell)}(x)}_2\mnorm{\Phi^{(\ell)}(y)}_2}{\mnorm{x}_2\mnorm{y}_2}
                  \leq \exp\left(\abs{\ln\left(\frac{\mnorm{\Phi^{(\ell)}(x)}_2\mnorm{\Phi^{(\ell)}(y)}_2}{\mnorm{x}_2\mnorm{y}_2}\right)}\right) 
                  \leq \exp(2L\eps) \leq 2,
            \end{align*}
            where the last step uses the assumption $\eps \leq \ln(2)/(2L)$.
            Note that we have $\abs{\delta -1} \leq 2\abs{\ln(\delta)}$ for $\delta \in (0,2]$.\footnote{
                  If $\delta \in (0,1]$, we have 
                  \[
                  \abs{\delta - 1} = 1 - \delta \leq - \ln(\delta) = \abs{\ln(\delta)} \leq 2\abs{\ln(\delta)},
                  \]
                  as follows from the well-known estimate $\ln(t) \leq t -1$ for $t>0$.
                  For $\delta \in (1,2]$, we note that according to the mean value theorem, there exists $\xi \in (1, \delta) $,
                  such that
                  \[
                        \frac{\ln(\delta)}{\delta - 1} = \frac{\ln(\delta) - \ln(1)}{\delta - 1} = \frac{1}{\xi}\geq 1/2,
                  \]
                  where the last step uses $\xi \leq 2$.
                  Consequently, 
                  \[
                  \abs{\delta - 1} = \delta -1 \leq 2 \ln(\delta) = 2 \abs{\ln(\delta)}.
                  \]
            }
            Hence, 
            \begin{align}\label{eq:closetoone}
                  \abs{\frac{\mnorm{\Phi^{(\ell)}(x)}_2\mnorm{\Phi^{(\ell)}(y)}_2}{\mnorm{x}_2\mnorm{y}_2} - 1} 
                  \leq 2\cdot \abs{\ln\left(\frac{\mnorm{\Phi^{(\ell)}(x)}_2\mnorm{\Phi^{(\ell)}(y)}_2}{\mnorm{x}_2\mnorm{y}_2}\right)} \leq 4(\ell+1)\eps.
            \end{align}
            Finally, we use
            \begin{align*}
                  &\norel\abs{\skpr{\Phi^{(\ell)}(x)}{\Phi^{(\ell)}(y)} - K_\ell(x,y)} \\
                  &= \abs{\skpr{\Phi^{(\ell)}(x)}{\Phi^{(\ell)}(y)} - \mnorm{x}_2\mnorm{y}_2 q_\ell(r_{-1}(x,y))} \\
                  &= \mnorm{x}_2\mnorm{y}_2 \cdot \abs{\frac{\skpr{\Phi^{(\ell)}(x)}{\Phi^{(\ell)}(y)}}{\mnorm{x}_2\mnorm{y}_2} - q_\ell(r_{-1}(x,y))} \\
                  &\leq \mnorm{x}_2\mnorm{y}_2 \cdot \left[\underbrace{\abs{r_\ell(x,y)}}_{\leq 1} \cdot \underbrace{\abs{\frac{\mnorm{\Phi^{(\ell)}(x)}_2\mnorm{\Phi^{(\ell)}(y)}_2}{\mnorm{x}_2\mnorm{y}_2} - 1}}_{\leq 4(\ell+1)\eps} + \underbrace{\abs{r_\ell(x,y)- q_\ell(r_{-1}(x,y))}}_{\leq 4(\ell+1)\eps}\right] \\
                  \overset{\text{\ref{item:sharp:iv},\eqref{eq:closetoone}}}&{\leq} 8(\ell+1)\eps \cdot \mnorm{x}_2\mnorm{y}_2.
            \end{align*}
            This proves \ref{item:sharp:v}.

            We note that 
            \begin{align*}
            \PP(\E_1 \cap \E_2 \cap \E_3) &\geq 1 - \frac{1}{2^{N-1}} - 6\exp(-c_1 \cdot \eps^2N)
            \geq 1 - 2\exp(- \ln(2) \cdot N)- 6\exp(-c_1 \cdot \eps^2 N) \\
            &\geq 1 - 8\exp(-c_1 \cdot \eps^2 N),
            \end{align*}
            where the last step uses that we may take $c_1 \leq 4\ln(2) \leq \ln(2)/\eps^2$.
            Up to now we have conditioned on $W^{(0)}, \ldots, W^{(\ell-1)}$. 
            Reintroducing the randomness over these random variables we obtain the desired claim by picking $c= c_1$.

            The final claim follows with the specific choice of $\ell = L-1$ by considering property \eqref{item:sharp:5}.
\end{proof}

\subsection{\texorpdfstring{Proof of \Cref{sharp:prop:direction}}{Proof of Proposition F.1}}
\label{subsec:technical}
Having established an approximation on the feature inner products in the previous subsection,
we can now prove \Cref{sharp:prop:direction}.
For $h > 0$ and $L \in \NN$ we define 
\begin{equation}\label{sharp:eq:Gamma}
 \Gamma_{L,h}(\rho):=\frac{1+h^2}{2h^2}
 \left[
  q_{L-1}\left(\frac{\rho+h^2}{1+h^2}\right)
  -q_{L-1}\left(\frac{\rho-h^2}{1+h^2}\right)
 \right],\qquad -1\leq\rho\leq1,
\end{equation}
where $q_L$ is as defined in \eqref{eq:qj}.
Using \Cref{prop:kl} we can show that $\Gamma_{L,h}$ serves as a deterministic 
estimator for the inner product of two finite difference vectors. 
\begin{lemma}\label{lem:gammalh}
There exists an absolute constant $c>0$ with the following property. 
Let $\Phi: \RR^d \to \RR$ be a random zero-bias ReLU neural network with $L$ hidden layers of width $N$
following \Cref{assum:1}.
Fix $x,y, \nu \in \SS^{d-1}$ with $x,y \in \nu^{\perp}$ and $h, \eps >0$
with $\eps < \ln(2)/(2L)$.
Set 
\begin{equation}\label{eq:vx}
V_x\defeq \frac{\Phi^{(L-1)}(x+h\nu)-\Phi^{(L-1)}(x-h\nu)}{2h}
       \in\RR^N
\end{equation}
and define $V_y \in \RR^N$ analogously.
Then, with probability at least $1-32L\exp(-c \cdot \eps^2 N)$, 
\[
\abs{\skpr{V_x}{V_y} - \Gamma_{L,h}(\skpr{x}{y})} \leq \frac{8L(1+h^2)}{h^2} \cdot \eps.
\]
\end{lemma}
\begin{proof}
Firstly, note that 
\[
\skpr{V_x}{V_y}= \frac{1}{4h^2} \cdot \sum_{\sigma, \tau \in \{-1,1\}} \sigma\tau \cdot \skpr{\Phi^{(L-1)}(x+\sigma h \nu)}{\Phi^{(L-1)}(y+\tau h \nu)}.
\]
The idea is now to approximate each of the four inner products using \Cref{prop:kl}.
Since $x,y \in \nu^\perp$, we note that for each choice of $\sigma, \tau \in \{-1,1\}$ we have 
\[
\mnorm{x+\sigma h \nu}_2 = \sqrt{1 + h^2} = \mnorm{y+\tau h \nu}_2 \quad \text{and} \quad 
\skpr{x+\sigma h \nu}{y+\tau h \nu} = \skpr{x}{y} + \sigma \tau h^2.
\]
Consequently, using the definition of $K_{L-1}$ as given in \eqref{sharp:eq:Kj},
\begin{align*}
      K_{L-1}(x+\sigma h \nu, y+\tau h \nu) = (1+h^2) \cdot q_{L-1}\left(\frac{\skpr{x}{y} + \sigma \tau h^2}{1+h^2}\right).
\end{align*}
We thus get 
\begin{align*}
      &\norel\frac{1}{4h^2} \cdot \sum_{\sigma, \tau \in \{-1,1\}} \sigma\tau \cdot K_{L-1}\left(x+\sigma h \nu,y+\tau h \nu\right) \\
      &= \frac{1 + h^2}{4h^2} \cdot \sum_{\sigma, \tau \in \{-1,1\}} \sigma\tau \cdot q_{L-1}\left(\frac{\skpr{x}{y} + \sigma \tau h^2}{1+h^2}\right) \\
      &= \frac{1 + h^2}{2h^2} \cdot \left(q_{L-1}\left(\frac{\skpr{x}{y} + h^2}{1+h^2}\right) - q_{L-1}\left(\frac{\skpr{x}{y} -  h^2}{1+h^2}\right) \right) \\
      &= \Gamma_{L,h}(\skpr{x}{y}).
\end{align*}
Combining everything, we obtain 
\begin{align*}
&\norel\abs{\skpr{V_x}{V_y} - \Gamma_{L,h}(\skpr{x}{y})} \\
&\leq \frac{1}{4h^2} \cdot \!\!\! \sum_{\sigma, \tau \in \{-1,1\}} \abs{\skpr{\Phi^{(L-1)}(x+\sigma h \nu)}{\Phi^{(L-1)}(y+\tau h \nu)} - K_{L-1}\left(x+\sigma h \nu,y+\tau h \nu\right)}
\end{align*}
and each of the summands in the last expression can be controlled using \Cref{prop:kl}.
Indeed, for fixed $\sigma, \tau \in \{-1,1\}$ we obtain that with probability 
at least $1-8L\exp(-c \cdot \eps^2 N)$, 
\begin{align*}
\abs{\skpr{\Phi^{(L-1)}(x+\sigma h \nu)}{\Phi^{(L-1)}(y+\tau h \nu)}\! - \!K_{L-1}\left(x+\sigma h \nu,y+\tau h \nu\right)}
\leq 8L(1+h^2) \eps.
\end{align*}
Using a union bound, we conclude that with probability at least $1-32L\exp(-c \cdot \eps^2 N)$,
\[
      \abs{\skpr{V_x}{V_y} - \Gamma_{L,h}(\skpr{x}{y})} \leq \frac{8L(1+h^2)}{h^2} \cdot \eps, 
\]
as desired. 
\end{proof}
In particular, \Cref{lem:gammalh} allows us to approximate 
\begin{equation}\label{eq:teaser}
\mnorm{V_x - V_y}_2^2 = \mnorm{V_x}_2^2 + \mnorm{V_y}_2^2 - 2 \skpr{V_x}{V_y}
\approx \Gamma_{L,h}(1) + \Gamma_{L,h}(1) - 2 \Gamma_{L,h}(\skpr{x}{y}),
\end{equation}
where the latter is a fully deterministic expression. 
Consequently, we study properties of $\Gamma_{L,h}$ in the following. 
Since the definition of $\Gamma_{L,h}$ involves $q_{L-1}$, we first establish 
the following result. 
\begin{lemma}\label{lem:qjdev}
For $j \in \{-1,0,1, \ldots\}$ let $q_{j}$ be defined as in \eqref{eq:qj}.
Then, for $t \in (-1,1)$,
\[
q_j'(t) = \prod_{r = -1}^{j - 1} \kappa'\left(q_r(t)\right) = \prod_{r = -1}^{j - 1} \left(1 - \frac{\arccos(q_r(t))}{\pi}\right).
\]
Moreover, the continuous extension of $q_j'$ onto $[-1,1]$ is nondecreasing. 
\end{lemma}
\begin{proof}
      The proof is via induction over $j$, where there is nothing to show for $j = -1$.
      If $j \in \NN_0$ and $j-1$ already satisfies the claim, then the chain rule yields
      \begin{align*}
            q_j'(t) = \kappa'(q_{j-1}(t)) \cdot q_{j-1}'(t) = \kappa'(q_{j-1}(t)) \cdot \prod_{r = -1}^{j - 2} \kappa'\left(q_r(t)\right) = \prod_{r = -1}^{j - 1} \kappa'\left(q_r(t)\right),
      \end{align*}
      as desired. 
      For the second equality, we use that $\kappa'$ was already computed in \eqref{sharp:eq:kappa-derivatives}.

      For the second part of the claim, we first note that $q_r$ is nondecreasing for every $r \in \{-1,0,1, \ldots\}$, 
      since $\kappa$ is nondecreasing as was shown in \Cref{lem:kappa}.
      Consequently, for $t_1, t_2 \in [-1,1]$ with $t_1 \leq t_2$ and since $\arccos$ is nonincreasing, 
      \[
      \arccos(q_r(t_1)) \geq \arccos(q_r(t_2)).
      \]
      From the representation of $q_j'$ that was just shown we obtain the claim. 
\end{proof}
\begin{lemma}\label{lem:arccos}
Let $L \in \NN$ and $\delta \in (0, 3/(2L)]$.
Then,
\[
\arccos(q_r(\cos(\delta))) \geq \delta/2
\quad 
\text{for every }r \in \{-1, \ldots, L-1\},
\]
where $q_r$ is as defined in \eqref{eq:qj}.
\end{lemma}
\begin{proof}
      For $\theta \in [0, \pi/2]$, we set $\psi(\theta) \defeq \arccos(\kappa(\cos(\theta)))$.
      Since $\kappa(\cos(\theta)) \geq \cos(\theta)$ as follows from \eqref{sharp:eq:kappa-increment} 
      and since $\arccos$ is nonincreasing, we obtain $\psi(\theta) \leq \theta$.
      Moreover, we have 
      \begin{align*}
            \cos(\psi(\theta)) - \cos(\theta) = \kappa(\cos(\theta)) - \cos(\theta)
            \overset{\text{\eqref{eq:integral}}}{=} \frac{1}{\pi}\int_0^\theta s \sin(s) \ \dd s
            \leq \frac{1}{\pi}\int_0^\theta s^2 \ \dd s = \theta^3/(3\pi),
      \end{align*}
      where the inequality uses $\sin(s) \leq s$ for all $s \geq0$.
      On the other hand, using the estimate $\sin(s) \geq 2s/\pi$ for $s \in [0, \pi/2]$, we compute 
      \begin{align*}
            \cos(\psi(\theta)) - \cos(\theta)&= \int_{\psi(\theta)}^\theta \sin(s) \ \dd s
            \geq \frac{2}{\pi} \cdot \int_{\psi(\theta)}^\theta s \ \dd s
            = \frac{1}{\pi} \cdot (\theta^2 - \psi^2(\theta))
            = \frac{1}{\pi}(\theta - \psi(\theta))\underbrace{\theta + \psi(\theta)}_{ \geq \theta} \\
            &\geq \frac{\theta}{\pi}(\theta - \psi(\theta)).
      \end{align*}
      Combining the two previous estimates, we obtain 
      \begin{equation}\label{eq:psilower}
      \psi(\theta) \geq \theta - \frac{\theta^2}{3}.
      \end{equation}

      For $\theta \in [0, \pi/2]$, we set $\eta_{-1}(\theta) \defeq \theta$ and recursively set $\eta_r \defeq \psi \circ \eta_{r-1}$
      for $r \in \NN_0$. 
      Then, 
      \begin{equation}\label{eq:arccosqr}
            \arccos(q_r(\cos(\theta))) = \eta_r(\theta) \quad \text{for all }
            r \in \{-1,0,1, \ldots\}.
      \end{equation}
      Indeed, this is trivially satisfied for $r = -1$ and if we assume that $r-1$ already satisfies 
      the claim, then
      \begin{align*}
            \eta_r(\theta) &= \arccos(\kappa(\cos(\eta_{r-1}(\theta))))
            = \arccos(\kappa(\cos(\arccos(q_{r-1}(\cos(\theta)))))) \\
            &= \arccos(\kappa(q_{r-1}(\cos(\theta)))) = \arccos(q_{r}(\cos(\theta))).
      \end{align*}
      Moreover, we have 
      \begin{equation}\label{eq:etarlower}
      \eta_r(\theta) \geq \theta - \frac{(r+1)\theta^2}{3} \quad \text{for all } r \in \{-1,0,1, \ldots\}.
      \end{equation}
      Indeed, this is again trivially satisfied for $r = -1$ and if $r-1$ already satisfies the claim, then 
      \begin{align*}
            \eta_r(\theta) = \psi(\eta_{r-1}(\theta)) \overset{\text{\eqref{eq:psilower}}}{\geq} \eta_{r-1}(\theta) - \frac{\eta^2_{r-1}(\theta)}{3}
            \overset{\psi(\theta) \leq \theta}{\geq} \eta_{r-1}(\theta) - \frac{\theta^2}{3} \geq \theta - \frac{r\theta^2}{3}- \frac{\theta^2}{3}
            = \theta - \frac{(r+1)\theta^2}{3}.
      \end{align*}
      Fix $r \in \{-1, \ldots, L-1\}$.
      Note that $\delta \leq 3/2 \leq \pi/2$.
      Putting everything together, we hence get 
      \begin{align*}
            \arccos(q_r(\cos(\delta))) \overset{\text{\eqref{eq:arccosqr}}}{=} \eta_r(\delta) 
            \overset{\text{\eqref{eq:etarlower}}}{\geq} \delta \left(1 - \frac{(r+1)\delta}{3}\right)
            \geq \delta \left(1 - \frac{L \delta}{3}\right) \geq \delta/2,
      \end{align*}
      where the last step uses the assumption $\delta \leq 3/(2L)$.
\end{proof}
We can now establish the desired bounds on $\Gamma_{L,h}$.
\begin{lemma}\label{prop:gammabound}
For $L \in \NN$ and $h \in (0,1]$, let $\Gamma_{L,h}$ be defined as in \eqref{sharp:eq:Gamma}.
Then, 
\[
1 - 2Lh/\pi \leq \Gamma_{L,h}(1) \leq 1.
\]
Moreover, if $\delta \in (0, 3/L]$ and $x,y \in \SS^{d-1}$ with $\mnorm{x-y}_2 \geq \delta$,
then
\[
\Gamma_{L,h}(\skpr{x}{y}) \leq \exp\left(-L\delta/(4\pi)\right).
\]
\end{lemma}
\begin{proof}
      Let $-1 \leq \rho \leq 1$.
      From the definition of $\Gamma_{L,h}$ and since $q_{L-1}'$ extends continuously onto $[-1,1]$,
      we conclude that 
      \begin{equation}\label{eq:intrepr}
      \Gamma_{L,h}(\rho)= \frac{1}{t_+ - t_-} \cdot \int_{t_-}^{t_+} q_{L-1}'(t) \ \dd t,
      \end{equation}
      where $t_\pm=(\rho\pm h^2)/(1+h^2)$.

      In particular, if $\rho = 1$, then $t_+ = 1$ and $t_- = (1-h^2)/(1+h^2)$.
      Let $t \in [(1-h^2)/(1+h^2),1]$ and write $t = \cos(\theta)$ for $\theta \in [0, \pi]$.
      In particular, this yields 
      \[
      \theta \leq \arccos((1-h^2)/(1+h^2)) = 2\arctan(h) \leq 2h.
      \]
      Here, the second step uses the double-angle identity and the last step uses the well-known estimate $\arctan(x) \leq x$ for $x \geq 0$.
      Moreover, since $\kappa(y) \geq y$ for every $y \in [-1,1]$ (see \eqref{sharp:eq:kappa-increment})
      and by definition of $q_r$, we conclude 
      $q_r(y) \geq y$ for every $y \in [-1,1]$ and $r \in \{-1, 0, 1 , \ldots\}$.
      In particular, since $\arccos$ is nonincreasing, we conclude 
      \begin{equation}\label{eq:qrbound}
      \arccos(q_r(\cos(\theta))) \leq \arccos(\cos(\theta)) = \theta \leq 2h
      \quad \text{for every } r \in \{-1,0,1, \ldots\}.
      \end{equation}
      Consequently, 
      \begin{align*}
            \Gamma_{L,h}(1) &= \frac{1}{t_+ - t_-} \cdot \int_{t_-}^{t_+} q_{L-1}'(t) \ \dd t
            \overset{\text{\cshref{lem:qjdev}}}{=} \frac{1}{t_+ - t_-} \cdot \int_{t_-}^{t_+} \prod_{r = -1}^{L-2} \left(1 - \frac{\arccos(q_r(t))}{\pi}\right) \ \dd t \\
            \overset{\text{\eqref{eq:qrbound}}}&{\geq} \frac{1}{t_+ - t_-} \cdot \int_{t_-}^{t_+}  \left(1 - \frac{2h}{\pi}\right)^L \ \dd t
            \geq 1- \frac{2hL}{\pi},
      \end{align*}
      where the last step uses Bernoulli's inequality. 
      On the other hand, since $q_{L-1}'(t) \leq 1$ for $t \in (-1,1)$ (as follows from \Cref{lem:qjdev}), we obtain
      $\Gamma_{L,h}(1) \leq 1$.
      This proves the first part of the claim. 

      For the second part, we set $\rho = \skpr{x}{y}$ and $t_+ = (\rho + h^2)/(1+h^2)$.
      Since $1 - \rho = \frac{1}{2}\mnorm{x-y}_2^2$, we obtain 
      \[
      1 - t_+ = \frac{1 - \rho}{1+h^2} = \frac{\mnorm{x-y}_2^2}{2(1+h^2)} \geq\frac{\delta^2}{2(1+h^2)}
      \]
      We pick $\theta_+ \in [0, \pi]$ with $\cos(\theta_+) = t_+$.
      Using the well-known estimate $1- \cos(z) \leq z^2/2$ for $z \in \RR$, we obtain 
      \[
            \frac{\theta_+^2}{2} \geq 1 - \cos(\theta_+)= 1 - t_+ \geq \frac{\delta^2}{2(1+h^2)},
      \]
      which implies 
      \[
      \theta_+ \geq \delta/\sqrt{1+h^2} \geq \delta/2,
      \]
      where the last inequality uses $h \leq 1$.
      Since $\cos$ is nonincreasing on $[0, \pi]$, this directly implies that $t_+ = \cos(\theta_+) \leq \cos(\delta /2)$.
      Using \eqref{eq:intrepr} and since $q_{L-1}'$ is nondecreasing as was shown in \Cref{lem:qjdev}, we get 
      \[
      \Gamma_{L,h}(\skpr{x}{y})=\Gamma_{L,h}(\rho) \leq q_{L-1}'(t_+) \leq q_{L-1}'(\cos(\delta/2)) = \prod_{r= -1}^{L-2} \left(1 - \frac{\arccos(q_r(\cos(\delta/2)))}{\pi}\right).
      \]
      Now, according to \Cref{lem:arccos}, we have 
      \[
      \arccos(q_r(\cos(\delta/2))) \geq \delta/4
      \]
      for all $r \in \{-1, \ldots, L-2\}$.
      This implies 
      \begin{align*}
          \Gamma_{L,h}(\skpr{x}{y}) \leq \left(1 - \frac{\delta}{4\pi}\right)^L \overset{1-z \leq \exp(-z)}{\leq}
          \exp(- L\delta /(4\pi)),  
      \end{align*}
      as was to be shown. 
\end{proof}
Since the squared Euclidean distance of $V_x$ and $V_y$ can be approximated using $\Gamma_{L,h}$
(see \eqref{eq:teaser}), we now obtain the following result. 
\begin{proposition}\label{prop:vxvy}
      There exists an absolute constant $c>0$ with the following property. 
      Let 
      \begin{equation}\label{eq:betadef}
      \beta \defeq \exp(-1/(48\pi\ee )) \in (0,1).
      \end{equation}
      Let $\Phi: \RR^d \to \RR$ be a random zero-bias ReLU network with $L$ hidden layers 
      of width $N$ following \Cref{assum:1}.
      Fix $x,y,\nu \in \SS^{d-1}$ with $x,y \in \nu^\perp$ and $\mnorm{x-y}_2\geq 1/(12 \ee L)$, set $h \defeq \pi(1-\beta)/(4L)$ and define
      $V_x$ and $V_y$ as in \eqref{eq:vx}. 
      Then, with probability at least $1-96L\exp(-c \cdot N/L^6)$, 
      \[
      \mnorm{V_x - V_y}_2^2 \geq \frac{1-\beta}{2} \quad \text{and}\quad  \max\{\mnorm{V_x}_2^2, \mnorm{V_y}_2^2\} \leq 2.
      \]
\end{proposition}
\begin{proof}
      We apply \Cref{lem:gammalh} with the choice of $\eps \defeq (1-\beta) \cdot h^2/(8\cdot 16L)$ together with a union 
      bound to obtain the existence of an absolute constant $c>0$ such that with probability at least
      $1-96L\exp(-c \cdot N/L^6)$, 
      \[
      \max_{z_1, z_2 \in \{x,y\}} \abs{\skpr{V_{z_1}}{V_{z_2}} - \Gamma_{L,h}(\skpr{z_1}{z_2})} \leq \frac{8L(1+h^2)}{h^2} \cdot \eps
      \overset{h \leq 1}{\leq} \frac{16L\eps}{h^2} = (1-\beta)/8.
      \]
      On this event, we obtain 
      \[
      \mnorm{V_x - V_y}_2^2 = \mnorm{V_x}_2^2 + \mnorm{V_y}_2^2 - 2 \skpr{V_x}{V_y}
      \geq 2 \Gamma_{L,h}(1) - 2 \Gamma_{L,h}(\skpr{x}{y}) - (1-\beta)/2.
      \]
      Hence, for the first part it suffices to show that $2 \Gamma_{L,h}(1) - 2 \Gamma_{L,h}(\skpr{x}{y}) \geq 1-\beta$.
      Using \Cref{prop:gammabound} and our choice of $h$, we get 
      \begin{align*}
            2 \Gamma_{L,h}(1) - 2 \Gamma_{L,h}(\skpr{x}{y}) 
            \geq 2 - \frac{4Lh}{\pi} - 2\beta = 2 - (1-\beta) - 2\beta= 1 -\beta,
      \end{align*}
      as desired.

      For the second part, we note that 
      \[
            \max\{\mnorm{V_x}_2^2, \mnorm{V_y}_2^2\} \leq \Gamma_{L,h}(1) + (1-\beta)/8 \overset{\text{\cshref{prop:gammabound}}}{\leq}1 + (1-\beta)/8  \leq 2. \qedhere
      \]
\end{proof}
We now have everything together to prove \Cref{sharp:prop:direction}.
We now prove the following fixed-step version, which immediately implies \Cref{sharp:prop:direction}.
\begin{proposition}\label{thm:direction}
There exist absolute constants $C,c>0$ with the following property. 
Let $\Phi: \RR^d \to \RR$ be a random zero-bias ReLU network with $L$ hidden layers 
      of width $N$ following \Cref{assum:1} and satisfying $d \geq 3$,
      \[
      N \geq C \cdot dL^6 \ln(\ee L) \quad \text{and} \quad d\ln(\ee L) \geq C.
      \] 
      Fix a direction $\nu \in \SS^{d-1}$
      and set
      \[
      h \defeq \frac{\pi(1- \exp(-1/(48\pi\ee)))}{4L} > 0.
      \]
      Then, with probability at least $1- \exp(-c \cdot d\ln(\ee L))$,
      \[
      \sup_{x \in \SS^{d-1} \cap \nu^\perp} \frac{\Phi(x+h\nu) - \Phi(x - h\nu)}{2h} \geq c \cdot\sqrt{d \ln(\ee L)}.
      \]
\end{proposition}
\begin{proof}
      For $x \in \SS^{d-1} \cap \nu^\perp$ we let $V_x \in \RR^N$ be as in \eqref{eq:vx} and note that 
      \[
      \sup_{x \in \SS^{d-1} \cap \nu^\perp} \frac{\Phi(x+h\nu) - \Phi(x - h\nu)}{2h} = 
      \sup_{x \in \SS^{d-1} \cap \nu^\perp} \skpr{(W^{(L)})^T}{V_x}.
      \]
      Since $\SS^{d-1} \cap \nu^\perp$ can be identified with $\SS^{d-2}$ and using \Cref{lem:sphere_cov} 
      together with \cite[Lemma~4.2.8]{vershynin_high-dimensional_2018}, we can choose a set $P \subseteq \SS^{d-1} \cap \nu^\perp$
      with $\# P \geq \frac{1}{6} \cdot (6 \ee L)^{d-2} \geq (\ee L)^{d-2}$ such that 
      \[
      \mnorm{x-y}_2 \geq 1/(12\ee L) \quad \text{for all } x,y \in P\text{ with } x \neq y.
      \]
      After possibly removing some elements from $P$,
      we can ensure that $\#P \leq (6 \ee L)^{d-2}$.
      With $\beta$ as defined in \eqref{eq:betadef} and using \Cref{prop:vxvy} together with a union 
      bound, we obtain the existence of an absolute constant $c_1 >0$ such 
      that with probability at least $1- 96 L(\# P)^2 \exp(-c_1 \cdot N/L^6)$,
      \begin{align}\label{eq:centralprop}
      \mnorm{V_x - V_y}_2 \geq \sqrt{\frac{1- \beta}{2}} \quad \text{for all }x,y \in P \text{ with } x\neq y \quad
      \text{and} \quad \max_{x \in P} \mnorm{V_x}_2^2 \leq 2.
      \end{align}
      The probability can be bounded by 
      \begin{align*}
            1- 96 L(\# P)^2 \exp(-c_1 \cdot N/L^6) \geq 1- \exp\left(\ln(96L) + 2d\ln(6 \ee L) - c_1\cdot N/L^6 \right)
            \geq 1- \exp\left( - c_2\cdot N/L^6 \right),
      \end{align*}
      where we used that $N \geq C \cdot dL^6 \ln(\ee L)$ and set $c_2 \defeq c_1 / 2$.
      We condition on $W^{(0)}, \ldots, W^{(L-1)}$ and assume that \eqref{eq:centralprop} is satisfied. 

      Using Sudakov's minoration inequality \cite[Theorem~7.4.1]{vershynin_high-dimensional_2018},
      we obtain the existence of absolute constants $c_3,c_4 > 0$ such that 
      \[
      \underset{W^{(L)}}{\EE} \left[\sup_{x \in P} \skpr{W^{(L)}}{V_x}\right]
      \geq c_3 \cdot \sqrt{\ln(\# P)} \geq c_4 \cdot \sqrt{d-2} \cdot \sqrt{\ln(\ee L)} \geq c_5 \cdot \sqrt{d \ln(\ee L)},
      \]
      where $c_5 \defeq c_4 / \sqrt{3}$ and where we used $d \geq 3$.
      We turn this into a high probability bound using \cite[Theorem~5.2.2]{vershynin_high-dimensional_2018}:
      The map 
      \[
      f: \quad \RR^N \to \RR, \quad y \mapsto \sup_{x \in P} \skpr{y}{V_x}
      \]
      is $\sqrt{2}$-Lipschitz. 
      Indeed, let $y_1, y_2 \in \RR^N$ and assume without loss of generality $f(y_1) \geq f(y_2)$.
      Pick $z \in P$ with $f(y_1) = \skpr{y_1}{V_{z}}$.
      Then,
      \[
      \abs{f(y_1) - f(y_2)} = f(y_1) - f(y_2) \leq \skpr{y_1 - y_2}{V_{z}} \leq \mnorm{y_1 - y_2}_2\mnorm{V_{z}}_2
      \leq \sqrt{2} \cdot \mnorm{y_1 - y_2}_2.
      \]
      From \cite[Theorem~5.2.2]{vershynin_high-dimensional_2018} we deduce the
      existence of an absolute constant $c_6 > 0$
      such that 
      \[
      \sup_{x \in \SS^{d-1} \cap \nu^\perp} \skpr{(W^{(L)})^T}{V_x} \geq
      \sup_{x \in P} \skpr{(W^{(L)})^T}{V_x} \geq c_7 \cdot \sqrt{d \ln(\ee L)}
      \]
      with probability at least $1-2\exp(-c_6 \cdot d\ln(\ee L))$, where $c_7 \defeq c_5/2$.

      Note that up to now we have conditioned on $W^{(0)}, \ldots, W^{(L-1)}$.
      Reintroducing the randomness over these random matrices, we obtain that 
      \[
      \sup_{x \in \SS^{d-1} \cap \nu^\perp} \skpr{(W^{(L)})^T}{V_x} \geq c_7 \cdot \sqrt{d \ln(\ee L)}
      \]
      with probability at least 
      \[
      1-\exp(-c_2 \cdot N/L^6) - 2\exp(-c_6 \cdot d\ln(\ee L)) \geq 1-\exp(-c \cdot d\ln(\ee L))
      \]
      with a suitable constant $c>0$.
      Here we used $N \geq C \cdot dL^6 \ln(\ee L)$ and $d \ln(\ee L) \geq C$.
      This completes the proof. 
\end{proof}

%% file: homogeneous_proof.tex
\renewcommand{\DD}{\widetilde{D}}
\section{\texorpdfstring{Proof of \Cref{thm:grad_diff}}{Proof of Theorem 6.2}} \label{sec:hom}
In this section, we show that for inputs with sufficiently large Euclidean norm, 
the gradient of a ReLU network (with potentially non-zero biases) 
at these inputs is, with high probability, very similar to the gradient of 
the associated \emph{homogeneous} network, i.e., the same network with zero biases. 
The main tool to derive the result is \Cref{prop:farout}, which states that for inputs with large 
Euclidean norm the bias term does \emph{not} have a significant 
influence on whether the ReLU is active or not. 

In this section, we study random ReLU networks following \Cref{assum:1} 
with \emph{fixed} bias vectors $b^{(\ell)} \in \RR^N$ for $\ell \in \{0,\ldots,L-1\}$.
We pick $\lambda' > 0$ such that  
\begin{equation}\label{eq:lambda}
\text{for all } \ell \in \{0,\ldots,L-1\}, \ i \in \{1,\ldots,N\}: \ \abs{b^{(\ell)}_i} \leq 
\frac{\lambda' \sqrt{2}}{\sqrt{N}}.
\end{equation}
The scaling by $\sqrt{2/N}$ is required to make \Cref{prop:farout} applicable, since the entries 
of the weight matrices are not standard Gaussian, but scaled by $\sqrt{2/N}$ too. 

We denote by $\widetilde{\Phi} : \RR^d \to \RR$
the homogeneous network associated to $\Phi$, i.e., the weight matrices of $\widetilde{\Phi}$ are the same
as for $\Phi$ and all the bias vectors are set to zero. 
For $j \in \{0,\ldots,L-1\}$ and $x \in \RR^d$, we let
\[
\DD^{(j)}(x) \defeq \Delta (W^{(j)} \widetilde{\Phi}^{(j-1)}(x)) \in \RR^{N \times N}.
\]
\newcommand{\aj}{A^{(j)}}
We then set 
\begin{equation}\label{eq:ajdef}
\aj(x) \defeq D^{(j)}(x) - \DD^{(j)}(x).
\end{equation}
The first goal is to show that $\aj(x) \in \mathcal{A}_\delta^{(1)}$ for $x \in \RR^d$ with sufficiently large Euclidean norm, where $\mathcal{A}_\delta^{(1)}$ is as defined in 
\eqref{eq:aj},
i.e., $A^{(j)}(x)$ is a diagonal matrix with rather sparse diagonal.
As already explained above, the idea is to apply \Cref{prop:farout}.
To this end, we aim to bound the covering numbers and the Gaussian width of certain sets. 
\begin{lemma}\label{lem:tbound}
        There exist absolute constants $C,c>0$ such that the following holds.
        Let $\Phi: \RR^d \to \RR$ be a random ReLU network with $L$ hidden layers of width $N$ satisfying \Cref{assum:1}
        with fixed biases. 
        Let $\lambda'$ be chosen as in \eqref{eq:lambda}.
        Assume $N \geq Cd$ and $N \geq C \cdot L^3d\ln(N/d)$.
        Let $\delta \in (0, \ee^{-1})$ with 
        \[
        N \geq C \cdot dL^2 \cdot \ln(1/\delta), 
        \quad 
        \delta N \geq C \cdot d  \cdot \ln(1/\delta),
        \quad
        \text{and}
        \quad
         \delta \cdot \ln(1/\delta) \leq c \cdot 1/L^2
        \] 
        and fix $j \in \{-1, \ldots, L-1\}$.
        Then the following holds with probability at least $1-\exp(-c \cdot \delta N)$, uniformly 
        over the choice of $\eps \in (0,\ee^{-1})$.
        We set $\theta \defeq 48 \cdot 3^L \cdot \lambda' \cdot \eps^{-1}$ and define
        \[
        T^{(j)} \!\defeq\! \left\{\!\frac{\Phi^{(j)}(x)}{\mnorm{\Phi^{(j)}(x)}_2}: \ x \in \RR^d, \mnorm{x}_2 \geq \theta, \Phi^{(j)}(x) \neq 0 \!\right\} \cup \left\{\!\frac{\tilde{\Phi}^{(j)}(x)}{\mnorm{\tilde{\Phi}^{(j)}(x)}_2}: \ x \in \RR^d, \mnorm{x}_2 \geq \theta, \tilde{\Phi}^{(j)}(x) \neq 0\!\right\},
        \]
        where $\widetilde{\Phi} : \RR^d \to \RR$ is the homogeneous network associated to $\Phi$. 
        Then,
        \[
        \ln(\ee \cdot \mathcal{N}(T^{(j)},\eps/2)) \leq  C \cdot d \cdot \ln(1/\eps) \quad 
        \text{and} \quad
        w((T^{(j)}-T^{(j)}) \cap B_\mu(0,\eps))\leq C \cdot  \sqrt{d} \cdot \eps \cdot \left(\sqrt{L}+\ln^{1/2}(1/\eps)\right).
        \]
        Here, 
        \[
        \mu \defeq \begin{cases}d & \text{ if }j = -1, \\ N & \text{ if $j \geq 0$.}\end{cases}
        \]
    \end{lemma}
    \begin{proof}
    We may assume $j \geq 0$, since
    \[
    T^{(-1)} = \SS^{d-1} \quad \text{and} \quad (T^{(-1)} - T^{(-1)}) \cap B_d(0,\eps) \subseteq B_d(0,\eps),
    \]
    so that the claim is trivially satisfied by \cite[Corollary~4.2.13~\&~Proposition~7.5.2(f)]{vershynin_high-dimensional_2018}.
    
    For simplicity, we write $T = T^{(j)}$ for the remainder of the proof. 
    Note that
    \begin{equation}\label{eq:con2}
    \PP\left(\text{for all } \ell \in \{0,\ldots,j\}: \ \op{W^{(\ell)}} \leq 3\right)
    \geq 1- \exp(-c_1 \cdot N),
    \end{equation}
    with an absolute constant $c_1 > 0$,
    as follows by \cite[Corollary~7.3.3]{vershynin_high-dimensional_2018} and a union bound, noting that $N \geq C \cdot\ln(L)$.
    Further,
    \begin{equation}\label{eq:con3}
    \text{for all } x \in \RR^d: \ \mnorm{\tilde{\Phi}^{(j)}(x)}_2 \geq \frac{\mnorm{x}_2}{4}
    \end{equation}
    with probability at least $1- \exp(-c_2 \cdot \delta N)$ 
    with an absolute constant $c_2 > 0$ by \Cref{prop:isom}.
    Lastly, note that 
    \begin{equation}\label{eq:con4}
    \lip_{2 \to 2}(\tilde{\Phi}^{(j)}) \leq C_1
    \end{equation}
    with probability at least $1- \exp(-c_3 \cdot \delta N)$ with absolute constants $C_1, c_3 > 0$, as follows by
    \Cref{prop:advanced_lip}.
    From now on, we assume that the properties defined in \Cref{eq:con2,eq:con3,eq:con4}
    are satisfied.
    This happens with probability at least 
    \begin{align*}
    &\norel 1- \exp(-c_1 \cdot N) - \exp(-c_2 \cdot \delta N)- \exp(-c_3 \cdot \delta N) 
    \geq 1 -3\exp(-c' \cdot \delta N) 
    = 1-\exp(\ln(3) - c' \cdot \delta N)
    \end{align*}
    with $c' \defeq \min\{c_1, c_2, c_3\}$.
    By taking $C$ large enough, we may assume 
    \[
    \delta N\geq C \cdot d \cdot \ln(1/\delta) \geq \frac{2}{c'} \cdot \ln(3)
    \]
    and can thus, by letting $c \defeq c'/2$, lower bound the probability by 
    \[
    1-\exp(-c \cdot \delta N).
    \]
    
    By our assumption \eqref{eq:lambda},
    \[
    \mnorm{b^{(\ell)}}_2^2 = \sum_{i=1}^N \abs{b^{(\ell)}_i}^2  \leq 
    2(\lambda')^2
    \]
    for every $\ell \in \{0,\ldots,j\}$.
    Using \eqref{eq:con2} we then observe, for any $x \in \RR^d$, that 
    \begin{align}
        \mnorm{\Phi^{(j)}(x)- \widetilde{\Phi}^{(j)}(x)}_2
        &\leq \mnorm{\relu(W^{(j)} \Phi^{(j-1)}(x) + b^{(j)}) - \relu(W^{(j)}\widetilde{\Phi}^{(j-1)}(x))}_2 \nonumber\\
        \overset{\text{ReLU 1-Lipschitz}}&{\leq} \mnorm{W^{(j)} \Phi^{(j-1)}(x) + b^{(j)} - W^{(j)}\widetilde{\Phi}^{(j-1)}(x)}_2 \nonumber\\
        &\leq \mnorm{b^{(j)}}_2 + \mnorm{W^{(j)}}_{2 \to 2} \cdot \mnorm{\Phi^{(j-1)}(x)- \widetilde{\Phi}^{(j-1)}(x)}_2 \nonumber\\
        &\leq \mnorm{b^{(j)}}_2 + 3 \cdot \mnorm{\Phi^{(j-1)}(x)- \widetilde{\Phi}^{(j-1)}(x)}_2 \nonumber\\
        &\leq \sum_{\ell = 0}^{j} 3^{j-\ell} \cdot \mnorm{b^{(\ell)}}_2 \nonumber\\
        &\leq \sqrt{2} \cdot \lambda' \cdot \frac{3^{j+1} - 1}{3-1} \nonumber\\
        \label{eq:hombiasdiff}
        &\leq \lambda' \cdot 3^L.
    \end{align}
    In particular, note that for $x \in \RR^d$ with $\mnorm{x}_2 \geq \theta$, 
    \begin{align*}
    \mnorm{\Phi^{(j)}(x)}_2 &\geq \mnorm{\tilde{\Phi}^{(j)}(x)}_2 - \mnorm{\Phi^{(j)}(x) - \tilde{\Phi}^{(j)}(x)}_2 \\
    &\geq \frac{\mnorm{x}_2}{4} - \lambda' \cdot 3^L \geq \frac{\theta}{4} - \lambda' \cdot 3^L 
    = 12 \cdot 3^L \cdot \lambda' \cdot \eps^{-1} - \lambda' \cdot 3^L >0.
    \end{align*}
    Due to this estimate and \Cref{eq:con3}, the conditions $\Phi^{(j)}(x) \neq 0$ and $\tilde{\Phi}^{(j)}(x) \neq 0$ in the definition 
    of $T^{(j)}$ are redundant.

        After these preliminary observations, let us move to the actual proof. 
        As a first step, we aim to bound the $\gamma$-covering numbers of $T$ 
        for every $\gamma \in (0,\eps/2]$.
        To this end, note  that $T = T_1 \cup T_2$ with 
        \[
        T_1 \defeq \left\{ \frac{\Phi^{(j)}(x)}{\mnorm{\Phi^{(j)}(x)}_2}: \ x \in \RR^d, \mnorm{x}_2 \geq \theta\right\}
        \]
        and 
        \[
        T_2 \defeq \left\{ \frac{\tilde{\Phi}^{(j)}(x)}{\mnorm{\tilde{\Phi}^{(j)}(x)}_2}: \ x \in \RR^d, \mnorm{x}_2 \geq \theta\right\} = \left\{ \frac{\tilde{\Phi}^{(j)}(x)}{\mnorm{\tilde{\Phi}^{(j)}(x)}_2}: \ x \in \SS^{d-1}\right\},
        \]
        whence it suffices to bound the $\gamma$-covering numbers of $T_1$ and $T_2$ separately. 

        \newcommand{\ngamma}{\mathcal{N}_1}
        We start with $T_1$.
        Let $\gamma \in (0,\eps/2]$.
        We then set $\theta_\gamma \defeq 24 \cdot 3^L \cdot \lambda' \cdot \gamma^{-1} \geq \theta_{\eps/2} = \theta$.
Let $\ngamma$ be a $\left(\frac{\gamma}{24C_1}\right)$-net of $\SS^{d-1}$ with $\#\ngamma \leq (72C_1/\gamma)^d$, which is possible using \cite[Corollary~4.2.13]{vershynin_high-dimensional_2018}, and define 
    \[
    \mathcal{N}_1' \defeq \left\{ \frac{\Phi^{(j)}(\theta_\gamma x^\ast)}{\mnorm{\Phi^{(j)}(\theta_\gamma x^\ast)}_2}: \ x^\ast \in \ngamma\right\} 
    \subseteq T_1.
    \]
    Moreover, let $\mathcal{N}_2$ be a $(2\lambda'\gamma)$-net of the annulus $A \defeq \{x \in \RR^d: \ \theta_\gamma \geq \mnorm{x}_2 \geq \theta\}$ with cardinality $\#\mathcal{N}_2 \leq \left(\frac{6\theta_\gamma}{2\lambda'\gamma}\right)^d
    = \left(\frac{72 \cdot 3^L}{\gamma^2}\right)^d$, 
    which exists according to \cite[Ex.~4.2.10 \& Cor.~4.2.13]{vershynin_high-dimensional_2018}.
    As above, we set 
    \[
    \mathcal{N}_2' \defeq \left\{ \frac{\Phi^{(j)}(x^\ast)}{\mnorm{\Phi^{(j)}(x^\ast)}_2}: \ x^\ast \in \mathcal{N}_2\right\} \subseteq T_1.
    \]
    We claim that $\mathcal{N}_1' \cup \mathcal{N}_2'$ is a $\gamma$-net of $T_1$.

    To this end, let $x \in \RR^d$ with $\mnorm{x}_2 \geq \theta$ be arbitrary. 
    We start with the case $\mnorm{x}_2 \geq \theta_\gamma$ and pick $x^\ast \in \ngamma$ satisfying 
    \[
    \mnorm{\frac{x}{\mnorm{x}_2} - x^\ast}_2 \leq \frac{\gamma}{24C_1}.
    \]
    We get
    \begin{align*}
        &\norel\mnorm{\frac{\Phi^{(j)}(x)}{\mnorm{\Phi^{(j)}(x)}_2} - \frac{\Phi^{(j)}(\theta_\gamma x^\ast)}{\mnorm{\Phi^{(j)}(\theta_\gamma x^\ast)}_2}}_2 \\
        &\leq \mnorm{\frac{\Phi^{(j)}(x)}{\mnorm{\Phi^{(j)}(x)}_2} - \frac{\tilde{\Phi}^{(j)}(x)}{\mnorm{\tilde{\Phi}^{(j)}(x)}_2}}_2
        + \mnorm{\frac{\tilde{\Phi}^{(j)}(x)}{\mnorm{\tilde{\Phi}^{(j)}(x)}_2} - \frac{\tilde{\Phi}^{(j)}(\theta_\gamma x^\ast)}{\mnorm{\tilde{\Phi}^{(j)}(\theta_\gamma x^\ast)}_2}}_2 \\
        &\norel+ \mnorm{\frac{\tilde{\Phi}^{(j)}(\theta_\gamma x^\ast)}{\mnorm{\tilde{\Phi}^{(j)}(\theta_\gamma x^\ast)}_2} - \frac{\Phi^{(j)}(\theta_\gamma x^\ast)}{\mnorm{\Phi^{(j)}(\theta_\gamma x^\ast)}_2}
        }_2 \\
        \overset{\text{\cshref{lem:diff_bound}}}&{\leq} 2\cdot \frac{\mnorm{\Phi^{(j)}(x) - \tilde{\Phi}^{(j)}(x)}_2}{\mnorm{\tilde{\Phi}^{(j)}(x)}_2}
        +  \mnorm{\frac{\tilde{\Phi}^{(j)}\left(x/ \mnorm{x}_2\right)}{\mnorm{\tilde{\Phi}^{(j)}\left(x/\mnorm{x}_2\right)}_2} - \frac{\tilde{\Phi}^{(j)}( x^\ast)}{\mnorm{\tilde{\Phi}^{(j)}(x^\ast)}_2}}_2 \\
        &\norel + 2 \cdot \frac{\mnorm{\Phi^{(j)}(\theta_\gamma x^\ast) - \tilde{\Phi}^{(j)}(\theta_\gamma x^\ast)}_2}{\mnorm{\tilde{\Phi}^{(j)}(\theta_\gamma x^\ast)}_2} \\
        \overset{\text{\cshref{lem:diff_bound}}}&{\leq} 2 \cdot \frac{\mnorm{\Phi^{(j)}(x) - \tilde{\Phi}^{(j)}(x)}_2}{\mnorm{\tilde{\Phi}^{(j)}(x)}_2}
        + 2 \cdot \frac{\mnorm{\tilde{\Phi}^{(j)}(x/\mnorm{x}_2) - \tilde{\Phi}^{(j)}(x^\ast)}_2}{\mnorm{\tilde{\Phi}^{(j)}(x^\ast)}_2} \\
        &\norel + 2 \cdot \frac{\mnorm{\Phi^{(j)}(\theta_\gamma x^\ast) - \tilde{\Phi}^{(j)}(\theta_\gamma x^\ast)}_2}{\mnorm{\tilde{\Phi}^{(j)}(\theta_\gamma x^\ast)}_2} \\
        &\leq \frac{8 \cdot 3^L \cdot \lambda'}{\theta_\gamma} + \frac{8C_1\gamma}{24C_1} + \frac{8 \cdot 3^L \cdot \lambda'}{\theta_\gamma}
        = \frac{\gamma}{3} + \frac{\gamma}{3} + \frac{\gamma}{3}
        = \gamma
    \end{align*}
    by choice of $\gamma$ and $\theta_\gamma$, using \Cref{eq:hombiasdiff,eq:con3,eq:con4} in the last inequality.  

    Now, assume that $x \in \{x \in \RR^d: \ \theta_\gamma \geq \mnorm{x}_2 \geq \theta\}$.
    Pick $x^\ast \in \mathcal{N}_2$ with $\mnorm{x - x^\ast}_2 \leq 2\lambda'\gamma$.
    This gives us 
    \begin{align}\label{eq:TsetminusT}
        \mnorm{\frac{\Phi^{(j)}(x)}{\mnorm{\Phi^{(j)}(x)}_2}-\frac{\Phi^{(j)}(x^\ast)}{\mnorm{\Phi^{(j)}(x^\ast)}_2}}_2 
        \overset{\text{\cshref{lem:diff_bound}}}{\leq}
        2 \cdot \frac{\mnorm{\Phi^{(j)}(x) - \Phi^{(j)}(x^\ast)}_2}{\mnorm{\Phi^{(j)}(x^\ast)}_2}
        \leq \frac{2 \cdot 3^L}{\mnorm{\Phi^{(j)}(x^\ast)}_2} \cdot \mnorm{x -x^\ast}_2,
    \end{align}
    using that $\Phi^{(j)}$ is $3^L$-Lipschitz from \eqref{eq:con2}.
    Note further that, using \eqref{eq:con3} and \eqref{eq:hombiasdiff}, 
    \begin{align*}
    \mnorm{\Phi^{(j)}(x^\ast)}_2 &\geq \mnorm{\tilde{\Phi}^{(j)}(x^\ast)}_2 - \mnorm{\Phi^{(j)}(x^\ast)-\tilde{\Phi}^{(j)}(x^\ast)}_2 \\
    &\geq \frac{\theta}{4} - 3^L \cdot \lambda' = 12 \cdot 3^L \cdot \lambda' \cdot \underbrace{\eps^{-1}}_{\geq 1} - 3^L \cdot \lambda'
    \geq 4 \cdot 3^L \cdot \lambda'.
    \end{align*}
    This bound together with \eqref{eq:TsetminusT} shows that 
    \[
    \mnorm{\frac{\Phi^{(j)}(x)}{\mnorm{\Phi^{(j)}(x)}_2}-\frac{\Phi^{(j)}(x^\ast)}{\mnorm{\Phi^{(j)}(x^\ast)}_2}}_2 \leq \gamma.
    \] 
    Overall, we have thus proven that $\mathcal{N}_1' \cup \mathcal{N}_2'$ is a $\gamma$-net
    of $T_1$. Consequently, 
    \[
    \mathcal{N}(T_1,\gamma) \leq \# \mathcal{N}_1 + \#\mathcal{N}_2 \leq \left(\frac{72C_1}{\gamma}\right)^d
    + \left(\frac{72 \cdot 3^L}{\gamma^2}\right)^d 
    \leq 2 \cdot  \left(\frac{72C_1 \cdot 3^L}{\gamma^2}\right)^d
    \]
    for $0<\gamma \leq \eps/2$.
    
    We continue by bounding the $\gamma$-covering number of $T_2$.
    Note that we may assume that $\tilde{\Phi}^{(j)}$ is $C_1$-Lipschitz by \eqref{eq:con4}.
        If we then choose a $\left(\frac{\gamma}{8C_1}\right)$-net $\mathcal{N}_3$ of $\SS^{d-1}$
        with $\#\mathcal{N}_3 \leq \left(\frac{24C_1}{\gamma}\right)^d$, see \cite[Cor.~4.2.13]{vershynin_high-dimensional_2018}, we can for arbitrary $x \in \SS^{d-1}$ choose
        $x^\ast \in \mathcal{N}_3$ with $\mnorm{x-x^\ast} \leq \frac{\gamma}{8C_1}$ and get 
        \begin{align*}
            \mnorm{\frac{\widetilde{\Phi}^{(j)}(x)}{\mnorm{\widetilde{\Phi}^{(j)}(x)}_2} 
            - 
            \frac{\widetilde{\Phi}^{(j)}(x^\ast)}{\mnorm{\widetilde{\Phi}^{(j)}(x^\ast)}_2}}_2
            \overset{\text{\cshref{lem:diff_bound}}}{\leq}
            \frac{2 \cdot \mnorm{\widetilde{\Phi}^{(j)}(x) - \widetilde{\Phi}^{(j)}(x^\ast)}_2}{\mnorm{\widetilde{\Phi}^{(j)}(x^\ast)}_2}
            \overset{\eqref{eq:con3}, \eqref{eq:con4}}{\leq}
            8C_1 \mnorm{x- x^\ast}_2 \leq \gamma.
        \end{align*}
        This proves that $\mathcal{N}(T_2, \gamma) \leq \left(\frac{24C_1}{\gamma}\right)^d$.
    
    In total, we hence get 
    \begin{equation*}
    \mathcal{N}(T,\gamma) \leq 
    \mathcal{N}(T_1,\gamma) + \mathcal{N}(T_2,\gamma)\leq 3 \cdot \left(\frac{72C_1 \cdot 3^L}{\gamma^2}\right)^d  \quad \text{for } \gamma \in (0,\eps/2].
    \end{equation*}
    Using \Cref{lem:log_b}, this implies
    \begin{equation}\label{eq:cov_bbb}
    \sqrt{\ln(\mathcal{N}(T,\gamma))} \leq C_2 \cdot \sqrt{d}\left(\sqrt{L}+ \ln^{1/2}(1/\gamma)\right)\quad \text{for } \gamma \in (0,\eps/2]
    \end{equation}
    with an appropriate constant $C_2 > 0$.
    Clearly,
    \[
    \diam((T-T) \cap B_N(0,\eps)) \leq \diam(B_N(0,\eps)) \leq 2\eps.
    \]
    Hence, using Dudley's inequality \cite[Theorem~8.1.3]{vershynin_high-dimensional_2018},
    we get with an absolute constant $C_3 > 0$ that 
    \begin{align*}
    w((T-T) \cap  B_N(0,\eps)) &\leq C_3 \cdot \int_0^{2\eps} \sqrt{\ln(\mathcal{N}((T-T) \cap B_N(0,\eps),\gamma))} \ \dd \gamma \\
    \overset{\remove{\text{\cite[Ex.~4.2.10]{vershynin_high-dimensional_2018}}}}&{\leq} C_3 \cdot \int_0^{2\eps} \sqrt{\ln(\mathcal{N}(T-T,\gamma/2))} \ \dd \gamma 
    \\
    &\leq C_3 \cdot \int_0^{2\eps} \sqrt{\ln(\mathcal{N}(T, \gamma/4)^2)} \ \dd \gamma \\
    \overset{\eqref{eq:cov_bbb}}&{\leq} \sqrt{2}C_2C_3 \cdot \sqrt{d} \cdot \left(2\eps \sqrt{L} + \int_0^{2\eps} \ln^{1/2}(4/\gamma) \ \dd \gamma\right) \\
    \overset{\text{\cshref{lem:int_b}}}&{\leq} 
    \sqrt{2}C_2C_3 \cdot \sqrt{d} \cdot \left(2\eps \sqrt{L} + 4\eps \cdot \ln^{1/2}(2/\eps)\right) \\
    &\leq  C \cdot \sqrt{d} \cdot \eps \cdot \left(\sqrt{L}+\ln^{1/2}(1/\eps)\right)
    \end{align*}
    with $C>0$ chosen large enough, applying \Cref{lem:log_b} in the last step. 
    This proves the bound on the Gaussian width.

    Moreover, if $\gamma = \eps/2$ we note that only $\mathcal{N}_1'$ is required to cover
    $T_1$. This yields 
    \[
    \mathcal{N}(T, \eps/2) \leq \# \mathcal{N}_1 + \# \mathcal{N}_3 \leq 2 \cdot \left(2 \cdot 72C_1/\eps\right)^d.
    \]
    This yields
    \[
    \ln( \ee \cdot \mathcal{N}(T,\eps /2)) \leq \ln\left(\ee \cdot \left(\frac{288C_1}{\eps}\right)^d\right) \overset{\text{\cshref{lem:log_b}}}{\leq} C_4 \cdot d \cdot \ln(1/\eps)
    \]
    with an absolute constant $C_4 > 0$. 
    \end{proof}
We can now prove that the matrices $A^{(j)}(x)$ are indeed sparse for inputs with large Euclidean norm. 
Note that \Cref{prop:theta} below can also be found as \Cref{lem:hom_trace} in the main part. 
\begin{proposition}\label{prop:theta}
    There exist absolute constants $C,c > 0$ such that the following holds.
    Let $\Phi: \RR^d \to \RR$ be a random ReLU network with $L$ hidden layers of width $N$ following \Cref{assum:1} with fixed biases.
    Assume $N \geq Cd$ and $N \geq C \cdot L^3d\ln(N/d)$.
        Let $\delta \in (0,\ee^{-1})$ with 
        \[
        N \geq C \cdot dL^2 \cdot \ln(1/\delta), 
        \quad 
        \delta N \geq C \cdot d\cdot\max\{ \ln(1/\delta), L\ln^{-1}(1/\delta)\}
        \quad
        \text{and}
        \quad
         \delta \cdot \ln(1/\delta) \leq c \cdot 1/L^2.
        \] 
    Pick $j \in \{0,\ldots,L-1\}$.
    Then with probability at least $1- \exp(-c \cdot \delta N)$, 
    \[
    \text{for all } x \in \RR^d, \ \mnorm{x}_2 \geq C \cdot 3^L \cdot \lambda' \cdot \delta^{-1} \cdot \ln^{1/2}(1/\delta): \quad \aj(x) \in \mathcal{A}_\delta^{(1)}.
    \]
    Here, $\mathcal{A}_\delta^{(1)}$ is as in \eqref{eq:aj}, $\lambda'$ is as in \eqref{eq:lambda} and $A^{(j)}(x)$ is
    as in \eqref{eq:ajdef}.
\end{proposition}
\begin{proof}
    Set 
    \[
    \mu \defeq \begin{cases} d& \text{if }j = 0, \\ N& \text{if }j > 0.\end{cases}
    \]
    Let $\widetilde{c_1}$ be taken according to \Cref{lem:recht_upper_original} and $\tilde{C_2},\widetilde{c_2}>0$ be taken according to 
    \Cref{prop:farout}. 
    Set $c' \defeq \min\{\widetilde{c_1}, \widetilde{c_2}\}$ and $\eps \defeq \frac{c'}{2\sqrt{2}} \cdot \delta \cdot \ln^{-1/2}(1/\delta)$.
    As in the beginning of the proof of \Cref{lem:tbound}, we pick constants $c_1, c_2 > 0$ that satisfy
    \Cref{eq:con2,eq:con3}.
    Further, in view of \Cref{lem:tbound}, we know
    \begin{align}
        \label{eq:extra_con}
        \ln(\ee \cdot \mathcal{N}(T^{(j-1)},\eps/2)) &\leq C_1 \cdot d \cdot \ln(1/\eps) \quad 
        \text{and} \\
        \label{eq:con5}
        w((T^{(j-1)}-T^{(j-1)}) \cap B_\mu(0,\eps))&\leq C_1 \cdot \sqrt{d} \cdot \eps \cdot \left(\sqrt{L}+\ln^{1/2}(1/\eps)\right)
    \end{align}
    with probability at least $1-\exp(-c_3 \cdot \delta N)$ with absolute constants $C_1,c_3 >0$.
    Here, $T^{(j-1)}$ is defined as in \Cref{lem:tbound}.
    Since \Cref{eq:con2,eq:con3,eq:con5,eq:extra_con} solely depend on the randomness in $W^{(0)}, \ldots, W^{(j-1)}$,
    we may condition on these random variables and assume that the mentioned properties (with index $j-1$) are satisfied. 
    Using \eqref{eq:hombiasdiff}, we see for every $x \in \RR^d$ with $\mnorm{x}_2 \geq \max\{16\tilde{C_2},48\} \cdot 3^L \cdot \lambda' \cdot \eps^{-1} =: \theta$,
    \begin{align}
    \mnorm{\Phi^{(j-1)}(x)}_2 &\geq \mnorm{\widetilde{\Phi}^{(j-1)}(x)}_2 - \mnorm{\Phi^{(j-1)}(x)-\widetilde{\Phi}^{(j-1)}(x)}_2\nonumber \\
    \overset{\eqref{eq:con3},\eqref{eq:hombiasdiff}}&{\geq} \frac{\theta}{4} - 3^L \cdot \lambda' \nonumber\\
    &\geq 4\tilde{C_2} \cdot 3^L \cdot \lambda' \cdot \eps^{-1}- 3^L \cdot \lambda' \nonumber\\
    &= 3^L \cdot \lambda' \cdot (4\tilde{C_2} \cdot \eps^{-1}-\underbrace{1}_{\leq 2\tilde{C_2} \cdot \eps^{-1}}) \nonumber\\
    &\geq 2\tilde{C_2} \cdot 3^L \cdot \lambda' \cdot \eps^{-1} \nonumber\\
    \label{eq:low_bound}
    \overset{\eps \leq \delta}&{\geq} \tilde{C_2} \cdot (\delta/2)^{-1} \lambda'.
    \end{align}
    The idea is to apply \Cref{prop:farout} with $\delta$ replaced by $\delta / 2$ to the set $T^{(j-1)}$.
    First of all, note that 
    \[
    \eps = c' \cdot \frac{\delta}{2} \cdot (2 \cdot \ln(1/\delta))^{-1/2}
    \overset{\text{\cshref{lem:log_b}}}{\leq} \widetilde{c_2} \cdot \frac{\delta}{2} \cdot \ln^{-1/2}(2/\delta).
    \]
    Moreover, using \eqref{eq:extra_con},
    \begin{align}
     (\delta/2)^{-1} \cdot \ln(\ee \cdot \mathcal{N}(T^{(j-1)},\eps)) 
     &\leq 2 \cdot \delta^{-1} \cdot \ln(\ee \cdot \mathcal{N}(T^{(j-1)},\eps/2)) \nonumber \\ 
    &\leq 2C_1 \cdot \delta^{-1} \cdot d \cdot \ln(1/\eps)  \nonumber\\
    &\leq 2C_1 \cdot \delta^{-1} \cdot d \cdot \ln\Bigg( \frac{2\sqrt{2}}{c'}\cdot \delta^{-1} \cdot \underbrace{\ln^{1/2}(1/\delta)}_{\leq 1/\delta}\Bigg) \nonumber\\
    &\leq 2C_1 \cdot \delta^{-1} \cdot d \cdot \ln\left( \frac{2\sqrt{2}}{c'} \cdot \delta^{-2}\right) \nonumber \\
    \overset{\text{\cshref{lem:log_b}}}&{\leq} C_2 \cdot \delta^{-1} \cdot d \cdot \ln(1/ \delta) \nonumber
    \end{align}
    with a suitable constant $C_2 > 0$.
    Further, using \eqref{eq:con5}, and $(a+b)^2 \leq 2(a^2 + b^2)$,
    \begin{align}
    &\norel (\delta/2)^{-3} \cdot w^2((T^{(j-1)}- T^{(j-1)}) \cap B_\mu(0,\eps))\nonumber \\
    &\leq  16 \cdot \delta^{-3}\cdot C_1^2 \cdot d \cdot \eps^2 \cdot 
    \left(L+\ln(1/\eps)\right) \nonumber\\
    & \leq  16C_1^2 \cdot \left(\frac{c'}{2\sqrt{2}}\right)^2 \cdot d \cdot \delta^{-3} \cdot \delta^2 \cdot 
    \ln^{-1}(1/\delta) \cdot \left(L+\ln\Bigg( \frac{2\sqrt{2}}{c'} \cdot \delta^{-1} \cdot \underbrace{\ln^{1/2}(1/\delta)}_{\leq 1/\delta}\Bigg)\right) \nonumber\\
    \label{eq:cond}
    \overset{\text{\cshref{lem:log_b}}}&{\leq} C_2 \cdot d \cdot \delta^{-1} \cdot \left(1+\ln^{-1}(1/\delta) \cdot L\right)\nonumber
    \end{align}
    after possibly increasing the size of $C_2$.
    Since by assumption 
    \[
    \delta N \geq C \cdot d \cdot \ln(1/\delta) \geq C
    \quad\text{and}\quad
    \delta N\geq C \cdot dL\ln^{-1}(1/\delta),
    \]
    we note that the assumptions of \Cref{prop:farout} are satisfied
    by taking $C$ large enough. 
    Note that 
    \[
    \{\alpha \cdot y : \ y \in T^{(j-1)}, \ \alpha \geq \tilde{C_2} \cdot (\delta/2)^{-1}\lambda'\} \supseteq \{\Phi^{(j-1)}(x): \ x \in \RR^d, \mnorm{x}_2 \geq \theta\}
    \]
    as follows from \eqref{eq:low_bound} and the definition of $T^{(j-1)}$, which thus implies 
    \begin{align*}
    &\norel \text{for all } x \in \RR^d \text{ with } \mnorm{x}_2 \geq \theta: \nonumber\\
    &\# \left\{ i \in \{1,\dots, N\}: \ \sgn_{-}\left((W^{(j)}\Phi^{(j-1)}(x) + b^{(j)})_i\right)
    \neq 
    \sgn_{-}\left((W^{(j)}\Phi^{(j-1)}(x))_i\right)
    \right\} \leq \frac{\delta}{2} \cdot N
    \end{align*}
    with probability at least $1-\exp(-c_4 \cdot \delta N)$, according to \Cref{prop:farout} with an absolute constant $c_4 >0$.
    
    Moreover, we compute
    \begin{align*}
        \mnorm{\frac{\Phi^{(j-1)}(x)}{\mnorm{\Phi^{(j-1)}(x)}_2} - \frac{\widetilde{\Phi}^{(j-1)}(x)}{\mnorm{\widetilde{\Phi}^{(j-1)}(x)}_2}}_2
        \overset{\text{\cshref{lem:diff_bound}}}&{\leq} 
        \frac{2\mnorm{\Phi^{(j-1)}(x)- \widetilde{\Phi}^{(j-1)}(x)}_2}{\underbrace{\mnorm{\widetilde{\Phi}^{(j-1)}(x)}}_{\geq \theta/4}} \\
        \overset{\eqref{eq:con3},\eqref{eq:hombiasdiff}}&{\leq} \frac{8 \cdot 3^L \cdot \lambda'}{\theta} \leq \eps/2
    \end{align*}
    for every $x \in \RR^d$ with $\mnorm{x}_2 \geq \theta$. 
    Since the assumptions of \Cref{lem:recht_upper_original} are satisfied (since these are the same conditions as for \Cref{prop:farout} and we may possibly increase $C$),
    we get
    \begin{align*}
    &\norel \text{for all } x \in \RR^d, \ \mnorm{x}_2 \geq \theta: \\
    &\# \left\{ i: \ \sgn_{-}\left((W^{(j)}\Phi^{(j-1)}(x))_i\right)
    \neq 
    \sgn_{-}\left((W^{(j)}\widetilde{\Phi}^{(j-1)}(x) )_i\right)
    \right\} \leq \frac{\delta}{2} \cdot N
    \end{align*}
    with probability at least $1-\exp(-c_5 \cdot \delta N)$ over the randomness in $W^{(j)}$ with an absolute constant $c_5 > 0$.
    Overall, we have 
    \begin{align}
        \Tr \abs{\aj (x)}
        &= 
        \# \left\{ i \in \{1, \dots, N\}: \ \sgn_{-}\left((W^{(j)}\Phi^{(j-1)}(x) + b^{(j)})_i\right)
    \neq 
    \sgn_{-}\left((W^{(j)}\widetilde{\Phi}^{(j-1)}(x))_i\right)
    \right\}  \nonumber\\
    &\leq \# \left\{ i \in \{1,\dots, N\}: \ \sgn_{-}\left((W^{(j)}\Phi^{(j-1)}(x) + b^{(j)})_i\right)
    \neq 
    \sgn_{-}\left((W^{(j)}\Phi^{(j-1)}(x))_i\right)
    \right\} \nonumber\\
    \label{eq:fin}
    & \norel +\# \left\{ i \in \{1,\dots,N\}: \ \sgn_{-}\left((W^{(j)}\Phi^{(j-1)}(x))_i\right)
    \neq 
    \sgn_{-}\left((W^{(j)}\widetilde{\Phi}^{(j-1)}(x))_i\right)
    \right\}
    \leq \delta N
    \end{align}
    uniformly over $x \in \RR^d$ with $\mnorm{x}_2 \geq \theta$ with probability at least $1-2\exp(-c_6 \cdot \delta  N)$ over the randomness in $W^{(j)}$ with $c_6 \defeq \min\{c_4,c_5\}$.

    In the end, the claim follows by reintroducing the randomness in $W^{(0)},\ldots,W^{(j-1)}$: 
    Note that \Cref{eq:con2,eq:con3,eq:con5,eq:extra_con} occur with probability at least 
    \[
    1- \exp(-c_1 \cdot N) - \exp(-c_2 \cdot \delta N) - \exp(-c_3 \cdot \delta N).
    \]
    Adding \eqref{eq:fin}, we get
    \begin{align*}
    &\norel 1- \exp(-c_1 \cdot N) - \exp(-c_2 \cdot \delta N) - \exp(-c_3 \cdot \delta N) - 2\exp(-c_6 \cdot \delta N) \\
    &\geq 1 - 5\exp(-c_7 \cdot \delta N),
    \end{align*}
    letting $c_7 \defeq \min\{c_1,c_2,c_3,c_6\}$.
    The claim follows letting $c \defeq c_7 / 2$, since we may assume 
    \[
    \delta N \geq C \geq \frac{2\ln(5)}{c_7}.\qedhere
    \]
\end{proof}

We can now prove the main result of this section. 
\begin{proof}[Proof of \Cref{thm:grad_diff}]
    Let $C_1,c_1 > 0$ be chosen according to \Cref{prop:theta}
    and choose $C_2, c_2>0$ according to \Cref{prop:advanced_lip}.
    Set $\theta \defeq C_1 \cdot 3^L \cdot  \lambda' \cdot \delta^{-1} \cdot \ln^{1/2}(1/\delta)$.
   We aim to show the following statement:
   For every $\ell \in \{0,\ldots,L\}$, 
    the event 
    \begin{align}
    &\norel\underset{\mnorm{x}_2 \geq \theta, A \in \mathcal{A}_\delta^{(j)}}{\sup}
    \op{W^{(L)}\left[\jac \Phi^{(j) \to (L-1)}(x)- \jac \tilde{\Phi}^{(j) \to (L-1)}(x)\right]A} \nonumber\\
    \label{eq:indshow}
    &\leq C' \cdot L \cdot  (\sqrt{N}\cdot\delta\ln(1/\delta) + \sqrt{\delta \ln(1/\delta)Ld\ln(N/d)}) \quad \text{for every } j \in \{\ell, \dots, L\}
    \end{align}
    occurs with probability at least $1- 3(L-\ell)\exp(-c' \cdot \delta N)$, where $C',c' > 0$ are absolute constants that are exactly determined later. 
    This will then imply the statement of the theorem by studying the special case $\ell = 0$
    and noting that $\delta N \geq C \cdot \ln(L)$.
    
    The proof is via induction over $\ell$, starting with the case $\ell = L$, for which there is nothing to show, since by definition we have 
    \[
    \jac \Phi^{(L) \to (L-1)}(x) = 
    \jac \widetilde{\Phi}^{(L) \to (L-1)}(x) = I_{N \times N} 
    \]
    for every $x \in \RR^d$.
    
    For the induction step, we take $\ell \in \{0,\dots,L-1\}$ and assume that $\ell + 1$ satisfies the claim.
    Let $\E_1$ be the event defined by 
    \[
    \text{for all } x \in \RR^d \text{ with } \mnorm{x}_2 \geq \theta \text{ and } j \in \{0, \dots, L-1\}: 
    \quad 
    A^{(j)}(x) \in \mathcal{A}_\delta^{(1)},
    \]
    which occurs with probability at least $1-\exp(-c_1 \cdot \delta N)$ according to  
    \Cref{prop:theta} under the application of a union bound (noting that $\delta N \geq C \cdot \ln(\ee L)$).
    Let the event $\E_2$ be defined by the property 
    \begin{align*}
    \underset{x \in \RR^d,A \in \mathcal{A}_\delta^{(1)}}{\sup} \quad \op{W^{(L)}\jac \tilde{\Phi}^{(j+1) \to (L-1)}(x)A} 
    \leq C_2 \cdot \sqrt{N \cdot \delta \ln(1/\delta)} \quad \text{for all } j \in \{0,\dots, L-1\},
    \end{align*}
    which occurs with probability at least $1-\exp(-c_2\cdot \delta N)$ according to \Cref{prop:advanced_lip}.
    Let $\E_3$ be the event defined by 
    \begin{align*}
    &\norel\underset{\mnorm{x}_2 \geq \theta, A \in \mathcal{A}_\delta^{(j)}}{\sup}
    \op{W^{(L)}\left[\jac \Phi^{(j) \to (L-1)}(x)- \jac \tilde{\Phi}^{(j) \to (L-1)}(x)\right]A} \\
    &\leq C' \cdot L \cdot  \left(\sqrt{N}\cdot\delta\ln(1/\delta) + \sqrt{\delta \ln(1/\delta)Ld\ln(N/d)}\right) \quad \text{for every } j \in \{\ell+1, \dots, L\}.
    \end{align*}
    According to the induction hypothesis, $\E_3$ occurs with probability at least $1- 3(L-\ell -1) \exp(-c' \cdot \delta N)$. 
    Lastly, let $\E_4$ be the event on which 
    \[
    \underset{x \in \RR^d, A_1 \in \mathcal{A}_\delta^{(1)}, A_2 \in \mathcal{A}_\delta^{(\ell)}}{\sup} \op{A_1 W^{(j)} \jac \tilde{\Phi}^{(\ell) \to (j-1)}(x)A_2}
    \leq C_3 \cdot \left(\sqrt{\delta \ln(1/\delta)} + \sqrt{\frac{Ld\ln(N/d)}{N}}\right)
    \]
    for every $j \in \{\ell, \dots, L-1\}$, which occurs with probability at least $1-\exp(-c_3 \cdot \delta N)$ according to \Cref{lem:auxil},
    with an absolute constant $c_3>0$.
    Here, we again used $\delta N \geq C \cdot \ln(\ee L)$.
    Again, let $A^{(j)}(x)$ be defined as in \eqref{eq:ajdef}.
    On the intersection $\E_1 \cap \E_2 \cap \E_3 \cap \E_4$,
    we use a decomposition similar to \Cref{lem:decomp} to get
    \begin{align*}
         &\norel\underset{\mnorm{x}_2 \geq \theta, A \in \mathcal{A}_\delta^{(\ell)}}{\sup}\op{W^{(L)}\left[\jac \Phi^{(\ell) \to (L-1)}(x)- \jac \tilde{\Phi}^{(\ell) \to (L-1)}(x)\right]A} \nonumber\\
         &\leq \sum_{j=\ell}^{L-1} \underset{\mnorm{x}_2 \geq \theta, A \in \mathcal{A}_\delta^{(\ell)}}{\sup}\op{W^{(L)}\jac \Phi^{(j+1) \to (L-1)}
         (x)A^{(j)}(x) W^{(j)}\jac \tilde{\Phi}^{(\ell) \to (j-1)}(x)A} \nonumber\\
         &\leq \sum_{j=\ell}^{L-1} \left(\underset{\mnorm{x}_2 \geq \theta}{\sup}\op{W^{(L)}\jac \Phi^{(j+1) \to (L-1)}
         (x)A^{(j)}(x)}\right) \nonumber\\
         &\hspace{1.2cm}\cdot\left(\underset{\mnorm{x}_2 \geq \theta, A \in \mathcal{A}_\delta^{(\ell)}}{\sup}\op{A^{(j)}(x)W^{(j)}\jac \tilde{\Phi}^{(\ell) \to (j-1)}(x)A}\right) \\
         \overset{\E_1}&{\leq}\sum_{j=\ell}^{L-1} \left(\underset{\mnorm{x}_2 \geq \theta, A_2 \in \mathcal{A}_\delta^{(1)}}{\sup}\op{W^{(L)}\jac \Phi^{(j+1) \to (L-1)}
         (x)A_2}\right) \nonumber\\
         &\hspace{1.2cm}\cdot\left(\underset{\mnorm{x}_2 \geq \theta, A_2 \in \mathcal{A}_\delta^{(1)}, A_1 \in \mathcal{A}_\delta^{(\ell)}}{\sup}\op{A_2W^{(j)}\jac \tilde{\Phi}^{(\ell) \to (j-1)}(x)A_1}\right) \\
         \overset{\E_4}&{\leq} C_3 \cdot \left(\sqrt{\delta \ln(1/\delta)} + \sqrt{\frac{Ld\ln(N/d)}{N}}\right) \cdot \sum_{j=\ell}^{L-1} \left(\underset{\mnorm{x}_2 \geq \theta, A_2 \in \mathcal{A}_\delta^{(1)}}{\sup}\op{W^{(L)}\jac \Phi^{(j+1) \to (L-1)}
         (x)A_2}\right).
    \end{align*}
    
    For arbitrary $j \in \{\ell, \ldots, L-1\}$,
    \begin{align*}
    &\norel \underset{\mnorm{x}_2 \geq \theta, A_2 \in \mathcal{A}_\delta^{(1)}}{\sup} \op{W^{(L)}\jac \Phi^{(j+1) \to (L-1)}(x)A_2} \\
    &\leq \underset{\mnorm{x}_2 \geq \theta, A_2 \in \mathcal{A}_\delta^{(1)}}{\sup}\op{W^{(L)}\left[\jac \Phi^{(j+1) \to (L-1)}(x) - \jac \tilde{\Phi}^{(j+1) \to (L-1)}(x)\right]A_2} 
            \\
            &\norel+ \underset{\mnorm{x}_2 \geq \theta, A_2 \in \mathcal{A}_\delta^{(1)}}{\sup}\op{W^{(L)}\jac \tilde{\Phi}^{(j+1) \to (L-1)}(x)A_2} \\
    \overset{\E_2 \cap \E_3}&{\leq} C' \cdot L \cdot  \left(\sqrt{N}\cdot\delta\ln(1/\delta) + \sqrt{\delta \ln(1/\delta)Ld\ln(N/d)}\right) + C_2 \cdot \sqrt{N \cdot \delta \ln(1/\delta)}.
    \end{align*}
    Since $\delta \cdot \ln(1/\delta) \leq c \cdot 1/L^2$ and $N \geq C \cdot L^3d\ln(N/d)$, we can, by possibly increasing $C$ and decreasing $c$, ensure
    \[
    \sqrt{N}\cdot\delta \ln(1/\delta) \leq \frac{C_2}{2C'} \cdot \frac{\sqrt{N\cdot \delta \ln(1/\delta)}}{L}
    \quad 
    \text{and}
    \quad 
    \sqrt{\delta \ln(1/\delta)Ld\ln(N/d)} \leq \frac{C_2}{2C'} \cdot \frac{\sqrt{N\cdot \delta \ln(1/\delta)}}{L},
    \]
    which implies 
    \[
    C' \cdot L \cdot  (\sqrt{N}\cdot\delta\ln(1/\delta) + \sqrt{\delta \ln(1/\delta)Ld\ln(N/d)}) \leq C_2 \sqrt{N \cdot \delta\ln(1/\delta)}.
    \]
    We thus get 
    \[
    \underset{\mnorm{x}_2 \geq \theta, A_2 \in \mathcal{A}_\delta^{(1)}}{\sup} \op{W^{(L)}\jac \Phi^{(j+1) \to (L-1)}(x)A_2} 
    \leq 2C_2 \cdot \sqrt{N \cdot \delta \ln(1/\delta)}
    \]
    for every $j \in \{\ell, \dots, L-1\}$ on $\E_1 \cap \E_2 \cap \E_3 \cap \E_4$.

    In total, on $\E_1 \cap \E_2 \cap \E_3 \cap \E_4$,
    \begin{align*}
        &\norel\underset{\mnorm{x}_2 \geq \theta, A \in \mathcal{A}_\delta^{(\ell)}}{\sup}\op{W^{(L)}\left[\jac \Phi^{(\ell) \to (L-1)}(x)- \jac \tilde{\Phi}^{(\ell) \to (L-1)}(x)\right]A} \\
        &\leq 2C_2C_3 \cdot L \cdot \left(\sqrt{N} \cdot \delta\ln(1/\delta) + \sqrt{Ld\ln(N/d) \delta \ln(1/\delta)}\right)
    \end{align*}
    and 
    \begin{align*}
    \PP(\E_1 \cap \E_2 \cap \E_3 \cap \E_4) &\geq 1 - \exp(-c_1 \cdot \delta N)- \exp(-c_2 \cdot \delta N)- 3(L\!-\! \ell\! -\! 1)\exp(-c' \cdot \delta N) - \exp(-c_3 \cdot \delta N) \\
    &\geq 1- 3(L-\ell)\exp(-c' \cdot \delta N),
    \end{align*}
    letting $c' \defeq \min\{c_1,c_2,c_3\}$, which proves the induction step by choosing $C' = 2C_2C_3$.
\end{proof}

Note that there is a trade-off between decreasing the size of $\delta$ (in order 
to make the bound sharper) and increasing the size of $\delta$ in order to make $\theta$
smaller.

\renewcommand{\DD}{\hat{D}}

%% file: bias_bounds_proofs.tex
\section{Bounds on the Lipschitz constants for networks with biases}\label{sec:bias_bounds_proofs}

In this section, we use the results from \Cref{sec:hom_diff} to prove bounds on the Lipschitz constant
of random ReLU nets with general biases. 
Precisely, in order to establish our upper bound (see \Cref{sec:upper_proof_bias}), we assume that the entries of the bias vectors are drawn independently from
(possibly different) symmetric distributions and satisfy a small ball property, 
i.e., 
we assume that there exists a constant $C_\tau > 0$ for which 
\[
\sup_{\ell \in \{0,\ldots,L-1\}, i \in \{1,\ldots,N\},t \in \RR, \eps > 0} 
\left(\eps^{-1} \cdot \PP\left(b_i^{(\ell)} \in (t-\eps, t + \eps)\right)\right) \leq C_\tau.
\]
For our lower bounds (see \Cref{sec:lower_general_proofs}), no special assumptions on the distribution of the biases are required.

\subsection{Upper bound}\label{sec:upper_proof_bias}
Let $R > 0$ be arbitrary and let $\Phi:\RR^d \to \RR$ be a random ReLU neural network
following \Cref{assum:3}.
The first goal is to establish a bound on 
\[
\underset{x \in B_d(0,R)}{\sup} \mnorm{\nablaa \Phi(x)}_2.
\]
This works very similarly to the proof of \Cref{prop:advanced_lip}, with the difference that \Cref{corr:tess} is used instead of 
\Cref{lem:recht_upper_original}.
We start by bounding the Jacobians of the intermediate maps $\Phi^{(\ell)}$.

\begin{lemma}\label{prop:not_last_layer}
    There exist absolute constants $C,c > 0$ such that the following holds. 
    Let $\Phi:\RR^d \to \RR$ be a random ReLU network following \Cref{assum:3} with $L$ hidden layers of width $N$ and 
    with small ball constant $C_\tau > 0$.
    Let $C_\tau' \defeq \sqrt{2/N} \cdot C_\tau$.
    Let further $R\geq (C_\tau')^{-1}$, $\delta \in (0, \ee^{-1})$ and assume $N \geq C \cdot d$ and
    \begin{align*}
    N \geq C \cdot L^3d\ln(N/d), \quad 
    \quad N \geq C \cdot L^2 d \ln(RC_\tau' /\delta), \quad
    \delta N \geq C \cdot d  \cdot \ln(RC_\tau' / \delta),\quad \delta\ln(1/\delta) \leq c \cdot 1/L^2.
    \end{align*}
    Let $0<\eps \leq c \cdot (C_\tau')^{-1} \cdot \delta \cdot \ln^{-1/2}(1/\delta)$.
    Then, for every $\ell_2 \in \{-1,\dots,L-1\}$ with probability at least $1- 3(\ell_2+1)\exp(-c \cdot \delta N)$,
    the following three properties are simultaneously satisfied:
    \begin{enumerate}
    \item{
    \label{item:a}
    $\displaystyle \sup_{x,y \in B_d(0,R), \mnorm{x-y}_2 \leq \eps} \Tr \abs{D^{(j)}(x) - D^{(j)}(y)} \leq \delta N \quad \text{for all } j \in \{0, \dots, \ell_2\}$,
    }
    \item{
\label{item:b}
    $\displaystyle \underset{A \in A_\delta^{(1)}}{\sup_{x \in B_d(0,R)}} \op{A W^{(j)} \jac \Phi^{(j-1)}(x)} \! \!\!\leq C \cdot \left(\sqrt{\delta \ln(1/\delta)} + \sqrt{\frac{Ld\ln(N/d)}{N}}\right) \quad \text{for all } j \in \{0, \dots, \ell_2\}$,
    }
       \item{ \label{item:c}$\displaystyle
         {\underset{x \in B_d(0,R)}{\sup}} \mnorm{\jac \Phi^{(j)}(x)}_{2 \to 2 }
    \leq C
         \quad \text{for all }  j \in \{-1, \dots, \ell_2\}$.}
    \end{enumerate}
\end{lemma}
\begin{proof}
    We let $C_1 > 0$ be the absolute constant from \Cref{lem:pw_2} and $c' > 0$ be the absolute constant from \Cref{lem:tess_bias_lip}.
    It suffices to establish the result for $\eps = \frac{c'}{2C_1} \cdot (C_\tau')^{-1} \cdot \delta \cdot \ln^{-1/2}(1/\delta)$.
The proof is via induction over $\ell_2$.
        In the case $\ell_2 = -1$, only \eqref{item:c} needs to be considered. 
        However, here we have, since necessarily $j = -1$,
        \[
        {\underset{x \in B_d(0,R)}{\sup}} \op{\jac \Phi^{(j)}(x)} = \op{I_{d \times d}}
        =  1.
        \]
        
        Let us move to the induction step. 
        Pick $\ell_2 \in \{0,\dots,L-1\}$ and assume that with probability at least $1- 3\ell_2\exp(-c \cdot \delta N)$,
        \begin{enumerate}[label=(\alph*)]
        \item{
        \label{item:aa}
        $\displaystyle \sup_{x,y \in B_d(0,R), \mnorm{x-y}_2 \leq \eps} \Tr \abs{D^{(j)}(x) - D^{(j)}(y)} \leq \delta N \quad \text{for all } j \in \{0, \dots, \ell_2-1\}$,
        }
        \item{
        \label{item:bb}
        $\displaystyle \underset{A \in A_\delta^{(1)}}{\sup_{x \in B_d(0,R)}} \op{A W^{(j)} \jac \Phi^{(j-1)}(x)} \leq C_2 \cdot \left(\sqrt{\delta \ln(1/\delta)} + \sqrt{\frac{Ld\ln(N/d)}{N}}\right) \quad  \\
        \text{for all } j \in \{0, \dots, \ell_2-1\}$,
        }
           \item{ \label{item:cc}$\displaystyle
             {\underset{x \in B_d(0,R)}{\sup}} \mnorm{\jac \Phi^{(j)}(x)}_{2 \to 2 }
        \leq 2C_1
             \quad \text{for all }  j \in \{-1, \dots, \ell_2-1\}$}
    \end{enumerate}
        where $C_2, c > 0$
        are appropriate constants to be exactly determined later and we recall that $C_1>0$ is the absolute constant from \Cref{lem:pw_2}.
        
        We now condition on the weights $W^{(0)}, \dots, W^{(\ell_2 -1)}$ and biases $b^{(0)}, \dots, b^{(\ell_2 - 1)}$ and assume that properties \ref{item:aa}-\ref{item:cc} are satisfied. 
        By \ref{item:cc} and \Cref{thm:glob}, this in particular implies 
        \begin{equation}\label{eq:lip_help_2}
        \lip_{2 \to 2}\left(\fres{\Phi^{(j)}}{B_d(0,R)}\right) \leq \sup_{x \in B_d(0,R)} \op{\jac \Phi^{(j)}(x)} \leq 2C_1 \quad \text{for all } j \in \{-1, \dots, \ell_2 -1\}. 
        \end{equation}
        
        We set $T \defeq \Phi^{(\ell_2 - 1)}(B_d(0,R))$.
        Note that for $x, y \in B_d(0,R)$ with $\mnorm{x-y}_2\leq \eps$ we have 
        \[
        \mnorm{\Phi^{(\ell_2 - 1)}(x)-\Phi^{(\ell_2 - 1)}(y)}_2 \leq 2C_1 \cdot \eps = c' \cdot (C_\tau ')^{-1} \cdot \delta \cdot \ln^{-1/2}(1/\delta).
        \]
        Hence, we may apply \Cref{lem:tess_bias_lip} by using that 
        \[
        N \geq C \cdot d \cdot \delta^{-1} \cdot \ln(RC_\tau'/\delta)
        \]
        and get 
        \begin{align}
        &\norel\underset{x,y \in B_d(0,R), \mnorm{x - y}_2 \leq \eps}{\sup}\Tr \abs{D^{(\ell_2)}(y) - D^{(\ell_2)}(x)} \nonumber\\
        \label{eq:ee_1_2}
        &\leq \underset{\mnorm{x-y}_2 \leq c' \cdot (C_\tau')^{-1}\cdot \delta \cdot \ln^{-1/2}(1/\delta)}{\sup_{x,y \in T,}} \hspace{-1cm}
        \#\left\{i \in \{1,\dots,N\}: \ \sgn_{-}((W^{(\ell_2)}x)_i + b^{(\ell_2)}_i) \neq \sgn_{-}((W^{(\ell_2)}y)_i+ b^{(\ell_2)}_i)\right\} \leq \delta N
        \end{align}
        with probability at least $1- \exp(-c' \cdot \delta N)$ over the randomness in $W^{(\ell_2)}$ and $b^{(\ell_2)}$. 
        We call the event on which \eqref{eq:ee_1_2} holds $\E_1$.

        Further, 
        set 
        \[
        \mathcal{V} \defeq \left\{ \jac \Phi^{(\ell_2-1)}(x): \ x \in B_d(0,R)\right\}.
        \]
        Note that we have $\op{M} \leq 2C_1$ and $\rang(M) \leq d \leq \delta N$ for every $M \in \mathcal{V}$ by \ref{item:33}. 
        Here, we used that $d \leq C \cdot d \cdot \ln(RC\tau'/\delta) \leq \delta N$.
        Hence, for fixed $x \in B_d(0,R)$ and $A \in \mathcal{A}_\delta^{(1)}$, we find, 
        using \Cref{lem:pw_4},
        \begin{align*}
            \op{AW^{(\ell_2)}\jac \Phi^{ (\ell_2-1)}(x)}
            &\leq \sqrt{2} \cdot 2C_1 \cdot \frac{\sqrt{\delta N} + \sqrt{\delta N}+ t}{\sqrt{N}}
        \end{align*}
        with probability at least $1- \exp(-c_1 \cdot t^2)$ for every $t \geq C_3$ for absolute constants $C_3, c_1 > 0$.
        Moreover, combining \Cref{lem:card_b} and \cite[Lemma~5.7]{geuchen2024upper}, we get 
        \begin{align*}
        \ln(\#\mathcal{V}) + \ln\left(\#\mathcal{A}_\delta^{(1)}\right)
        &\leq L(d+1) \cdot \ln\left(\frac{\ee N}{d+1}\right) +  \delta N \cdot \ln(4\ee /\delta)\nonumber\\
        \overset{\text{\cshref{lem:log_b}}}&{\leq} 2\ee L d \cdot \ln(N/d) + 4\ee \cdot N \cdot \delta \ln(1/\delta) \nonumber\\
        &\leq 4\ee  (Ld \ln(N/d) + N \delta \ln(1/\delta)).
        \end{align*}
        Hence, via a union bound, we observe 
        \[
        \underset{A \in \mathcal{A}_\delta^{(1)}}{\underset{x \in B_d(0,R)}{\sup}} \op{AW^{(\ell_2)}\jac \Phi^{ (\ell_2-1)}(x)}
        \leq \sqrt{2} \cdot 2C_1 \cdot \frac{\sqrt{\delta N} + \sqrt{\delta N}+ t}{\sqrt{N}}
        \]
        with probability at least 
        \[
        1- \exp(4\ee  ( Ld\ln(N/d) +  N\delta\ln(1/\delta))-c_1t^2)
        \]
        for every $t \geq C_3$.
        We then explicitly pick 
        \[
        t \defeq \sqrt{c_1^{-1} \cdot 5\ee \cdot (Ld \ln(N/d) + N\delta \cdot \ln(1/\delta))}, 
        \]
        which yields the existence of an absolute constant $C_2 > 0$, such that 
        \begin{align}\label{eq:h3}
        \underset{A \in \mathcal{A}_\delta^{(1)}}{\underset{x \in B_d(0,R)}{\sup}} \op{AW^{(\ell_2)}\jac \Phi^{(\ell_2-1)}(x)}
        \leq C_2  \cdot \frac{\sqrt{N \cdot \delta \cdot \ln(1/\delta)} + \sqrt{Ld\ln(N/d)}}{\sqrt{N}}
        \end{align}
        with probability at least 
        \[
        1-\exp(-\ee \cdot ( Ld\ln(N/d) + N\delta \ln(1/\delta))) \geq 1 - \exp(- \delta N)
        \]
        over the randomness in $W^{(\ell_2)}$.
        We call the event defined by property \eqref{eq:h3} $\E_2$.

        Moreover, 
        let $\neps\subseteq B_d(0,R)$ be an $\eps$-net of $B_d(0,R)$ with 
        \begin{align}
        \ln(\#\neps) &\leq d \cdot \ln(6R/\eps) 
        \remove{= d \cdot \ln\left(\frac{12C_1\cdot RC_\tau' \cdot \ln^{1/2}(1/\delta)}{c' \cdot \delta}\right)
        \overset{\ln^{1/2}(1/\delta) \leq 1/\delta}{\leq} d \cdot \ln\left(\frac{12C_1 \cdot RC_\tau'}{c' \cdot \delta^2}\right)} \nonumber\\
        \label{eq:neps_b_22}
        &\remove{\leq 2 \cdot d \cdot \ln(12C_1 \cdot RC_\tau'/(c'\delta))} \leq 
        C_4 \cdot d \cdot \ln(RC_\tau'/\delta)
        \end{align}
        with an absolute constant $C_4 > 0$ using \Cref{lem:log_b}. 
        We may assume that $0 \notin \neps$.\footnote{To see this, let $\neps$ be an $\eps$-net
        of $B_d(0,R) \setminus \{0\}$ with $\# \neps \leq (6R/\eps)^d$ (exists due to \cite[Ex.~4.2.10, Corr.~4.2.13]{vershynin_high-dimensional_2018}).
        Since $\neps$ is finite, the map $x \mapsto \min_{x^\ast \in \neps} \mnorm{x-x^\ast}_2$ is continuous.
        Letting $x \to 0$ proves $\min_{x^\ast \in \neps}\mnorm{x^\ast}_2 \leq \eps$.
        Thus, $\neps$ is already an $\eps$-net of $B_d(0,R)$.}
        For $x \in B_d(0,R)$ we denote by $\pi(x) \in \neps$ a net point with $\mnorm{x- \pi(x)}_2 \leq \eps$.
        On $\E_1 \cap \E_2$, by \Cref{lem:decomp},
        \begin{align*}
            &\norel\underset{x \in B_d(0,R)}{\sup}
            \op{\jac \Phi^{ (\ell_2)}(\pi(x)) - 
            \jac \Phi^{(\ell_2)}(x)} \nonumber\\
            &\leq {\underset{x \in B_d(0,R)}{\sup}} \sum_{j= 0}^{\ell_2} 
            \op{\jac \Phi^{(j+1) \to (\ell_2)}(\pi(x))(D^{(j)}(\pi(x)) - D^{(j)}(x))
            W^{(j)}\jac \Phi^{ (j-1)}(x)} \nonumber\\
            &\leq 
            {\underset{x \in B_d(0,R)}{\sup}} \Bigg( 
            \sum_{j= 0}^{\ell_2} \op{\jac \Phi^{(j+1) \to (\ell_2)}(\pi(x))(D^{(j)}(\pi(x)) - D^{(j)}(x))} \nonumber\\
            & \hspace{4.5cm}\cdot 
            \op{\abs{D^{(j)}(\pi(x)) - D^{(j)}(x)}W^{(j)}\jac \Phi^{ (j-1)}(x)} \Bigg)\nonumber\\
            &\leq \sum_{j= 0}^{\ell_2} \left( \underset{x \in B_d(0,R)}{\sup} \op{\jac \Phi^{(j+1) \to (\ell_2)}(\pi(x))(D^{(j)}(\pi(x)) - D^{(j)}(x))} \right) \nonumber\\
            & \hspace{1.2cm}\cdot 
            \left( {\underset{x \in B_d(0,R)}{\sup}} \op{\abs{D^{(j)}(\pi(x)) - D^{(j)}(x)}W^{(j)}\jac \Phi^{ (j-1)}(x)}\right) \\
            \overset{\E_1, \ref{item:aa}}&{\leq} \sum_{j= 0}^{\ell_2} \left(\sup_{x^\ast \in \neps, A \in \mathcal{A}_\delta^{(1)}} \op{\jac \Phi^{(j+1) \to (\ell_2)}(x^\ast)A}\right) 
            \cdot \left(\underset{A \in \mathcal{A}_\delta^{(1)}}{\sup_{x \in B_d(0,R)}} \op{A W^{(j)} \jac \Phi^{ (j-1)}(x)}\right) \\
            \overset{\E_2, \ref{item:bb}}&{\leq} C_2 \cdot \left(\sqrt{\delta \ln(1/\delta)} + \sqrt{\frac{Ld\ln(N/d)}{N}}\right) \cdot \sum_{j= 0}^{\ell_2} \left(\sup_{x^\ast \in \neps, A_2 \in \mathcal{A}_\delta^{(1)}} \op{\jac \Phi^{(j+1) \to (\ell_2)}(x^\ast)A_2}\right).
        \end{align*}
        Since $\delta \ln(1/\delta) \leq c/L^2$ and $N \geq C \cdot L^3 d \ln(N/d)$, 
        \[
        C_2 \cdot \left(\sqrt{\delta \ln(1/\delta)} + \sqrt{\frac{Ld\ln(N/d)}{N}}\right) \leq \frac{1}{L}.
        \]
        Hence, on $\E_1 \cap \E_2$, 
        \begin{align*}
        &\norel \underset{x \in B_d(0,R)}{\sup}
            \op{\jac \Phi^{(\ell_2)}(\pi(x)) - 
            \jac \Phi^{(\ell_2)}(x)} \\
            &\leq \frac{1}{L} 
            \cdot \sum_{j= 0}^{\ell_2} \left(\sup_{x^\ast \in \neps, A \in \mathcal{A}_\delta^{(1)}} \op{\jac \Phi^{(j+1) \to (\ell_2)}(x^\ast)A}\right).
        \end{align*}
        
        Recall that up to now we have conditioned on $W^{(0)}, \dots, W^{(\ell_2 -1)}$ and $b^{(0)}, \dots, b^{(\ell_2 -1)}$. 
        Reintroducing the randomness over these random matrices, we note that the following four properties
        are simultaneously satisfied with
        probability at least $1- 3\ell_2\exp(-c \cdot \delta N) - \exp(-c' \cdot \delta N) - \exp(- \delta N)$:
        \begin{enumerate}[label=(\roman*)]
            \item{
        \label{item:aaa}
        $\displaystyle \sup_{x,y \in B_d(0,R), \mnorm{x-y}_2 \leq \eps} \Tr \abs{D^{(j)}(x) - D^{(j)}(y)} \leq \delta N \quad \text{for all } j \in \{0, \dots, \ell_2\}$,
        }
        \item{
        \label{item:bbb}
        $\displaystyle \underset{A \in A_\delta^{(1)}}{\sup_{x \in B_d(0,R)}} \op{A W^{(j)} \jac \Phi^{ (j-1)}(x)} \leq C_2 \cdot \left(\sqrt{\delta \ln(1/\delta)} + \sqrt{\frac{Ld\ln(N/d)}{N}}\right) \quad  \\
        \text{for all }  j \in \{0, \dots, \ell_2\}$,
        }
        \item{\label{item:ccc}
            $\displaystyle \underset{x \in B_d(0,R)}{\sup}
            \op{\jac \Phi^{ (\ell_2)}(\pi(x)) - 
            \jac \Phi^{ (\ell_2)}(x)} \\ 
            \hspace{1cm}\leq \frac{1}{L} 
            \cdot \sum_{j= 0}^{\ell_2} \left(\sup_{x^\ast \in \neps, A_2 \in \mathcal{A}_\delta^{(1)}} \op{\jac \Phi^{(j+1) \to (\ell_2)}(x^\ast)A_2}\right),$
        }
           \item{ \label{item:ddd}$\displaystyle
             {\underset{x \in B_d(0,R)}{\sup}} \mnorm{\jac \Phi^{(j)}(x)}_{2 \to 2 }
        \leq 2C_1
             \quad \text{for all } j \in \{0, \dots, \ell_2-1\}$.}
        \end{enumerate}
        We call the event defined by properties \ref{item:aaa}-\ref{item:ddd} $\E_3$, noting that the randomness is now with respect to $W^{(0)}, \dots, W^{(\ell_2)}$ and $b^{(0)}, \dots, b^{(\ell_2)}$.

        For fixed $x^\ast \in \neps$, $j \in \{0, \dots, L\}$ and $A_2 \in \mathcal{A}_\delta^{(j)}$, we get
        \[
        \op{\jac \Phi^{(j) \to (\ell_2)}(x^\ast)A_2} \leq C_1
        \]
        with probability at least $1-\exp(-c_1 \cdot N/L^2)$ using \Cref{lem:pw_2} with an absolute constant $c_1 > 0$.
        Since 
        \[
        \ln\left(\#\mathcal{A}_\delta^{(j)}\right) \leq \delta N \cdot \ln(4\ee / \delta) \overset{\text{\cshref{lem:log_b}}}{\leq} 4\ee \cdot \delta N \cdot \ln(1/\delta),
        \]
        a union bound yields that 
        \begin{equation}\label{eq:e4_2}
           \op{\jac \Phi^{(j) \to (\ell_2)}(x^\ast)A_2} \leq C_1 \quad \text{for all } x^\ast \in \neps, j \in \{0, \dots, L\} \text{ and } A_2 \in \mathcal{A}_\delta^{(j)}. 
        \end{equation}
        with probability at least 
        \[
        1- \exp(\ln(L+1) + 4\ee \cdot \delta N \cdot \ln(1/\delta) + C_4 \cdot d \cdot \ln(RC_\tau'/\delta) - c_1 \cdot N/L^2) \geq 1 - \exp(-c_2 \cdot N/L^2).
        \]
        For the last step, we used $N \geq C \cdot L^3 d \ln(N/d) \geq C \cdot L^2 \ln(\ee L), \delta \ln(1/\delta) \leq c/L^2$ and $N \geq C \cdot L^2 d \ln(RC_\tau'/\delta)$ and set $c_2 \defeq c_1 / 2$. 
        We call the event defined by \eqref{eq:e4_2} $\E_4$.
        On $\E_5 \defeq \E_3 \cap \E_4$, by \ref{item:ccc}, 
        \[
        \underset{x \in B_d(0,R)}{\sup}
            \op{\jac \Phi^{(\ell_2)}(\pi(x)) - 
            \jac \Phi^{ (\ell_2)}(x)} \leq C_1
        \]
        and therefore 
        \begin{align*}
        &\norel {\underset{x \in B_d(0,R)}{\sup}} \mnorm{\jac \Phi^{(\ell_2)}(x)}_{2 \to 2 } \\
        &\leq {\underset{x^\ast \in \neps}{\sup}} \mnorm{\jac \Phi^{ (\ell_2)}(x^\ast)}_{2 \to 2 } \!\!+\!\!\underset{x \in B_d(0,R)}{\sup}
            \op{\jac \Phi^{ (\ell_2)}(\pi(x)) - 
            \jac \Phi^{(\ell_2)}(x)} \leq 2C_1.
        \end{align*}

        Altogether, all the desired properties are satisfied on $\E_5$.
        Note that we get 
        \begin{align*}
            \PP(\E_5) &\geq 1- 3\ell_2\exp(-c \cdot \delta N) - \exp(-c' \cdot \delta N) - \exp(- \delta N) - \exp(-c_2 \cdot N/L^2) \\
            &\geq 1 -3(\ell_2 + 1)\exp(-c \cdot \delta N)
        \end{align*}
        by taking $c \defeq \min \{c', c_2\}$ and using $\delta \leq 1/L^2$.
\end{proof}

We can now incorporate the final layer to get a bound on the gradient of $\Phi$ on the ball $B_d(0,R)$.
This bound consists of two summands:
The first one is obtained by bounding the gradient on an $\eps$-net of $B_d(0,R)$, 
the second one stems from controlling the deviation when moving from a net point 
to an arbitrary input that is within $\eps$-distance of the net point. 
\begin{proposition}\label{thm:r_ball_delta}
There exist positive absolute constants $C,c > 0$ such that the following holds. 
    Let $\Phi:\RR^d \to \RR$ be a random ReLU network following \Cref{assum:3} with $L$ hidden layers of width $N$ and 
    with small ball constant $C_\tau > 0$. 
    Set $C_\tau ' \defeq \sqrt{2/N} \cdot C_\tau$. 
    Let further $R\geq (C_\tau')^{-1}$, $\delta \in (0, \ee^{-1})$ and assume 
    $N \geq Cd$, $N \geq C \cdot L^3d\ln(N/d)$ and  
    \begin{gather*}
    \delta N \geq C \cdot d  \cdot \ln(RC_\tau' / \delta),
    \quad N \geq C \cdot L^2 d \ln(RC_\tau' /\delta), \quad \delta\ln(1/\delta) \leq c \cdot 1/L^2, 
    \quad \text{and} \quad d \ln(RC_\tau'/\delta) \geq C.
    \end{gather*}
    Then the event 
    \[
         {\underset{x \in B_d(0,R)}{\sup}} \mnorm{\nablaa \Phi(x)}_{2}
    \leq C \cdot \left(\sqrt{d \cdot \ln(RC_\tau'/\delta)}+L \cdot \left(\sqrt{N} \cdot \delta\ln(1/\delta) + \sqrt{Ld\ln(N/d) \cdot \delta\ln(1/\delta)}\right)\right)
        \]
       occurs with probability at least $1 -\exp(-c \cdot d\ln(RC_\tau'/\delta))$.
\end{proposition}
\begin{proof}
    \textbf{Step 1 (Point estimate):} We start with a pointwise estimate:
        Let $x \in B_d(0,R) \setminus \{0\}$ be fixed, $j \in \{0, \ldots, L\}$ and $A \in \mathcal{A}_\delta^{(j)}$. 
        Using \Cref{lem:pw_3}, we get the existence of absolute constants $C_1,c_1 > 0$ such that 
        \begin{align*}
            \mnorm{W^{(L)}\jac \Phi^{(j) \to (L-1)}(x)A}_{2 \to 2 } \leq C_1 \cdot (\sqrt{\Tr \abs{A}} + t)
        \end{align*}
        with probability at least $1-\exp(-c_1t^2)$ for every $C_1 \leq t \leq \frac{\sqrt{N}}{L}$.
        Note that we may assume $d \leq N/L^2$ and  
        \[
         \delta N \leq \delta \cdot \ln(1/\delta) \cdot N \leq c \cdot N/L^2 \leq N/L^2,
        \]
        by taking $C$ sufficiently large and $c$ sufficiently small, which implies $\Tr \abs{A} \leq N/L^2$ for every $A \in \mathcal{A}_\delta^{(j)}$.

        \medskip
        \textbf{Step 2 (Union bound):} We continue by taking a union bound:
        We let $\tilde{c}>0$ be the small constant 
        provided by \Cref{prop:not_last_layer} and 
        set $\eps \defeq \tilde{c} \cdot (C_\tau')^{-1} \cdot \delta \cdot \ln^{-1/2}(1/\delta)$.
        Let $\neps$ 
        be an $\eps$-net of $B_d(0,R)$ with 
        \begin{equation*}
        \ln(\#\neps) \leq
        C_2 \cdot d \cdot \ln(C_\tau' R/\delta)
        \end{equation*}
        with an absolute constant $C_2 > 0$ as in \eqref{eq:neps_b_22}.
        As in \Cref{prop:not_last_layer}, we may assume that $0 \notin \neps$.
        Let $j \in \{0, \ldots, L\}$. 
        Then we have 
        \begin{equation*}
        \ln\left(\#\mathcal{A}_\delta^{(j)}\right) \leq 4\ee \cdot \delta N \cdot \ln(1/\delta)
        \end{equation*}
        according to \Cref{lem:card_b}.
        Since by assumption $\delta N \geq C \cdot d \cdot\ln(C_\tau' R/\delta)$, we get 
        \[
        \ln(\#\neps) + \ln\left(\#\mathcal{A}_\delta^{(j)}\right) \leq C_3 \cdot N \cdot \delta\ln(1/\delta)
        \]
        with an absolute constant $C_3 > 0$.
        Hence, taking a union bound, 
        we get 
        \[
        \underset{A \in \mathcal{A}_\delta^{(j)}}{\underset{x^\ast \in \neps}{\sup}}
        \mnorm{W^{(L)}\jac \Phi^{(j) \to (L-1)}(x^\ast)A}_{2 \to 2 } \leq C_1 \cdot (\sqrt{\delta N} + t),
        \]
        with probability at least $1- \exp(C_3 \cdot \delta N \cdot \ln(1/\delta)-c_1 t^2)$
        for every $C_1 \leq t \leq \frac{\sqrt{N}}{L}$.
        We explicitly pick 
        \[
        t \defeq \sqrt{2C_3 c_1^{-1} \cdot \delta N \cdot \ln(1/\delta)}
        \]
        and hence get 
        \begin{equation}\label{eq:unif_b_22}
        \underset{A \in \mathcal{A}_\delta^{(j)}}{\underset{x^\ast \in \neps}{\sup}}
        \mnorm{W^{(L)}\jac \Phi^{(j) \to (L-1)}(x^\ast)A}_{2 \to 2 } \leq C_4 \cdot \sqrt{N \cdot \delta\ln(1/\delta)} 
        \end{equation}
        with probability at least 
        \[
        1- \exp(- N \cdot \delta \ln(1/\delta)) \geq 1 -\exp(-\delta N)
        \]
        with a constant $C_4>0$. 
        Further, note that $ t \leq \frac{\sqrt{N}}{L}$ holds by taking $C$ large and $c$ small enough, since we have
        \[
        \delta \ln(1/\delta) \leq c \cdot 1/L^2
        \quad 
        \text{and}
        \quad 
        N \geq C \cdot L^2d\ln(1/\delta)
        \]
        and $t \geq C_1$ is trivially satisfied. 

        In the case $j = 0$, we may leave out the union bound over $\mathcal{A}_\delta^{(j)}$ 
        (since $\#\mathcal{A}_\delta^{(0)} = 1$) and get 
        \[
        \sup_{x^\ast \in \neps} \mnorm{\nablaa \Phi(x^\ast)}_2 =
        \underset{A \in \mathcal{A}_\delta^{(0)}}{\underset{x^\ast \in \neps}{\sup}}
        \mnorm{W^{(L)}\jac \Phi^{(L-1)}(x^\ast)A}_{2 \to 2 } \leq C_1 \cdot (\sqrt{d} + t)
        \]
        with probability at least $1- \exp(C_3 \cdot d \cdot \ln(RC_\tau'/\delta) - c_1t^2)$. 
        Taking 
        \[
        t \defeq \sqrt{2C_3 c_1^{-1} \cdot d\ln(RC_\tau'/\delta)},
        \]
        we obtain 
        \begin{equation}\label{eq:final_layer_12}
        \sup_{x^\ast \in \neps} \mnorm{\nablaa \Phi(x^\ast)}_2 \leq C_4 \cdot \sqrt{d \cdot \ln(RC_\tau'/\delta)}
        \end{equation}
        with probability at least $1- \exp(-d \ln(RC_\tau'/\delta))$, after possibly increasing the constant $C_4$.
        
        \medskip

         \textbf{Step 3 (Controlling the deviation):}
        
         It remains to control the deviation when moving from an arbitrary point $x \in B_d(0,R)$ to a net point $\pi(x) \in \neps$ with $\mnorm{x-\pi(x)}_2 \leq \eps$. 

        We let $\E_1$ be the event defined by the properties 
        \begin{enumerate}
        \item{
    \label{item:aa_1}
    $\displaystyle \sup_{x,y \in B_d(0,R), \mnorm{x-y}_2 \leq \eps} \Tr \abs{D^{(j)}(x) - D^{(j)}(y)} \leq \delta N \quad \text{for all } j \in \{0, \dots, L-1\}$,
    }
    \item{
\label{item:bb_2}
    $\displaystyle \underset{A \in A_\delta^{(1)}}{\sup_{x \in B_d(0,R)}} \op{A W^{(j)} \jac \Phi^{(j-1)}(x)} \leq C \cdot \left(\sqrt{\delta \ln(1/\delta)} + \sqrt{\frac{Ld\ln(N/d)}{N}}\right) \quad \\ \text{for all }j \in \{0, \dots, L-1\}$.
    }
    \end{enumerate}
    By \Cref{prop:not_last_layer}, $\PP(\E_1) \geq 1 - 3L\exp(-c' \cdot \delta N) \geq 1 - \exp(-c'' \cdot \delta N)$ with absolute constants $c',c'' > 0$, since $\delta N \geq C \cdot \ln(\ee L)$.
        Using \Cref{lem:decomp}, on $\E_1$, we get 
        \begin{align*}
            &\norel{\underset{x \in B_d(0,R)}{\sup}}
            \mnorm{\nablaa \Phi(\pi(x)) - \nablaa \Phi(x)}_2 \nonumber\\
            &\leq \sum_{j= 0}^{L-1} \left( \underset{x \in B_d(0,R)}{\sup} \op{W^{(L)}\jac \Phi^{(j+1) \to (L-1)}(\pi(x))(D^{(j)}(\pi(x)) - D^{(j)}(x))} \right) \nonumber\\
            & \hspace{1.2cm}\cdot 
            \left( {\underset{x \in B_d(0,R)}{\sup}} \op{(D^{(j)}(\pi(x)) - D^{(j)}(x))W^{(j)}\jac \Phi^{ (j-1)}(x)}\right) \\
            &\leq \sum_{j= 0}^{L-1} \left( \underset{x^\ast \in \neps, A \in \mathcal{A}_\delta^{(1)}}{\sup} \op{W^{(L)}\jac \Phi^{(j+1) \to (L-1)}(x^\ast)A} \right) \nonumber\\
            & \hspace{1.2cm}\cdot 
            \left( \underset{A \in \mathcal{A}_\delta^{(1)}}{{\underset{x \in B_d(0,R),}{\sup}}} \op{AW^{(j)}\jac \Phi^{(j-1)}(x)}\right) \\
            &\leq C_4 \cdot \left(\sqrt{\delta \ln(1/\delta)} + \sqrt{\frac{Ld\ln(N/d)}{N}}\right) \cdot \sum_{j= 0}^{L-1} \left( \underset{x^\ast \in \neps, A \in \mathcal{A}_\delta^{(1)}}{\sup} \op{W^{(L)}\jac \Phi^{(j+1) \to (L-1)}(x^\ast)A} \right).
        \end{align*}

        We let $\E_2$ be the event defined by 
        \[
        \underset{A \in \mathcal{A}_\delta^{(j)}}{\underset{x^\ast \in \neps}{\sup}}
        \mnorm{W^{(L)}\jac \Phi^{(j) \to (L-1)}(x^\ast)A}_{2 \to 2 } \leq C_4 \cdot \sqrt{N \cdot \delta\ln(1/\delta)} \quad \text{for all } j \in \{1, \dots , L\}.
        \]
        Using \eqref{eq:unif_b_22} and a union bound,
        \[
        \PP(\E_2) \geq 1- \exp(\ln(L) - \delta N) \geq 1 - \exp(-c_2 \cdot \delta N)
        \]
        with an absolute constant $c_2 > 0$, since $\delta N \geq C \cdot \ln(\ee L)$.
        On $\E_1 \cap \E_2$, 
        \begin{align}
            &\norel{\underset{x \in B_d(0,R)}{\sup}}
            \mnorm{\nablaa \Phi(\pi(x)) - \nablaa \Phi(x)}_2 \nonumber\\
            \label{eq:e1cape22}
            &\leq C_5 \cdot L \cdot \left(\sqrt{N} \cdot \delta\ln(1/\delta) + \sqrt{Ld\ln(N/d) \cdot \delta\ln(1/\delta)}\right).
        \end{align}

        \medskip
        \textbf{Step 4 (Concluding the proof):} 
        Let $\E_3$ be the event defined by 
        \[
        \sup_{x^\ast \in \neps} \mnorm{\nablaa \Phi(x^\ast)}_2 \leq C_4 \cdot \sqrt{d \cdot \ln(RC_\tau'/\delta)}.
        \]
        By \eqref{eq:final_layer_12}, 
        \[
        \PP(\E_3) \geq 1 - \exp(-d\ln(RC_\tau'/\delta)).
        \]
        On $\E_1 \cap \E_2 \cap \E_3$, we thus have 
        \begin{align*}
        \sup_{x \in B_d(0,R)} \mnorm{\nablaa \Phi(x)}_2 &\leq \sup_{x^\ast \in \neps} \mnorm{\nablaa \Phi(x^\ast)}_2 + {\underset{x \in B_d(0,R)}{\sup}}
            \mnorm{\nablaa \Phi(x) - 
           \nablaa\Phi(\pi(x))}_2 \\
           &\leq C \cdot \left(\sqrt{d \cdot \ln(RC_\tau'/\delta)}+L \cdot \left(\sqrt{N} \cdot \delta\ln(1/\delta) + \sqrt{Ld\ln(N/d) \cdot \delta\ln(1/\delta)}\right)\right)
        \end{align*}
        if $C$ is chosen sufficiently large.
        To obtain the claim, note
        \begin{align*}
        \PP(\E_1 \cap \E_2 \cap \E_3) &\geq 1 - \exp(-c'' \cdot \delta N) - \exp(-c_2 \cdot \delta N)- \exp(-d\ln(RC_\tau'/\delta)) \\
        &\geq 1-3\exp(-d\ln(RC_\tau'/\delta)) \\
        &\geq 1 -\exp(-c_4 \cdot d\ln(RC_\tau'/\delta))
        \end{align*}
        with an absolute constant $c_4 > 0$.
        Here we used $\delta N \geq C \cdot d\ln(RC_\tau'/\delta)$ and $d \ln(RC_\tau'/\delta) \geq C$.
        This proves the claim. 
\end{proof}
We can now bound the gradient of a random ReLU network following \Cref{assum:3} globally. 
\begin{proposition}\label{prop:global_bound_delta}
    There exist absolute constants $C,c>0$ satisfying the following. 
     Let $\Phi:\RR^d \to \RR$ be a random ReLU network following \Cref{assum:3} with $L$ hidden layers of width $N$  
    with small ball constant $C_\tau > 0$. 
    Let $\lambda \geq C_\tau^{-1}$ and choose $\eta_\lambda \in [0,1]$ such that 
    \[
    \PP\left(\underset{j \in \{1, \ldots, N\}}{\sup_{\ell \in \{0, \ldots, L-1\}}} \abs{b_j^{(\ell)}} \leq \lambda\right) \geq 1 - \eta_\lambda.
    \]
    Let $\delta \in (0, \ee^{-1})$ and assume that 
    \begin{gather*}
    \delta N \geq C \cdot d  \cdot \max\{L, \ln(\lambda C_\tau/\delta)\},
    \quad N \geq C \cdot dL^2 \ln(\lambda C_\tau/\delta), \quad \delta\ln(1/\delta) \leq c \cdot 1/L^2\\ 
         d \max\{L, \ln(\lambda C_\tau/\delta)\} \geq C \quad \text{and} \quad
         N \geq C \cdot dL^4 \ln^2(\ee L).
    \end{gather*}
    Fix $u \geq 0$ and assume $N \geq C \cdot u^2L^3\ln(N/d)\ln(L\ln(N/d))$.
    Then,
    \begin{align*}
        \PP\Bigg(\sup_{x \in \RR^d} \mnorm{\nablaa \Phi(x)}_2
        &\leq C \Big(\sqrt{d(L+\ln(\lambda C_\tau /\delta))} \\
        &\hspace{0.5cm}+ L \cdot \left(\sqrt{N}\cdot \delta \ln(1/\delta) + \sqrt{\delta \ln(1/\delta)Ld\ln(N/d)}\right) + u\Big)\Bigg) \\
        &\hspace{-3cm}\geq 1- \exp(-c \cdot d(L + \ln(\lambda C_\tau/\delta))) - 4\exp(-c \cdot u^2) - \eta_\lambda.
    \end{align*}
\end{proposition}
\begin{proof}
    In the beginning of the proof, we set $\lambda' \defeq \sqrt{N/2} \cdot \lambda$ 
and $C_\tau' \defeq \sqrt{2/N} \cdot C_\tau$.
Let the constant $C_1 > 0$ be taken from \Cref{thm:grad_diff} and set 
$R \defeq  C_1 \cdot 3^L \cdot \lambda' \cdot \delta^{-1} \cdot \ln^{1/2}(1/\delta)$.
Moreover, note that $\lambda' C_\tau ' = \lambda C_\tau$ by construction. 

\medskip

\textbf{Step 1: Control the gradient on the ball} 
We check that the conditions of \Cref{thm:r_ball_delta} are satisfied.
To this end, let $C_2,c_2 > 0$ be the constants from \Cref{thm:r_ball_delta}.
Note that $R \geq \lambda' \geq (C_\tau')^{-1}$ by assumption. 
We first note that 
\begin{align}
    \ln(RC_\tau'/\delta) &= \ln \left(\frac{ C_1 \cdot 3^L \cdot \lambda'C_\tau' \cdot \ln^{1/2}(1/\delta)}{\delta^2}\right)
    = \ln(3) \cdot L + \ln\left(C_1 \cdot \lambda' C_\tau' \cdot \ln^{1/2}(1/\delta)/\delta^2\right).
     \nonumber
\end{align}
Since $\ln^{1/2}(1/\delta) \geq 1$, this clearly yields 
\begin{equation}\label{eq:ln_lower}
    \ln(RC_\tau'/\delta) \geq L + \ln(\lambda'C_\tau'/\delta).
\end{equation}
On the other hand, since $\ln^{1/2}(1/\delta) \leq 1/\delta$, 
\begin{align}
    \ln(RC_\tau'/\delta) &\leq \ln(3) \cdot L + \ln\left(C_1 \cdot \lambda'C_\tau'/\delta^3\right)
    \leq \ln(3) \cdot L + \ln\left(\left(C_1 \cdot \lambda'C_\tau'/\delta\right)^3\right) \nonumber\\
    \label{eq:ln_upper}
    \overset{\text{\cshref{lem:log_b}}}&{\leq} C' \cdot \left(L + \ln(\lambda'C_\tau'/\delta)\right)
\end{align}
with an absolute constant $C' >0$.
\eqref{eq:ln_lower} and \eqref{eq:ln_upper} together imply that the conditions of \Cref{thm:r_ball_delta}
are satisfied. 
We can thus apply \Cref{thm:r_ball_delta} and get 
\begin{align*}
\underset{x \in B_d(0,R)}{\sup} \mnorm{\nablaa \Phi(x)}_{2}\!
    \overset{\eqref{eq:ln_upper}}&{\leq}\! C_2 \!\cdot\! \left(\sqrt{d \cdot (L \!+\! \ln(\lambda C_\tau/\delta))}+L \cdot \left(\sqrt{N} \cdot \delta\ln(1/\delta) + \sqrt{Ld\ln(N/d) \cdot \delta\ln(1/\delta)}\right)\right)
\end{align*}
with probability at least 
$1- \exp(-c_2 \cdot d\ln(RC_\tau'/\delta)) \overset{\eqref{eq:ln_lower}}{\geq} 1- \exp(-c_2 \cdot d(L + \ln(\lambda C_\tau/\delta)))$.

\medskip

\textbf{Step 2: Control the difference to zero-bias network}
We use \Cref{thm:grad_diff} to control the difference of the gradient of $\Phi$ to the gradient of its zero-bias counterpart $\widetilde{\Phi}$.
To this end, we first condition on the bias vectors $b^{(0)}, \ldots, b^{(L-1)}$ and assume that 
\[
\underset{j \in \{1,\ldots,N\}}{\underset{\ell \in \{0,\ldots,L-1\}}{\sup}} \abs{b^{(\ell)}_j} \leq \lambda = 
    \frac{\lambda' \sqrt{2}}{\sqrt{N}}.
\]
Since the biases are now fixed, $R$ was chosen appropriately and the conditions of \Cref{thm:grad_diff} hold by assumption, we get 
\[
\underset{\mnorm{x}_2 \geq R}{\sup}
    \mnorm{\nablaa \Phi(x) - \nablaa \widetilde{\Phi}(x)}_2 \leq C_3 \cdot \left(L \cdot \left(\sqrt{N} \cdot \delta\ln(1/\delta) + \sqrt{Ld\ln(N/d) \cdot \delta\ln(1/\delta)}\right)\right)
\]
with probability at least $1-\exp(-c_3 \cdot \delta N)$ with absolute constants $C_3, c_3 > 0$.
Recall that we conditioned on $b^{(0)}, \ldots, b^{(L-1)}$ so far. 
Reintroducing the randomness over the bias vectors, we obtain 
\[
\underset{\mnorm{x}_2 \geq R}{\sup}
    \mnorm{\nablaa \Phi(x) - \nablaa \widetilde{\Phi}(x)}_2 \leq C_3 \cdot \left(L \cdot \left(\sqrt{N} \cdot \delta\ln(1/\delta) + \sqrt{Ld\ln(N/d) \cdot \delta\ln(1/\delta)}\right)\right)
\]
with probability at least $1- \exp(-c_3 \cdot \delta N) - \eta_\lambda$.

\medskip

\textbf{Step 3: Control the gradient of the zero-bias network}
We use \Cref{eq:u} and get that 
\[
\sup_{x \in \RR^d} \mnorm{\nablaa \tilde{\Phi}(x)}_2 \leq C_4 \cdot \left(\sqrt{d \ln(\ee L)} + u \right)
\leq C_4 \cdot (\sqrt{dL} + u)
\]
with probability at least $1- 4\exp(-c_4 \cdot u^2)$ for absolute constants $C_4, c_4 > 0$.

The final claim follows from a union bound over the three individual steps. 
\end{proof}
\Cref{thm:main_upper_bias} now follows from \Cref{prop:global_bound_delta} with the explicit 
choice of $\delta \asymp \frac{d}{N}(L + \ln(\lambda C_\tau N/d))$.
\begin{proof}[Proof of \Cref{thm:main_upper_bias}]
    We let $K >0$ be an absolute constant that we will exactly determine later and set $\delta \defeq K \cdot \frac{d}{N}(L + \ln(\lambda C_\tau N/d))$.
    Moreover, we set $u = \sqrt{d(L + \ln(\lambda C_\tau N/d))}$.
    
    Firstly, we check that the conditions of \Cref{prop:global_bound_delta} are satisfied. 
    To this end, let $C_1, c_1 >0$ be the constants provided by \Cref{prop:global_bound_delta}.
    Since $\delta \geq K \cdot \frac{dL}{N}$, we clearly have $\delta N \geq C_1 \cdot dL$ if 
    $K \geq C_1$.
    Moreover, by using $\delta \geq d/N$ we have $\ln(\lambda C_\tau /\delta) = \ln(\lambda C_\tau) + \ln(1/\delta)\leq \ln(\lambda C_\tau N/d)$
    and consequently, 
    \[
    C_1 \cdot d \ln(\lambda C_\tau/\delta) \leq C_1 \cdot d\ln(\lambda C_\tau N/d) \leq \delta N.
    \]
    Altogether, we have verified the first condition of \Cref{prop:global_bound_delta}.

    For the second one, we again use $\ln(\lambda C_\tau /\delta) \leq \ln(\lambda C_\tau N/d)$
    to obtain 
    \[
    C_1 \cdot dL^2\ln(\lambda C_\tau/\delta) \leq C_1 \cdot dL^2 \ln(\lambda C_\tau N/d)
    \leq C_1 \cdot dL^2\left(\ln(\lambda C_\tau) + \ln(N/d)\right)
    \leq N,
    \]
    where the last step follows from the assumptions of the theorem. 

    Regarding the third condition, we note that our assumptions imply $dL^3\ln^2(\ee L) \leq N/C $ and
    \[
    dL^2\ln^2(\ee L)\ln(\lambda C_\tau N/d) 
    \leq dL^3 \ln(L \ln(N/d)) \cdot \ln(\lambda C_\tau) + dL^2\ln^2(\ee L) \ln(N/d)\leq N/C'
    \]
    for an absolute constant $C'>0$ that depends on $C$.
    Here, the last step follows from the assumption and from \eqref{eq:loghelper}.
    Moreover, note that $C'$ can be made arbitrarily large by increasing $C$.
    Summing these two inequalities and assuming $C' \leq C$ (after replacing $C'$ by $\min\{C,C'\}$, if necessary) 
    yields 
    \[
    dL^2\ln^2(\ee L)(L + \ln(\lambda C_\tau N/d)) \leq 2N/C'
    \]
    or, equivalently, 
    \[
    \frac{N}{d(L +\ln(\lambda C_\tau N/d))} \geq \frac{C'}{2} \cdot L^2\ln^2(\ee L).
    \]
    This gives us 
    \[
    \delta^{-1} \geq \frac{K^{-1}C'}{2} \cdot L^2\ln^2(\ee L). 
    \]
    In particular, we may assume $\delta \in (0, \ee^{-1})$ for sufficiently large $C$.
    Since $x \mapsto \ln^2(x)/x$ is nonincreasing for $x \in [\ee^2, \infty)$,
    we get 
    \begin{equation}\label{eq:square_bound}
    \delta \ln^2(1/\delta) = \frac{\ln^2(1/\delta)}{1/\delta} 
    \leq \frac{\ln^2\left(\frac{K^{-1}C'}{2} \cdot L^2\ln^2(\ee L)\right)}{\frac{K^{-1}C'}{2} \cdot L^2\ln^2(\ee L)}
    \leq 2K\cdot \frac{(3 + \ln(K^{-1}C'/2))^2}{C'} \cdot \frac{\ln^2(\ee L)}{L^2\ln^2(\ee L)} \leq c_1 /L^2
    \end{equation}
    if $C$ is chosen sufficiently large (after $K$ has been fixed).
    This proves the third condition. 

    The fourth condition is trivially satisfied: Since $\delta^{-1} \geq \frac{K^{-1}C'}{2}$, we get
    \[
    \ln(\lambda C_\tau/\delta) \geq \ln(K^{-1}C'/2) \geq C_1.
    \]
    The fifth condition is trivially satisfied too. 

    It remains to check that $N \geq C_1 \cdot u^2L^3\ln(N/d)\ln(L \ln(N/d))$.
    By plugging in the definition of $u$, it suffices to check that 
    \begin{align*}
    N &\geq 3C_1\cdot  dL^4 \cdot \ln(N/d) \ln(L \ln(N/d)), 
    \quad N \geq 3C_1\cdot  dL^3 \ln(\lambda C_\tau) \ln(N/d)\ln(L \ln(N/d)) \\
    N &\geq 3C_1\cdot  dL^3\ln^2(N/d)\ln(L \ln(N/d)).
    \end{align*}
    The first condition follows using $\ln(L \ln(N/d)) \leq \ln(N/d)$ as in \eqref{eq:loghelper}.
    The second condition is precisely one of the conditions of the theorem.
    For the third one, we use
    \[
    dL^3\ln^2(N/d)\ln(L \ln(N/d))
    = dL^3\ln^2(N/d)(\ln(L) + \ln(\ln(N/d)))
    \leq dL^4 \ln^2(N/d) + dL^3\ln^3(N/d).
    \]
    The first summand can again be controlled using \eqref{eq:loghelper}.
    For the second one, we use the monotonicity of $x \mapsto x/\ln^3(x)$ on $[\ee^3,\infty)$ to get 
    \begin{equation}\label{eq:L3}
    \frac{N/d}{\ln^3(N/d)} \geq \frac{CL^4}{\ln^3(CL^4)} \geq \frac{C}{(\ln(C)+4)^3}\frac{L^4}{\ln^3(\ee L)}
    \geq 6C_1  \cdot L^3
    \end{equation}
    for a sufficiently large choice of $C$.

    We can thus apply \Cref{prop:global_bound_delta}. 
    It remains to estimate the resulting upper bound for the $\sup_{x \in \RR^d} \mnorm{\nablaa \Phi(x)}_2$ itself. 
    Note that $\delta \geq d/N$ directly yields 
    \[
    \sqrt{d(L + \ln(\lambda C_\tau/\delta))} \leq \sqrt{d(L + \ln(\lambda C_\tau N/d))}.
    \]
    
    Regarding the second summand, we first use the definition of $\delta$ to get the identity
    \[
    \frac{[L\sqrt{N}\delta\ln(1/\delta)]^2}{d[L+\ln(\lambda C_\tau N/d)]} = KL^2\delta \ln^2(1/\delta)
    \]
    and further obtain, using \eqref{eq:square_bound},
    \[
    \frac{[L\sqrt{N}\delta\ln(1/\delta)]^2}{d[L+\ln(\lambda C_\tau N/d)]} \leq Kc_1,
    \]
    which is the desired estimate. 

    Regarding the third summand, we note that 
    \[
    L \sqrt{\delta \ln(1/\delta)Ld\ln(N/d)}
    \leq \sqrt{K}\cdot \frac{dL^{3/2}}{\sqrt{N}} \cdot \sqrt{L + \ln(\lambda C_\tau N/d)} \cdot \ln(N/d)
    \leq \sqrt{d(L + \ln(\lambda C_\tau N/d))},
    \]
    where the last step uses the condition $N \geq C \cdot dL^3\ln^2(N/d)$, which follows again 
    similarly to \eqref{eq:loghelper}.
    
    To conclude the proof, it remains to estimate the final probability. 
    Here, we note that 
    \[
    L + \ln(\lambda C_\tau / \delta) = L + \ln(\lambda C_\tau N/d) - \ln(K(L + \ln(\lambda C_\tau N/d)))
    \geq \frac{1}{2}\cdot \left(L+\ln(\lambda C_\tau N/d)\right),
    \]
    where the last inequality is a result from $\ln(Kx) \leq x/2$ for sufficiently large $x$, for fixed $K$.
    This concludes the proof. 

\end{proof}
\subsection{Lower bound}\label{sec:lower_general_proofs}
In this section, we prove lower bounds for the Lipschitz constants of random ReLU networks following \Cref{assum:1} with general biases. 
We emphasize that no assumptions on the distribution of the biases are required.
The proof uses the fact that $\lip_p(\Phi) \geq \lip_p(\tilde{\Phi})$ (where $\tilde{\Phi}$ is the associated zero-bias network)
so that lower bounds can directly be transferred from \Cref{thm:main_lower}.
\begin{proof}[Proof of \Cref{thm:main_lower_bias}]
    We let $\tilde{\Phi}: \RR^d \to \RR$ denote the zero-bias network associated with $\Phi$.
    We first prove the following auxiliary statement.
    For every fixed realization of $\Phi$, $x \in \RR^d$ and $\ell \in \{-1, \ldots, L-1\}$, we have 
    \[
    \lim_{r \to \infty} \frac{\Phi^{(\ell)}(rx)}{r} = \tilde{\Phi}^{(\ell)}(x).
    \]
    The proof is via induction over $\ell$ and there is nothing to show for $\ell = -1$. 
    For the induction step, we pick $\ell \in \{0, \ldots, L-1\}$ and assume that the claim is already 
    true for $\ell-1$. We then get
    \begin{align*}
        \mnorm{\frac{\Phi^{(\ell)}(rx)}{r} - \tilde{\Phi}^{(\ell)}(x)}_2
        &= \mnorm{\relu\left(\frac{1}{r}W^{(\ell)}\Phi^{(\ell - 1)}(rx) + b^{(\ell)}/r\right) - \relu\left(W^{(\ell)}\tilde{\Phi}^{(\ell -1)}(x)\right)}_2 \\
        &\leq \mnorm{\frac{1}{r}W^{(\ell)}\Phi^{(\ell - 1)}(rx) + b^{(\ell)}/r - W^{(\ell)}\tilde{\Phi}^{(\ell - 1)}(x)}_2 \\
        &\leq \op{W^{(\ell)}} \underbrace{\mnorm{\frac{\Phi^{(\ell - 1)}(rx)}{r} - \tilde{\Phi}^{(\ell - 1)}(x)}_2}_{\to 0} + \underbrace{\frac{\mnorm{b^{(\ell)}}_2}{r}}_{\to 0} \to 0 \quad (r \to \infty).
    \end{align*}
    By induction, the claim follows for every $\ell \in \{-1, \ldots, L-1\}$. 
    Adding the output layer, we obtain 
    \begin{align*}
        \mnorm{\frac{\Phi(rx)}{r} - \tilde{\Phi}(x)}_2 
        &= \abs{\frac{W^{(L)}\Phi^{(L-1)}(rx)}{r} + b^{(L)} /r- W^{(L)}\tilde{\Phi}^{(L-1)}(x)} \\
        &\leq \mnorm{W^{(L)}}_2 \cdot \mnorm{\frac{\Phi^{(L-1)}(rx)}{r} - \tilde{\Phi}^{(L-1)}(x)}_2 + \abs{b^{(L)}}/r \to 0 \quad (r \to \infty).
    \end{align*}
    Hence, we have 
    \[
    \lim_{r \to \infty} \frac{\Phi(rx)}{r} = \tilde{\Phi}(x) \quad \text{for all } x \in \RR^d.
    \]

    Fix $x,y \in \RR^d$ with $x \neq y$ and $p \in [1,\infty]$.
    Then, 
    \begin{align*}
        \frac{\abs{\tilde{\Phi}(x) - \tilde{\Phi}(y)}}{\mnorm{x-y}_p} = \lim_{r \to \infty}\frac{\abs{{\Phi}(rx) - {\Phi}(ry)}}{r\mnorm{x-y}_p} \leq \lip_p(\Phi).
    \end{align*}
    Taking the supremum over $x$ and $y$ gives $\lip_p(\tilde{\Phi}) \leq \lip_p(\Phi)$.
    The claim then follows from \Cref{thm:main_lower}.
\end{proof}

%% file: auxiliary.tex
\section{Other auxiliary results} \label{sec:auxil}
In this appendix, we provide some additional auxiliary results that are used over the course of the present paper. 
First, we show that the small-ball property in \Cref{assum:3} is indeed equivalent 
to each bias having an absolutely continuous distribution with probability density function bounded by half the small-ball constant. 
\begin{lemma}\label{thm:abs_cont}
    Let $X$ be a real-valued random variable and $T > 0$. Then the following 
    are equivalent. 
    \begin{enumerate}
    \item $X$ has an absolutely continuous distribution with probability density function
    bounded by $T/2$ almost everywhere.
    \item For every $t \in \RR$ and $\eps > 0$,
    \[
    \PP(X \in (t-\eps, t +\eps)) \leq T \cdot \eps.
    \]
    \end{enumerate}
\end{lemma}
\begin{proof}
    The direction "(1)$\Rightarrow$(2)" is trivial. 
    Hence, it remains to show the opposite direction.
    As a first step, let $A \subseteq \RR$ with $\metalambda(A) = 0$ and $\eps > 0$.
    According to \cite[Lemma~1.17]{folland1984real}, there exist $\left((a_j,b_j)\right)_{j =1}^\infty$ with 
    \[
    A \subseteq \bigcup_{j=1}^\infty (a_j,b_j) \quad \text{and} \quad \sum_{j=1}^\infty (b_j - a_j) \leq \frac{2\eps}{T}.
    \]
    This implies 
    \[
    \PP(X \in A) \leq \sum_{j=1}^\infty \PP(X \in (a_j,b_j)) \leq \frac{T}{2} \cdot \sum_{j=1}^\infty (b_j - a_j) \leq \eps.
    \]
    Hence, since $\eps>0$ was arbitrary, 
    \[
    \PP(X \in A) = 0.
    \]
    According to the Radon-Nikodym theorem \cite[Theorem~3.8]{folland1984real}
    there thus exists a measurable and non-negative probability 
    density function $f: \RR \to [0, \infty)$
    of $X$. 
    According to the Lebesgue differentiation theorem \cite[Theorem~3.21]{folland1984real}, for almost every $x \in \RR$,
    \[
    f(x) = \lim_{\eps \to 0} \frac{1}{2\eps} \int_{(x-\eps, x+\eps)} f(y) \ \dd y.
    \]
    This implies the claim, since for every $x \in \RR$ and $\eps >0$,
    \[
    \frac{1}{2\eps} \int_{(x-\eps, x+\eps)} f(y) \ \dd y 
    = \frac{1}{2\eps} \cdot \PP\big(X \in (x-\eps, x +\eps)\big)\leq \frac{1}{2\eps} \cdot T \cdot \eps = T/2. \qedhere
    \]
\end{proof}
The next lemma provides a bound on the Euclidean distance of two normalized vectors in terms of the original vectors. 
\begin{lemma}\label{lem:diff_bound}
Let $x,y \in \RR^d \setminus \{0\}$. Then 
\begin{equation*}
\mnorm{\frac{x}{\mnorm{x}_2} - \frac{y}{\mnorm{y}_2}}_2 \leq \frac{2}{\max\{\Vert x \Vert_2, \Vert y \Vert_2\}} \mnorm{x-y}_2.
\end{equation*}
\end{lemma}
\begin{proof}
By symmetry, without loss of generality, 
we may assume $\Vert x \Vert_2 \geq \Vert y \Vert_2$. We then get
\begin{align*}
\mnorm{\frac{x}{\mnorm{x}_2} - \frac{y}{\mnorm{y}_2}}_2 &\leq 
\mnorm{\frac{x}{\mnorm{x}_2} - \frac{y}{\mnorm{x}_2}}_2 + \mnorm{\frac{y}{\mnorm{x}_2} - \frac{y}{\mnorm{y}_2}} \\
&= \frac{1}{\mnorm{x}_2} \cdot \mnorm{x-y}_2 + \mnorm{y}_2 \cdot \abs{\frac{1}{\mnorm{x}_2}- \frac{1}{\mnorm{y}_2}} \\
&=  \frac{1}{\mnorm{x}_2} \cdot \mnorm{x-y}_2 + \mnorm{y}_2 \cdot \frac{\abs{\mnorm{y}_2 - \mnorm{x}_2}}{\mnorm{x}_2\mnorm{y}_2} \\
&\leq \frac{2}{\mnorm{x}_2} \cdot \mnorm{x-y}_2,
\end{align*}
as was to be shown.
\end{proof}
\renewcommand*{\proofname}{Proof}
For the sake of completeness, we state and prove the following elementary lemma, which allows to
pull constants outside of a logarithm. 
\begin{lemma}\label{lem:log_b}
    Let $x \geq \ee$ and $C\geq 1$. Then,
    \[
    \ln(C \cdot x) \leq C \cdot \ln(x).
    \]
\end{lemma}
\begin{proof}
    Using the well-known estimate $\ln(C) \leq C-1$ for $C > 0$, we compute
    \[
    \ln(C \cdot x)= \ln(C) + \ln(x) \leq C-1 + \ln(x) \overset{x \geq \ee}{\leq} (C-1)\ln(x) + \ln(x) = C \cdot \ln(x). \qedhere
    \]
\end{proof}
Next, we provide a useful integral bound which occurs often when bounding the Gaussian width of a set using Dudley's inequality. 
\begin{lemma}\label{lem:int_b}
    Let $C, \alpha > 0$ with $C / \alpha \geq \ee$. 
    Then,
    \[
    \int_0^\alpha \sqrt{\ln(C/\gamma)} \ \dd \gamma \leq 2 \cdot \alpha \cdot \sqrt{\ln(C/\alpha)}.
    \]
\end{lemma}
\begin{proof}
    We first substitute $\sigma = \ln(C/\gamma)$ and get 
    \begin{align*}
      \int_0^\alpha \sqrt{\ln(C/\gamma)} \ \dd \gamma
      =-C \cdot\int_\infty^{\ln(C/\alpha)} \sqrt{\sigma} \cdot \exp(-\sigma) \ \dd \sigma = C \cdot \int_{\ln(C/\alpha)}^\infty \sqrt{\sigma} \cdot \exp(-\sigma) \ \dd \sigma.
    \end{align*}
    Using integration by parts, we further get 
    \begin{align*}
        \int_{\ln(C/\alpha)}^\infty \sqrt{\sigma} \cdot \exp(-\sigma) \ \dd \sigma &= \left[-\sqrt{\sigma} \cdot \exp(-\sigma)\right]_{\ln(C/\alpha)}^\infty 
        + \frac{1}{2} \cdot \int_{\ln(C/\alpha)}^\infty \frac{\exp(-\sigma)}{\sqrt{\sigma}} \ \dd\sigma \\
        &= \frac{\alpha}{C} \cdot \ln^{1/2}(C/\alpha) + \frac{1}{2} \cdot \int_{\ln(C/\alpha)}^\infty \frac{\exp(-\sigma)}{\sqrt{\sigma}} \ \dd\sigma.
    \end{align*}
    We study the last integrand further. 
    Using the substitution $\sigma = u^2/2$, we obtain 
    \begin{align*}
        \int_{\ln(C/\alpha)}^\infty \frac{\exp(-\sigma)}{\sqrt{\sigma}} \ \dd\sigma &= 
        \sqrt{2} \cdot \int_{\sqrt{2\ln(C/\alpha)}}^\infty \exp(-u^2/2) \ \dd u \\
        &= \sqrt{2} \cdot \sqrt{2\pi} \cdot \underset{g \sim \mathcal{N}(0,1)}{\PP}(g \geq \sqrt{2\ln(C/\alpha)}) \\
        \overset{\text{\cite[Prop.~2.1.2]{vershynin_high-dimensional_2018}}}&{\leq} \frac{\sqrt{2}}{\sqrt{2\ln(C/\alpha)}}
        \cdot \frac{\alpha}{C} \overset{C/\alpha \geq \ee}{\leq} \frac{\alpha}{C}.
    \end{align*}
    This proves the claim. 
\end{proof}

Next, we provide an auxiliary result about the spectral norm of a Gaussian matrix which is multiplied from the left and the right with (deterministic) matrices. 
\begin{lemma}\label{lem:pw_4}
There exist absolute constants $C, c > 0$ such that the following holds.
For $k,\ell,m,n\in \NN$, let $A_1 \in \RR^{k \times \ell}, A_2 \in \RR^{m \times n}$ be fixed matrices and $W \in \RR^{\ell \times m}$ be a random matrix with 
$W_{i,j} \iid \mathcal{N}(0,1)$.
Then, for arbitrary $t \geq C$,
\begin{equation*}
\op{A_1WA_2} \leq  \op{A_1} \cdot \op{A_2}\cdot \left(\sqrt{\rang(A_1)} + \sqrt{\rang(A_2)} + t\right)
\end{equation*}
with probability at least $1-\exp(-c \cdot t^2)$. 
\end{lemma}
\begin{proof}
Let $A_1 = U_1 \Sigma_1 V_1^T$ be the singular value decomposition of $A_1$, where $U_1 \in \RR^{k \times k}$ and $V_1 \in \RR^{\ell \times \ell}$ are orthogonal matrices and $\Sigma_1 \in \RR^{k \times \ell}$ contains the singular values of $A_1$.
Moreover, let $A_2 = U_2 \Sigma_2 V_2^T$ be the singular decomposition of $A_2$ with orthogonal matrices $U_2 \in \RR^{m \times m}$,
$V_2 \in \RR^{n \times n}$ and $\Sigma_2 \in \RR^{m \times n}$ containing the singular values of $A_2$. 
We then get by the rotational invariance of the normal distribution and of the spectral norm that 
\begin{equation*}
\op{A_1WA_2} = \op{U_1\Sigma_1 V_1^TWU_2\Sigma_2 V_2^T} \d \op{\Sigma_1 W\Sigma_2}.
\end{equation*}
Let $\alpha \defeq \rang(A_1)$, $\beta \defeq \rang(A_2)$ and let $\sigma_1 \geq \dots\geq \sigma_\alpha > 0$ and $\sigma_1 ' \geq\dots\geq \sigma_\beta '>0$ be the singular values of $A_1$ and $A_2$, respectively.
We set $\tilde{\Sigma_1} \defeq \diag(\sigma_1, \dots, \sigma_\alpha) \in \RR^{\alpha \times \alpha}$ and, moreover, $\tilde{\Sigma_2}
\defeq \diag(\sigma_1',\dots,\sigma_\beta') \in \RR^{\beta \times \beta}$.
We then get $\Sigma_1 = I_{k \times \alpha}\tilde{\Sigma_1}I_{\alpha \times \ell}$ and 
$\Sigma_2 = I_{m \times \beta} \tilde{\Sigma_2} I_{\beta \times n}$.
Hence, we see
\begin{align*}
\op{\Sigma_1 W\Sigma_2} &= \op{I_{k \times \alpha}\tilde{\Sigma}I_{\alpha \times \ell}WI_{m \times \beta} \tilde{\Sigma_2} I_{\beta \times n}} \leq \op{I_{\alpha \times \ell}WI_{m \times \beta}} \cdot \sigma_1 \cdot \sigma_1'\\
&= \op{I_{\alpha \times \ell}WI_{m \times \beta}} \cdot \op{A_1} \cdot \op{A_2} \d \op{W'} \cdot \op{A_1}\cdot\op{A_2}
\end{align*}
where $W' \in \RR^{\alpha \times \beta}$ has independent $\mathcal{N}(0, 1)$-entries. 
Hence, from \cite[Corollary~7.3.3]{vershynin_high-dimensional_2018} we see 
\begin{equation*}
\op{W'} \leq \sqrt{\rang(A_1)} + \sqrt{\rang(A_2)} + t
\end{equation*}
with probability at least $1-2\exp(-\tilde{c} \cdot t^2) = 1 - \exp(\ln(2) - \tilde{c} \cdot t^2)$ with an absolute constant $\tilde{c} > 0$.
By taking $C$ sufficiently large, we have 
\[
\ln(2) \leq \frac{\tilde{c}}{2}\cdot t^2 \quad \text{for } t \geq C,
\]
so that the claim follows letting $c \defeq \tilde{c}/2$.
\end{proof}

In the following lemma, we bound the cardinality of the set of $(N \times N)$-diagonal matrices with entries on the diagonal either $0$, $1$, or $-1$ and where the number of non-zero entries on the diagonal is bounded by $\delta \cdot N$.
\begin{lemma}\label{lem:card_b}
    For given $\delta \in (0,1)$ and $N \in \NN$, let 
    \[
    \mathcal{A} \defeq \left\{ A \in \diag\{0,1,-1\}^{N \times N}: \ \Tr \abs{A} \leq \delta \cdot N\right\},
    \]
    where we recall that $\abs{A}$ is obtained from $A$ by applying the absolute value 
    to every entry.
    Then,
    \[
    \#\mathcal{A} \leq \left(\frac{4\ee}{\delta}\right)^{  \delta \cdot N}.
    \]
\end{lemma}
\begin{proof}
    In the case $\delta N < 1$, the set $\mathcal{A}$ only consists of the zero matrix, so that the claim is satisfied since $4\ee /\delta \geq 1$. 
    Hence, we may assume $\delta N \geq 1$.
    
    Note that the cardinality of $\mathcal{A}$ can be upper bounded by the number of subsets of $\{1,\ldots,N\}$ with cardinality at most $\lfloor  \delta \cdot N \rfloor$, multiplied by $2^{\lfloor  \delta \cdot N \rfloor}$, since we can choose between $1$ and $-1$ for each diagonal entry that is nonzero. 
    Therefore, using \cite[Ex.~0.0.5]{vershynin_high-dimensional_2018}, we observe 
    \begin{align*}
        \#\mathcal{A} &\leq 2^{\lfloor  \delta \cdot N \rfloor} \cdot 
        \sum_{k=0}^{\lfloor \delta N \rfloor}\binom{N}{k}
        \leq 2^{  \lfloor\delta \cdot N \rfloor} \cdot \left(\frac{\ee N}{\lfloor  \delta \cdot N \rfloor}\right)^{\lfloor  \delta \cdot N \rfloor} \\
        \overset{\lfloor x \rfloor \geq x/2 \text{ for } x \geq 1}&{\leq}
        2^{  \delta \cdot N } \cdot \left(\frac{2\ee N}{  \delta \cdot N }\right)^{  \delta \cdot N }  
        = \left(\frac{4\ee}{\delta}\right)^{ \delta \cdot N}. \qedhere
    \end{align*}
\end{proof}
Next, we show a lower bound on the covering number of the Euclidean sphere $\SS^{n-1}$.
This statement is probably well-known, but since we could not locate a convenient reference, we include the short proof. 
\begin{lemma}\label{lem:sphere_cov}
For any $\eps \in (0,1)$,
\[
\mathcal{N}(\SS^{n-1}, \eps ) \geq \frac{1}{6} \cdot \left(\frac{1}{2\eps}\right)^{n-1}.
\]
\end{lemma}
\begin{proof}
Let $Z_1$ be an $\eps$-net of $\SS^{n-1}$ and $Z_2$ be an $\eps$-net of $[-1,1]$ with $\# Z_2 \leq \frac{3}{\eps}$.
$Z_2$ exists due to \cite[Corollary~4.2.13]{vershynin_high-dimensional_2018}.
We define the surjective mapping
\[
\varphi: \quad [-1,1] \times \SS^{n-1} \to \B_n(0,1), \quad (t,x) \mapsto tx
\]
and set 
\[
Z_3 \defeq \{\varphi(t,x) : x \in Z_1, t \in Z_2\}.
\]
Let $y_0 \in \B_n(0,1)$ be arbitrary and pick $t_0 \in [-1,1], x_0 \in \SS^{n-1}$ with $\varphi(t_0,x_0) = y_0$. 
Moreover, pick $\tilde{t} \in Z_2, \tilde{x} \in Z_1$ with $\abs{\tilde{t} - t_0}, \mnorm{\tilde{x} - x_0}_2 \leq \eps$. 
Then,
\[
\mnorm{\varphi(\tilde{t}, \tilde{x}) - \varphi(t_0, x_0)}_2 = \mnorm{\tilde{t}\tilde{x} - t_0x_0}_2 \leq \mnorm{\tilde{t}\tilde{x} - \tilde{t} x_0}_2 + \mnorm{\tilde{t}x_0 - t_0x_0}_2
\leq \mnorm{\tilde{x} - x_0}_2 + \abs{\tilde{t} - t_0} \leq 2\eps.
\]
Hence, $Z_3$ is a $(2\eps)$-net of $\B_n(0,1)$. 
By \cite[Corollary~4.2.13]{vershynin_high-dimensional_2018}, we get $\# Z_3 \geq \left(\frac{1}{2\eps}\right)^n$.
On the other hand, by the definition of $Z_3$,
\[
\# Z_3 \leq \# Z_1 \cdot \# Z_2 \leq \# Z_1 \cdot \frac{3}{\eps}.
\]
Altogether, we get 
\[
\# Z_1 \geq \frac{\eps}{3} \cdot \left(\frac{1}{2\eps}\right)^n = \frac{1}{6} \cdot \left(\frac{1}{2\eps}\right)^{n-1}.
\]
Since $Z_1$ was an arbitrary $\eps$-net of $\SS^{n-1}$, the claim is shown. 
\end{proof}